\pdfoutput=1

\documentclass{article}
\usepackage{arxiv}
\usepackage{graphicx}
\usepackage{float}

\usepackage{amsmath}
\usepackage{amssymb}
\usepackage{mathtools}
\usepackage{amsthm}
\usepackage{bm}
\usepackage{dsfont}
\usepackage{times}

\usepackage{tabularx}
\usepackage{amsfonts}       % blackboard math symbols
\usepackage{nicefrac}       % compact symbols for 1/2, etc.
\usepackage{microtype}      % microtypography
\usepackage{xcolor}         % colors
\usepackage[colorlinks=true,citecolor=blue]{hyperref}       % hyperlinks
\usepackage{enumitem}
\usepackage{subcaption}
\usepackage[english]{babel}
\usepackage{cleveref}
\usepackage{parskip}
\usepackage{overpic}
\usepackage{cite}
\usepackage{multirow}
\usepackage{makecell}
\usepackage{tablefootnote}

\theoremstyle{remark}
\newtheorem{remark}{Remark}
\theoremstyle{plain}
\newtheorem{theorem}{Theorem}
\newtheorem{proposition}{Proposition}
\newtheorem{corollary}{Corollary}
\newtheorem{lemma}{Lemma}

\theoremstyle{definition}
\newtheorem{definition}{Definition}
\newtheorem{assumption}{Assumption}

\crefname{equation}{Equation}{Equations}
\crefname{assumption}{Assumption}{Assumptions}
\crefname{definition}{Definition}{Definitions}
\crefname{theorem}{Theorem}{Theorems}
\crefname{lemma}{Lemma}{Lemmas}
\crefname{remark}{Remark}{Remarks}
\crefname{corollary}{Corollary}{Corollaries}

\usepackage[toc,page,header]{appendix}
\usepackage{titletoc}

\usepackage{comment}

\providecommand{\ind}{\mathds{1}}
\providecommand{\E}{\mathbb{E}}

\newcommand{\ceil}[1]{\left\lceil #1 \right\rceil}
\newcommand{\floor}[1]{\left\lfloor #1 \right\rfloor}
\newcommand{\norm}[1]{\left\lVert #1 \right\rVert}
\newcommand{\abs}[1]{\left| #1 \right|}

\providecommand{\KL}[2]{\mathrm{KL}( #1 \| #2 )}
\providecommand{\TV}[2]{\mathrm{TV}( #1 , #2 )}
\newcommand{\brc}[1]{\left( #1 \right)}
\newcommand{\sbrc}[1]{\left[ #1 \right]}
\newcommand{\cbrc}[1]{\left\{ #1 \right\}}
\providecommand{\eps}{\varepsilon}
\providecommand{\T}{\intercal}
\renewcommand{\d}{\mathrm{d}}
\providecommand{\rmW}{\mathrm{W}}
\providecommand{\Tr}{\mathrm{Tr}}
\providecommand{\mbR}{\mathbb{R}}
\providecommand{\mbN}{\mathbb{N}}
\providecommand{\calN}{\mathcal{N}}
\providecommand{\calL}{\mathcal{L}}

\providecommand{\calO}{\mathcal{O}}

\newcommand{\range}{\mathrm{range}}

\DeclareMathOperator*{\argmin}{arg\,min}

\providecommand{\yuchen}[1]{{#1}}
\let\citep\cite

\title{Theory on Score-Mismatched Diffusion Models and Zero-Shot Conditional Samplers}

\author{Yuchen Liang$^\dagger$, \quad Peizhong Ju$^\ddagger$, \quad Yingbin Liang$^\dagger$, \quad Ness Shroff$^\dagger$ \\
$^\dagger$The Ohio State University \qquad $^\ddagger$University of Kentucky
}

\begin{document}

\maketitle

\begin{abstract}
The denoising diffusion model has recently emerged as a powerful generative technique, capable of transforming noise into meaningful data. While theoretical convergence guarantees for diffusion models are well established when the target distribution aligns with the training distribution, practical scenarios often present mismatches. One common case is in the zero-shot conditional diffusion sampling, where the target conditional distribution is different from the (unconditional) training distribution. These score-mismatched diffusion models remain largely unexplored from a theoretical perspective. In this paper, we present the first performance guarantee with explicit dimensional dependencies for general score-mismatched diffusion samplers, focusing on target distributions with finite second moments. We show that score mismatches result in an asymptotic distributional bias between the target and sampling distributions, proportional to the accumulated mismatch between the target and training distributions. This result can be directly applied to zero-shot conditional samplers for any conditional model, irrespective of measurement noise. Interestingly, the derived convergence upper bound offers useful guidance for designing a novel bias-optimal zero-shot sampler in linear conditional models that minimizes the asymptotic bias. For such bias-optimal samplers, we further establish convergence guarantees with explicit dependencies on dimension and conditioning, applied to several interesting target distributions, including those with bounded support and Gaussian mixtures. Our findings are supported by numerical studies.
\end{abstract}

\section{Introduction}

Generative modeling stands as a cornerstone in deep learning, with the goal of producing samples whose distribution emulates that of the training data. Traditional approaches encompass variational autoencoders (VAE) \citep{kingma2013vae}, generative adversarial networks (GANs) \citep{goodfellow2014gan}, normalizing flows \citep{rezende2015normalizingflow}, and others. Recently, diffusion models, especially the denoising diffusion probabilistic models (DDPMs) \citep{sohldickstein2015,ho2020ddpm}, have emerged as particularly compelling generative models, gaining widespread acclaim for their stable and cutting-edge performance across various tasks, such as image and video generation \citep{dalle2,rombach2022stable-diffusion}. 
% We refer interested readers to for general surveys of diffusion models in diverse applications \cite{diffusion-survey-song,diffusion-survey-medical}.

In ideal situations, the training and target distributions of generative models match each other. However, this often does not hold in practice, where distributional mismatch between the training and target distributions can occur due to various reasons such as possible privacy constraints, need for computational efficiency, and knowledge gap between training and sampling processes. Specifically for diffusion models, such mismatches exhibit between the scores obtained from the training data and the scores of the target distribution from which we want to generate samples. One common scenario that existing studies primarily focus on is \textbf{conditional} diffusion models in image generation tasks (see \cite{diffusion-survey-croitoru,diffusion-survey-img-restore,diffusion-survey-super-res} for surveys of diffusion models in computer vision).
% In practical applications, the success in diffusion models can be witnessed, in particular, in tasks of {\bf conditional} image generation (e.g., \cite{dalle2,rombach2022stable-diffusion}). 
% Tasks of conditional image generation include image super-resolution, inpainting, editing, text-to-image synthesis, to name a few.
Different from unconditional image generation, conditional image samplers aim to generate images that are consistent with the given information, either be a text-prompt (as in text-to-image synthesis) or a sub-image (as in image super-resolution). For example, in image super-resolution, given the input of a low-resolution image, the goal is not to generate some arbitrary high-resolution image but the one whose corresponding low-resolution part matches the given input. Here the diffusion models are well-trained on the \textit{unconditional} distribution of high-resolution images, whereas the target distribution is the \textit{conditional} distribution given the low-resolution input. If one uses these well-trained unconditional scores to generate conditional samples, there will be a mismatch at each step of the sampling process.
% \yl{please explain the example in the context of what is training distribution and what is the target, and what causes their mismatch}

One class of methods to tackle the conditional sampling problem is to include {\em extra-guided training}, where a modified score function is trained with the \textit{extra} knowledge of the conditioning information \citep{dhariwal2021diff-beat-gan,ho2022classifierfree}. On the theory side, several recent works \citep{yuan2023reward-directed,wu2024guidance,fu2024guidance} 
%on conditional diffusion samplers investigate the {\em extra-guided training} method, 
provided performance guarantees for such conditional diffusion samplers, where a score guidance is obtained through {\em extra} training based on the conditional information. However, the additional guided training in these samplers requires {\em extra} computations and needs to be conducted for every image conditioning, which may not be efficient in practice.

Alternatively, \textit{zero-shot} conditional image samplers arise as a prevalent approach (e.g., \cite{choi2021ilvr,chung2022ccdf,chung2022mcg,chung2023dps,wang2023ddnm,song2023pgdm,fei2023gdp}) for {\em training-free} conditional generation given well-trained unconditional scores. For each conditioned image, zero-shot samplers require no additional training to modify the scores. Instead, they adjust the scores during sampling by calculating rectified scores based on conditional information to mitigate the mismatch between the oracle conditional scores and the approximated ones.\footnote{Note, however, that some zero-shot methods, such as DPS \citep{chung2023dps} and $\Pi$GDM \citep{song2023pgdm}, might induce additional computational costs during sampling.} Despite their empirical promise, theoretical guarantee on these zero-shot samplers is largely unexplored. 
%In contrast, \yuchen{only \cite{gupta2024post-samp-intract,xu2024dpnp} studied the {\em zero-shot} case from the theoretical perspective. 
In \cite{gupta2024post-samp-intract}, the authors provided a super-polynomial lower bound for zero-shot sampling as a converse result. In \cite{xu2024dpnp}, the authors proposed and analyzed a {\em plug-and-play} conditional sampler. Their analysis relies on the properties of the Markov transition kernel specific to their plug-and-play model, which does not appear to be applicable to several widely used zero-shot samplers, such as Come-Closer-Diffuse-Faster (CCDF) \citep{chung2022ccdf} and Denoising Diffusion Null-space Model (DDNM) \citep{wang2023ddnm}. Therefore, there is a need to provide the performance guarantee for those popular zero-shot conditional samplers.

In this paper, we address two key theoretical research gaps in zero-shot score-mismatched diffusion models: (i) We provide performance guarantees for general score-mismatched diffusion models, extending their applicability beyond the primary focus of existing theoretical studies on the special case of conditional image generation. (ii) We analyze zero-shot conditional diffusion models, which are generally applicable to existing zero-shot samplers such as CCDF \citep{chung2022ccdf} and DDNM \citep{wang2023ddnm} to which the analysis in \cite{xu2024dpnp} is not applicable (as discussed above).

\subsection{Our Contributions}\label{sec:maincontributions}

%In this paper, we make a first step towards analyzing the performance of score-mismatched DDPM samplers, with a particular emphasis on the zero-shot DDPM conditional samplers.

Technically, the main challenge due to mismatched scores is to analyze the expected tilting factor \citep{liang2024discrete} under a mean-perturbed Gaussian, providing an upper bound of the asymptotic orders of all Gaussian non-centralized moments. Our detailed contributions are as follows.

\textbf{Convergence of General Score-Mismatched DDPM}: We provide the \textit{first} non-asymptotic convergence bound on the KL divergence between the target and generated distributions when there is mismatch between the sampling and target scores in DDPM samplers, for general target distributions having finite second moments. We show that the score mismatch at each diffusion step introduces an asymptotic distributional bias that is proportional to the accumulated mismatch. We also provide the first explicit dimensional dependency when the sixth moment of the target distribution exists. Our result is applicable to general forms of mismatch between the target and training scores, which greatly extend the focus of the existing theoretical research on conditional score-mismatch diffusion models. 

We then apply our results to zero-shot conditional DDPM samplers, as long as the conditioning involves certain deterministic or random transformations of the data. This provides the first theoretical guarantees for several existing zero-shot samplers, such as CCDF \citep{chung2022ccdf} and DDNM \citep{wang2023ddnm}. Notably, the theory in \cite{xu2024dpnp} does not apply to these samplers, as their analysis relies on the properties of the Markov transition kernel specific to their plug-and-play model. In contrast, our approach is based on the tilting-factor analysis from \cite{liang2024discrete}, which is applicable to general score-mismatched DDPM models. Moreover, the theory in \cite{xu2024dpnp} is limited to cases where the measurement log-likelihood function is differentiable and bounded and does not provide explicit dependencies on the data dimension. In contrast, our results do not require the measurement log-likelihood function to be differentiable or bounded and explicitly characterize the dependencies on the data dimension.

%Note that our theory is completely different from that in \cite{xu2024dpnp} and is the first one that can be applied to many existing zero-shot samplers, such as CCDF \citep{chung2022ccdf} and DDNM \citep{wang2023ddnm}. In particular, different from \cite{xu2024dpnp}, we do not require the measurement log-likelihood function to be differentiable or bounded (i.e., such assumption is violated when there is no measurement noise).

\textbf{Novel Design of Bias-Optimal Zero-shot Sampler BO-DDNM}: Inspired by our convergence analysis of score-mismatched DDPM, we design a novel zero-shot conditional sampler, called the BO-DDNM sampler, which minimizes the asymptotic bias for linear conditional models. Such a sampler coincides with the regular DDNM sampler \citep{wang2023ddnm} when there is no presence of measurement noise, and achieves faster convergence than both the DDNM and DDNM$^+$ samplers under measurement noise, as shown by our theory and numerical simulations.

\textbf{Theory for BO-DDNM with Explicit Parameter Dependencies}: We provide the convergence bound for the proposed BO-DDNM sampler with explicit dependencies on the dimension $d$ as well as the conditional information $y$, for various interesting classes of target distributions including those having bounded support and Gaussian mixture. 
For the case of Gaussian mixture, we further show that three factors positively affect the asymptotic bias: (1) the variance of the measurement noise, (2) the averaged distance between $y$ and the mean of each Gaussian component, and (3) the corresponding correlation coefficient for each component.

\subsection{Related Work}

We provide a summary of works addressing {\em unconditional} and {\em score-matched} diffusion models in \Cref{sec:intro-works}. Below we discuss related works on {\em conditional} diffusion models which are closely related to our study here.

\textbf{Extra-Guided Training:} In order to achieve conditional sampling using DDPM models in practice, one method is to introduce conditional guided training, where one either uses an existing classifier (a.k.a., classifier guidance) \citep{dhariwal2021diff-beat-gan} or jointly trains the unconditional and conditional scores (a.k.a., classifier-free guidance (CFG)) \citep{ho2022classifierfree}. 
Here a guidance term is obtained to ``guide'' the diffusion sampling process at each step such that the sampling scores correspond to the true conditional scores. 
% The major downside of this method is the extra computational cost to train the modified score when the unconditional scores have been well trained. For example, given a new conditioning, a new classifier needs to be trained based on the conditioning \citep{dhariwal2021diff-beat-gan}. 
% Also, early theoretical studies of conditional diffusion models also assume training of conditional score functions. For example, 
On the theory side, \cite{wu2024guidance} investigates the effect of the guidance strength in CFG on Gaussian mixtures, \cite{bradley2024cfg-pred-corr} shows that CFG is an instance of predictor-corrector methods, and \cite{chidambaram2024guidance} finds that CFG might fail to sample correctly on certain mixture targets. There are other theoretical works that investigate sample complexity bounds for conditional score matching for a variety of target distribution models, including the conditional Ising models \citep{mei2023graphical}, those supported on a low-dimensional linear subspace \citep{yuan2023reward-directed}, and H\"older smooth and sub-Gaussian conditional models \citep{fu2024guidance}. Other than stochastic samplers, a conditional ODE sampler is proposed and studied in \cite{chang2024cond-ode}, which also requires extra training of the conditional score function.

\textbf{Zero-shot Samplers:} To achieve conditional DDPM sampling, a popular method is to use zero-shot conditional samplers, with which one generates a conditional sample using approximated scores. These scores are calculated from the unconditional score estimates and the conditional information using simple (usually linear) functions {\em without extra-training} \citep{choi2021ilvr,chung2022mcg,chung2022ccdf,chung2023dps,wang2023ddnm,song2023pgdm,fei2023gdp}. 
% See \cite[Section~V.E]{diffusion-survey-super-res} for a list of zero-shot conditional image samplers. \yl{we need to cite those papers here}
% Equivalently, this problem can be casted as a posterior sampling problem, where one aims to sample from $q(x|y)$ given knowledge of the prior $q(y)$ and the likelihood $q(y|x)$. In many computer vision applications, the likelihood $q(y|x)$ usually adopts the simple Gaussian linear model (such as image inpainting and super-resolution). 
The only theoretical works on the performance of zero-shot DDPM conditional samplers are \cite{xu2024dpnp,gupta2024post-samp-intract}. In \cite{xu2024dpnp}, a diffusion plug-and-play sampler is proposed which alternates between a diffusion sampling step and a consistency sampling step. The difference of our results from those in \cite{xu2024dpnp} has been thoroughly discussed in \Cref{sec:maincontributions}.
%Our result is different from theirs in that (1) our approach is based on the tilting-factor analysis in \cite{liang2024discrete} which is applicable to score-mismatched DDPM models, while theirs is based on properties of the Markov transition kernel which is applicable to their particular plug-and-play model; (2) our convergence bound is applicable when there is no measurement noise, while theirs is restricted to differentiable log-likelihood, which implies positive measurement noise; (3) our result features explicit dimensional dependency, as well as that on the conditioning, while their dimensional and other system parameter dependencies are implicit. 
% \yuchen{Emphasize the difference in theory, not in algorithm (as summary). Put more details in related works.} which approaches the problem via the lens of Bayesian posterior sampling and assumes smooth measurement noise. This work poses many open questions/issues for zero-shot conditional samplers which call for future research. (i) In \cite{xu2024dpnp} a strictly positive measurement noise is assumed. Is there a unified analysis that can be applied whether or not there is measurement noise? (ii) A specific plug-and-play model is employed in \cite{xu2024dpnp}. Is there some theory to analyze general zero-shot diffusion samplers, not only the plug-and-play ones? (iii) Since the dimensional dependencies are only implicit in \cite{xu2024dpnp}, is it possible to obtain explicit dependencies on the data dimension, as well as on other system parameters?
From an alternative perspective, \cite{gupta2024post-samp-intract} shows that the sampling complexity with zero-shot samplers can take super-polynomial time for some worst-case distribution (among the set of distributions where smooth scores can be efficiently estimated). Instead, our result shows a consistent fact that there exists a non-vanishing asymptotic distributional bias within polynomial time.

\section{Problem Setup}

%In this paper, we are interested in analyzing the performance of \textit{zero-shot} conditional DDPMs. 
In this section, we first provide some background on the score-{\em matched} DDPMs. Then, we introduce the score-{\em mismatched} DDPM samplers and, as a special example, the {\em conditional} sampling problem and zero-shot samplers.

\subsection{Background of Score-Matched DDPMs}
\label{sec:ddpm-bg}

The goal of the score-matched sampling problem is to generate a sample whose distribution is close to the data distribution. To this end, the DDPM algorithm \citep{ho2020ddpm} is widely used, which consists of a forward process and a reverse process of latent variables. Let $x_0 \in \mbR^d$ be the initial data, and let $x_t \in \mbR^d, \forall 1 \leq t \leq T$ be the latent variables.
Let $Q_0$ be the data distribution, and let $Q_t$ (resp., $Q_{t,t-1}$) be the marginal (resp., joint) latent distribution for all $ 1 \leq t \leq T$. 
% Denote $Q := Q_{0,\dots,T}$ as the overall joint distribution.

%\subsubsection{Forward Process}
{\bf Forward Process:}
In the forward process, white Gaussian noise is gradually added to the data: $x_t = \sqrt{1-\beta_t} x_{t-1} + \sqrt{\beta_t} w_t,~\forall 1 \leq t \leq T$,
where $w_t \stackrel{i.i.d.}{\sim} \calN(0,I_d)$. Equivalently, this can be expressed as:
% a conditional distribution at each time $t$ given $x_{t-1}$:
\begin{equation} \label{eq:def_forward_proc}
    Q_{t|t-1}(x_t|x_{t-1}) = \calN(x_t; \sqrt{1-\beta_t} x_{t-1}, \beta_t I_d),
\end{equation}
which means that under $Q$, the Markov chain $X_0 \to X_1 \to \dots \to X_T$ holds. 
% Here $\beta_t \in (0,1)$ captures the ``amount'' of noise that is injected at time $t$, and $\beta_t$'s are called the \textit{noise schedule}. 
Define $\alpha_t := 1-\beta_t$ and $\Bar{\alpha}_t := \prod_{i=1}^t \alpha_i$ for all $1 \leq t \leq T$.
An immediate result by accumulating the steps is that $Q_{t|0}(x_t|x_0) = \calN(x_t; \sqrt{\Bar{\alpha}_t} x_{0}, (1-\Bar{\alpha}_t) I_d)$,
or, written equivalently, $x_t = \sqrt{\Bar{\alpha}_t} x_0 + \sqrt{1-\Bar{\alpha}_t} \Bar{w}_t,~\forall 1 \leq t \leq T$,
where $\Bar{w}_t \sim \calN(0,I_d)$ denotes the \textit{aggregated} noise at time $t$ and is independent of $x_0$.
% Intuitively, for large $T$, since $Q_{T|0} \approx \calN(0,I_d)$ (which is independent of $x_0$), it is expected that $Q_{T} \approx \calN(0,I_d)$ when $T$ becomes large, as long as the variance under $Q_0$ is finite.
Finally, since each $\Bar{w}_t$ is Gaussian, each $Q_t~(t \geq 1)$ is absolutely continuous w.r.t. the Lebesgue measure. Let the p.d.f. of each $Q_t$ be $q_t$, and $q_{t,t-1}$, $q_{t|t-1}$, and $q_{t-1|t}$ for $t \geq 1$ are similarly defined. 
% In case $Q_0$ is also absolutely continuous w.r.t. the Lebesgue measure, let $q_0$ be the corresponding p.d.f. of $Q_0$.

% \yuchen{Shall we mention SDE here?}

%\subsubsection{Reverse Process}

% The goal of DDPM is to generate samples approximately from the data distribution $Q_0$.
{\bf Reverse Process:} In the reverse process, the latent variable at time $T$ is first drawn from a standard Gaussian distribution: $x_T \sim \calN(0,I_d) =: P_T$.
Then, each forward step is approximated by a reverse sampling step.
% Define
% \begin{align}\label{eq:def_mu_sigma_t}
%     \mu_t(x_t):= \frac{1}{\sqrt{\alpha_t}} \brc{ x_t + (1-\alpha_t) \nabla \log q_t(x_t) },\quad \sigma_t^2:= \frac{1-\alpha_t}{\alpha_t}.
% \end{align}
At each time $t = T,T-1,\dots,1$, define the \textit{true} reverse process as $x_{t-1} = \mu_t(x_t) + \sigma_t z_t$,
% \[  x_{t-1} = \frac{1}{\sqrt{\alpha_t}} \brc{ x_t + (1-\alpha_t) \nabla \log q_t(x_t) } + \sqrt{\frac{1-\alpha_t}{\alpha_t}} z = \mu_t(x_t) + \sigma_t z \]
where $z_t \sim \calN(0,I_d)$. Here $\sigma_t^2 := \frac{1-\alpha_t}{\alpha_t}$. For the typical DDPM sampling process, $\mu_t(x_t) = \frac{1}{\sqrt{\alpha_t}} \brc{ x_t + (1-\alpha_t) \nabla \log q_{t}(x_t) } $. Equivalently, $P_{t-1|t} = \calN(x_{t-1};\mu_t(x_t), \sigma_t^2 I_d)$.
Here $\nabla \log q_t(x)$ is called the \textit{score} of $q_t$, and $\mu_t(x_t)$ is a function of the score. Let $P_t$ be the marginal distributions of $x_t$ in the true reverse process, and let $p_t$ be its corresponding p.d.f. w.r.t. the Lebesgue measure. Define $p_{t-1|t}$ and $p_{t|t-1}$ in a way similar to the forward process. 

In practice, one does not have access to $\nabla \log q_t(x_t)$ and thus $\mu_t(x_t)$. Instead, an estimate of $\nabla \log q_t(x_t)$, denoted as $s_t(x_t)$, is used, which results in an estimated $\widehat{\mu}_t(x_t)$ and the \textit{estimated} reverse process: $x_{t-1} = \widehat{\mu}_t(x_t) + \sigma_t z_t$.
% \[  x_{t-1} = \frac{1}{\sqrt{\alpha_t}} \brc{ x_t + (1-\alpha_t) s_t(x_t) } + \sqrt{\frac{1-\alpha_t}{\alpha_t}} z = \widehat{\mu}_t(x_t) + \sigma_t z \]
% in which we define
% \begin{equation}\label{eq:def_mu_hat_t}
%     \widehat{\mu}_t(x_t) := x_t + (1-\alpha_t) s_t(x_t).
% \end{equation}
Let $\widehat{P}_t$ be the marginal distributions of $x_t$ in the estimated reverse process with the corresponding p.d.f. $\widehat{p}_t$. Note that $\widehat{P}_{t-1|t} = \calN(x_{t-1};\widehat{\mu}_t(x_t), \sigma_t^2 I_d)$ and $\widehat{P}_T = P_T$. Hence, under $P$ and $\widehat{P}$, $X_T \to X_{T-1} \to \dots \to X_0$ holds.

{\bf Performance Metrics:} In the case where $Q_0$ is absolutely continuous w.r.t. the Lebesgue measure, we are interested in measuring the sampling performance through the KL divergence between $Q_0$ and $\widehat{P}_0$, defined as
\[\textstyle \KL{Q}{P} := \int \log \frac{\d Q}{\d P} \d Q = \E_{X \sim Q} \sbrc{\log \frac{q(X)}{p(X)}} \geq 0. \]
Indeed, from Pinsker's inequality, the total-variation (TV) distance
% defined as
% \[  \TV{Q_0}{\widehat{P}_0} := \sup_{A \subseteq \mbR_d} \abs{Q_0(A)-\widehat{P}_0(A)}. \]
can be upper bounded as $\TV{Q_0}{\widehat{P}_0}^2$ $\leq \frac{1}{2} \KL{Q_0}{\widehat{P}_0}$. When $q_0$ does not exist, we use the Wasserstein-2 distance to measure the one-step perturbed performance, which is defined as
\begin{equation*}
\textstyle     \rmW_2(Q,P) :=\cbrc{\min_{\Gamma\in\Pi(Q,P)} \int_{\mbR^d \times \mbR^d} \norm{x-y}^2 \d \Gamma (x, y) }^{1/2},
\end{equation*}
where $\Pi(Q,P)$ is the set of all joint probability measures on $\mbR^d \times \mbR^d$ with marginal distributions $Q$ and $P$, respectively. Both metrics are widely adopted \citep{chen2023improved,benton2023linear}. 

\subsection{Score-Mismatched DDPMs}
\label{subsec:diff-bias-score}

Differently from the score-{\em matched} sampling problem, the goal of the score-\textit{mismatched} problem is to sample from a \textbf{different} target distribution from the training distribution with which we estimate the scores. Thus, there will be a mismatch between the target score and the estimated score at each diffusion step.
% Here we need to differentiate between the \textit{training} distribution and the \textit{target} distribution. 
Let $Q_t$ ($t \geq 0$) be the \textit{training} distributions used for training the score.
% and let $P_t$ ($t \geq 0$) be those generated from the DDPM sampler. 
Let $\Tilde{Q}_0$ be the \textit{target} distribution that one hopes to generate samples from, and let $\Tilde{Q}_t$ ($t \geq 1$) be its Gaussian-perturbed distributions according to the forward process in \eqref{eq:def_forward_proc}.
Define the posterior mean under the target distributions as $m_{t}(\Tilde{x}_t) := \E_{\Tilde{X}_{t-1} \sim \Tilde{Q}_{t-1|t}} [\Tilde{X}_{t-1}|\Tilde{x}_t]$. Note that by Tweedie's formula \citep{tweedie2011efron}, $m_{t}(\Tilde{x}_t) = \frac{1}{\sqrt{\alpha_t}} \brc{ \Tilde{x}_t + (1-\alpha_t) \nabla \log \Tilde{q}_{t}(\Tilde{x}_t) }$. \yuchen{Recall that $P_t$ and $\widehat{P}_t$ are the \textit{sampling} distributions of the true and estimated reverse process, respectively.}
For general score-mismatched DDPMs, we leave the generic definition of $\mu_{t}(x_t)$ without providing any particular expression. An example of $\mu_t(x_t)$ is given later in \eqref{eq:def_cond_spl}, yet
% An example is that $\mu_{t}(x_t) = \frac{1}{\sqrt{\alpha_t}} \brc{ x_t + (1-\alpha_t) g_{t}(x_t) }$, where $g_{t}(x_t)$ is a simple function of $\nabla \log q_t(x_t)$. 
our general analysis does not require any particular form for $\mu_t$. 
% Define $\widehat{\mu}_{t}$ likewise.
With these notations, the \textit{score mismatch} at each step $t \geq 1$ can be defined as
% Employing similar definitions as \eqref{eq:def_delta}, we have
% \begin{equation} \label{eq:def_delta_general}
%     \mu_{t}(x_t) = \E_{X_{t-1} \sim P_{t-1|t}} [X_{t-1}], \quad \Delta_t(x_t) = \frac{\sqrt{\alpha_t}}{1-\alpha_t} (m_t(x_t) - \mu_t(x_t)).
% \end{equation}
\begin{equation} \label{eq:def_delta_general}
\textstyle    \Delta_{t}(x_t) := \frac{\sqrt{\alpha_t}}{1-\alpha_t} \brc{\E_{X_{t-1} \sim \Tilde{Q}_{t-1|t}} [X_{t-1}|x_t] - \E_{X_{t-1} \sim P_{t-1|t}} [X_{t-1}|x_t] } = \frac{\sqrt{\alpha_t}}{1-\alpha_t} (m_{t}(x_t) - \mu_{t}(x_t)). 
    % = \nabla \log \Tilde{q}_{t}(x_t) - g_{t}(x_t).
\end{equation}
\yuchen{The goal, then, is to provide an upper bound on the distributional dissimilarity between the target distribution $\Tilde{Q}_0$ and the sampling distribution $\widehat{P}_0$.} We use the same metrics as those defined in \Cref{sec:ddpm-bg} to evaluate the performance of the score-mismatched DDPM.
% Below, our main results in \Cref{sec:main_res} are applicable to the diffusion problem under general score mismatch, for which we provide a complete setup in \Cref{app:prob-general}. Specifically, define $m_{t,y}(x_t) := \E_{X_{t-1} \sim Q_{t-1|t,y}} [X_{t-1}]$ to be the true posterior mean.
% In general, $\Delta_{t}(x_t)$ can be computed for any DDPM diffusion problem with score mismatch, e.g., 

% Another scenario is where the scores are trained on a different distribution from the true target $Q_0$.

\subsection{Zero-Shot Conditional DDPMs}
\label{subsec:cond-ddpm}

One interesting example of score mismatch is the zero-shot conditional sampling problem.
% when $y = h(x_0)$ where $h(\cdot)$ is some arbitrary (deterministic or random) function of only $x_0$.
Differently from the unconditional counterpart, the conditional sampling problem aims to obtain a sample that aligns in particular with the provided conditioning. Define $y \in \mbR^p$ to be the conditioned information about $x_0$. Specifically, let $y = h(x_0)$, where $h(\cdot)$ is some arbitrary (deterministic or random) function of only $x_0$ (apart from independent noise). Note that general score-mismatched DDPMs can be specialized to zero-shot conditional samplers with the following notations:
% Note that our result for score-mismatched DDPMs in \Cref{sec:prob-general} can be directly applied to zero-shot conditional samplers with the following notations: \yl{should we say ``Note that score-mismatched DDPMs can specialize to zero-shot conditional samplers with the following notations:"?}
\begin{equation} \label{eq:thm1_abbrv_notation}
    \Tilde{Q}_t = Q_{t|y},~~m_{t} = m_{t,y},~~\mu_{t} = \mu_{t,y},~~\text{and}~\Delta_t = \Delta_{t,y}.
\end{equation}
% \yl{Did you define $g_t$ in the mismatch?}

%\subsubsection{Linear Conditional Models}

{\bf Linear Conditional Models:} In practice, one commonly adopted model is the linear conditional model \citep{jalal2021mri,wang2023ddnm,song2023pgdm}, defined as
\begin{equation} \label{eq:def_cond_model}
    y := H x_0 + n,
\end{equation}
where $H \in \mbR^{p \times d}$ ($p \leq d$) is a deterministic matrix and $n \sim \calN(0, \sigma_y^2 I_p)$ is the measurement noise, which is assumed to be Gaussian and independent of $x_0$. For the case where there is no measurement noise, let $\sigma_y^2 = 0$ and thus $n = 0$ almost surely.
% extensions to non-linear models can be sought \citep{song2023pgdm,kawar2022jpeg}.
In applications like image super-resolution and inpainting \citep{wang2023ddnm}, $H$ admits a simple form of a 0-1 diagonal matrix, where the 1's occur only on the diagonal and at those locations corresponding to the provided pixels. 
In these scenarios, both $H$ and $y$ are fixed and given. 
The linear conditional model is studied in \Cref{sec:delta_ty}. 
% Given our theory on general score-mismatched DDPM, we will further propose an optimized conditional sampler for such linear models.

%\subsubsection{Conditional Forward Process for Linear Models}

{\bf Conditional Forward Process for Linear Models:} Write the Moore–Penrose pseudo-inverse of $H$ as $H^\dagger$, and note that $H^\dagger H$ is an orthogonal projection matrix. With this notation, under \eqref{eq:def_cond_model}, we can re-express the forward process in \eqref{eq:def_forward_proc} as
\[ x_t = \sqrt{\Bar{\alpha}_t} (I_d - H^\dagger H) x_0 + \sqrt{\Bar{\alpha}_t} H^\dagger y - \sqrt{\Bar{\alpha}_t} H^\dagger n + \sqrt{1-\Bar{\alpha}_t} \Bar{w}_t. \]
% Define $Q_{t|y}$ as the conditional distribution given $y$ at time $t \geq 0$. 
Here, since $n$ is independent of $\Bar{w}_t$, for fixed $x_0$ and $y$, we have that, for all $t \geq 1$,
\begin{equation} \label{eq:def_cond_fwd2}
    Q_{t|0,y}(x_t|x_0,y) = \calN(x_t; \sqrt{\Bar{\alpha}_t} (I_d - H^\dagger H) x_0 + \sqrt{\Bar{\alpha}_t} H^\dagger y, \Bar{\alpha}_t \sigma_y^2 H^\dagger (H^\dagger)^\T + (1-\Bar{\alpha}_t) I_d).
\end{equation}
Also, since the forward process is a Markov chain, we have that $Q_{t|t-1,y} = Q_{t|t-1}$ for all $t \geq 1$.
% Since $Q_{0|y}$ is supported on a low-dimension manifold, we are interested in measuring the performance of conditional diffusion models through $\KL{Q_{1|y}}{\widehat{P}_{1|y}}$ (as a function of $y$).

%\subsubsection{Zero-shot Conditional Sampler for Linear Models}\label{sec:intro-zero-shot}

% Suppose that one is given the conditional score, $\nabla \log q_{t|y} (x_t|y)$, then the conditional sampling problem becomes a matched sampling problem. 
% Nevertheless, even if one has an approximated unconditional score, it still requires extra training to obtain the conditional score for each incoming $y$, which creates computational burden. Instead, 
{\bf Zero-shot Conditional Sampler for Linear Models:} We employ the \textit{zero-shot} conditional sampler for linear conditional models in the following form: $x_{t-1} = \mu_{t,y}(x_t) + \sigma_t z_t$, where
\begin{equation} \label{eq:def_cond_spl}
\textstyle    \mu_{t,y}(x_t) = \frac{1}{\sqrt{\alpha_t}} \brc{ x_t + (1-\alpha_t) g_{t,y}(x_t) },\quad g_{t,y} := (I_d-H^\dagger H) \nabla \log q_t (x_t) + f_{t,y}(x_t).
\end{equation}
Here $f_{t,y}(x_t)$ is a simple function of $y$ and $x_t$ computable \textit{without extra training} and such that $(I_d-H^\dagger H) f_{t,y}(x) \equiv 0$ for all $x \in \mbR^d$. 
% We omit dependencies of $y$ when they are obvious.
Intuitively, $f_{t,y}$ characterizes the score rectification in the range space of $H^\dagger H$.
Indeed, many zero-shot samplers in the literature have such $f_{t,y}(x_t)$'s that satisfy \eqref{eq:def_cond_spl} (see \Cref{app:fy-literature}). 
Now, with the linear model in \eqref{eq:def_cond_model} and the zero-shot conditional sampler in \eqref{eq:def_cond_spl}, the score mismatch at each time $t \geq 1$ is equal to 
\begin{equation} \label{eq:def_delta_linear_cond}
    \Delta_{t,y}(x_t) = (I_d-H^\dagger H) (\nabla \log q_{t|y} (x_t) - \nabla \log q_t (x_t)) + (H^\dagger H) \nabla \log q_{t|y} (x_t) - f_{t,y}(x_t).
\end{equation}
% and the choice in \eqref{eq:def_fy} will be shown to enjoy some theoretical benefit (See \Cref{lem:cond_score_proj}). 
% Equivalently, we can write $P_{t-1|t,y} = \calN(x_{t-1};\mu_{t,y}(x_t), \sigma_t^2 I_d)$. Similarly define the \textit{approximated} estimated reverse process and $\widehat{P}_{t-1|t,y}$ with $\nabla \log q_t (x_t)$ replaced by $s_t(x_t)$. Let $\widehat{P}_{T|y} = \calN(0,I_d)$.

% \yuchen{Is it possible to extend to non-linear models?}

\section{DDPM under General Score Mismatch}
\label{sec:prob-general}

In this section, we provide convergence guarantees for general score-mismatched DDPM samplers under a general target distribution $\Tilde{Q}_0$. Throughout this section we keep the generic definition for score mismatch $\Delta_{t}$ as in \eqref{eq:def_delta_general}, without assuming any particular expression for $\mu_t$. 
% Note that all the results in this section can be extended to arbitrary $y = h(x_0)$, where $h(\cdot)$ is some arbitrary (deterministic or random) function of only $x_0$ (in which case $Q_{t|y}$ is still well defined for all $t \geq 0$). We have provided a more general problem formulation in \Cref{app:prob-general}. 

\subsection{Technical Assumptions}

We will analyze general score-mismatched DDPMs under the following technical assumptions.

\begin{assumption}[Finite Second Moment] \label{ass:m2-general}
    There exists a constant $M_2 < \infty$ (that does not depend on $d$ and $T$) such that $\E_{X_0 \sim \Tilde{Q}_{0}} \norm{X_0}^2 \leq d M_2$.
\end{assumption}

The first \Cref{ass:m2-general} is commonly adopted in the analyses of score-matched DDPM samplers \citep{chen2023improved,chen2023sampling,liang2024discrete}.

\begin{assumption}[Posterior Mean Estimation] \label{ass:score-general}
%Let $\widehat{\mu}_t$ be the estimated posterior mean at time $t = 1,\dots,T$. The $\widehat{\mu}_t$ estimates 
The estimated posterior mean $\widehat{\mu}_t$ at $t = 1,\dots,T$ satisfy
\[ \textstyle \frac{1}{T} \sum_{t=1}^T \frac{\alpha_t}{(1-\alpha_t)^2} \E_{X_t \sim \Tilde{Q}_{t}} \norm{\widehat{\mu}_t(X_t) - \mu_t(X_t)}^2 \leq \eps^2,~\text{where}~\eps^2 = \Tilde{O}(T^{-2}). \]
\end{assumption}

The above \Cref{ass:score-general} is made for the score estimation error for the general mismatched setting, where we leave generic definitions of $\mu_t$ and $\widehat{\mu}_t$. \yuchen{While the expectation is over $\Tilde{Q}_t$, \Cref{ass:score-general} is very likely to hold when $\Tilde{Q}_t$ is close to $Q_t$, i.e., when the score mismatches are moderate.} For zero-shot conditional samplers in linear models, this assumption is weaker than that for the estimation error for unconditional scores (see \eqref{eq:zs-est-err}).
\yuchen{Compared with the score-matched case, the estimation error needs to be achieved at a higher accuracy because of the extra error term when there is score mismatch (\Cref{lem:score-est-ptb}). Such a higher level of estimation accuracy also occurs in previous theoretical studies for accelerated DDPM samplers \citep{li2024accl-prov}.} 

\begin{assumption}[Regular Derivatives] \label{ass:regular-drv-general}
For all $t \geq 1$ where $\Tilde{q}_{t-1}$ exists, 
% $m \geq \frac{1}{2}$, 
$\ell \geq 1$, and $\bm{a} \in [d]^p$ where $\abs{\bm{a}} = p \geq 1$,
\yuchen{\[ \E_{X_t \sim \Tilde{Q}_t} \abs{\partial_{\bm{a}}^p \log \Tilde{q}_{t}(X_t) }^\ell = O\brc{1},\quad \E_{X_t \sim \Tilde{Q}_t} \abs{\partial_{\bm{a}}^p \log \Tilde{q}_{t-1}(m_t(X_t)) }^\ell = O\brc{1}. \]}
% \begin{align*}
%     &(1-\alpha_t)^{m} \E_{X_t \sim \Tilde{Q}_{t}} \abs{\partial_{\bm{a}}^p \log \Tilde{q}_{t}(X_t) }^\ell = O\brc{(1-\alpha_t)^{m}},\\
%     &(1-\alpha_t)^{m} \E_{X_t \sim \Tilde{Q}_{t}} \abs{\partial_{\bm{a}}^p \log \Tilde{q}_{t-1}(m_{t}(X_t)) }^\ell = O\brc{(1-\alpha_t)^{m}}.
% \end{align*}
\end{assumption}

The above \Cref{ass:regular-drv-general} is useful for our tilting-factor based analysis, which guarantees that all (higher-order) Taylor polynomials of $\log \Tilde{q}_{t}$ are well controlled in expectation. 
\yuchen{It is rather soft, and it can be verified when $\Tilde{Q}_0$ has finite variance (under early-stopping) \citep{liang2024discrete}. }

\begin{assumption}[Bounded Mismatch] \label{ass:bdd-mismatch-general}
    For all $t \geq 1$ where $\Tilde{q}_{t-1}$ exists, and $\ell \geq 2$,
    % and $m \geq 1$,
    \yuchen{\[ \E_{X_t \sim \Tilde{Q}_{t}}\norm{\Delta_{t}(X_t) }^\ell = O(\Bar{\alpha}_t). \]}
    % \[ (1-\alpha_t)^{m} \E_{X_t \sim \Tilde{Q}_{t}}\norm{\Delta_{t}(X_t) }^\ell = O(\Bar{\alpha}_t (1-\alpha_t)^{m}). \]
\end{assumption}

The above \Cref{ass:bdd-mismatch-general} is used to characterize the amount of mismatch at each time $t \geq 1$. The $\Bar{\alpha}_t := \prod_{i=1}^t \alpha_i$ is necessary for the overall bias to be bounded. 
% In the particular problem of zero-shot sampling under linear conditional models in \eqref{eq:def_cond_model}, \Cref{ass:bdd-mismatch-general} will be shown for a variety of target distributions, including those having bounded support and Gaussian mixtures (see \Cref{sec:delta_ty}). 

\yuchen{In the paper, \Cref{ass:regular-drv-general,ass:bdd-mismatch-general} have been established in two cases of zero-shot conditional sampling: (i) where $Q_0$ has bounded support for any $H$, using a special $\alpha_t$ in \eqref{eq:alpha_genli} (see the proof of \Cref{lem:norm2_bd_bdsupp}); and (ii) where $Q_0$ is Gaussian mixture and $H = \begin{pmatrix} I_p & 0 \end{pmatrix}$ (see \Cref{lem:bdd-mismatch-gauss-mix}). For Case (i), the assumption that $Q_0$ has bounded support has wide applicability in practice (e.g., images \citep{ho2020ddpm,wang2023ddnm}) and is commonly made in many theoretical investigations of the score-matched DDPM \citep{li2024accl-prov,li2023faster}. }
% Also, where $Q_0$ is Gaussian mixture, the proof of \Cref{lem:bdd-mismatch-gauss-mix} can be easily extended for general $H$.
Finally, note that when $\Tilde{q}_0$ does not exist (e.g., for images \citep{ho2020ddpm,wang2023ddnm}), \Cref{ass:regular-drv-general,ass:bdd-mismatch-general} are required only for $t \geq 2$.

\subsection{Convergence Bound}

Before presenting the main result, we first define a set of noise schedule as follows.
\begin{definition}[Noise Schedule] \label{def:noise_smooth}
    For all sufficiently large $T$, set the step size $\alpha_t$'s to satisfy
    \begin{equation*}
\textstyle        1-\alpha_t \lesssim \frac{\log T}{T},~\forall 1 \leq t \leq T,\quad \text{and} \quad \Bar{\alpha}_T := \prod_{t=1}^T \alpha_t = o\brc{\frac{1}{T}}. % o\brc{T^{-1}}.
    \end{equation*}
\end{definition}
An example of $\alpha_t$ that satisfies \cref{def:noise_smooth} is $1-\alpha_t \equiv \frac{c \log T}{T},~\forall t \geq 1$ with $c > 1$. Then, $\Bar{\alpha}_T = \brc{1 - \frac{c \log T}{T}}^{T} = \exp\brc{T \log \brc{1 - \frac{c \log T}{T} } } = O\brc{e^{T \frac{- c \log T}{T} } } = o\brc{ T^{-1} }$.

The following \Cref{thm:main-general} \yuchen{provides an upper bound on the KL-divergence between the target distribution $\Tilde{Q}_0$ and the sampling distribution $\widehat{P}_0$, as a function of (general) score-mismatch $\Delta_t$ at each time $t \geq 1$.} \Cref{thm:main-general} is the \textit{first} convergence result for score-mismatched DDPM samplers for any smooth $\Tilde{Q}_{0}$ that has finite second moment (along with some mild regularity conditions).

\begin{theorem}[DDPM under Score Mismatch] \label{thm:main-general}
Suppose that $\Tilde{Q}_0$ has a p.d.f. $\Tilde{q}_0$ which is analytic, and suppose that \Cref{ass:m2-general,ass:regular-drv-general,ass:bdd-mismatch-general,ass:score-general} are satisfied. Then, with the $\alpha_t$ chosen to satisfy \Cref{def:noise_smooth}, the distribution $\widehat{P}_{0}$ from the score-mismatched DDPM satisfies 
\[\KL{\Tilde{Q}_{0}}{\widehat{P}_{0}} \lesssim \mathcal{W}_{\text{oracle}} + \mathcal{W}_{\text{bias}} + \mathcal{W}_{\text{vanish}}, \quad\text{where}\] %where
\begin{align*}
\mathcal{W}_{\text{oracle}} & \textstyle= \sum_{t=1}^T \frac{(1-\alpha_t)^2}{2 \alpha_t} \E_{X_t \sim \Tilde{Q}_t} \sbrc{ \Tr\Big(\nabla^2 \log \Tilde{q}_{t-1}(m_{t}(X_t))\nabla^2\log \Tilde{q}_{t}(X_t)  \Big) } + (\log T) \eps^2\\
\mathcal{W}_{\text{bias}} & \textstyle= \sum_{t=1}^T (1-\alpha_t) \E_{X_t \sim \Tilde{Q}_t} \norm{\Delta_{t}(X_t)}^2\\
\mathcal{W}_{\text{vanish}} & \textstyle= \sum_{t=1}^T \frac{1-\alpha_t}{\sqrt{\alpha_t}} \E_{X_t \sim \Tilde{Q}_t} \bigg[(\nabla \log \Tilde{q}_{t-1}(m_{t}(X_t)) - \sqrt{\alpha_t} \nabla \log \Tilde{q}_{t}(X_t))^\T \Delta_{t}(X_t)\bigg]\\
    &\textstyle- \sum_{t=1}^T \frac{(1-\alpha_t)^2}{2 \alpha_t} \E_{X_t \sim \Tilde{Q}_t}\sbrc{\Delta_{t}(X_t)^\T \nabla^2 \log \Tilde{q}_{t-1}(m_{t}(X_t)) \Delta_{t}(X_t) } \\
    &\textstyle+ \sum_{t=1}^T \frac{(1-\alpha_t)^2}{3! {\alpha_t}^{3/2}} \E_{X_t \sim \Tilde{Q}_t} \bigg[ 3 \sum_{i=1}^d \partial^3_{iii} \log \Tilde{q}_{t-1}(m_{t}(X_t)) \Delta_{t}(X_t)^i \\
    &\textstyle\qquad + \sum_{\substack{i,j=1 \\ i\neq j}}^d \partial^3_{iij} \log \Tilde{q}_{t-1}(m_{t}(X_t)) \Delta_{t}(X_t)^j \bigg] + \max_{t \geq 1}\sqrt{\E_{X_t \sim \Tilde{Q}_t} \norm{\Delta_{t}(X_t)}^2} (\log T) \eps.
\end{align*}
\end{theorem}

When $\Tilde{Q}_0$ does not have a p.d.f., a similar upper bound is applied to $\KL{\Tilde{Q}_{1}}{\widehat{P}_{1}}$ such that $\rmW_2(\Tilde{Q}_{1}, \Tilde{Q}_{0})^2 \lesssim (1-\alpha_1) d$ (see \Cref{cor:main-general} in \Cref{app:cor-main-general}). 

To explain the three error terms in \Cref{thm:main-general}, $\mathcal{W}_{\text{oracle}}$ captures the error assuming that one has access to (a close estimate of) $\nabla \log \Tilde{q}_t,~\forall t \geq 1$. This error is independent of the score mismatch $\Delta_{t}$, and it decays as $\Tilde{O}(T^{-1})$ under \Cref{ass:regular-drv-general} \citep[Theorem~1]{liang2024discrete}.
The remaining two error terms $\mathcal{W}_{\text{bias}}$ and $\mathcal{W}_{\text{vanish}}$ arise from the mismatched sampling process. Both terms become zero if $\Delta_{t} \equiv 0$ for all $t \geq 1$, which corresponds to the score-matched case. 
Under \Cref{ass:regular-drv-general,ass:bdd-mismatch-general}, $\mathcal{W}_{\text{vanish}}$ decays as $\Tilde{O}(T^{-1})$ under an additional mild condition (see \Cref{lem:small-grad-diff-genli} in \Cref{app:proof_thm_gen_q0_kl_general}), and $\mathcal{W}_{\text{bias}}$ asymptotically approaches a constant. 
% \yuchen{Shall we say through numerical simulations?}. 
Combining all three terms, score mismatch causes an asymptotic distributional bias between $\Tilde{Q}_{0}$ and $\widehat{P}_{0}$. 
% Note that this theorem can be directly extended to the problem of DDPM under general score mismatch (See \Cref{app:prob-general}).
% In \Cref{sec:delta_ty}, we will further specialize this theorem to the zero-shot conditional sampling problem and propose an optimal conditional sampler that minimizes the asymptotic bias.

% The introduction of the score mismatch, characterized by $\Delta_{t,y}$, has the following effects. First, the score estimation needs to be achieved at higher accuracy ($\eps^2 = \Tilde{O}(T^{-2})$) than when the exact conditional score is estimated ($\eps^2 = \Tilde{O}(T^{-1})$). Second, there emerges a non-vanishing distributional bias as $T \to \infty$.
To further understand $\mathcal{W}_{\text{bias}}$, note that $1-\alpha_t$ is usually summable under \Cref{ass:bdd-mismatch-general} (cf. \Cref{lem:nonvanish_coef_sum,lem:nonvanish_coef_sum_genli}). Thus, the bias can be further upper-bounded by the maximum step-wise mismatch $\max_{t \geq 1} \E_{X_t \sim \Tilde{Q}_{t}} \norm{\Delta_{t}(X_t)}^2$. In case that $\mu_{t}(x_t) = \frac{1}{\sqrt{\alpha_t}} \brc{ x_t + (1-\alpha_t) g_{t}(x_t) }$ (as for the zero-shot sampler in \eqref{eq:def_cond_spl}), define a measure $\Tilde{P}_t$ such that $g_{t} (x_t) = \nabla \log \Tilde{p}_t(x_t)$.
Then, from \eqref{eq:def_delta_general},
% we have $\Delta_{t}(x_t) = \nabla \log \Tilde{q}_{t} (x_t) - \nabla \log \Tilde{p}_t(x_t)$, and
\[ \textstyle \E_{X_t \sim \Tilde{Q}_{t}} \norm{\Delta_{t}(X_t)}^2 = \E_{X_t \sim \Tilde{Q}_{t}} \norm{\nabla \log \frac{\Tilde{q}_{t} (X_t)}{\Tilde{p}_t(X_t)}}^2 =: \mathcal{F}(\Tilde{Q}_{t} \| \Tilde{P}_t). \]
% \yl{should the second ``=" be ``:="?}
where $\mathcal{F}(Q \| P)$ denotes the \textit{Fisher divergence} (or called relative Fisher information) between $Q$ and $P$.
In \Cref{sec:delta_ty}, this distributional bias $\mathcal{W}_{\text{bias}}$ inspires us to design a novel zero-shot DDPM sampler, the BO-DDNM sampler, that minimizes the asymptotic bias.

Next we provide an upper bound with explicit dimensional dependency, for any $Q_0$ that has finite sixth moment such as Gaussian mixture $Q_0$'s and those $Q_0$'s having bounded support.
% We leave all discussions on \Cref{ass:bdd-mismatch-general} to \Cref{sec:delta_ty} as $\Delta_{t,y}$ depends on the actual choice of $f_{t,y}$.
To this end, we consider
a special noise schedule first proposed in \cite{li2023faster}:
\begin{equation} \label{eq:alpha_genli}
\textstyle  1-\alpha_1 = \delta,\quad 1-\alpha_t = \frac{c \log T}{T} \min\cbrc{\delta \brc{1 + \frac{c \log T}{T}}^t, 1},~\forall 2 \leq t \leq T 
\end{equation}
for any constants $(c,\delta)$ such that $c > 1$ and $\delta e^c > 1$. \yuchen{Note that this noise schedule corresponds to early-stopping in the literature \citep{chen2023improved,benton2023linear}.}
With the $\alpha_t$ in \eqref{eq:alpha_genli}, we can 
show that the regularity condition \Cref{ass:regular-drv-general} holds for a quite general set of distributions (see \Cref{lem:small-grad-diff-genli} in \Cref{app:proof_thm_gen_q0_kl_general}).
% \yl{should we claim Assumption 3 holds as a lemma in Appendix, and remark this after Theorem 2? This way, Theorem 2 forcuses on the convergence result}
\begin{theorem} \label{thm:gen_q0_kl_general}
Suppose that $\E_{X_0 \sim \Tilde{Q}_{0}} \norm{X_0}^6 \lesssim d^3$. Further, suppose that $\Delta_{t}$ satisfies that $\E_{X_t \sim \Tilde{Q}_{t}} \norm{\Delta_{t}(X_t)}^4 \lesssim \frac{\Bar{\alpha}_t^2}{(1-\Bar{\alpha}_t)^{2 r}} d^{2 \gamma}$ with some $\gamma, r \geq 1$ for all $t \geq 2$. Then, if the estimation error satisfies \Cref{ass:score-general} and if $\Delta_t$ satisfies \Cref{ass:bdd-mismatch-general}, with the $\alpha_t$ in \eqref{eq:alpha_genli} such that $\delta \ll 1$ and $c \asymp \log(1/\delta)$, 
% \Cref{ass:regular-drv-general} is satisfied, and 
we have, \yuchen{for some $\Tilde{Q}_{1}$ such that $\rmW_2(\Tilde{Q}_{1}, \Tilde{Q}_{0})^2 \lesssim \delta d$,}
\begin{multline*}
\textstyle    \KL{\Tilde{Q}_{1}}{\widehat{P}_{1}} \lesssim d^{\gamma} \delta^{-r} \brc{1 - \frac{2 \log(1/\delta) \log T}{T}} \\
\textstyle    + \max \{ d^{(3+\gamma)/2} \delta^{-\frac{r+2}{2}}, d^{1+\gamma} \delta^{-(r-1)} \} \frac{(\log T)^2}{T} + d^{\gamma/2} \delta^{-r/2} (\log T) \eps.
\end{multline*}
\end{theorem}

% \begin{proof}
%     See \Cref{app:proof_thm_gen_q0_kl_general}.
% \end{proof}

Note that \Cref{thm:gen_q0_kl_general} provides the \textit{first} performance guarantee with \textit{explicit} dimensional dependence for general score-mismatched DDPMs. Here the finite sixth moment is a technical condition to guarantee small expected difference of the first-order Taylor polynomial in case of mismatched scores (see \Cref{lem:small-grad-diff-genli} in \Cref{app:proof_thm_gen_q0_kl_general}). Later, \Cref{thm:gen_q0_kl_general} will be useful to provide guarantees for zero-shot conditional samplers under linear models (\Cref{lem:norm2_bd_bdsupp}).

\section{Zero-shot Conditional DDPM Samplers}
\label{sec:delta_ty}

As we discuss before, an important scenario of score-mismatched diffusion models is the zero-shot conditional problem, where certain information $y$ is given. In this section, we apply our general results for score-mismatch DDPMs in \Cref{sec:prob-general} to studying zero-shot conditional DDPM samplers.
% with the notations in \eqref{eq:thm1_abbrv_notation} for $y = h(x_0)$ for some arbitrary (deterministic or random) $h(\cdot)$.
%Now we analyze the performance of zero-shot conditional DDPM samplers. First note that according to the notations in \eqref{eq:thm1_abbrv_notation}, all results in \Cref{sec:prob-general} can be straightforwardly applied when $y = h(x_0)$ for some arbitrary (deterministic or random) $h(\cdot)$.
In the following, we are particularly interested in the linear conditional model in \eqref{eq:def_cond_model}. We take the same \Cref{ass:m2-general,ass:regular-drv-general,ass:bdd-mismatch-general} (albeit with changed notations), and further adopt the following common assumption on the \textit{unconditional} score estimation \citep{chen2023improved,chen2023sampling,liang2024discrete}.

\begin{assumption}[Estimation Error of Unconditional Score] \label{ass:score}
Suppose that $s_t$ 
% of $\nabla \log q_t(x_t)$
%the $s_t$ estimates 
satisfies
    \[ \textstyle \frac{1}{T} \sum_{t=1}^T \E_{X_t \sim Q_{t|y}} \norm{s_t(X_t) - \nabla \log q_t(X_t)}^2 \leq \eps^2,~\text{where}~\eps^2 = \Tilde{O}(T^{-2}). \]
\end{assumption}
Note that, with the zero-shot sampler defined in \eqref{eq:def_cond_spl}, since $\norm{I_d - H^\dagger H} = 1$, we have, $\forall x \in \mbR^d$,
% \begin{align} \label{eq:zs-est-err}
%     \norm{\widehat{\mu}_{t,y}(x) - \mu_{t,y}(x)}^2 &= \frac{(1-\alpha_t)^2}{\alpha_t} \norm{(I_d - H^\dagger H)(s_t(x) - \nabla \log q_t(x))}^2 \nonumber\\
%     &\leq \frac{(1-\alpha_t)^2}{\alpha_t} \norm{s_t(x) - \nabla \log q_t(x)}^2,~\forall x \in \mbR^d.
% \end{align}
\begin{equation} \label{eq:zs-est-err}
    \norm{\widehat{\mu}_{t,y} - \mu_{t,y}}^2 = \frac{(1-\alpha_t)^2}{\alpha_t} \norm{(I_d - H^\dagger H)(s_t - \nabla \log q_t)}^2 \leq \frac{(1-\alpha_t)^2}{\alpha_t} \norm{s_t - \nabla \log q_t}^2.
\end{equation}
Therefore, \Cref{ass:score} directly implies \Cref{ass:score-general}, and thus \Cref{thm:main-general} (as well as \Cref{cor:main-general}) still holds under \Cref{ass:m2-general,ass:regular-drv-general,ass:bdd-mismatch-general,ass:score}.

\subsection{A Novel Bias-Optimal Zero-shot Sampler}

%In this section, we investigate how particular choices of zero-shot samplers, in particular $f_{t,y}$, influence the overall convergence through $\Delta_{t,y}$.

Guided by the performance guarantee characterized in \Cref{thm:main-general}, we will propose a novel \textit{optimized} zero-shot condition sampler. With the zero-shot sampler defined in \eqref{eq:def_cond_spl}, the goal is to choose the $f_{t,y}$ function that minimizes the convergence error for each $y \in \mbR^p$ and $t \geq 1$.

Specifically, it is observed in \Cref{thm:main-general} that the convergence error in terms of the KL-divergence will have an asymptotic distributional bias given by $\mathcal{W}_{\text{bias}}$. As follows, we characterize an optimal $f_{t,y}$ that minimizes $\mathcal{W}_{\text{bias}}$, which thus yields a corresponding bias-optimal zero-shot sampler.
%by choosing $f_{t,y}$ that minimizes $\mathcal{W}_{\text{bias}}$.

\begin{theorem} \label{lem:cond_score_proj}
Define $\Sigma_{t|0,y} := \Bar{\alpha}_t \sigma_y^2 H^\dagger (H^\dagger)^\T + (1-\Bar{\alpha}_t) I_d$. For any $Q_0$ and $t \geq 1$, we have
\[ \textstyle \nabla \log q_{t|y}(x_t|y) = \Sigma_{t|0,y}^{-1} (\sqrt{\Bar{\alpha}_t} H^\dagger y - x_t) + \frac{\sqrt{\Bar{\alpha}_t}}{1 - \Bar{\alpha}_t} (I_d - H^\dagger H) \E_{Q_{0|t,y}} [X_0 | x_t, y]. \]
Also, recall the sampler in \eqref{eq:def_cond_spl} and define $f_{t,y}^*$ as
\begin{equation} \label{eq:def_fy_star}
    f_{t,y}^*(x_t) := \Sigma_{t|0,y}^{-1} \brc{\sqrt{\Bar{\alpha}_t} H^\dagger y - H^\dagger H x_t}.
\end{equation} 
Also recall $\Delta_{t,y}$ from \eqref{eq:def_delta_linear_cond}. Then, $f_{t,y}^*$ satisfies that, for all $t \geq 1$ and fixed $y \in \mbR^{p}$,
\[ f_{t,y}^* \in \argmin_{f_{t,y}: (I_d - H^\dagger H) f_{t,y} \equiv 0} \norm{\Delta_{t,y}}^2,\quad Q_{t|y} \text{--almost surely}. \]
\end{theorem}

The sampler $f_{t,y}^*(x_t)$ defined in \eqref{eq:def_fy_star} provides a bias-optimal zero-shot conditional DDPM sampler.
%\textbf{Bias-Optimal (BO) DDNM} sampler. 
In the case with $\sigma_y = 0$, such an optimal sampler coincides with the regular DDNM sampler in \cite{wang2023ddnm} (see \Cref{app:fy-literature}). Thus, we call this sampler as \textbf{Bias-Optimal (BO) DDNM} sampler. With \eqref{eq:def_fy_star}, we can also calculate the minimum step-wise mismatch as
\begin{align*}
    \min_{f_{t,y}: (I_d - H^\dagger H) f_{t,y} \equiv 0} \E_{X_t \sim Q_{t|y}} \norm{\Delta_{t,y}}^2
    & \textstyle = \E_{X_t \sim Q_{t|y}} \norm{\nabla \log \frac{q_{t|y} (X_t)}{q_t (X_t)}}^2_{(I_d - H^\dagger H)}, 
    %=: \mathcal{F}_{(I_d - H^\dagger H)}(Q_{t|y} || Q_t) \\
    % &\leq \E_{X_t \sim Q_{t|y}} \norm{\nabla \log \frac{q_{t|y} (X_t)}{q_t (X_t)}}^2 =: \mathcal{F}(Q_{t|y} || Q_t)
\end{align*}
which is the projected Fisher divergence between $Q_{t|y}$ and $Q_t$ on $\range(I_d - H^\dagger H)$.

\begin{figure}[ht]
\begin{minipage}{0.49\linewidth}
\centering
% \begin{tabular}{cc}
%     \includegraphics[height=2.2cm]{Figures/gauss_kl_sigy.png} &
%     \includegraphics[height=2.2cm]{Figures/mixture_kl_sigy.png}
% \end{tabular}
\includegraphics[height=3cm]{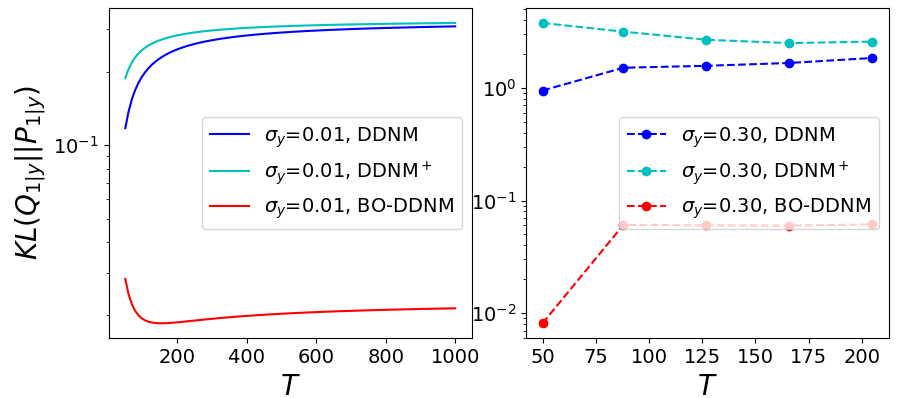}
%\framebox[4.0in]{$\;$}
% \fbox{\rule[-.5cm]{0cm}{4cm} \rule[-.5cm]{4cm}{0cm}}
\caption{
Comparison of BO-DDNM, DDNM and DDNM$^+$ for Gaussian (left) and Gaussian mixture (right) $Q_0$ under measurement noise.} 
% The same $Q_0$'s as in \Cref{fig:gauss_kl_non_asymp} are used.
\label{fig:gauss_kl_non_asymp_noisy}
\end{minipage}
\hfill
\begin{minipage}{0.49\linewidth}
\centering
% \begin{tabular}{cc}
%     \includegraphics[height=2.2cm]{Figures/gauss_kl_y.png} &
%     \includegraphics[height=2.2cm]{Figures/gauss_kl_rho.png}
% \end{tabular}
\includegraphics[height=3cm]{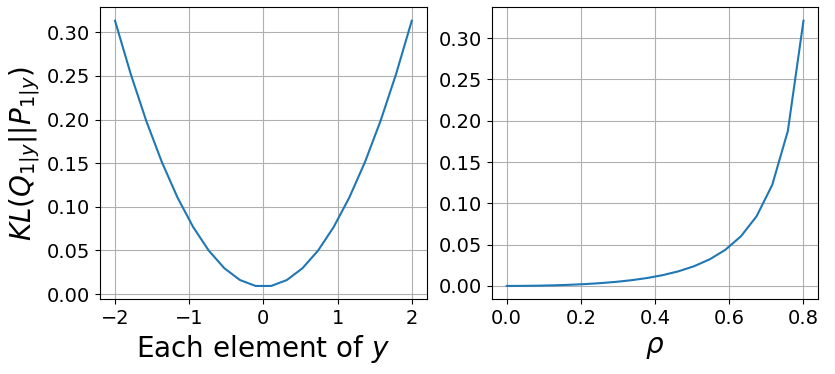}
%\framebox[4.0in]{$\;$}
% \fbox{\rule[-.5cm]{0cm}{4cm} \rule[-.5cm]{4cm}{0cm}}
\caption{Distributional bias as a function of the conditioning $y$ (left) and the correlation coefficient $\rho$ (right) for Gaussian $Q_0$.}
\label{fig:gauss_kl_limit}
\end{minipage}%
\end{figure}

In the following lemma, we provide the performance bound for BO-DDNM when $Q_0$ has bounded support. For comparison, we also derive the theoretical performance of vanilla DDNM, denoted as $f_{t,y}^N$.

\begin{theorem}[BO-DDNM vs.\ DDNM] \label{lem:norm2_bd_bdsupp}
Suppose that $\norm{X_0}^2 \leq R^2 d~~a.s.$ under $Q_0$. Suppose that \Cref{ass:m2-general,ass:score} hold. 
% Then, in the case where $\sigma_y^2$ is known, with the conditional sampler $f_{t,y}^*$ in \eqref{eq:def_fy_star}, \Cref{thm:gen_q0_kl_general} holds with $\gamma = 1$ and $r = 2$ for all $\sigma_y^2 \geq 0$. In the case that $\sigma_y^2 > 0$ is unknown, with the conditional sampler $f_{t,y}^N$, suppose further that $\norm{H^\dagger} \lesssim 1$, and then \Cref{thm:gen_q0_kl_general} still holds with $\gamma = 1$ and $r = 4$.
Then, with the conditional sampler $f_{t,y}^*$ in \eqref{eq:def_fy_star}, \Cref{thm:gen_q0_kl_general} holds with $\gamma = 1$ and $r = 2$. Also, with the conditional sampler $f_{t,y}^N := (1-\Bar{\alpha}_t)^{-1} \brc{\sqrt{\Bar{\alpha}_t} H^\dagger y - H^\dagger H x_t}$, if further $\norm{H^\dagger} \lesssim 1$, then \Cref{thm:gen_q0_kl_general} holds with $\gamma = 1$ and $r = 4$.
\end{theorem}

\Cref{lem:norm2_bd_bdsupp} establishes the \textit{first} result applicable to DDNM-type zero-shot conditional samplers for any linear conditional models on those target distributions having bounded support.

% Here the upper bound assumption on $\norm{H^\dagger}$ can be achieved in many applications, such as in image super-resolution (where $H^\dagger$ is a diagonal matrix with 0 or 1's with $\norm{H^\dagger} = 1$) and in image deblurring and colorization (where $H^\dagger$ is a block-diagonal matrix, whose spectral norm is independent of $d$). Note that a similar approach can be applied to Gaussian and Gaussian mixture $Q_0$'s.

{\bf Advantage of BO-DDNM over DDNM and DDNM$^+$:} When there is positive measurement noise, %although the dimensional dependency is the same, \Cref{lem:norm2_bd_bdsupp} predicts a higher asymptotic bias using $f_{t,y}^N$ than using $f_{t,y}^*$ with the $\alpha_t$ in \eqref{eq:alpha_genli} ($r=4$ vs. $r=2$). 
\Cref{lem:norm2_bd_bdsupp} indicates that our BO-DDNM sampler that uses $f_{t,y} = f_{t,y}^*$ enjoys a smaller asymptotic bias than DDNM that uses $f_{t,y}^N$ with the $\alpha_t$ in \eqref{eq:alpha_genli} ($\delta^{-2}$ vs.\ $\delta^{-4}$). Note that the DDNM sampler corresponds to $f_{t,y} = f_{t,y}^N$ (see \Cref{app:fy-literature}). Such an advantage is also demonstrated by our numerical experiment. In \Cref{fig:gauss_kl_non_asymp_noisy}, we numerically compared modified conditional zero-shot sampler (as given in \eqref{eq:def_fy_star}) with the DDNM and DDNM$^+$ sampler for both Gaussian and Gaussian mixture $Q_0$'s at different levels of measurement noise. It is observed that 
% larger $\sigma_y^2$ results in higher asymptotic bias when $f_{t,y} = f_{t,y}^*$, a phenomenon that coincides with \Cref{lem:norm2_bd_gauss,lem:norm2_bd_gauss_mix}. Also, 
the optimal BO-DDNM sampler achieves a much lower bias than both the DDNM and the DDNM$^+$ samplers numerically, especially when $\sigma_y^2$ becomes large.

\subsection{BO-DDNM Sampler for Gaussian Mixture \texorpdfstring{$Q_0$}{Q0}}

In this section, we focus on the convergence dependency on other system parameters of the BO-DDNM sampler, including the chosen $y$. In particular, we restrict our attention to Gaussian mixture $Q_0$'s and to a special conditional model, where $H = \begin{pmatrix} I_p & 0 \end{pmatrix}$.
% , and thus $H^\dagger H = \begin{pmatrix}
% I_p & 0 \\
% 0 & 0 
% \end{pmatrix}$. 
This choice can be seen in many applications, such as image super-resolution and inpainting (after reorganizing the pixels), where $I_p$ corresponds to the locations of the given pixels \citep{wang2023ddnm,song2023pgdm}. We assume positive measurement noise. 
%Starting from this subsection, 
We introduce the notation $[\Sigma_0]_{a b}$ to denote the variance components that correspond to the space of $a \times b$ where $a,b \in \{y, \Bar{y}\}$.

The following \Cref{lem:norm2_bd_gauss_mix} gives an upper bound on the asymptotic bias for Gaussian mixture $Q_0$.

\begin{proposition} \label{lem:norm2_bd_gauss_mix}
Suppose that $Q_0$ is Gaussian mixture with equal variance, whose p.d.f. is given by $q_0(x_0) = \sum_{n=1}^N \pi_n q_{0,n}(x_0)$, where $q_{0,n}$ is the p.d.f. of $\calN(\mu_{0,n}, \Sigma_0)$ and $\pi_n \in [0,1]$ is the mixing coefficient with $\sum_{n=1}^N \pi_n = 1$. Suppose that $H = \begin{pmatrix} I_p & 0 \end{pmatrix}$, and adopt $f_{t,y}^*$ in \eqref{eq:def_fy_star} and $\alpha_t$ in \Cref{def:noise_smooth}. Write $\lambda_1 \geq \dots \geq \lambda_d > 0$ and $\Tilde{\lambda}_1 \geq \dots \geq \Tilde{\lambda}_{d-p} > 0$ as the eigenvalues of $\Sigma_0$ and $[\Sigma_0]_{\Bar{y}\Bar{y}}$, respectively.
Then, 
% \Cref{ass:bdd-mismatch-general} is satisfied, \yl{should we separate out the claim of Assumption 4; state it in Appendix, and remark it after this lemma?} and
\begin{align*}
    &\E_{X_t \sim Q_{t|y}} \norm{\Delta_{t,y}(X_t)}^2 \\
    &\textstyle \lesssim \Bar{\alpha}_t d + \Bar{\alpha}_t^2 \frac{\norm{[\Sigma_0]_{y \Bar{y}}}^2}{\min\{\Tilde{\lambda}_{d-p}, 1\}^2 \min\{\lambda_d, 1\}^2 } \max\cbrc{d (\lambda_1 + \sigma_y^2) + \sum_{n=1}^N \pi_n \norm{H^\dagger y - H^\dagger H \mu_{0,n}}^2, d }\\
    &\textstyle \lesssim \Bar{\alpha}_t \brc{d + \sum_{n=1}^N \pi_n \norm{H^\dagger y - H^\dagger H \mu_{0,n}}^2}.
\end{align*}
\end{proposition}
% \yl{should we call above result as lemma or proposition?}

\Cref{lem:norm2_bd_gauss_mix} indicates that three factors affect (an upper bound on) the asymptotic bias. 
%\textbf{(i) Dependency on $y$:} 
(i) The measurement noise variance $\sigma_y^2$ determined by the system nature has an increasing effect on the bias.
(ii) The averaged distance $\sum_{n=1}^N \pi_n \norm{H^\dagger y - H^\dagger H \mu_{0,n}}^2$ between $H^\dagger y$ and $H^\dagger H \mu_{0,n}$ captures the quadratic dependency in $y$, as illustrated in the left plot of \Cref{fig:gauss_kl_limit}. (iii) The correlation between $H x_0$ and $(I_d-H^\dagger H) x_0$ of each mixture component contributes positively to the bias, which is contained in the factor $\frac{\norm{[\Sigma_0]_{y \Bar{y}}}^2}{\min\{\lambda_d, 1\}^2 }$.
% For the last factor, if $H x_0$ and $(I_d-H^\dagger H) x_0$ are uncorrelated (thus independent due to Gaussianity), we have $\Delta_{t,y}(x_t) \equiv 0$, and the asymptotic mismatch equals zero. 
%
To see this, consider $\sigma_y^2 = 0$ and a specific Gaussian example with $d=2$, $p=1$, and $\Sigma_0 = \begin{pmatrix}
\sigma_{11}^2 & \rho \sigma_{11} \sigma_{22}\\
\rho \sigma_{11} \sigma_{22} & \sigma_{22}^2
\end{pmatrix}$.
As the correlation coefficient $\rho$ increases, $\Sigma_0$ becomes closer to be singular, and thus 
% $\lambda_1$ increases to $\Tr(\Sigma_0) < \infty$ and
$\lambda_d$ decreases to 0. Also, $\norm{[\Sigma_0]_{y \Bar{y}}}^2 = \rho^2 \sigma_{11}^2 \sigma_{22}^2$ increases quadratically with $\rho$. Hence, this factor $\frac{\norm{[\Sigma_0]_{y \Bar{y}}}^2}{\min\{\lambda_d, 1\}^2 }$ grows unboundedly as $\rho \to 1$, as does $\E_{X_t \sim Q_{t|y}} \norm{\Delta_{t,y}(X_t)}^2$.
% This phenomenon can also be observed at general $d$ and $p$.
% At first sight, this result might seem counter-intuitive, since one would expect better performance when the unfinished part of an image is more relevant (in other words, more correlated) with the conditioned part. To see why larger $\rho$ may cause larger bias, 
% %the commonly adopted metric $\KL{Q_{1|y}}{\widehat{P}_{1|y}}$. 
% observe that in the extreme scenario of complete correlation where $y$ is a deterministic function of $(I_d-H^\dagger H) x_0$, $Q_{1|y}$ becomes very close to a singleton, which pushes the asymptotic bias in KL-divergence to near infinity for any $\widehat{P}_{1|y}$.
Such dependency on the correlation is illustrated numerically in the right plot of \Cref{fig:gauss_kl_limit}.
% \yl{this paragraph needs further elaboration, particular its connection to Figure 2}

The following theorem characterizes the conditional KL divergence when $Q_0$ is mixture Gaussian. In particular, we can show \Cref{ass:bdd-mismatch-general} holds with any $\alpha_t$ that satisfies \Cref{def:noise_smooth} when $Q_0$ is Gaussian mixture (see \Cref{lem:bdd-mismatch-gauss-mix} in \Cref{app:thm5-proof}).

\begin{theorem} \label{thm:kl_bd_gauss_mix}
Suppose 
% that $Q_0$ is Gaussian mixture, $H = \begin{pmatrix} I_p & 0 \end{pmatrix}$,
the same conditions in \Cref{lem:norm2_bd_gauss_mix} hold 
and $\sigma_y^2 > 0$. Suppose that \Cref{ass:score} holds. Take $f_{t,y}^*$ in \eqref{eq:def_fy_star} and $\alpha_t$ that further satisfies $\sum_{t=1}^T (1-\alpha_t) \Bar{\alpha}_t = 1 + o(1)$. Then,
% are used where $\delta \ll 1$ and $c \asymp \log(1/\delta)$.
% If $H^\dagger H = \begin{pmatrix}
% I_p & 0 \\
% 0 & 0 
% \end{pmatrix}$
\begin{align*}
    % &\KL{Q_{1|y}}{\widehat{P}_{1|y}} \lesssim  \brc{d + \sum_{n=1}^N \pi_n \norm{H^\dagger y - H^\dagger H \mu_{0,n}}^2} \brc{1 - \frac{2 \log(1/\delta) \log T}{T}} \\
&\textstyle \KL{Q_{0|y}}{\widehat{P}_{0|y}} \lesssim \brc{d + \sum_{n=1}^N \pi_n \norm{H^\dagger y - H^\dagger H \mu_{0,n}}^2}+ \\
&\textstyle \brc{d^2 + \sum_{n=1}^N \pi_n \norm{H^\dagger y - H^\dagger H \mu_{0,n}}^4} \frac{(\log T)^2}{T} + \sqrt{d + \sum_{n=1}^N \pi_n \norm{H^\dagger y - H^\dagger H \mu_{0,n}}^2} (\log T) \eps.
\end{align*}
% Here $\rmW_2(Q_{1|y}, Q_{0|y})^2 \lesssim \delta d$.
\end{theorem}

\yuchen{Although \Cref{lem:norm2_bd_gauss_mix,thm:kl_bd_gauss_mix} assume $H = \begin{pmatrix} I_p & 0 \end{pmatrix}$, extension to general $H$ is straightforward by modifying the proof of \Cref{lem:bdd-mismatch-gauss-mix} and using the fact that $\norm{H^\dagger H} = \norm{I_d - H^\dagger H} = 1$.}

This is the \textit{first} convergence result for zero-shot samplers where explicit dependency on the conditioning $y$ is derived for Gaussian mixture targets. Note that the extra condition on $\alpha_t$ can be verified for both constant $\alpha_t$ (\Cref{lem:nonvanish_coef_sum}) and that in \eqref{eq:alpha_genli} (\Cref{lem:nonvanish_coef_sum_genli}). 
Among the three terms in \Cref{thm:kl_bd_gauss_mix}, the first term is the asymptotic bias analyzed in \Cref{lem:norm2_bd_gauss_mix}.
% , which approaches a non-zero limit of order $d$ as $T \to \infty$. 
Since the last two terms decrease to zero as $T \to \infty$, the asymptotic KL divergence will also approach some non-zero limit of order $d$.
% that depends on $H^\dagger y$ quadratically, as well as on the correlation of $Q_0$. 
% Also note that some constant coefficients are omitted in the bound. Thus, in practice when $\delta$ is not too small, it might be observed that $\KL{Q_{1|y}}{\widehat{P}_{1|y}}$ first decreases to a non-zero minimum and then increases to the limit as $T$ gradually increases.
% This phenomenon is verified numerically in \Cref{fig:gauss_kl_limit}.

% To interpret this dependency, note that
% \[ \sum_{n=1}^N \pi_n \norm{H^\dagger y - H^\dagger H \mu_{0,n}}^2 \lesssim \norm{H^\dagger y - H^\dagger H \mu_{0}}^2 + \sum_{n=1}^N \pi_n \norm{H^\dagger H \mu_0 - H^\dagger H \mu_{0,n}}^2 \]
% where we define $\mu_{0}$ as the mean of $Q_0$. Therefore, the asymptotic bias in the KL divergence depends on (1) the distance between $H^\dagger y$ and $H^\dagger H \mu_{0}$ (the projected mean) and (2) the variance under $Q_0$ (along the projected direction).
% Also, compared with Gaussian $Q_0$, the dependence on $y$ in the vanishing bound becomes higher in case of Gaussian mixture (of orders $d$ vs. $d^2$).

The proof of \Cref{thm:kl_bd_gauss_mix} is non-trivial because from \Cref{thm:main-general} we need to figure out the dependency on $y$ in all first three orders of partial derivatives of a Gaussian mixture density, which is generally hard to express. To this end, we restrict focus to a particular linear model where explicit dependency can be sought.
The result can be extended to the case of $\sigma_y^2 = 0$ with the $\alpha_t$ in \eqref{eq:alpha_genli} (see \Cref{rmk:gauss_mix_kl_genli}).

\section{Proof Sketch of \texorpdfstring{\Cref{thm:main-general}}{Thm 1}}
\label{sec:prf-skc}

We now provide a proof sketch of Theorem 1 to describe the idea of our analysis approach. The main technical challenge due to mismatched scores is to analyze the expected tilting factor under a mean-perturbed Gaussian, providing an upper bound of the asymptotic orders of all Gaussian non-centralized moments. See the full proof in \Cref{app:thm1-proof}.

To begin, with \Cref{lem:rev-err-tilt-factor}, we decompose the total error as $\KL{\Tilde{Q}_0}{\widehat{P}_0} \leq \E_{X_T \sim \Tilde{Q}_T} \sbrc{\log \frac{\Tilde{q}_T(X_T)}{\widehat{p}_T(X_T)}} + \sum_{t=1}^T \E_{X_t,X_{t-1} \sim \Tilde{Q}_{t,t-1}} \sbrc{\log \frac{p_{t-1|t}(X_{t-1}|X_t)}{\widehat{p}_{t-1|t}(X_{t-1}|X_t)}} + \sum_{t=1}^T \E_{X_t,X_{t-1} \sim \Tilde{Q}_{t,t-1}} \sbrc{\log \frac{p_{t-1|t}(X_{t-1}|X_t)}{\widehat{p}_{t-1|t}(X_{t-1}|X_t)}}$. These three terms correspond respectively to the \textit{initialization error}, \textit{estimation error}, and \textit{reverse-step error}. 
% \begin{multline*}
%     \KL{\Tilde{Q}_0}{\widehat{P}_0} \leq \underbrace{\textstyle \E_{X_T \sim \Tilde{Q}_T} \sbrc{\log \frac{\Tilde{q}_T(X_T)}{\widehat{p}_T(X_T)}}}_{\text{initialization error}} + \underbrace{\textstyle \sum_{t=1}^T \E_{X_t,X_{t-1} \sim \Tilde{Q}_{t,t-1}} \sbrc{\log \frac{p_{t-1|t}(X_{t-1}|X_t)}{\widehat{p}_{t-1|t}(X_{t-1}|X_t)}}}_{\text{estimation error}} \\
%     % + \underbrace{\textstyle \sum_{t=1}^T \E_{X_t,X_{t-1} \sim \Tilde{Q}_{t,t-1}} \sbrc{\log \frac{\Tilde{q}_{t-1|t}(X_{t-1}|X_t)}{p_{t-1|t}(X_{t-1}|X_t)}}}_{\text{reverse-step error}}.
%     + \underbrace{\textstyle \sum_{t=1}^T \brc{\E_{\Tilde{Q}_{t,t-1}}  [\Tilde{\zeta}_{t,t-1}] - \E_{\Tilde{Q}_{t} \times P_{t-1|t}} [\Tilde{\zeta}_{t,t-1}]} }_{\text{reverse-step error}}.
% \end{multline*}
The initialization error can be bounded by $\Bar{\alpha}_T d$ in order using \cite[Lemma~3]{liang2024discrete} under \Cref{ass:m2-general}. Below we focus on the remaining two terms.

\textbf{Step 1: Bounding estimation error under mismatch (\Cref{lem:score-est-ptb}).} At each time $t=1,\dots,T$, $\log (p_{t-1|t}(x_{t-1}|x_t) / \widehat{p}_{t-1|t}(x_{t-1}|x_t))$ has an explicit expression since they are conditional Gaussians with the same variance. However, differently from the typical matched case, the mean of $P_{t-1|t}$ (i.e., $\mu_t(x_t)$) is no longer equal to the posterior mean of $\Tilde{Q}_{t-1|t}$ (i.e., $m_t(x_t)$). Their difference is contained in $\Delta_t(x_t)$, whose asymptotic order needs to be upper-bounded in light of \Cref{ass:score-general}.

\textbf{Step 2: Decomposing reverse-step error under mismatch (\Cref{eq:rev_step_decomp}).}
% The key in proving \Cref{thm:main-general} is to evaluate the expected difference of the tilting factor. To this end, 
% We decompose the total reverse-step error as follows.
\yuchen{First we decompose the tilting factor as $\Tilde{\zeta}_{t,t-1}(x_t,x_{t-1}) = \Tilde{\zeta}_{\text{mis}} (x_t,x_{t-1}) + \Tilde{\zeta}_{\text{van}} (x_t,x_{t-1})$, where 
\begin{align*}
    \Tilde{\zeta}_{\text{mis}} &:= \sqrt{\alpha_t} \Delta_t(x_t)^\T (x_{t-1} - m_t(x_t)) \\
    \Tilde{\zeta}_{\text{van}} &:= (\nabla \log \Tilde{q}_{t-1}(m_t(x_t)) - \sqrt{\alpha_t} \nabla \log \Tilde{q}_{t}(x_t))^\T (x_{t-1} - m_t(x_t)) + \sum_{p=2}^\infty T_p(\log \Tilde{q}_{t-1}, x_t, m_t(x_t)).
\end{align*}
% Here $T_p(f,x,\mu)$ is the $p$-th order Taylor power term of function $f$ around $x=\mu$.
Here $\Tilde{\zeta}_{\text{mis}}$ captures the factor that contributes to the total bias within $\Tilde{\zeta}_{t,t-1}$.
Define the oracle sampling process as $\Tilde{P}_{t-1|t} = \calN(m_{t,y},\sigma_t^2 I_d)$. Then, the reverse-step error can be decomposed as 
% $\E_{\Tilde{Q}_{t,t-1}}  [\Tilde{\zeta}_{t,t-1}] - \E_{\Tilde{Q}_{t} \times P_{t-1|t}} [\Tilde{\zeta}_{t,t-1}] = \mathcal{W}_{\text{oracle, rev-step}} + \mathcal{W}_{\text{bias, rev-step}} + \mathcal{W}_{\text{vanish, rev-step}}$, where $\mathcal{W}_{\text{oracle, rev-step}} := \E_{\Tilde{Q}_{t,t-1}} [\Tilde{\zeta}_{t,t-1}] - \E_{\Tilde{Q}_{t} \times P_{t-1|t}} [\Tilde{\zeta}_{t,t-1}]$, $\mathcal{W}_{\text{bias, rev-step}} := \E_{\Tilde{Q}_{t} \times \Tilde{P}_{t-1|t}} [\Tilde{\zeta}_{\text{mis}}] - \E_{\Tilde{Q}_{t} \times P_{t-1|t}} [\Tilde{\zeta}_{\text{mis}}]$, and $\mathcal{W}_{\text{vanish, rev-step}} := \E_{\Tilde{Q}_{t} \times \Tilde{P}_{t-1|t}} [\Tilde{\zeta}_{\text{van}}] - \E_{\Tilde{Q}_{t} \times P_{t-1|t}} [\Tilde{\zeta}_{\text{van}}]$.
\begin{multline*}
    \E_{\Tilde{Q}_{t,t-1}}  [\Tilde{\zeta}_{t,t-1}] - \E_{\Tilde{Q}_{t} \times P_{t-1|t}} [\Tilde{\zeta}_{t,t-1}] = \underbrace{\brc{ \E_{\Tilde{Q}_{t,t-1}} [\Tilde{\zeta}_{t,t-1}] - \E_{\Tilde{Q}_{t} \times P_{t-1|t}} [\Tilde{\zeta}_{t,t-1}]} }_{\mathcal{W}_{\text{oracle, rev-step}}}\\
   + \underbrace{ \brc{\E_{\Tilde{Q}_{t} \times \Tilde{P}_{t-1|t}} [\Tilde{\zeta}_{\text{mis}}] - \E_{\Tilde{Q}_{t} \times P_{t-1|t}} [\Tilde{\zeta}_{\text{mis}}] } }_{\mathcal{W}_{\text{bias, rev-step}}} + \underbrace{\brc{ \E_{\Tilde{Q}_{t} \times \Tilde{P}_{t-1|t}} [\Tilde{\zeta}_{\text{van}}] - \E_{\Tilde{Q}_{t} \times P_{t-1|t}} [\Tilde{\zeta}_{\text{van}}] } }_{\mathcal{W}_{\text{vanish, rev-step}}}.
\end{multline*}
}

\textbf{Step 3: Bounding $\mathcal{W}_{\text{oracle, rev-step}}$ and $\mathcal{W}_{\text{bias, rev-step}}$ (\Cref{eq:mismatch_zeta_0,eq:mismatch_zeta_1}).} 
Under \Cref{ass:regular-drv-general}, the dominant term of $\mathcal{W}_{\text{oracle, rev-step}}$ is given by \cite[Theorem~1]{liang2024discrete}. Also, the calculation of $\mathcal{W}_{\text{bias, rev-step}}$ is reduced to the difference in conditional mean, which is proportional to $\norm{\Delta_t(x_t)}^2$. 

\textbf{Step 4: Bounding $\mathcal{W}_{\text{vanish, rev-step}}$ (\Cref{lem:tilde-zeta-van-expr,lem:tilde-zeta-van-higher}).} 
To upper-bound $\mathcal{W}_{\text{vanish, rev-step}}$, 
% since the posterior mean of $\Tilde{\zeta}_{\text{van}}$ under $\Tilde{P}_{t-1|t}$ has been considered in 
with results on the matched case in \cite{liang2024discrete}, we need only to characterize the mean of $\Tilde{\zeta}_{\text{van}}$ under the mismatched posterior $P_{t-1|t}$.
We determine the dominant order in the expected values of all Taylor polynomials, which includes calculating all non-centralized moments. We first calculate the first three non-centralized moments (\Cref{lem:tilde-zeta-van-expr}) and then determine the asymptotic order of all higher moments (\Cref{lem:tilde-zeta-van-higher}). With these, we can finally locate the terms of dominating order in $\mathcal{W}_{\text{vanish, rev-step}}$. 
%The proof is now complete. 
\qedhere
% \end{proof}

\section{Conclusion}

In this paper, we have provided convergence guarantees for the general score-mismatched diffusion models, which are specialized to zero-shot conditional samplers. For linear conditional models, we also designed an optimal BO-DDNM sampler that minimizes the asymptotic bias, for which we showed the dependencies on the system parameters. One future direction is to explore zero-shot samplers that use higher-order derivatives of the log-densities, which might achieve better convergence results.

\section*{Acknowledgements}
    This work has been supported in part by the U.S. National Science Foundation under the grants: CCF-1900145, NSF AI Institute (AI-EDGE) 2112471, CNS-2312836, CNS-2223452, CNS-2225561, and was sponsored by the Army Research Laboratory under Cooperative Agreement Number W911NF-23-2-0225. The views and conclusions contained in this document are those of the authors and should not be interpreted as representing the official policies, either expressed or implied, of the Army Research Laboratory or the U.S. Government. The U.S. Government is authorized to reproduce and distribute reprints for Government purposes notwithstanding any copyright notation herein.
\bibliography{diffusion}

\begin{thebibliography}{10}

\bibitem{kingma2013vae}
Diederik~P Kingma and Max Welling.
\newblock Auto-encoding variational bayes.
\newblock {\em arXiv preprint arXiv:1312.6114}, 2022.

\bibitem{goodfellow2014gan}
Ian Goodfellow, Jean Pouget-Abadie, Mehdi Mirza, Bing Xu, David Warde-Farley, Sherjil Ozair, Aaron Courville, and Yoshua Bengio.
\newblock Generative adversarial nets.
\newblock In {\em Advances in Neural Information Processing Systems}, volume~27, 2014.

\bibitem{rezende2015normalizingflow}
Danilo Rezende and Shakir Mohamed.
\newblock Variational inference with normalizing flows.
\newblock In {\em Proceedings of the 32nd International Conference on Machine Learning}, volume~37, pages 1530--1538, 2015.

\bibitem{sohldickstein2015}
Jascha Sohl-Dickstein, Eric Weiss, Niru Maheswaranathan, and Surya Ganguli.
\newblock Deep unsupervised learning using nonequilibrium thermodynamics.
\newblock In {\em Proceedings of the 32nd International Conference on Machine Learning}, volume~37, pages 2256--2265, 2015.

\bibitem{ho2020ddpm}
Jonathan Ho, Ajay Jain, and Pieter Abbeel.
\newblock Denoising diffusion probabilistic models.
\newblock In {\em Advances in Neural Information Processing Systems}, volume~33, pages 6840--6851, 2020.

\bibitem{dalle2}
Aditya Ramesh, Prafulla Dhariwal, Alex Nichol, Casey Chu, and Mark Chen.
\newblock Hierarchical text-conditional image generation with clip latents.
\newblock {\em arXiv preprint arXiv:2204.06125}, 2022.

\bibitem{rombach2022stable-diffusion}
Robin Rombach, Andreas Blattmann, Dominik Lorenz, Patrick Esser, and Bj\"orn Ommer.
\newblock High-resolution image synthesis with latent diffusion models.
\newblock In {\em Proceedings of the IEEE/CVF Conference on Computer Vision and Pattern Recognition (CVPR)}, pages 10684--10695, June 2022.

\bibitem{diffusion-survey-croitoru}
F.~Croitoru, V.~Hondru, R.~Ionescu, and M.~Shah.
\newblock Diffusion models in vision: A survey.
\newblock {\em IEEE Transactions on Pattern Analysis \& Machine Intelligence}, 45(9):10850--10869, Sep 2023.

\bibitem{diffusion-survey-img-restore}
Xin Li, Yulin Ren, Xin Jin, Cuiling Lan, Xingrui Wang, Wenjun Zeng, Xinchao Wang, and Zhibo Chen.
\newblock Diffusion models for image restoration and enhancement -- a comprehensive survey.
\newblock {\em arXiv preprint arXiv:2308.09388}, 2023.

\bibitem{diffusion-survey-super-res}
Brian~B. Moser, Arundhati~S. Shanbhag, Federico Raue, Stanislav Frolov, Sebasti{\'a}n~M. Palacio, and Andreas Dengel.
\newblock Diffusion models, image super-resolution and everything: A survey.
\newblock {\em arXiv preprint arXiv:2401.00736}, 2024.

\bibitem{dhariwal2021diff-beat-gan}
Prafulla Dhariwal and Alex Nichol.
\newblock Diffusion models beat gans on image synthesis.
\newblock {\em arXiv preprint arXiv:2105.05233}, 2021.

\bibitem{ho2022classifierfree}
Jonathan Ho and Tim Salimans.
\newblock Classifier-free diffusion guidance.
\newblock {\em arXiv preprint arXiv:2207.12598}, 2022.

\bibitem{yuan2023reward-directed}
Hui Yuan, Kaixuan Huang, Chengzhuo Ni, Minshuo Chen, and Mengdi Wang.
\newblock Reward-directed conditional diffusion: Provable distribution estimation and reward improvement.
\newblock In {\em Thirty-seventh Conference on Neural Information Processing Systems}, 2023.

\bibitem{wu2024guidance}
Yuchen Wu, Minshuo Chen, Zihao Li, Mengdi Wang, and Yuting Wei.
\newblock Theoretical insights for diffusion guidance: A case study for gaussian mixture models.
\newblock {\em arXiv preprint arXiv:2403.01639}, 2024.

\bibitem{fu2024guidance}
Hengyu Fu, Zhuoran Yang, Mengdi Wang, and Minshuo Chen.
\newblock Unveil conditional diffusion models with classifier-free guidance: A sharp statistical theory.
\newblock {\em arXiv preprint arXiv:2403.11968}, 2024.

\bibitem{choi2021ilvr}
Jooyoung Choi, Sungwon Kim, Yonghyun Jeong, Youngjune Gwon, and Sungroh Yoon.
\newblock Ilvr: Conditioning method for denoising diffusion probabilistic models.
\newblock In {\em Proceedings of the IEEE/CVF International Conference on Computer Vision (ICCV)}, pages 14367--14376, October 2021.

\bibitem{chung2022ccdf}
Hyungjin Chung, Byeongsu Sim, and Jong~Chul Ye.
\newblock Come-closer-diffuse-faster accelerating conditional diffusion models for inverse problems through stochastic contraction.
\newblock {\em IEEE/CVF Conference on Computer Vision and Pattern Recognition (CVPR)}, 2022.

\bibitem{chung2022mcg}
Hyungjin Chung, Byeongsu Sim, Dohoon Ryu, and Jong~Chul Ye.
\newblock Improving diffusion models for inverse problems using manifold constraints.
\newblock In {\em Advances in Neural Information Processing Systems}, 2022.

\bibitem{chung2023dps}
Hyungjin Chung, Jeongsol Kim, Michael~Thompson Mccann, Marc~Louis Klasky, and Jong~Chul Ye.
\newblock Diffusion posterior sampling for general noisy inverse problems.
\newblock In {\em The Eleventh International Conference on Learning Representations}, 2023.

\bibitem{wang2023ddnm}
Yinhuai Wang, Jiwen Yu, and Jian Zhang.
\newblock Zero-shot image restoration using denoising diffusion null-space model.
\newblock In {\em The Eleventh International Conference on Learning Representations}, 2023.

\bibitem{song2023pgdm}
Jiaming Song, Arash Vahdat, Morteza Mardani, and Jan Kautz.
\newblock Pseudoinverse-guided diffusion models for inverse problems.
\newblock In {\em International Conference on Learning Representations}, 2023.

\bibitem{fei2023gdp}
Ben Fei, Zhaoyang Lyu, Liang Pan, Junzhe Zhang, Weidong Yang, Tianyue Luo, Bo~Zhang, and Bo~Dai.
\newblock Generative diffusion prior for unified image restoration and enhancement.
\newblock In {\em Proceedings of the IEEE/CVF Conference on Computer Vision and Pattern Recognition (CVPR)}, pages 9935--9946, June 2023.

\bibitem{gupta2024post-samp-intract}
Shivam Gupta, Ajil Jalal, Aditya Parulekar, Eric Price, and Zhiyang Xun.
\newblock Diffusion posterior sampling is computationally intractable.
\newblock {\em arXiv preprint arXiv:2402.12727}, 2024.

\bibitem{xu2024dpnp}
Xingyu Xu and Yuejie Chi.
\newblock Provably robust score-based diffusion posterior sampling for plug-and-play image reconstruction.
\newblock {\em arXiv preprint arXiv:2403.17042}, 2024.

\bibitem{liang2024discrete}
Yuchen Liang, Peizhong Ju, Yingbin Liang, and Ness Shroff.
\newblock Non-asymptotic convergence of discrete-time diffusion models: New approach and improved rate.
\newblock {\em arXiv preprint arXiv:2402.13901}, 2024.

\bibitem{bradley2024cfg-pred-corr}
Arwen Bradley and Preetum Nakkiran.
\newblock Classifier-free guidance is a predictor-corrector.
\newblock {\em arXiv preprint arXiv:2408.09000}, 2024.

\bibitem{chidambaram2024guidance}
Muthu Chidambaram, Khashayar Gatmiry, Sitan Chen, Holden Lee, and Jianfeng Lu.
\newblock What does guidance do? a fine-grained analysis in a simple setting.
\newblock {\em arXiv preprint arXiv:2409.13074}, 2024.

\bibitem{mei2023graphical}
Song Mei and Yuchen Wu.
\newblock Deep networks as denoising algorithms: Sample-efficient learning of diffusion models in high-dimensional graphical models.
\newblock {\em arXiv preprint arXiv:2309.11420}, 2023.

\bibitem{chang2024cond-ode}
Jinyuan Chang, Zhao Ding, Yuling Jiao, Ruoxuan Li, and Jerry~Zhijian Yang.
\newblock Deep conditional generative learning: Model and error analysis.
\newblock {\em arXiv preprint arXiv:2402.01460}, 2024.

\bibitem{chen2023improved}
Hongrui Chen, Holden Lee, and Jianfeng Lu.
\newblock Improved analysis of score-based generative modeling: user-friendly bounds under minimal smoothness assumptions.
\newblock In {\em Proceedings of the 40th International Conference on Machine Learning}, 2023.

\bibitem{benton2023linear}
Joe Benton, Valentin~De Bortoli, Arnaud Doucet, and George Deligiannidis.
\newblock Nearly $d$-linear convergence bounds for diffusion models via stochastic localization.
\newblock In {\em The Twelfth International Conference on Learning Representations}, 2024.

\bibitem{tweedie2011efron}
Bradley Efron.
\newblock Tweedie's formula and selection bias.
\newblock {\em Journal of the American Statistical Association}, 106(496):1602--1614, 2011.

\bibitem{jalal2021mri}
Ajil Jalal, Marius Arvinte, Giannis Daras, Eric Price, Alexandros~G Dimakis, and Jon Tamir.
\newblock Robust compressed sensing mri with deep generative priors.
\newblock In {\em Advances in Neural Information Processing Systems}, 2021.

\bibitem{chen2023sampling}
Sitan Chen, Sinho Chewi, Jerry Li, Yuanzhi Li, Adil Salim, and Anru Zhang.
\newblock Sampling is as easy as learning the score: theory for diffusion models with minimal data assumptions.
\newblock In {\em The Eleventh International Conference on Learning Representations}, 2023.

\bibitem{li2024accl-prov}
Gen Li, Yu~Huang, Timofey Efimov, Yuting Wei, Yuejie Chi, and Yuxin Chen.
\newblock Accelerating convergence of score-based diffusion models, provably.
\newblock {\em arXiv preprint arXiv:2403.03852}, 2024.

\bibitem{li2023faster}
Gen Li, Yuting Wei, Yuxin Chen, and Yuejie Chi.
\newblock Towards faster non-asymptotic convergence for diffusion-based generative models.
\newblock In {\em The Twelfth International Conference on Learning Representations}, 2024.

\bibitem{hyvarinen2005scorematch}
Aapo Hyv{{\"a}}rinen.
\newblock Estimation of non-normalized statistical models by score matching.
\newblock {\em Journal of Machine Learning Research}, 6(24):695--709, 2005.

\bibitem{debortoli2021bridge}
Valentin De~Bortoli, James Thornton, Jeremy Heng, and Arnaud Doucet.
\newblock Diffusion schr\"{o}dinger bridge with applications to score-based generative modeling.
\newblock In {\em Advances in Neural Information Processing Systems}, volume~34, pages 17695--17709, 2021.

\bibitem{lee2023general}
Holden Lee, Jianfeng Lu, and Yixin Tan.
\newblock Convergence of score-based generative modeling for general data distributions.
\newblock In {\em Proceedings of The 34th International Conference on Algorithmic Learning Theory}, volume 201, pages 946--985, 2023.

\bibitem{pedrotti2023predcorr}
Francesco Pedrotti, Jan Maas, and Marco Mondelli.
\newblock Improved convergence of score-based diffusion models via prediction-correction.
\newblock {\em arXiv preprint arXiv:2305.14164}, 2023.

\bibitem{conforti2023fisher}
Giovanni Conforti, Alain Durmus, and Marta~Gentiloni Silveri.
\newblock Score diffusion models without early stopping: finite fisher information is all you need.
\newblock {\em arXiv preprint arXiv:2308.12240}, 2023.

\bibitem{bruno2023wass}
Stefano Bruno, Ying Zhang, Dong-Young Lim, {\"O}mer~Deniz Akyildiz, and Sotirios Sabanis.
\newblock On diffusion-based generative models and their error bounds: The log-concave case with full convergence estimates.
\newblock {\em arXiv preprint arXiv:2311.13584}, 2023.

\bibitem{gao2023wass}
Xuefeng Gao, Hoang~M. Nguyen, and Lingjiong Zhu.
\newblock Wasserstein convergence guarantees for a general class of score-based generative models.
\newblock {\em arXiv preprint arXiv:2311.11003}, 2023.

\bibitem{gao2024wass}
Xuefeng Gao and Lingjiong Zhu.
\newblock Convergence analysis for general probability flow odes of diffusion models in wasserstein distances.
\newblock {\em arXiv preprint arXiv:2401.17958}, 2024.

\bibitem{chen2023ddim}
Sitan Chen, Giannis Daras, and Alexandros~G. Dimakis.
\newblock Restoration-degradation beyond linear diffusions: a non-asymptotic analysis for ddim-type samplers.
\newblock In {\em Proceedings of the 40th International Conference on Machine Learning}, 2023.

\bibitem{chen2023probode}
Sitan Chen, Sinho Chewi, Holden Lee, Yuanzhi Li, Jianfeng Lu, and Adil Salim.
\newblock The probability flow {ODE} is provably fast.
\newblock In {\em Thirty-seventh Conference on Neural Information Processing Systems}, 2023.

\bibitem{huang2024pfode}
Daniel~Zhengyu Huang, Jiaoyang Huang, and Zhengjiang Lin.
\newblock Convergence analysis of probability flow ode for score-based generative models.
\newblock {\em arXiv preprint arXiv:2404.09730}, 2024.

\bibitem{cheng2023flow}
Xiuyuan Cheng, Jianfeng Lu, Yixin Tan, and Yao Xie.
\newblock Convergence of flow-based generative models via proximal gradient descent in wasserstein space.
\newblock {\em arXiv preprint arXiv:2310.17582}, 2023.

\bibitem{benton2024flow}
Joe Benton, George Deligiannidis, and Arnaud Doucet.
\newblock Error bounds for flow matching methods.
\newblock {\em Transactions on Machine Learning Research}, 2024.

\bibitem{jiao2024flow_match}
Yuling Jiao, Yanming Lai, Yang Wang, and Bokai Yan.
\newblock Convergence analysis of flow matching in latent space with transformers.
\newblock {\em arXiv preprint arXiv:2404.02538}, 2024.

\bibitem{gao2024flow-match}
Yuan Gao, Jian Huang, Yuling Jiao, and Shurong Zheng.
\newblock Convergence of continuous normalizing flows for learning probability distributions.
\newblock {\em arXiv preprint arXiv:2404.00551}, 2024.

\bibitem{lyu2024cst-model}
Junlong Lyu, Zhitang Chen, and Shoubo Feng.
\newblock Sampling is as easy as keeping the consistency: convergence guarantee for consistency models, 2024.

\bibitem{li2024consistency}
Gen Li, Zhihan Huang, and Yuting Wei.
\newblock Towards a mathematical theory for consistency training in diffusion models.
\newblock {\em arXiv preprint arXiv:2402.07802}, 2024.

\bibitem{song2023cm}
Yang Song, Prafulla Dhariwal, Mark Chen, and Ilya Sutskever.
\newblock Consistency models.
\newblock In {\em Proceedings of the 40th International Conference on Machine Learning}, 2023.

\bibitem{oko2023dist-est}
Kazusato Oko, Shunta Akiyama, and Taiji Suzuki.
\newblock Diffusion models are minimax optimal distribution estimators.
\newblock In {\em ICLR 2023 Workshop on Mathematical and Empirical Understanding of Foundation Models}, 2023.

\bibitem{shah2023gauss_mm_learning}
Kulin Shah, Sitan Chen, and Adam Klivans.
\newblock Learning mixtures of gaussians using the {DDPM} objective.
\newblock In {\em Thirty-seventh Conference on Neural Information Processing Systems}, 2023.

\bibitem{gatmiry2024gm-learning}
Khashayar Gatmiry, Jonathan Kelner, and Holden Lee.
\newblock Learning mixtures of gaussians using diffusion models.
\newblock {\em arXiv preprint arXiv:2404.18869}, 2024.

\bibitem{chen2024gm-learning}
Sitan Chen, Vasilis Kontonis, and Kulin Shah.
\newblock Learning general gaussian mixtures with efficient score matching.
\newblock {\em arXiv preprint arXiv:2404.18893}, 2024.

\bibitem{cole2024subgauss}
Frank Cole and Yulong Lu.
\newblock Score-based generative models break the curse of dimensionality in learning a family of sub-gaussian distributions.
\newblock In {\em The Twelfth International Conference on Learning Representations}, 2024.

\bibitem{zhang2024minimax}
Kaihong Zhang, Heqi Yin, Feng Liang, and Jingbo Liu.
\newblock Minimax optimality of score-based diffusion models: Beyond the density lower bound assumptions.
\newblock {\em arXiv preprint arXiv:2402.15602}, 2024.

\bibitem{chen2023lowdim}
Minshuo Chen, Kaixuan Huang, Tuo Zhao, and Mengdi Wang.
\newblock Score approximation, estimation and distribution recovery of diffusion models on low-dimensional data.
\newblock In {\em Proceedings of the 40th International Conference on Machine Learning}, 2023.

\bibitem{gauss-mix-char}
John~Torjus Fl{\aa}m.
\newblock The linear model under gaussian mixture inputs: Selected problems in communications, 2013.

\end{thebibliography}

\newpage
\appendix
\begin{center}
    \huge \textbf{Appendix}
\end{center}

% \addcontentsline{toc}{section}{Appendix} % Add the appendix text to the document TOC
% \part{Appendix} % Start the appendix part
% \parttoc % Insert the appendix TOC

\allowdisplaybreaks
\startcontents[section]
{
\hypersetup{linkcolor=blue}
\printcontents[section]{l}{1}{\setcounter{tocdepth}{2}}
}

\crefalias{section}{appendix} % uncomment if you are using cleveref

% \section{Proof of Theorem 1}

% This is a boring technical proof.

% \section{Proof of Theorem 2}

% This is a complete version of a proof sketched in the main text.

\begin{center}
    \huge \textbf{Appendix}
\end{center}
\allowdisplaybreaks
\startcontents[section]
{
\hypersetup{linkcolor=blue}
\printcontents[section]{l}{1}{\setcounter{tocdepth}{2}}
}
\crefalias{section}{appendix}

\section{Full List of Notations}\label{app:notations}

For any two functions $f(d,\delta,T)$ and $g(d,\delta,T)$, we write $f(d,\delta,T) \lesssim g(d,\delta,T)$ (resp. $f(d,\delta,T) \gtrsim g(d,\delta,T)$) for some universal constant (not depending on $\delta$, $d$ or $T$) $L < \infty$ (resp. $L > 0$) if $\limsup_{T \to \infty} | f(d,\delta,T)/$ $g(d,\delta,T) | \leq L$ (resp. $\liminf_{T \to \infty} | f(d,\delta,T) / g(d,\delta,T) | \geq L$). We write $f(d,\delta,T) \asymp g(d,\delta,T)$ when both $f(d,\delta,T) \lesssim g(d,\delta,T)$ and $f(d,\delta,T) \gtrsim g(d,\delta,T)$ hold. Note that the dependence on $\delta$ and $d$ is retained with $\lesssim, \gtrsim, \asymp$. We write $f(d,\delta,T) = O(g(T))$ (resp. $f(d,\delta,T) = \Omega(g(T))$) if $f(d,\delta,T) \lesssim L(d,\delta) g(T)$ (resp. $f(d,\delta,T) \gtrsim L(d,\delta) g(T)$) holds for some $L(d,\delta)$ (possibly depending on $\delta$ and $d$). We write $f(d,\delta,T) = o(g(T))$ if $\limsup_{T \to \infty} |f(d,\delta,T)$ $/g(T)|= 0$. We write $f(d,\delta,T) = \Tilde{O}(g(T))$ if $f(d,\delta,T) = O(g(T) (\log g(T))^k)$ for some constant $k$. Note that the big-$O$ notation omits the dependence on $\delta$ and $d$. In the asymptotic when $\eps^{-1} \to \infty$, we write $f(d,\eps^{-1}) = \calO(g(d,\eps^{-1}))$ if $f(d,\delta,\eps^{-1}) \lesssim g(d,\delta,\eps^{-1}) (\log g(\eps^{-1}))^k$ for some constant $k$. 
Unless otherwise specified, we write $x^i (1\leq i \leq d)$ as the $i$-th element of a vector $x \in \mbR^d$ and $[A]^{ij}$ as the $(i,j)$-th element of a matrix $A$. For a function $f(x): \mbR^d \to \mbR$, we write $\partial_i f(z)$ as a shorthand for $\frac{\partial}{\partial x^i} f(x)\Big|_{x=z}$, and similarly for higher moments. 
For a vector (resp. matrix), all norms, if not explicitly specified, are referred to 2-norm (resp. spectral norm). For a vector $x$ and matrix $P$, define $\norm{x}_P := \sqrt{a^\T P a}$.
For matrices $A,B$, $\Tr(A)$ is the trace of $A$, and $A \preceq B$ means that $B-A$ is positive semi-definite. For a positive integer $n$, $[n] := \{1,\dots,n\}$.

\section{Related Works on Unconditional DDPM Samplers}
\label{sec:intro-works}

Given time-averaged $L^2$ unconditional score estimation error \citep{hyvarinen2005scorematch}, polynomial-time convergence guarantees have been established for wide families of target distributions \citep{debortoli2021bridge,chen2023sampling,lee2023general,chen2023improved,benton2023linear,pedrotti2023predcorr,conforti2023fisher}. For all target distributions with finite second moment, under $L^2$ score estimation error, $\calO(d \log(1/\delta)^2 / \eps^2)$ number of steps are sufficient to achieve $\eps^2$ KL divergence between the $\delta$-perturbed target distribution and the generated distribution using the specially designed exponential-decay-then-constant step-sizes \citep{benton2023linear,conforti2023fisher}. 
The analysis usually involves applying the Girsanov change-of-measure framework and the Fokker-Plank equation \citep{chen2023sampling,chen2023improved} to either the original SDE diffusion process or some transformed process \citep{benton2023linear,conforti2023fisher}, followed by an analysis of the discretization of the continuous-time process. 
More recently, similar convergence guarantees have been established using non-SDE-type techniques, such as with typical sets \citep{li2023faster} and with tilting factor representations \citep{liang2024discrete}. 
% The DT parameterization has a closer relationship to \cite{ho2020ddpm}, whose parameterization is further employed in many zero-shot conditional samplers (e.g., \cite{chung2022ccdf,wang2023ddnm,fei2023gdp}).
Here the new analysis introduced in \cite{liang2024discrete} is applicable to a larger set of step-sizes (equivalently, noise schedules) than the ones commonly used in previous analyses \citep{chen2023improved,benton2023linear,conforti2023fisher}. In this paper, we employ the same analytical framework as in \cite{liang2024discrete}.

Some other works analyzed sampling errors using a different measure (the Wasserstein-2 distance) \citep{bruno2023wass,gao2023wass,gao2024wass}. Beyond stochastic samplers, another line of studies provided theoretical guarantees for the deterministic sampler corresponding to DDPM \citep{chen2023ddim,chen2023probode,huang2024pfode}. 
% \cite{cao2023opt-coef} compared the performance of SDE and PF-ODE and investigated conditions where  one might outperform the other. 
Besides, \cite{cheng2023flow,benton2024flow,jiao2024flow_match,gao2024flow-match} provided guarantees for the closely-related flow-matching model, which learns a deterministic coupling between any two distributions. 
% \cite{chang2024cond-ode} proposed a novel ODE for sampling from a conditional distribution.
Also, \cite{lyu2024cst-model,li2024consistency} provided convergence guarantees for the closely-related consistency models \citep{song2023cm}. Finally, in order to achieve an end-to-end analysis, several works also developed sample complexity bounds to achieve the $L^2$ score estimation error for a variety of distributions \citep{oko2023dist-est,shah2023gauss_mm_learning,gatmiry2024gm-learning,chen2024gm-learning,cole2024subgauss,zhang2024minimax,mei2023graphical,chen2023lowdim}.

\section{Details of Numerical Simulations}

In \Cref{fig:gauss_kl_non_asymp_noisy}, we compared the performances of our optimal BO-DDNM sampler (with the $f_{t,y}^*$ in \eqref{eq:def_fy_star}) against the DDNM and DDNM$^+$ samplers \citep{wang2023ddnm} at different levels of $\sigma_y^2$. For Gaussian, we use $\mu_0 = 0$, $d=4$, $p=2$, and $y=\begin{pmatrix} 0.5 & 0.5 \end{pmatrix}$. We first randomly generate a positive definite matrix $\Sigma_0$ and uniformly sample $\rho \in [0.4,0.7)$, and then this correlation coefficient is enforced for any $[\Sigma_0]^{ij}$ where $i \in [p]$ and $j \in \{p+1,\dots,d\}$. We use the noise schedule in \eqref{eq:alpha_genli} with $c=3$ and $\delta=0.0001$ for Gaussian $Q_0$. For Gaussian mixture, we use $N=2$, $d=2$, $p=1$, and $y=1$. We set $\pi_n = \begin{pmatrix} 0.4 & 0.6 \end{pmatrix}$, $\text{diag}(\Sigma_0) = \begin{pmatrix} 0.1 & 1 \end{pmatrix}$, and $\rho = 0.6$. We further uniformly sample $\{\mu_{0,n}\}_{n=1}^N$ in the space $[-1,1) \times [-1,1)$. We use the noise schedule in \eqref{eq:alpha_genli} with $c=4$ and $\delta=0.02$ for Gaussian mixture $Q_0$. We use 150000 samples to estimate the divergence when $Q_0$ is Gaussian mixture.

In \Cref{fig:gauss_kl_limit}, we numerically verify the exact bias in KL divergence as a function of $y$ and $\rho$ for Gaussian $Q_0$. Here $Q_0 = \calN(0,\Sigma_0)$, $d=4$ and $p=2$. Suppose that $\sigma_y^2 = 0$. We assume that each element of $y$ has equal values. The correlation coefficient $\rho$ is enforced for any pair of $x^i$ and $x^j$ where $i \in [p]$ and $j \in \{p+1,\dots,d\}$. We first randomly generate a positive definite matrix $\Sigma_0$ and then enforce the correlation condition for any $x^i$ and $x^j$ where $i \in [p]$ and $j \in \{p+1,\dots,d\}$. We use a sufficiently large number of steps $T = 20000$. The conditional sampler is set as $f_{t,y} = f_{t,y}^*$ given in \eqref{eq:def_fy_star}. The noise schedule in \eqref{eq:alpha_genli} with $c=3$ and $\delta = 0.0001$ is used.

\section{Derivation of Score Bias for Existing Zero-shot DDPM Samplers}
\label{app:fy-literature}

In this section we show some examples of zero-shot conditional samplers proposed in the literature and in particular how they are related to the formulation of interest in \eqref{eq:def_cond_spl}. We recall the notations $H$, $y$, and $\sigma_t$ from \Cref{subsec:cond-ddpm}. Also denote
\[ \mu_t := \frac{1}{\sqrt{\alpha_t}} x_t + \frac{1- \alpha_t}{\sqrt{\alpha_t}} \nabla \log q_t(x_t) = \E_{X_{t-1} \sim Q_{t-1|t}(\cdot|x_t)} [X_{t-1}|x_t]  \]
which is the mean of the unconditional reverse-step at time $t \geq 1$. 

\subsection{Come-Closer-Diffuse-Faster (CCDF)}

We first examine the Come-Closer-Diffuse-Faster (CCDF) algorithm \citep{chung2022ccdf}. The CCDF algorithm using DDPM samplers gives that
\begin{align*}
    x_{t-1}' &= \mu_t + \sigma_t z_{t,1},\\
    % \hat{x}_t &= \sqrt{\Bar{\alpha}_t} H^\dagger y + \sqrt{1-\Bar{\alpha}_t} z_{t,2}\\
    x_{t-1} &= (I-H^\dagger H) x_{t-1}' + \sqrt{\Bar{\alpha}_{t}} H^\dagger y + \sqrt{1 - \Bar{\alpha}_t} z_{t,2}
\end{align*}
where $z_{t,1}, z_{t,2} \stackrel{i.i.d.}{\sim} \calN(0,I_d)$ are standard Gaussian random variables. Thus, the conditional mean of the update is
\begin{align*}
    \mu_{t,y} &= (I-H^\dagger H) \mu_t + \sqrt{\Bar{\alpha}_{t-1}} H^\dagger y\\
    &= \frac{1}{\sqrt{\alpha_t}} x_t + \frac{1- \alpha_t}{\sqrt{\alpha_t}} (I-H^\dagger H) \nabla \log q_t(x_t) + \sqrt{\Bar{\alpha}_{t}} H^\dagger y - \frac{1}{\sqrt{\alpha_t}} H^\dagger H x_t
\end{align*}
in which
\[ f_{t,y}(x_t) = \frac{1}{1-\alpha_t} \brc{ \sqrt{\alpha_t} \sqrt{\Bar{\alpha}_{t}} H^\dagger y - H^\dagger H x_t }. \]

\subsection{DDNM and \texorpdfstring{DDNM$^+$}{DDNM-plus}}

Next, we examine the DDNM algorithm and its modified version DDNM$^+$ \citep{wang2023ddnm}. We first note that the unconditional DDPM satisfies that (cf. \cite[Equations (7) and (11)]{ho2020ddpm}),
\begin{align} \label{eq:ho2020ddpm_post_mean}
    \mu_t &:= \frac{\sqrt{\Bar{\alpha}_{t-1}} (1-\alpha_t)}{1-\Bar{\alpha}_t} x_{0|t} + \frac{\sqrt{\alpha_t} (1-\Bar{\alpha}_{t-1})}{1-\Bar{\alpha}_t} x_t, \nonumber\\
    x_{0|t} &:= \frac{1}{\sqrt{\Bar{\alpha}_t}} x_t + \frac{1- \Bar{\alpha}_t}{\sqrt{\Bar{\alpha}_t}} \nabla \log q_t(x_t) = \E_{X_0 \sim Q_{0|t}(\cdot|x_t)} [X_0|x_t] .
\end{align}
Combining these two lines, we have $\mu_t = \frac{1}{\sqrt{\alpha_t}} (x_t + (1-\alpha_t) \nabla \log q_t(x_t))$.
In DDNM, $x_{0|t}$ is projected along the direction of the given $y$, which yields
\[ x_{0|t,y} := H^\dagger y + (I_d - H^\dagger H) x_{0|t}, \]
and the corresponding conditional mean of the update becomes
\[ \mu_{t,y} = \frac{\sqrt{\Bar{\alpha}_{t-1}} (1-\alpha_t)}{1-\Bar{\alpha}_t} x_{0|t,y} + \frac{\sqrt{\alpha_t} (1-\Bar{\alpha}_{t-1})}{1-\Bar{\alpha}_t} x_t. \]
Thus,
\begin{align*}
    \mu_{t,y} &= \frac{\sqrt{\Bar{\alpha}_{t-1}} (1-\alpha_t)}{1-\Bar{\alpha}_t} \brc{H^\dagger y + (I_d - H^\dagger H) x_{0|t}} + \frac{\sqrt{\alpha_t} (1-\Bar{\alpha}_{t-1})}{1-\Bar{\alpha}_t} x_t \\
    &\stackrel{(i)}{=} (I_d - H^\dagger H) \mu_t + H^\dagger \brc{\frac{\sqrt{\Bar{\alpha}_{t-1}} (1-\alpha_t)}{1-\Bar{\alpha}_t} y + \frac{\sqrt{\alpha_t} (1-\Bar{\alpha}_{t-1})}{1-\Bar{\alpha}_t} H x_t } \\
    &= \frac{1}{\sqrt{\alpha_t}} x_t + \frac{1-\alpha_t}{\sqrt{\alpha_t}} (I_d - H^\dagger H) \nabla \log q_t(x_t) \\
    &\qquad + \brc{\frac{\sqrt{\Bar{\alpha}_{t-1}} (1-\alpha_t)}{1-\Bar{\alpha}_t} H^\dagger y + \frac{\sqrt{\alpha_t} (1-\Bar{\alpha}_{t-1})}{1-\Bar{\alpha}_t} H^\dagger H x_t - \frac{1}{\sqrt{\alpha_t}} H^\dagger H x_t }
\end{align*}
where $(i)$ follows from \eqref{eq:ho2020ddpm_post_mean}.
Thus, to express this conditional mean in the form of \eqref{eq:def_cond_spl}, 
\begin{align*}
    f_{t,y}(x_t) &= \frac{\sqrt{\Bar{\alpha}_t}}{1-\Bar{\alpha}_t} H^\dagger y + \frac{1}{1-\alpha_t} \brc{\frac{\alpha_t (1-\Bar{\alpha}_{t-1})}{1-\Bar{\alpha}_t} - 1} H^\dagger H x_t \\
    &= \frac{1}{1-\Bar{\alpha}_t} \brc{\sqrt{\Bar{\alpha}_{t}} H^\dagger y - H^\dagger H x_t}.
\end{align*}
Here note that $f_{t,y}(x_t)$ is supported on $\range(H^\dagger H)$. Also note that for DDNM, $f_{t,y} = f_{t,y}^*$, which is the BO-DDNM sampler defined in \eqref{eq:def_fy_star}, when there is no measurement noise (i.e., $\sigma_y^2 = 0$).

Next we investigate its modified version, DDNM$^+$, in particular when $H = \begin{pmatrix} I_p & 0 \end{pmatrix}$. To relate the notations of \cite[Section~3.3 and Appendix~I]{wang2023ddnm} with ours, note that $\Sigma = A = H$, $U = I_p$, $V = I_d$, $s_1,\dots,s_p = 1$, $s_{p+1},\dots,s_{d} = 0$, and $a = \frac{\sqrt{\Bar{\alpha}_{t-1}} (1-\alpha_t)}{1 - \Bar{\alpha}_t}$. If $\sigma_t \geq a \sigma_y$, we have
\[ \Sigma_t = I_d,\quad \Phi_t = \begin{pmatrix} (\sigma_t^2 - a^2 \sigma_y^2) I_p & 0 \\ 0 & \sigma_t^2 I_{d-p} \end{pmatrix}. \]
Otherwise, if $\sigma_t < a \sigma_y$, we have
\[ \Sigma_t = \begin{pmatrix} \frac{\sigma_t}{a \sigma_y} I_p & 0 \\ 0 & I_{d-p} \end{pmatrix},\quad \Phi_t = \begin{pmatrix} 0 & 0 \\ 0 & \sigma_t^2 I_{d-p} \end{pmatrix}. \]
Observe that the only difference is on the space that supports $H^\dagger H$.
% Also we note that $\sigma_t^2 = \frac{1-\alpha_t}{\alpha_t}$.

From \cite[Equations (17) and (18)]{wang2023ddnm}, we can write
\[ \hat{x}_{0|t,y} := (I_d - H^\dagger H) x_{0|t} + \underbrace{\Sigma_t H^\dagger y + (I_d - \Sigma_t) H^\dagger H x_{0|t}}_{\text{supported on } \range(H^\dagger H)}, \]
Thus, with similar arguments above,
\begin{align*}
    \mu_{t,y} &= \frac{\sqrt{\Bar{\alpha}_{t-1}} (1-\alpha_t)}{1-\Bar{\alpha}_t} \hat{x}_{0|t,y} + \frac{\sqrt{\alpha_t} (1-\Bar{\alpha}_{t-1})}{1-\Bar{\alpha}_t} x_t \\
    % &= (I_d - H^\dagger H) \mu_t + H^\dagger \brc{\frac{\sqrt{\Bar{\alpha}_{t-1}} (1-\alpha_t)}{1-\Bar{\alpha}_t} y + \frac{\sqrt{\alpha_t} (1-\Bar{\alpha}_{t-1})}{1-\Bar{\alpha}_t} H x_t } \\
    &= \frac{1}{\sqrt{\alpha_t}} x_t + \frac{1-\alpha_t}{\sqrt{\alpha_t}} (I_d - H^\dagger H) \nabla \log q_t(x_t) \\
    &+ \frac{1-\alpha_t}{\sqrt{\alpha_t} (1-\Bar{\alpha}_t)} \underbrace{\brc{ \sqrt{\Bar{\alpha}_{t}} (\Sigma_t H^\dagger y + (I_d - \Sigma_t) H^\dagger H x_{0|t}) - H^\dagger H x_t } }_{\text{supported on } \range(H^\dagger H)}
\end{align*}
where $f_{t,y}$ is again supported on $\range(H^\dagger H)$.

\subsection{Samplers Using Higher-Order Derivatives}

Before we end this section, we note that the formulation in \eqref{eq:def_cond_spl} only uses (estimates of) first-order derivatives of (unconditional) log-p.d.f.s (a.k.a. unconditional score functions). This might not correspond to the optimal zero-shot sampler, and in practice there have been methods that use both first- and second-order derivatives (namely, in $\partial x_{0|t}(x_t) / \partial x_t$) to achieve better zero-shot sampling performance \citep{chung2023dps,song2023pgdm}. Nevertheless, the second-order derivatives might be hard to obtain, which require extra machine time and memory in the calculation. We leave investigations to use second-order derivatives in zero-shot conditional samplers as future work.

\section{Proof of \texorpdfstring{\Cref{thm:main-general}}{Theorem 1} and \texorpdfstring{\Cref{cor:main-general}}{Corollary 1}}
\label{app:thm1-proof}

Overall, the structure of the proof of \Cref{thm:main-general} is similar to that for \cite[Theorem~1]{liang2024discrete}. To start, we note that with similar arguments in \cite[Equation~13]{liang2024discrete},
an upper bound on $\KL{\Tilde{Q}_0}{\widehat{P}_0}$ is given by
\begin{align} \label{eq:proof-main-general-total-err}
    &\KL{\Tilde{Q}_0}{\widehat{P}_0} \nonumber\\
    &= \KL{\Tilde{Q}_T}{\widehat{P}_T} + \sum_{t=1}^T \E_{X_t \sim \Tilde{Q}_{t}} \sbrc{ \KL{\Tilde{Q}_{t-1|t}(\cdot|X_t)}{\widehat{P}_{t-1|t}(\cdot|X_t)} } \nonumber\\
    &\qquad - \sum_{t=1}^T \E_{X_{t-1} \sim \Tilde{Q}_{t-1}} \sbrc{ \KL{\Tilde{Q}_{t|t-1}(\cdot|X_{t-1})}{\widehat{P}_{t|t-1}(\cdot|X_{t-1})} } \nonumber\\
    &\leq \KL{\Tilde{Q}_T}{\widehat{P}_T} + \sum_{t=1}^T \E_{X_t \sim \Tilde{Q}_{t}} \sbrc{ \KL{\Tilde{Q}_{t-1|t}(\cdot|X_t)}{\widehat{P}_{t-1|t}(\cdot|X_t)} } \nonumber\\
    &= \underbrace{\E_{X_T \sim \Tilde{Q}_T} \sbrc{\log \frac{\Tilde{q}_T(X_T)}{\widehat{p}_T(X_T)}}}_{\text{Term 1: initialization error}} + \underbrace{\sum_{t=1}^T \E_{X_t,X_{t-1} \sim \Tilde{Q}_{t,t-1}} \sbrc{\log \frac{p_{t-1|t}(X_{t-1}|X_t)}{\widehat{p}_{t-1|t}(X_{t-1}|X_t)}}}_{\text{Term 2: estimation error}} \nonumber\\
    &\qquad + \underbrace{\sum_{t=1}^T \E_{X_t,X_{t-1} \sim \Tilde{Q}_{t,t-1}} \sbrc{\log \frac{\Tilde{q}_{t-1|t}(X_{t-1}|X_t)}{p_{t-1|t}(X_{t-1}|X_t)}}}_{\text{Term 3: reverse-step error}}.
\end{align}
The last equality holds because $\widehat{p}_T = p_T$.
Now, we provide an upper bound for the reverse-step error that is ready for further analysis. 
% is to investigate the relationship between $\Tilde{q}_{t-1|t}$ and $p_{t-1|t}$ for any fixed $t \geq 1$ and $x_t \in \mbR^d$.
In the following lemma, we show that the mismatched $\Tilde{q}_{t-1|t}$ is an exponentially tilted form of $p_{t-1|t}$.
% which further enables us to proceed to upper-bound the expected log-likelihood ratio.

\begin{lemma} \label{lem:rev-err-tilt-factor}
Fixed $t \geq 1$. For any fixed $x_t \in \mbR^d$, as long as $\Tilde{q}_{t-1}$ exists, we have
\[ \Tilde{q}_{t-1|t}(x_{t-1}|x_t) = \frac{p_{t-1|t}(x_{t-1}|x_t) e^{\Tilde{\zeta}_{t,t-1}(x_t,x_{t-1})}}{\E_{X_{t-1}\sim P_{t-1|t}}[e^{\Tilde{\zeta}_{t,t-1}(x_t,X_{t-1})}]} \]
where
\begin{align*}
    &\Tilde{\zeta}_{t,t-1}(x_t,x_{t-1}) \\
    &:= \sqrt{\alpha_t} \Delta_t(x_t)^\T (x_{t-1} - m_t(x_t)) + (\nabla \log \Tilde{q}_{t-1}(m_t(x_t)) - \sqrt{\alpha_t} \nabla \log \Tilde{q}_{t}(x_t))^\T (x_{t-1} - m_t(x_t)) \\
    % &+ \frac{1}{2} (x_{t-1} - m_t(x_t))^\T \nabla^2 \log \Tilde{q}_{t-1}(m_t(x_t)) (x_{t-1} - m_t(x_t)) 
    &+ \sum_{p=2}^\infty T_p(\log \Tilde{q}_{t-1}, x_{t-1}, m_t(x_t)).
\end{align*}
Here we define the $p$-th order term in the Taylor expansion of $f(x)$ around $\mu$ as
\[ T_p (f,x,\mu) := \frac{1}{p!} \sum_{\gamma \in \mbN^d:\sum_i \gamma^i = p} \partial^p_{\bm{a}} f(\mu) \prod_{i=1}^d (x^i-\mu^i)^{\gamma^i} \]
where $\bm{a} \in [d]^p$ are the indices of differentiation in which the multiplicity of $i \in [d]$ is $\gamma^i$.
\end{lemma}

\begin{proof}
    See \Cref{app:lem-rev-err-tilt-factor-proof}.
\end{proof}

We abbreviate $\Tilde{\zeta}_{t,t-1} = \Tilde{\zeta}_{t,t-1}(x_t,x_{t-1})$. Given the expression of $\Tilde{\zeta}_{t,t-1}$, the conditional reverse-step error can be upper-bounded for any fixed $x_t$ as
\begin{align} \label{eq:rev_err_zeta}
    &\E_{X_{t-1} \sim \Tilde{Q}_{t-1|t}} \sbrc{\log \frac{\Tilde{q}_{t-1|t}(X_{t-1}|x_t)}{p_{t-1|t}(X_{t-1}|x_t)}} \nonumber\\
    &= \E_{X_{t-1} \sim \Tilde{Q}_{t-1|t}} \sbrc{\Tilde{\zeta}_{t,t-1} - \log \E_{X_{t-1} \sim P_{t-1|t}} [e^{\Tilde{\zeta}_{t,t-1}}] } \nonumber\\
    &\stackrel{(i)}{\leq} \E_{X_{t-1} \sim \Tilde{Q}_{t-1|t}} \sbrc{\Tilde{\zeta}_{t,t-1}} + \E_{X_{t-1} \sim P_{t-1|t}} \sbrc{ - \log e^{\Tilde{\zeta}_{t,t-1}} } \nonumber\\
    &= \E_{X_{t-1} \sim \Tilde{Q}_{t-1|t}} [\Tilde{\zeta}_{t,t-1}] -  \E_{X_{t-1} \sim P_{t-1|t}} [\Tilde{\zeta}_{t,t-1}]
\end{align}
where in $(i)$ we use Jensen's inequality and note that $-\log(\cdot)$ is convex. Thus, from \eqref{eq:proof-main-general-total-err}, we have an upper bound as
\begin{multline*}
    \KL{\Tilde{Q}_0}{\widehat{P}_0} \leq \underbrace{\E_{X_T \sim \Tilde{Q}_T} \sbrc{\log \frac{\Tilde{q}_T(X_T)}{\widehat{p}_T(X_T)}}}_{\text{Term 1: initialization error}} + \underbrace{\sum_{t=1}^T \E_{X_t,X_{t-1} \sim \Tilde{Q}_{t,t-1}} \sbrc{\log \frac{p_{t-1|t}(X_{t-1}|X_t)}{\widehat{p}_{t-1|t}(X_{t-1}|X_t)}}}_{\text{Term 2: estimation error}} \\
    + \underbrace{\sum_{t=1}^T \E_{X_{t-1} \sim \Tilde{Q}_{t-1|t}} [\Tilde{\zeta}_{t,t-1}] -  \E_{X_{t-1} \sim P_{t-1|t}} [\Tilde{\zeta}_{t,t-1}]}_{\text{Term 3: reverse-step error}}.
\end{multline*}

Here, using \cite[Lemma~3]{liang2024discrete}, the initialization error can be upper-bounded as, when $T \to \infty$,
\[ \E_{X_T \sim \Tilde{Q}_T} \sbrc{\log \frac{\Tilde{q}_T(X_T)}{\widehat{p}_T(X_T)}} \leq \frac{1}{2} \E_{X_0 \sim \Tilde{Q}_0} \norm{X_0}^2 \Bar{\alpha}_T + O\brc{\Bar{\alpha}_T^2}. \]
This implies that, under \Cref{ass:m2-general} and if $c > 1$,
\[ \E_{X_T \sim \Tilde{Q}_T} \sbrc{\log \frac{\Tilde{q}_T(X_T)}{p_T(X_T)}} = o(T^{-1}). \]
Also, under \cref{ass:regular-drv-general}, the higher-order Taylor polynomials enjoy exponential rate of decay in expectation, which is contained in powers of $(1-\alpha_t)$. Thus, we are allowed to exchange the limit (of Taylor expansion) and the expectation operators (cf. \cite[Lemma~11]{liang2024discrete}).

Now, we upper-bound the estimation error and reverse-step error under score mismatch separately. 

\subsection{Step 1: Bounding estimation error under mismatch}

The following lemma provides an upper bound for the estimation error under score mismatch.

\begin{lemma} \label{lem:score-est-ptb}
Under \cref{ass:score-general,ass:bdd-mismatch-general}, with the $\alpha_t$ satisfying \cref{def:noise_smooth}, we have
\[ \sum_{t=1}^T \E_{X_t,X_{t-1} \sim \Tilde{Q}_{t,t-1}} \sbrc{\log \frac{p_{t-1|t}(X_{t-1}|X_t)}{\widehat{p}_{t-1|t}(X_{t-1}|X_t)}} \lesssim \max_{t \geq 1}\sqrt{\E_{X_t \sim \Tilde{Q}_t} \norm{\Delta_t(X_t)}^2} (\log T) \eps + (\log T) \eps^2. \]
\end{lemma}

\begin{proof}
    See \Cref{app:lem-score-est-ptb}.
\end{proof}

\subsection{Step 2: Decomposing reverse-step error under mismatch}

Now, we decompose $\Tilde{\zeta}_{t,t-1}(x_t,x_{t-1}) = \Tilde{\zeta}_{\text{mis}} + \Tilde{\zeta}_{\text{van}}$ where
\begin{align} \label{eq:def-tilde-zeta-mis-van}
    \Tilde{\zeta}_{\text{mis}} &:= \sqrt{\alpha_t} \Delta_t(x_t)^\T (x_{t-1} - m_t(x_t)), \nonumber \\
    \Tilde{\zeta}_{\text{van}} &:= (\nabla \log \Tilde{q}_{t-1}(m_t(x_t)) - \sqrt{\alpha_t} \nabla \log \Tilde{q}_{t}(x_t))^\T (x_{t-1} - m_t(x_t)) + \sum_{p=2}^\infty T_p(\log \Tilde{q}_{t-1}, x_{t-1}, m_t(x_t)).
\end{align}
Here $\Tilde{\zeta}_{\text{van}}$ is the same tilting factor without score bias (cf. \cite{liang2024discrete}).
Also, define an auxiliary conditional probability $\Tilde{P}_{t-1|t}$ such that
\[ \Tilde{P}_{t-1|t} := \calN\brc{ m_t(x_t), \frac{1-\alpha_t}{\alpha_t} I_d}, \]
which corresponds to the oracle reverse process that knows the true scores of the perturbed target distributions. Thus, we can decompose the expected value of \eqref{eq:rev_err_zeta} in the following way:
\begin{align} \label{eq:rev_step_decomp}
    &\E_{X_t \sim \Tilde{Q}_{t}} \brc{ \E_{X_{t-1} \sim \Tilde{Q}_{t-1|t}} - \E_{X_{t-1} \sim P_{t-1|t}} } [\Tilde{\zeta}_{t,t-1}] \nonumber\\
    &= \underbrace{\E_{X_t \sim \Tilde{Q}_{t}} \brc{ \E_{X_{t-1} \sim \Tilde{Q}_{t-1|t}} - \E_{X_{t-1} \sim \Tilde{P}_{t-1|t}} } [\Tilde{\zeta}_{t,t-1}] }_{\mathcal{W}_{\text{oracle, rev-step}}} \nonumber\\
    &\quad + \underbrace{\E_{X_t \sim \Tilde{Q}_{t}} \brc{ \E_{X_{t-1} \sim \Tilde{P}_{t-1|t}} - \E_{X_{t-1} \sim P_{t-1|t}} } [\Tilde{\zeta}_{\text{mis}}]}_{\mathcal{W}_{\text{bias, rev-step}}} \nonumber\\
    &\quad + \underbrace{\E_{X_t \sim \Tilde{Q}_{t}} \brc{ \E_{X_{t-1} \sim \Tilde{P}_{t-1|t}} - \E_{X_{t-1} \sim P_{t-1|t}} } [\Tilde{\zeta}_{\text{van}}]}_{\mathcal{W}_{\text{vanish, rev-step}}}.
\end{align}

\subsection{Step 3: Bounding \texorpdfstring{$\mathcal{W}_{\text{oracle, rev-step}}$}{oracle error} and \texorpdfstring{$\mathcal{W}_{\text{bias, rev-step}}$}{bias error} }

Among the terms above, \cite[Theorem~1]{liang2024discrete} shows that, under \Cref{ass:regular-drv-general} and using the $\alpha_t$ in \Cref{def:noise_smooth},
\begin{equation} \label{eq:mismatch_zeta_0}
    \mathcal{W}_{\text{oracle, rev-step}} \lesssim \sum_{t=1}^T (1-\alpha_t)^2 \E_{X_t \sim \Tilde{Q}_t}\sbrc{\Tr\Big(\nabla^2\log \Tilde{q}_{t-1}(m_t(X_t)) \nabla^2\log \Tilde{q}_t(X_t)\Big)}.
\end{equation}

Also, for $\mathcal{W}_{\text{bias, rev-step}}$, since direct calculation yields
\begin{align*}
    \E_{X_{t-1} \sim \Tilde{P}_{t-1|t}} [\Tilde{\zeta}_{\text{mis}} (x_t,X_{t-1})] &= \E_{X_{t-1} \sim \Tilde{Q}_{t-1|t}}[\Tilde{\zeta}_{\text{mis}} (x_t,X_{t-1})] = 0, \\
    \E_{X_{t-1} \sim P_{t-1|t}} [\Tilde{\zeta}_{\text{mis}} (x_t,X_{t-1})] &= -(1-\alpha_t) \norm{\Delta_t(x_t)}^2,
\end{align*}
we have
\begin{align} \label{eq:mismatch_zeta_1}
    \mathcal{W}_{\text{bias, rev-step}} &= \E_{X_t \sim \Tilde{Q}_{t}} \brc{\E_{X_{t-1} \sim \Tilde{P}_{t-1|t}} - \E_{X_{t-1} \sim P_{t-1|t}}}[\Tilde{\zeta}_{\text{mis}} (X_t,X_{t-1})] \nonumber\\
    &= (1-\alpha_t) \E_{X_t \sim \Tilde{Q}_{t}} \norm{\Delta_t(X_t)}^2.
\end{align}

\subsection{Step 4: Bounding \texorpdfstring{$\mathcal{W}_{\text{vanish, rev-step}}$}{vanishing error} }

Next, for $\mathcal{W}_{\text{vanish, rev-step}}$, we first note that $\Tilde{P}_{t-1|t}$ is conditional Gaussian. Thus, under \Cref{ass:regular-drv-general}, we are able to exchange the limit (from Taylor series) and the expectation due to Gaussian-like moments (cf. \cite[Lemma~11]{liang2024discrete}), which gives us
\begin{align*}
    &\E_{X_{t-1} \sim \Tilde{P}_{t-1|t}} [\Tilde{\zeta}_{\text{van}} (x_t,X_{t-1})]\\
    &\quad = \frac{1-\alpha_t}{2 \alpha_t} \Tr(\nabla^2 \log \Tilde{q}_{t-1}(m_t(x_t))) + \sum_{p=4}^\infty \E_{X_{t-1} \sim \Tilde{P}_{t-1|t}} \sbrc{ T_p(\log \Tilde{q}_{t-1}, X_{t-1}, m_t(x_t))}.
\end{align*}
Here the expected value at $p=3$ is zero because all odd-order centralized moments of Gaussian vanish.
% from \cite[Lemmas~7 and 10]{liang2024discrete}, under \Cref{ass:regular-drv},
% \begin{align*}
%     &\E_{X_{t-1} \sim \Tilde{P}_{t-1|t}} [\Tilde{\zeta}_{\text{van}} (x_t,X_{t-1})]\\
%     &\quad = \frac{1-\alpha_t}{2 \alpha_t} \Tr(\nabla^2 \log \Tilde{q}_{t-1}(m_t(x_t))) + \sum_{p=4}^\infty \E_{X_{t-1} \sim \Tilde{P}_{t-1|t}} \sbrc{ T_p(\log \Tilde{q}_{t-1}, x_{t-1}, m_t(x_t))} \\
%     % &+ \frac{1}{8} \brc{\sum_{i=1}^d \partial^4_{i i i i} \log q_{t-1}(\mu_t)} \brc{\frac{1-\alpha_t}{\alpha_t}}^2 + \frac{1}{4!} \brc{\sum_{\substack{i,j=1 \\ i\neq j}}^d \partial^4_{i i j j} \log q_{t-1}(\mu_t)} \brc{\frac{1-\alpha_t}{\alpha_t}}^2 + O\brc{(1-\alpha_t)^3}.\\
%     & \E_{X_{t} \sim \Tilde{Q}_{t}} \sbrc{ \E_{X_{t-1} \sim \Tilde{Q}_{t-1|t}} [\Tilde{\zeta}_{\text{van}} (X_t,X_{t-1})] - \E_{X_{t-1} \sim \Tilde{P}_{t-1|t}} [\Tilde{\zeta}_{\text{van}} (X_t,X_{t-1})] }\\
%     &\quad = \frac{(1-\alpha_t)^2}{2 \alpha_t} \E_{X_{t} \sim \Tilde{Q}_{t}} \Tr\brc{\nabla^2\log \Tilde{q}_{t-1}(m_t(X_t)) \nabla^2\log \Tilde{q}_t(X_t)} + O\brc{(1-\alpha_t)^3}.
% \end{align*} 

Now it remains to characterize the expectation of $\Tilde{\zeta}_{\text{van}} (x_t,x_{t-1})$ under $P_{t-1|t}$. To this end, we introduce the following notation.
\begin{definition} [Big-O in $\calL^p$ space] \label{def:bigO-Lp}
For a random variable $Z_T$, we say that $Z_T(x) = O_{\calL^p(Q)}(1)$ if $\brc{\E_{X \sim Q} \abs{Z_T(X)}^p}^{1/p} = O(1)$ for all $p \geq 1$ as $T \to \infty$. Define $\Tilde{O}_{\calL^p(Q)}$ likewise.
\end{definition}
One property is that if $Z_T(x) = O_{\calL^p(Q)}(1)$ then $\E_{X \sim Q}\abs{Z_T(X)} = O(1)$.
Another property is that if $Z_1 = O_{\calL^p(Q)}(a_T)$ and $Z_2 = O_{\calL^p(Q)}(b_T)$ for all $p \geq 1$, applying Cauchy-Schwartz inequality we get, for all $p \geq 1$,
\[ \brc{\E_{X \sim Q} \abs{Z_1 Z_2}^p}^{1/p} \leq \brc{\E_{X \sim Q} Z_1^{2 p} \E_{X \sim Q} Z_2^{2 p}}^{1/(2 p)} = O(a_T b_T), \]
which implies that $O_{\calL^p(Q)}(a_T) O_{\calL^p(Q)}(b_T) = O_{\calL^p(Q)}(a_T b_T)$. Now, with this notation, the first lines of \Cref{ass:regular-drv-general} can be equivalently written as
\begin{align*}
    (1-\alpha_t)^{m} \abs{\partial_{\bm{a}}^k \log q_{t}(X_t) } &= O_{\calL^p(\Tilde{Q}_t)}\brc{(1-\alpha_t)^{m}},~\forall p \geq 1, \\
    (1-\alpha_t)^{m} \abs{\partial_{\bm{a}}^k \log q_{t-1}(m_t(X_t)) } &= O_{\calL^p(\Tilde{Q}_t)}\brc{(1-\alpha_t)^{m}},~\forall p \geq 1.
\end{align*}
Also, \Cref{ass:bdd-mismatch-general} can be equivalently written as
\[ (1-\alpha_t)^{m} \norm{\Delta_{t,y}(X_t) } = O_{\calL^p(\Tilde{Q}_t)}(\Bar{\alpha}_t (1-\alpha_t)^{m}),~\forall p \geq 1. \]
With these notations, the following lemma characterizes the expectation of $\Tilde{\zeta}_{\text{van}} (x_t,x_{t-1})$ under $P_{t-1|t}$, which involves non-centralized Gaussian moments.

\begin{lemma} \label{lem:tilde-zeta-van-expr}
As long as $\Tilde{q}_{t-1}$ is defined, with the definition of $\Tilde{\zeta}_{\text{van}}$ in \eqref{eq:def-tilde-zeta-mis-van}, under \Cref{ass:regular-drv-general,ass:bdd-mismatch-general}, we have $\forall \ell \geq 1$,
\begin{align*}
    &\E_{X_{t-1} \sim P_{t-1|t}} [\Tilde{\zeta}_{\text{van}} (x_t,X_{t-1})]\\
    &= - \frac{1-\alpha_t}{\sqrt{\alpha_t}} (\nabla \log \Tilde{q}_{t-1}(m_t(x_t)) - \sqrt{\alpha_t} \nabla \log \Tilde{q}_{t}(x_t))^\T \Delta_t(x_t)\\
    &\quad + \frac{1-\alpha_t}{2 \alpha_t} \Tr(\nabla^2 \log \Tilde{q}_{t-1}(m_t)) + \frac{(1-\alpha_t)^2}{2 \alpha_t} \Delta_t(x_t)^\T \nabla^2 \log \Tilde{q}_{t-1}(m_t) \Delta_t(x_t)\\
    &\quad - \frac{1}{3!} \brc{\frac{(1-\alpha_t)^2}{{\alpha_t}^{3/2}}} \brc{ 3 \sum_{i=1}^d \partial^3_{iii} \log \Tilde{q}_{t-1}(m_t) \Delta_t^i + \sum_{\substack{i,j=1 \\ i\neq j}}^d \partial^3_{iij} \log \Tilde{q}_{t-1}(m_t) \Delta_t^j } \\
    &\quad + \sum_{p=4}^\infty \E_{X_{t-1} \sim P_{t-1|t}} \sbrc{ T_p(\log \Tilde{q}_{t-1}, X_{t-1}, m_t(x_t))} + O_{\calL^\ell(\Tilde{Q}_t)}\brc{(1-\alpha_t)^3}.
\end{align*}
\end{lemma}
\begin{proof}
    See \Cref{app:lem-tilde-zeta-van-expr-proof}.
\end{proof}

The following lemma provides the rate of decay of the difference in expectation of all Taylor polynomials with order $p \geq 4$.

\begin{lemma} \label{lem:tilde-zeta-van-higher}
As long as $\Tilde{q}_{t-1}$ is defined, under \Cref{ass:regular-drv-general,ass:bdd-mismatch-general}, we have, $\forall p \geq 4,\ell \geq 1$,
\[  \brc{ \E_{X_{t-1} \sim \Tilde{P}_{t-1|t}} - \E_{X_{t-1} \sim P_{t-1|t}}} [T_p(\log \Tilde{q}_{t-1}, X_{t-1}, m_t(X_t))] = O_{\calL^\ell(\Tilde{Q}_t)}\brc{ (1-\alpha_t)^{3} }. \]
\end{lemma}

\begin{proof}
    See \Cref{app:lem-tilde-zeta-van-higher}.
\end{proof}

Thus, with the help of \Cref{lem:tilde-zeta-van-expr,lem:tilde-zeta-van-higher}, we can identify the dominating terms in $\mathcal{W}_{\text{vanish, rev-step}}$ when $1-\alpha_t$ is small. The dominating term is
\begin{align} \label{eq:mismatch_zeta_2}
    \mathcal{W}_{\text{vanish, rev-step}} &= \E_{X_t \sim \Tilde{Q}_t} \brc{\E_{X_{t-1} \sim \Tilde{P}_{t-1|t}} - \E_{X_{t-1} \sim P_{t-1|t}}}[\Tilde{\zeta}_{\text{van}} (X_t,X_{t-1})] \nonumber\\
    &= \frac{1-\alpha_t}{\sqrt{\alpha_t}} \E_{X_t \sim \Tilde{Q}_t}  \sbrc{(\nabla \log \Tilde{q}_{t-1}(m_t(X_t)) - \sqrt{\alpha_t} \nabla \log \Tilde{q}_{t}(X_t))^\T \Delta_t(X_t) } \nonumber\\
    &\quad - \frac{(1-\alpha_t)^2}{2 \alpha_t} \E_{X_t \sim \Tilde{Q}_t} \sbrc{ \Delta_t(X_t)^\T \nabla^2 \log \Tilde{q}_{t-1}(m_t(X_t)) \Delta_t(X_t) } \nonumber\\
    &\quad + \frac{1}{3!} \brc{\frac{(1-\alpha_t)^2}{{\alpha_t}^{3/2}}} \E_{X_t \sim \Tilde{Q}_t} \bigg[ 3 \sum_{i=1}^d \partial^3_{iii} \log \Tilde{q}_{t-1}(m_t(X_t)) \Delta_t(X_t)^i \nonumber \\
    &\qquad \quad + \sum_{\substack{i,j=1 \\ i\neq j}}^d \partial^3_{iij} \log \Tilde{q}_{t-1}(m_t(X_t)) \Delta_t(X_t)^j \bigg] \nonumber \\
    &\quad + O((1-\alpha_t)^{3}).
\end{align}
Therefore, with the decomposition in \eqref{eq:rev_step_decomp} in mind, an upper bound on the reverse-step error is achieved by summing up \eqref{eq:mismatch_zeta_0}, \eqref{eq:mismatch_zeta_1} and \eqref{eq:mismatch_zeta_2}. The proof of \Cref{thm:main-general} is now complete.

\subsection{\texorpdfstring{\Cref{cor:main-general}}{Corollary 1} and its proof}
\label{app:cor-main-general}

Below we state and prove a corollary of \Cref{thm:main-general} when $\Tilde{q}_0$ does not exist. By \cite[Lemma~6]{liang2024discrete}, $\Tilde{q}_1$ always exists, which provides us with the following convergence result with early-stopping.

\begin{corollary} \label{cor:main-general}
Suppose that \cref{ass:m2-general,ass:regular-drv-general,ass:bdd-mismatch-general,ass:score-general} are satisfied. Then, \yuchen{suppose that the $\alpha_t$ satisfies \Cref{def:noise_smooth} at $t \geq 2$}, the distribution $\widehat{P}_{1}$ from the discrete-time DDPM under score bias satisfies
\begin{align*}
    &\KL{\Tilde{Q}_{1}}{\widehat{P}_{1}} \\
    &\lesssim \sum_{t=2}^T (1-\alpha_t) \E_{X_t \sim \Tilde{Q}_{t}} \norm{\Delta_{t}(X_t)}^2 \\
    &+ \sum_{t=2}^T \frac{1-\alpha_t}{\sqrt{\alpha_t}} \E_{X_t \sim \Tilde{Q}_{t}} \bigg[(\nabla \log \Tilde{q}_{t-1}(m_{t}(X_t)) - \sqrt{\alpha_t} \nabla \log \Tilde{q}_{t}(X_t))^\T \Delta_{t}(X_t)\bigg]\\
    &+ \sum_{t=2}^T \frac{(1-\alpha_t)^2}{2 \alpha_t} \E_{X_t \sim \Tilde{Q}_{t}}\sbrc{\Tr\Big(\nabla^2 \log \Tilde{q}_{t-1}(m_{t}(X_t)) \brc{ \nabla^2\log \Tilde{q}_t(X_t) - \Delta_{t}(X_t) \Delta_{t}(X_t)^\T } \Big) } \\
    &+ \sum_{t=2}^T \frac{(1-\alpha_t)^2}{3! {\alpha_t}^{3/2}} \E_{X_t \sim \Tilde{Q}_t} \sbrc{ 3 \sum_{i=1}^d \partial^3_{iii} \log \Tilde{q}_{t-1}(m_t(X_t)) \Delta_t(X_t)^i + \sum_{\substack{i,j=1 \\ i\neq j}}^d \partial^3_{iij} \log \Tilde{q}_{t-1}(m_t(X_t)) \Delta_t(X_t)^j }\\
    &+ \max_{t \geq 2}\sqrt{\E_{X_t \sim \Tilde{Q}_t} \norm{\Delta_t(X_t)}^2} (\log T) \eps + (\log T) \eps^2,
\end{align*}
where $\rmW_2(\Tilde{Q}_{1}, \Tilde{Q}_{0})^2 \lesssim (1-\alpha_1) d$.
\end{corollary}
\begin{proof}
    The result directly follows with the same arguments as in the proof of \Cref{thm:main-general}. The only difference is the guarantee under the Wasserstein distance, which can be obtained using \cite[Lemma~12]{liang2024discrete}.
\end{proof}

\section{Proof of \texorpdfstring{\Cref{thm:gen_q0_kl_general}}{Theorem 2}}
\label{app:proof_thm_gen_q0_kl_general}

We first recall some of the properties of the noise schedule in \eqref{eq:alpha_genli}. By \Cref{lem:alpha_genli_rate}, the noise schedule in \eqref{eq:alpha_genli} satisfies that $\frac{1 - \alpha_t}{(1 - \Bar{\alpha}_{t-1})^p} \lesssim \frac{\log T \log(1/\delta)}{\delta^{p-1} T}$ for all $p \geq 1$ while $\Bar{\alpha}_T = o(T^{-1})$, and thus such $\alpha_t$ satisfies \cref{def:noise_smooth} when $t \geq 2$.
Further, with the $\alpha_t$ in \eqref{eq:alpha_genli}, \cite[Lemmas 15 and 17]{liang2024discrete} show that for 
any $Q_0$ with finite variance under early-stopping, $\forall p,\ell \geq 1$,
\begin{align*}
    \E_{X_t \sim \Tilde{Q}_t} \abs{\partial_{\bm{a}}^p \log \Tilde{q}_t(X_t) }^\ell &= O\brc{\frac{1}{(1 - \Bar{\alpha}_t)^{p \ell / 2 }}},\\
    \E_{X_t \sim \Tilde{Q}_t} \abs{\partial_{\bm{a}}^p \log \Tilde{q}_{t-1}(m_t(X_t)) }^\ell &= O\brc{\frac{1}{(1 - \Bar{\alpha}_{t-1})^{p\ell/2}}}.
\end{align*}
% \yuchen{This is in the updated AISTATS version. Not yet on arXiv yet...}
Thus, using \Cref{lem:alpha_genli_rate}, \Cref{ass:regular-drv-general} is satisfied (since $\delta$ is constant). In the following, we further verify the last relationship in \Cref{ass:regular-drv-general} holds.

% \subsection{Verifying Assumption}

% In order to invoke \Cref{cor:main-general}, since \Cref{ass:m2-general,ass:score-general,ass:bdd-mismatch-general} have been satisfied, we only need to verify \Cref{ass:regular-drv-general}.
% In the following lemma, we verify \Cref{ass:regular-drv-general} when $Q_0$ has finite sixth moment.

% \subsection{Explicit Dimensional Dependence}
Therefore, since \Cref{ass:m2-general,ass:score-general,ass:bdd-mismatch-general} have been satisfied, we can invoke \Cref{cor:main-general} and get $\KL{\Tilde{Q}_{1}}{\widehat{P}_{1}} \lesssim \mathcal{W}_{\text{oracle}} + \mathcal{W}_{\text{bias}} + \mathcal{W}_{\text{vanish}}$.
Now, we investigate the dimensional dependence for each term of the upper bound in \Cref{cor:main-general}.

To start, from \cite[Theorem~3]{liang2024discrete}, for any $\Tilde{Q}_0$ having finite variance, with the $\alpha_t$ in \eqref{eq:alpha_genli}, we have
\[ \mathcal{W}_{\text{oracle}} \lesssim \frac{d^2 \log^2(1/\delta) (\log T)^2}{T} + (\log T) \eps^2. \]
Also, since by assumption $\E_{X_t \sim \Tilde{Q}_t} \norm{\Delta_t(X_t)}^2 \lesssim \frac{\Bar{\alpha}_t}{(1-\Bar{\alpha}_t)^r} d^\gamma$, with the $\alpha_t$ in \eqref{eq:alpha_genli}, we have from \Cref{lem:nonvanish_coef_sum_genli} that when $\delta \ll 1$,
\[ \mathcal{W}_{\text{bias}} \lesssim \frac{d^\gamma}{\delta^r} \brc{1 - \frac{2 \log(1/\delta) \log T}{T}}. \]

Now we investigate each term in $\mathcal{W}_{\text{vanish}}$. The following lemma is useful to determine the rate of difference of the first-order Taylor polynomials.

\begin{lemma} \label{lem:small-grad-diff-genli}
When $\E_{X_0 \sim \Tilde{Q}_{0}} \norm{X_0}^6 \lesssim d^3$, with the $\alpha_t$ in \eqref{eq:alpha_genli}, we have
\[ (1-\alpha_t) \sqrt{ \E_{X_t \sim \Tilde{Q}_t} \norm{ \nabla \log \Tilde{q}_{t-1}(m_t(X_t)) - \sqrt{\alpha_t} \nabla \log \Tilde{q}_t(X_t)}^2 } \lesssim \frac{d^{3/2} (1-\alpha_t)^2}{(1-\Bar{\alpha}_{t-1})^3}. \]
As a result, \Cref{ass:regular-drv-general} holds.
\end{lemma}

\begin{proof}
See \Cref{app:proof-lem-small-grad-diff-genli}.
\end{proof}

In other words, combining \Cref{lem:small-grad-diff-genli} and \Cref{lem:alpha_genli_rate}, we have
\begin{equation} \label{eq:small-grad-diff}
   (1-\alpha_t) \sqrt{ \E_{X_t \sim \Tilde{Q}_t} \norm{ \nabla \log \Tilde{q}_{t-1}(m_t(X_t)) - \sqrt{\alpha_t} \nabla \log \Tilde{q}_t(X_t)}^2 } = \Tilde{O}\brc{\frac{1}{T^2}}.
\end{equation}
Now, by Cauchy-Schwartz inequality and \Cref{lem:small-grad-diff-genli},
\begin{align*}
    &\sum_{t=2}^T \frac{1-\alpha_t}{\sqrt{\alpha_t}} \E_{X_t \sim \Tilde{Q}_t} \bigg[(\nabla \log \Tilde{q}_{t-1}(m_t(X_t)) - \sqrt{\alpha_t} \nabla \log \Tilde{q}_t(X_t))^\T \Delta_t(X_t)\bigg] \nonumber\\
    &\leq \sum_{t=2}^T \frac{1-\alpha_t}{\sqrt{\alpha_t}} \sqrt{\E_{X_t \sim \Tilde{Q}_t} \norm{\Delta_t(X_t)}^2} \times \nonumber\\
    &\qquad \sqrt{ \E_{X_t \sim \Tilde{Q}_t} \norm{\nabla \log \Tilde{q}_{t-1}(m_t(X_t)) - \sqrt{\alpha_t} \nabla \log \Tilde{q}_t(X_t)}^2 } \nonumber\\
    &\lesssim \frac{d^{\gamma / 2}}{(1-\Bar{\alpha}_t)^{r/2}} \times \frac{d^{3/2} (1-\alpha_t)^2}{(1-\Bar{\alpha}_{t-1})^3} \\
    &\lesssim \sum_{t=2}^T \frac{d^{\frac{3+\gamma}{2}} \log(1/\delta)^2 (\log T)^2}{\delta^{1 + r/2} T^2} \\
    &\leq \frac{d^{\frac{3+\gamma}{2}} \log(1/\delta)^2 (\log T)^2}{\delta^{1 + r/2} T}.
\end{align*}

To proceed for higher orders of Taylor polynomials, we first note that from \cite[Section~G.2]{liang2024discrete}, the second and third derivatives of $\log \Tilde{q}_t$ are
\begin{align*}
    \nabla^2 \log \Tilde{q}_t(x) &= - \frac{1}{1-\Bar{\alpha}_t} I_d + \frac{1}{(1-\Bar{\alpha}_t)^2} \bigg(\E_{X_0\sim \Tilde{Q}_{0|t}(\cdot|x)} \sbrc{(x-\sqrt{\Bar{\alpha}_t} X_0) (x-\sqrt{\Bar{\alpha}_t} X_0)^\T} \nonumber\\
    &\quad - \brc{\E_{X_0\sim \Tilde{Q}_{0|t}(\cdot|x)} \sbrc{x-\sqrt{\Bar{\alpha}_t} X_0 }} \brc{\E_{X_0\sim \Tilde{Q}_{0|t}(\cdot|x)} \sbrc{x-\sqrt{\Bar{\alpha}_t} X_0 }}^\T \bigg)\\
    \partial^3_{ijk} \log \Tilde{q}_t (x) &= - \int z^i z^j z^k \d \Tilde{Q}_{0|t}(x_0|x) \\
    &\qquad + \sum_{\substack{a_1 = i,j,k \\ a_2 < a_3,~ a_2,a_3 \neq a_1}} \int z^{a_1} \d \Tilde{Q}_{0|t}(x_0|x) \int z^{a_2} z^{a_3} \d \Tilde{Q}_{0|t}(x_0|x) \\
    &\qquad - 2 \int z^i \d \Tilde{Q}_{0|t}(x_0|x) \int z^j \d \Tilde{Q}_{0|t}(x_0|x) \int z^k \d \Tilde{Q}_{0|t}(x_0|x)
\end{align*}
where $z := \frac{x - \sqrt{\Bar{\alpha}_t} x_0}{1-\Bar{\alpha}_t}$.
Thus, in order to provide an upper bound on the expected norm of the second-order derivative of $\log \Tilde{q}_{t-1}(m_t)$, we can calculate
\begin{align*}
    &\E_{X_t \sim \Tilde{Q}_t} \norm{\E_{X_0\sim \Tilde{Q}_{0|t-1}(\cdot|m_t)} \sbrc{(m_t - \sqrt{\Bar{\alpha}_{t-1}} X_0) (m_t-\sqrt{\Bar{\alpha}_{t-1}} X_0)^\T}}_F^2\\
    &\leq \E_{X_t \sim \Tilde{Q}_t, X_0\sim \Tilde{Q}_{0|t-1}(\cdot|m_t) } \norm{(m_t-\sqrt{\Bar{\alpha}_{t-1}} X_0) (m_t-\sqrt{\Bar{\alpha}_{t-1}} X_0)^\T}_F^2\\
    &= \E_{X_t \sim \Tilde{Q}_t, X_0\sim \Tilde{Q}_{0|t-1}(\cdot|m_t)} \norm{m_t-\sqrt{\Bar{\alpha}_{t-1}} X_0}^4\\
    &\stackrel{(i)}{\lesssim} d^2 (1-\Bar{\alpha}_{t-1})^2,
\end{align*}
and
\begin{align*}
    &\E_{X_t \sim \Tilde{Q}_t} \norm{\brc{\E_{X_0\sim \Tilde{Q}_{0|t-1}} \sbrc{m_t-\sqrt{\Bar{\alpha}_{t-1}} X_0 }} \brc{\E_{X_0\sim \Tilde{Q}_{0|t-1}} \sbrc{m_t-\sqrt{\Bar{\alpha}_{t-1}} X_0 }}^\T}_F^2\\
    &= \E_{X_t \sim \Tilde{Q}_t} \norm{\E_{X_0\sim \Tilde{Q}_{0|t-1}(\cdot|m_t)} \sbrc{m_t-\sqrt{\Bar{\alpha}_{t-1}} X_0 }}^4\\
    &\leq \E_{X_t \sim \Tilde{Q}_t, X_0\sim \Tilde{Q}_{0|t-1}(\cdot|m_t) } \norm{m_t-\sqrt{\Bar{\alpha}_{t-1}} X_0}^4\\
    &\stackrel{(ii)}{\lesssim} d^2 (1-\Bar{\alpha}_{t-1})^2,
\end{align*}
where both $(i)$ and $(ii)$ follow from \cite[Lemma~16]{liang2024discrete}.
Thus,
\[ \E_{X_t \sim \Tilde{Q}_t} \norm{\nabla^2\log \Tilde{q}_{t-1}(m_t(X_t))}_F^2 \lesssim \frac{1}{(1-\Bar{\alpha}_{t-1})^2} d^2. \]
For third-order derivatives, we can similarly use \cite[Lemma~16]{liang2024discrete} and get (cf. \cite[Section~G.2]{liang2024discrete})
% \[ \E_{X_t \sim \Tilde{Q}_t} \sbrc{ \sum_{i,j,k=1}^d ( \partial^3_{ijk} \log \Tilde{q}_{t-1}(m_t(X_t)) )^2 } \lesssim \frac{d^3}{(1-\Bar{\alpha}_{t-1})^3}. \]
\begin{align*}
    \E_{X_t \sim \Tilde{Q}_t} \sbrc{ \sum_{i=1}^d ( \partial^3_{iii} \log \Tilde{q}_{t-1}(m_t(X_t)) )^2 } &\lesssim \frac{1}{(1-\Bar{\alpha}_{t-1})^3} \sum_{i=1}^d \E (Z^i)^6 \lesssim \frac{d}{(1-\Bar{\alpha}_{t-1})^3},\\
    \E_{X_t \sim \Tilde{Q}_t} \sbrc{ \sum_{i,j=1}^d ( \partial^3_{iij} \log \Tilde{q}_{t-1}(m_t(X_t)) )^2 } &\lesssim \frac{1}{(1-\Bar{\alpha}_{t-1})^3} \sum_{i,j=1}^d \brc{\E (Z^i)^6}^{2/3} \brc{\E (Z^j)^6}^{1/3} \\
    &\lesssim \frac{d^2}{(1-\Bar{\alpha}_{t-1})^3}.
\end{align*}
Here we denote $Z \sim \calN(0,I_d)$, and note that $\E (Z^i)^6, \E (Z^j)^6 \lesssim 1$.

Therefore, by Cauchy-Schwartz inequality,
\begin{align*}
    &\sum_{t=2}^T \frac{(1-\alpha_t)^2}{2 \alpha_t} \E_{X_t \sim \Tilde{Q}_t}\sbrc{\Delta_t(X_t)^\T \nabla^2 \log \Tilde{q}_{t-1}(m_t(X_t)) \Delta_t(X_t) } \\
    &\leq \sum_{t=2}^T \frac{(1-\alpha_t)^2}{2 \alpha_t} \sqrt{ \E_{X_t \sim \Tilde{Q}_t} \norm{\Delta_t(X_t)}^4 } \sqrt{ \E_{X_t \sim \Tilde{Q}_t} \norm{\nabla^2 \log \Tilde{q}_{t-1}(m_t(X_t)) }^2 }\\
    &\lesssim \sum_{t=2}^T \frac{(1-\alpha_t)^2}{2 \alpha_t} \frac{d^\gamma}{(1-\Bar{\alpha}_t)^r} \sqrt{ \E_{X_t \sim \Tilde{Q}_t} \norm{\nabla^2 \log \Tilde{q}_{t-1}(m_t(X_t)) }_F^2 }\\
    &\lesssim \sum_{t=2}^T \frac{(1-\alpha_t)^2}{2 \alpha_t} \frac{d^\gamma}{(1-\Bar{\alpha}_t)^r} \sqrt{ \frac{d^2}{(1-\Bar{\alpha}_{t-1})^2} }\\
    &\lesssim \frac{d^{1 + \gamma} \log(1/\delta)^2 (\log T)^2}{\delta^{r-1} T},
\end{align*}
and
\begin{align*}
    &\sum_{t=2}^T \frac{(1-\alpha_t)^2}{3! {\alpha_t}^{3/2}} \E_{X_t \sim \Tilde{Q}_t} \sbrc{ 3 \sum_{i=1}^d \partial^3_{iii} \log \Tilde{q}_{t-1}(m_t(X_t)) \Delta_t(X_t)^i } \\
    &\leq \sum_{t=2}^T \frac{3 (1-\alpha_t)^2}{3! {\alpha_t}^{3/2}} \sqrt{ \E_{X_t \sim \Tilde{Q}_t} \norm{\Delta_t(X_t)}^2 } \sqrt{ \E_{X_t \sim \Tilde{Q}_t} \sum_{i=1}^d ( \partial^3_{iii} \log \Tilde{q}_{t-1}(m_t(X_t)) )^2 }\\
    &\lesssim \sum_{t=2}^T (1-\alpha_t)^2 \frac{d^{\gamma/2}}{(1-\Bar{\alpha}_t)^{r/2}} \sqrt{\frac{d}{(1-\Bar{\alpha}_{t-1})^3}}\\
    &\lesssim \frac{d^{\frac{1+\gamma}{2}} \log(1/\delta)^2 (\log T)^2}{\delta^{\frac{r-1}{2}} T},
\end{align*}
and, with $M$ being a matrix such that $M^{ij} (x):= \partial^3_{iij} \log \Tilde{q}_{t-1}(m_t(x))$,
\begin{align*}
    &\sum_{t=2}^T \frac{(1-\alpha_t)^2}{3! {\alpha_t}^{3/2}} \E_{X_t \sim \Tilde{Q}_t} \sbrc{ \sum_{\substack{i,j=1 \\ i\neq j}}^d \partial^3_{iij} \log \Tilde{q}_{t-1}(m_t(X_t)) \Delta_t(X_t)^j }\\
    &\leq \sum_{t=2}^T \frac{(1-\alpha_t)^2}{3! {\alpha_t}^{3/2}} \E_{X_t \sim \Tilde{Q}_t} \norm{ M(X_t) \Delta_t(X_t) }_1\\
    &\leq \sum_{t=2}^T \frac{(1-\alpha_t)^2}{3! {\alpha_t}^{3/2}} \sqrt{d} \sqrt{\E_{X_t \sim \Tilde{Q}_t} \norm{ \Delta_t(X_t) }^2 } \sqrt{\E_{X_t \sim \Tilde{Q}_t} \norm{ M(X_t) }^2 }\\
    &\leq \sum_{t=2}^T \frac{(1-\alpha_t)^2}{3! {\alpha_t}^{3/2}} \sqrt{d} \sqrt{ \E_{X_t \sim \Tilde{Q}_t} \norm{ \Delta_t(X_t)}^2 } \sqrt{\E_{X_t \sim \Tilde{Q}_t} \sbrc{\sum_{i,j=1}^d (\partial^3_{iij} \log \Tilde{q}_{t-1}(m_t(X_t)))^2} }\\
    &\lesssim \sum_{t=2}^T (1-\alpha_t)^2 \frac{d^{(1+\gamma)/2}}{(1-\Bar{\alpha}_t)^{r/2}} \sqrt{\frac{d^2}{(1-\Bar{\alpha}_{t-1})^3}}\\
    &\lesssim \frac{d^{\frac{3+\gamma}{2}} \log(1/\delta)^2 (\log T)^2}{\delta^{\frac{r-1}{2}} T},
\end{align*}
and
\[ \max_{t \geq 2}\sqrt{\E_{X_t \sim \Tilde{Q}_t} \norm{\Delta_t(X_t)}^2} (\log T) \eps \lesssim \frac{d^{\gamma/2}}{\delta^{r/2}} (\log T) \eps. \]
Therefore, combining all the above, we get
\begin{align*}
    \KL{\Tilde{Q}_{1}}{\widehat{P}_{1}} &\lesssim d^{\gamma} \delta^{-r} \brc{1 - \frac{2 \log(1/\delta) \log T}{T}} \\
    &\qquad + \max \{ d^{(3+\gamma)/2} \delta^{-\frac{r+2}{2}}, d^{1+\gamma} \delta^{-(r-1)} \} \frac{(\log T)^2}{T} \\
    &\qquad + d^{\gamma/2} \delta^{-r/2} (\log T) \eps.
\end{align*}

\section{Auxiliary Lemmas and Proofs in \texorpdfstring{\Cref{sec:prob-general}}{Section 3}}

\subsection{Proof of \texorpdfstring{\Cref{lem:rev-err-tilt-factor}}{Lemma 1}}
\label{app:lem-rev-err-tilt-factor-proof}

We remind readers that throughout this proof $x_t$ is fixed. For brevity write $m_t = m_t(x_t)$, $\mu_t = \mu_t(x_t)$, and $\Delta_t(x_t) = \Delta_t$. Recall that $m_t = x_t / \sqrt{\alpha_t} + (1-\alpha_t)/\sqrt{\alpha_t} \nabla \log \Tilde{q}_t(x_t)$.
By Bayes' rule, we have
\begin{align*}
    &\Tilde{q}_{t-1|t}(x_{t-1}|x_t)\\
    &= \frac{\Tilde{q}_{t|t-1}(x_t|x_{t-1}) \Tilde{q}_{t-1}(x_{t-1})}{\Tilde{q}_t(x_t)}\\
    &\propto \Tilde{q}_{t-1}(x_{t-1}) \Tilde{q}_{t|t-1}(x_t|x_{t-1})\\
    &\stackrel{(i)}{\propto} \Tilde{q}_{t-1}(x_{t-1}) \exp\brc{-\frac{\norm{x_t - \sqrt{\alpha_t} x_{t-1}}^2}{2(1-\alpha_t)}} \\
    &\propto \Tilde{q}_{t-1}(x_{t-1}) p_{t-1|t}(x_{t-1}|x_t) \exp\brc{\frac{\norm{x_{t-1}-\mu_t}^2 - \norm{x_{t-1}-x_t/\sqrt{\alpha_t}}^2}{2 (1-\alpha_t)/\alpha_t}}\\
    &= \Tilde{q}_{t-1}(x_{t-1}) p_{t-1|t}(x_{t-1}|x_t) \exp\brc{\frac{\norm{x_{t-1}-x_t/\sqrt{\alpha_t} + x_t/\sqrt{\alpha_t}-\mu_t}^2 - \norm{x_{t-1}-x_t/\sqrt{\alpha_t}}^2}{2 (1-\alpha_t)/\alpha_t}}\\
    &\propto \Tilde{q}_{t-1}(x_{t-1}) p_{t-1|t}(x_{t-1}|x_t) \exp\brc{\frac{(x_{t-1}-x_t/\sqrt{\alpha_t})^\T (x_t/\sqrt{\alpha_t}-\mu_t)}{(1-\alpha_t)/\alpha_t}}
    % &\propto p_{t-1|t}(x_{t-1}|x_t) \exp\brc{\log \Tilde{q}_{t-1}(x_{t-1}) - \sqrt{\alpha_t} x_{t-1}^\T \nabla \log \Tilde{q}_t(x_t)}\\
    % &\propto p_{t-1|t}(x_{t-1}|x_t) \exp\brc{\zeta_{t,t-1}(x_t,x_{t-1})},
\end{align*}
where $(i)$ follows because the forward process is Markov and $\Tilde{q}_{t|t-1} = q_{t|t-1}$.
Here, the exponent is equal to
\begin{align*}
    &\frac{(x_{t-1}-x_t/\sqrt{\alpha_t})^\T (x_t/\sqrt{\alpha_t}-\mu_t)}{(1-\alpha_t)/\alpha_t} \\
    &= \frac{(x_{t-1}-x_t/\sqrt{\alpha_t})^\T (m_t-\mu_t)}{(1-\alpha_t)/\alpha_t} - \frac{(x_{t-1}-x_t/\sqrt{\alpha_t})^\T ((1-\alpha_t)/\sqrt{\alpha_t}) \nabla \log \Tilde{q}_t(x_t)}{(1-\alpha_t)/\alpha_t}\\
    &= \sqrt{\alpha_t} \Delta_t^\T (x_{t-1}-x_t/\sqrt{\alpha_t}) - \sqrt{\alpha_t} (x_{t-1}-x_t/\sqrt{\alpha_t})^\T \nabla \log \Tilde{q}_t(x_t).
\end{align*}
Thus,
\[ \Tilde{q}_{t-1|t}(x_{t-1}|x_t) \propto p_{t-1|t}(x_{t-1}|x_t) \exp\brc{\Tilde{\zeta}_{t,t-1}(x_t,x_{t-1})} \]
where
\[ \Tilde{\zeta}_{t,t-1}(x_t,x_{t-1}) = \sqrt{\alpha_t} \Delta_t^\T (x_{t-1}-m_t) + \log \Tilde{q}_{t-1}(x_{t-1}) - \sqrt{\alpha_t} (x_{t-1} - m_t)^\T \nabla \log \Tilde{q}_t(x_t). \]
Finally, since all partial derivatives of $\Tilde{q}_{t-1}$ exists for any $t \geq 2$ (See \cite[Lemma~6]{liang2024discrete}), the Taylor expansion of $\log \Tilde{q}_{t-1}$ around $m_t$ gives the desirable result.

\subsection{Proof of \texorpdfstring{\Cref{lem:score-est-ptb}}{Lemma 2}}
\label{app:lem-score-est-ptb}

For each $t = 1,\dots,T$,
\begin{align*}
    \log \frac{p_{t-1|t}(x_{t-1}|x_t)}{\widehat{p}_{t-1|t}(x_{t-1}|x_t)} &= \frac{\alpha_t}{2(1-\alpha_t)} \brc{\norm{x_{t-1}-\widehat{\mu}_t(x_t)}^2 - \norm{x_{t-1}-\mu_t(x_t)}^2}\\
    &= \frac{\alpha_t}{(1-\alpha_t)} (x_{t-1}-\mu_t(x_t))^\T (\mu_t(x_t)-\widehat{\mu}_t(x_t)) + \frac{\alpha_t}{2 (1-\alpha_t)} \norm{\mu_t(x_t)-\widehat{\mu}_t(x_t)}^2.
\end{align*}
For the first term above,
\begin{align*}
    &\E_{X_t,X_{t-1} \sim \Tilde{Q}_{t,t-1}} \sbrc{(X_{t-1}-\mu_t(X_t))^\T (\mu_t(X_t)-\widehat{\mu}_t(X_t))}\\
    % &= \E_{X_t \sim Q_t} \sbrc{(\E_{X_{t-1} \sim Q_{t-1|t}} [X_{t-1}|X_t]-\mu_t(X_t))^\T (\mu_t(X_t)-\widehat{\mu}_t(X_t))} \\
    &= \E_{X_t \sim \Tilde{Q}_t} \sbrc{(m_t(X_t)-\mu_t(X_t))^\T (\mu_t(X_t)-\widehat{\mu}_t(X_t))}\\
    &= \frac{1-\alpha_t}{\sqrt{\alpha_t}} \E_{X_t \sim \Tilde{Q}_t} \sbrc{\Delta_t(X_t)^\T (\mu_t(X_t)-\widehat{\mu}_t(X_t))}\\
    &\leq \frac{1-\alpha_t}{\sqrt{\alpha_t}} \sqrt{\E_{X_t \sim \Tilde{Q}_t} \norm{\Delta_t(X_t)}^2 \E_{X_t \sim \Tilde{Q}_t} \norm{\mu_t(X_t)-\widehat{\mu}_t(X_t)}^2 }.
    % &\lesssim \frac{1-\alpha_t}{\sqrt{\alpha_t}} \sqrt{\E_{X_t \sim \Tilde{Q}_t} \norm{\Delta_t(X_t)}^2 \times \E_{X_t \sim \Tilde{Q}_t} \norm{\mu_t(X_t)-\widehat{\mu}_t(X_t)}^2 }
\end{align*}
% where the last line follows from \Cref{ass:bdd-mismatch}.
Here we recall the definition of $\Delta_t$ from \eqref{eq:def_delta_general} where $m_t(x) - \mu_t(x) = \frac{1-\alpha_t}{\sqrt{\alpha_t}} \Delta_t(x)$.
Thus,
\begin{align*}
    &\sum_{t=1}^T \E_{X_t,X_{t-1} \sim \Tilde{Q}_{t,t-1}} \sbrc{\log \frac{p_{t-1|t}(X_{t-1}|X_t)}{\widehat{p}_{t-1|t}(X_{t-1}|X_t)}} \\
    &\lesssim \sum_{t=1}^T \brc{ \sqrt{\alpha_t} \sqrt{\E_{X_t \sim \Tilde{Q}_t} \norm{\Delta_t(X_t)}^2 \E_{X_t \sim \Tilde{Q}_t} \norm{\mu_t(X_t)-\widehat{\mu}_t(X_t)}^2 } + \frac{\alpha_t}{1-\alpha_t} \E_{X_t \sim Q_t} \norm{\mu_t(X_t)-\widehat{\mu}_t(X_t)}^2 }\\
    &= \sum_{t=1}^T (1-\alpha_t) \sqrt{\E_{X_t \sim \Tilde{Q}_t} \norm{\Delta_t(X_t)}^2} \times \sqrt{\frac{\alpha_t}{(1-\alpha_t)^2} \E_{X_t \sim \Tilde{Q}_t} \norm{\mu_t(X_t)-\widehat{\mu}_t(X_t)}^2 } \\
    &\qquad + \sum_{t=1}^T (1-\alpha_t) \frac{\alpha_t}{(1-\alpha_t)^2} \E_{X_t \sim Q_t} \norm{\mu_t(X_t)-\widehat{\mu}_t(X_t)}^2 \\
    &\stackrel{(i)}{\lesssim} \frac{\log T}{T} \sum_{t=1}^T \sqrt{\E_{X_t \sim \Tilde{Q}_t} \norm{\Delta_t(X_t)}^2} \times \sqrt{\frac{\alpha_t}{(1-\alpha_t)^2} \E_{X_t \sim \Tilde{Q}_t} \norm{\mu_t(X_t)-\widehat{\mu}_t(X_t)}^2 } \\
    &\qquad + \frac{\log T}{T} \sum_{t=1}^T \frac{\alpha_t}{(1-\alpha_t)^2} \E_{X_t \sim Q_t} \norm{\mu_t(X_t)-\widehat{\mu}_t(X_t)}^2 \\
    &\leq \max_{t \geq 1} \sqrt{\E_{X_t \sim \Tilde{Q}_t} \norm{\Delta_t(X_t)}^2} \frac{\log T}{T} \sum_{t=1}^T \sqrt{\frac{\alpha_t}{(1-\alpha_t)^2} \E_{X_t \sim \Tilde{Q}_t} \norm{\mu_t(X_t)-\widehat{\mu}_t(X_t)}^2 } \\
    &\qquad + \frac{\log T}{T} \sum_{t=1}^T \frac{\alpha_t}{(1-\alpha_t)^2} \E_{X_t \sim Q_t} \norm{\mu_t(X_t)-\widehat{\mu}_t(X_t)}^2 \\
    &\stackrel{(ii)}{\lesssim} \max_{t \geq 1} \sqrt{\E_{X_t \sim \Tilde{Q}_t} \norm{\Delta_t(X_t)}^2} (\log T) \sqrt{\frac{1}{T} \sum_{t=1}^T \frac{\alpha_t}{(1-\alpha_t)^2} \E_{X_t \sim \Tilde{Q}_t} \norm{\mu_t(X_t)-\widehat{\mu}_t(X_t)}^2 } \\
    &\qquad + \frac{\log T}{T} \sum_{t=1}^T \frac{\alpha_t}{(1-\alpha_t)^2} \E_{X_t \sim Q_t} \norm{\mu_t(X_t)-\widehat{\mu}_t(X_t)}^2 \\
    &\stackrel{(iii)}{\lesssim} \max_{t \geq 1}\sqrt{\E_{X_t \sim \Tilde{Q}_t} \norm{\Delta_t(X_t)}^2} (\log T) \eps + (\log T) \eps^2
\end{align*}
where $(i)$ follows from 
% \Cref{ass:bdd-mismatch} and 
\Cref{def:noise_smooth}, $(ii)$ follows from the fact that for any non-negative sequence $a_t$, $\frac{1}{T} \sum_{t=1}^T \sqrt{a_t} \leq \sqrt{\frac{1}{T} \sum_{t=1}^T a_t}$ by Jensen's inequality, and $(iii)$ follows from \Cref{ass:score-general}. The proof is complete.

\subsection{Proof of \texorpdfstring{\Cref{lem:tilde-zeta-van-expr}}{Lemma 3}}
\label{app:lem-tilde-zeta-van-expr-proof}

Recall that $P_{t-1|t} = \calN(\mu_t, \frac{1-\alpha_t}{\alpha_t} I_d)$, and thus $\E_{X_{t-1} \sim P_{t-1|t}}[X_{t-1} - m_t(x_t)] = - \frac{1-\alpha_t}{\sqrt{\alpha_t}} \Delta_t(x_t)$. 
Note that we can change the limit and the expectation under \Cref{ass:regular-drv-general}.
Now, we can calculate that
\begin{align*}
    &\E_{X_{t-1} \sim P_{t-1|t}} [\Tilde{\zeta}_{\text{van}} (x_t,X_{t-1})]\\
    &= - \frac{1-\alpha_t}{\sqrt{\alpha_t}} (\nabla \log \Tilde{q}_{t-1}(m_t(x_t)) - \sqrt{\alpha_t} \nabla \log \Tilde{q}_{t}(x_t))^\T \Delta_t(x_t)\\
    &+ \sum_{p=2}^\infty \E_{X_{t-1} \sim P_{t-1|t}} \sbrc{ T_p(\log \Tilde{q}_{t-1}, X_{t-1}, m_t(x_t))}.
\end{align*}
Below we write $m_t = m_t(x_t)$. Since $T_2$ is in quadratic form, the expected value under $P_{t-1|t}$ for this term is
\begin{align*}
    &\E_{X_{t-1} \sim P_{t-1|t}} \sbrc{ T_2(\log \Tilde{q}_{t-1}, X_{t-1}, m_t)} \\
    &\quad = \frac{1-\alpha_t}{2 \alpha_t} \Tr(\nabla^2 \log \Tilde{q}_{t-1}(m_t)) + \frac{(1-\alpha_t)^2}{2 \alpha_t} \Delta_t(x_t)^\T \nabla^2 \log \Tilde{q}_{t-1}(m_t) \Delta_t(x_t).
\end{align*}
Recall the formula for Gaussian non-centralized third moment. If $Z \sim \calN(\mu,\sigma^2)$, then $\E[Z^2] = \mu^2 + \sigma^2$ and $\E[Z^3] = \mu^3 + 3 \mu \sigma^2$. Thus, the expected value under $P_{t-1|t}$ for $T_3$ is
\begin{align*}
    &\E_{X_{t-1} \sim P_{t-1|t}} \sbrc{ T_3(\log \Tilde{q}_{t-1}, X_{t-1}, m_t)}\\ &= \frac{1}{3!} \sum_{i=1}^d \partial^3_{iii} \log \Tilde{q}_{t-1}(m_t) \E_{X_{t-1} \sim P_{t-1|t}} (X_{t-1}^i-m_t^i)^3 \\
    &\quad + \frac{1}{3!} \sum_{\substack{i,j=1 \\ i\neq j}}^d \partial^3_{iij} \log \Tilde{q}_{t-1}(m_t) \E_{X_{t-1} \sim P_{t-1|t}} (X_{t-1}^i-m_t^i)^2 (X_{t-1}^j-m_t^j) \nonumber\\
    &\quad + \frac{1}{3!} \sum_{\substack{i,j,k=1 \\ i,j,k \text{ all differ}}}^d \partial^3_{ijk} \log \Tilde{q}_{t-1}(m_t) \E_{X_{t-1} \sim P_{t-1|t}} (X_{t-1}^i-m_t^i) (X_{t-1}^j-m_t^j) (X_{t-1}^k-m_t^k)\\
    &= \frac{1}{3!} \sum_{i=1}^d \partial^3_{iii} \log \Tilde{q}_{t-1}(m_t) \brc{\brc{- \frac{1-\alpha_t}{\sqrt{\alpha_t}} \Delta_t^i}^3 + 3 \brc{- \frac{1-\alpha_t}{\sqrt{\alpha_t}} \Delta_t^i} \brc{\frac{1-\alpha_t}{\alpha_t}}} \\
    &\quad + \frac{1}{3!} \sum_{\substack{i,j=1 \\ i\neq j}}^d \partial^3_{iij} \log \Tilde{q}_{t-1}(m_t) \brc{\brc{- \frac{1-\alpha_t}{\sqrt{\alpha_t}} \Delta_t^i}^2 + \brc{\frac{1-\alpha_t}{\alpha_t}}} \brc{- \frac{1-\alpha_t}{\sqrt{\alpha_t}} \Delta_t^j} \nonumber\\
    &\quad + \frac{1}{3!} \sum_{\substack{i,j,k=1 \\ i,j,k \text{ all differ}}}^d \partial^3_{ijk} \log \Tilde{q}_{t-1}(m_t) \brc{- \frac{1-\alpha_t}{\sqrt{\alpha_t}} \Delta_t^i} \brc{- \frac{1-\alpha_t}{\sqrt{\alpha_t}} \Delta_t^j} \brc{- \frac{1-\alpha_t}{\sqrt{\alpha_t}} \Delta_t^k}\\
    &= \frac{1}{3!} \brc{- \frac{(1-\alpha_t)^2}{{\alpha_t}^{3/2}}} \brc{ 3 \sum_{i=1}^d \partial^3_{iii} \log \Tilde{q}_{t-1}(m_t) \Delta_t^i + \sum_{\substack{i,j=1 \\ i\neq j}}^d \partial^3_{iij} \log \Tilde{q}_{t-1}(m_t) \Delta_t^j } \\
    &\quad + O_{\calL^\ell(\Tilde{Q}_t)}\brc{(1-\alpha_t)^3},~\forall \ell \geq 1.
\end{align*}
Here the last line follows because $(1-\alpha_t)^3 \abs{\partial^3_{ijk} \log \Tilde{q}_{t-1}(m_t)} = O_{\calL^\ell(\Tilde{Q}_t)}\brc{(1-\alpha_t)^3}$ under \Cref{ass:regular-drv-general} and $(1-\alpha_t)^3 \norm{\Delta_t} = O_{\calL^\ell(\Tilde{Q}_t)}((1-\alpha_t)^3)$ under \Cref{ass:bdd-mismatch-general}, both for all $\ell \geq 1$. The proof is now complete.

\subsection{Proof of \texorpdfstring{\Cref{lem:tilde-zeta-van-higher}}{Lemma 4}}
\label{app:lem-tilde-zeta-van-higher}

% The result is a direct application to Gaussian non-centralized moments. 
Fix $x_t \in \mbR^d$. For brevity write $m_t = m_t(x_t)$, $\mu_t = \mu_t(x_t)$, and $\Delta_t = \Delta_t(x_t)$. Recall that
\[ T_p (\log \Tilde{q}_{t-1},x_{t-1},m_t) = \frac{1}{p!} \sum_{\gamma \in \mbN^d:\sum_i \gamma^i = p} \partial^p_{\bm{a}} \log \Tilde{q}_{t-1}(m_t) \prod_{i=1}^d (x_{t-1}^i-m_t^i)^{\gamma^i} \]
where $\bm{a} \in [d]^p$ are the indices of differentiation in which the multiplicity of $i$ is $\gamma^i$.
First, for the expectation under $\Tilde{P}_{t-1|t}$ (i.e., Gaussian centralized moments),
\begin{align*}
    \E_{X_{t-1} \sim \Tilde{P}_{t-1|t}} \sbrc{ \prod_{i=1}^d (X_{t-1}^i-m_t^i)^{\gamma^i} } &= \prod_{i=1}^d \E_{X_{t-1} \sim \Tilde{P}_{t-1|t}} \sbrc{ (X_{t-1}^i-m_t^i)^{\gamma^i} }\\
    &= \prod_{i=1}^d \brc{\frac{1-\alpha_t}{\alpha_t}}^{\gamma^i/2} (\gamma^i-1)!! \ind\{\gamma^i \text{ is even}\}\\
    &= \brc{\frac{1-\alpha_t}{\alpha_t}}^{p/2} \prod_{i=1}^d (\gamma^i-1)!! \ind\{\gamma^i \text{ is even}\},
\end{align*}
where we use the convention that $(-1)!! = 1$.
Next, for the expectation under $P_{t-1|t}$ (i.e., Gaussian non-centralized moments),
\begin{align*}
    &\E_{X_{t-1} \sim P_{t-1|t}} \sbrc{ \prod_{i=1}^d (X_{t-1}^i-m_t^i)^{\gamma^i} }\\
    &= \prod_{\substack{i=1}}^d \E_{X_{t-1} \sim P_{t-1|t}}\sbrc{ (X_{t-1}^i-\mu_t^i - \frac{1-\alpha_t}{\sqrt{\alpha_t}} \Delta_t^i)^{\gamma^i} } \\
    &= \prod_{\substack{i=1}}^d \sum_{\substack{\ell=0 \\ \ell \text{ even}}}^{\gamma^i} \binom{\gamma^i}{\ell} \brc{- \frac{1-\alpha_t}{\sqrt{\alpha_t}} \Delta_t^i}^{\gamma^i-\ell} \brc{\frac{1-\alpha_t}{\alpha_t}}^{\ell/2} (\ell-1)!!
\end{align*}
% \begin{align*}
%     &\abs{\E_{X_{t-1} \sim P_{t-1|t}} \sbrc{ \prod_{i=1}^d (X_{t-1}^i-m_t^i)^{\gamma^i} } }\\
%     &= \prod_{\substack{i=1 \\ \gamma^i > 0}}^d \abs{\E_{X_{t-1} \sim P_{t-1|t}}\sbrc{ (X_{t-1}^i-\mu_t^i - \frac{1-\alpha_t}{\sqrt{\alpha_t}} \Delta_t^i)^{\gamma^i} }} \\
%     &\leq \prod_{\substack{i=1 \\ \gamma^i > 0}}^d \E_{X_{t-1} \sim P_{t-1|t}}\sbrc{ \abs{ X_{t-1}^i-\mu_t^i - \frac{1-\alpha_t}{\sqrt{\alpha_t}} \Delta_t^i }^{\gamma^i} } \\
%     &\stackrel{(i)}{\leq} \prod_{\substack{i=1 \\ \gamma^i > 0}}^d 2^{\gamma^i-1} \brc{ \E_{X_{t-1} \sim P_{t-1|t}}\sbrc{ \abs{ X_{t-1}^i-\mu_t^i}^{\gamma^i} } + \abs{\frac{1-\alpha_t}{\sqrt{\alpha_t}} \Delta_t^i }^{\gamma^i}} \\
%     &\stackrel{(ii)}{\leq} \prod_{\substack{i=1 \\ \gamma^i > 0}}^d 2^{\gamma^i-1} \brc{ (\gamma^i-1)!! \brc{\frac{1-\alpha_t}{\alpha_t}}^{\gamma^i/2}  + \brc{\frac{1-\alpha_t}{\sqrt{\alpha_t}}}^{\gamma^i} \abs{\Delta_t^i }^{\gamma^i}} \\
%     &\stackrel{(iii)}{=} O\brc{\brc{\frac{1-\alpha_t}{\alpha_t}}^{p/2}}
% \end{align*}
% where $(i)$ follows because $(a+b)^p \leq 2^{p-1} (a^p + b^p)$ for all $p \geq 1$, $(ii)$ follows from Gaussian centralized absolute moments, and $(iii)$ follows from upper-bounded $\norm{\Delta_{t}}$ under \Cref{ass:bdd-mismatch}.
To investigate their difference, we divide into the following few cases. Note that under \Cref{ass:bdd-mismatch-general}, $(1-\alpha_t)^m \norm{\Delta_t(x_t)} = O_{\calL^\ell(\Tilde{Q}_t)}((1-\alpha_t)^m)$ for any $m \geq 1/2$ and $\ell \geq 1$.
\begin{enumerate}
    \item Case 1: $p$ is even and all elements of $\gamma^i$ are even. Then,
    \begin{align*}
        &\E_{X_{t-1} \sim P_{t-1|t}} \sbrc{ \prod_{i=1}^d (X_{t-1}^i-m_t^i)^{\gamma^i} } \\
        &= \prod_{\substack{i=1}}^d \brc{\brc{\frac{1-\alpha_t}{\alpha_t}}^{\gamma^i/2} (\gamma^i-1)!! + O_{\calL^\ell(\Tilde{Q}_t)}\brc{(1-\alpha_t)^{\gamma^i/2+1}} }\\
        &= \brc{\frac{1-\alpha_t}{\alpha_t}}^{p/2} \prod_{i=1}^d (\gamma^i-1)!! + O_{\calL^\ell(\Tilde{Q}_t)}\brc{(1-\alpha_t)^{p/2+1}}
    \end{align*}
    \item Case 2: $p$ is even and $\exists i^*$ such that $\gamma^{i^*}$ is odd. Since $\sum_i \gamma^i = p$, there exists $j^*$ such that $\gamma^{j^*}$ is also odd. Then,
    \begin{align*}
        \E_{X_{t-1} \sim P_{t-1|t}} \sbrc{(X_{t-1}^{i^*}-m_t^{i^*})^{\gamma^{i^*}} } &= O_{\calL^\ell(\Tilde{Q}_t)}\brc{\brc{1-\alpha_t}^{(\gamma^{i^*}+1)/2}},\\
        \E_{X_{t-1} \sim P_{t-1|t}} \sbrc{(X_{t-1}^{j^*}-m_t^{j^*})^{\gamma^{j^*}} } &= O_{\calL^\ell(\Tilde{Q}_t)}\brc{\brc{1-\alpha_t}^{(\gamma^{j^*}+1)/2}},
    \end{align*}
    which implies that
    \[ \E_{X_{t-1} \sim P_{t-1|t}} \sbrc{ \prod_{i=1}^d (X_{t-1}^i-m_t^i)^{\gamma^i} } = O_{\calL^\ell(\Tilde{Q}_t)}\brc{\brc{1-\alpha_t}^{p/2+1}} \]
    \item Case 3: $p$ is odd and $\exists i^*$ such that $\gamma^{i^*}$ is odd. Then,
    \[ \E_{X_{t-1} \sim P_{t-1|t}} \sbrc{(X_{t-1}^{i^*}-m_t^{i^*})^{\gamma^{i^*}} } = O_{\calL^\ell(\Tilde{Q}_t)}\brc{\brc{1-\alpha_t}^{(\gamma^{i^*}+1)/2}}, \]
    which implies that
    \[ \E_{X_{t-1} \sim P_{t-1|t}} \sbrc{ \prod_{i=1}^d (X_{t-1}^i-m_t^i)^{\gamma^i} } = O_{\calL^\ell(\Tilde{Q}_t)}\brc{\brc{1-\alpha_t}^{(p+1)/2}} \]
\end{enumerate}
Combining these cases, we get
\begin{multline*}
    \brc{\E_{X_{t-1} \sim \Tilde{P}_{t-1|t}} - \E_{X_{t-1} \sim P_{t-1|t}}} \sbrc{ \prod_{i=1}^d (X_{t-1}^i-m_t^i)^{\gamma^i} } \\
    = \begin{cases}
    O_{\calL^\ell(\Tilde{Q}_t)}\brc{(1-\alpha_t)^{\frac{p}{2} + 1}},&\forall p \geq 4~\text{even}\\
    O_{\calL^\ell(\Tilde{Q}_t)}\brc{(1-\alpha_t)^{\frac{p+1}{2}}},&\forall p \geq 4~\text{odd}
\end{cases}
\end{multline*}

% Finally, note that under \Cref{ass:regular-drv,ass:bdd-mismatch}, given any $\bm{a} \in [d]^m$ (where $m \leq p$),
% \begin{align*}
%     &\E_{X_{t} \sim \Tilde{Q}_{t}} \sbrc{ \abs{ \partial^p_{\bm{a}} \log \Tilde{q}_{t-1}(m_t(X_t)) } \abs{\prod_{i \in \bm{a}} \Delta_t(X_t)^i } } \\
%     &\stackrel{(i)}{\leq} \sqrt{ \E_{X_{t} \sim \Tilde{Q}_{t}} \abs{ \partial^p_{\bm{a}} \log \Tilde{q}_{t-1}(m_t(X_t)) }^2 \prod_{i \in \bm{a}} \brc{ \E_{X_{t} \sim \Tilde{Q}_{t}} \abs{ \Delta_t(X_t)^i }^{2m} }^{\frac{1}{m}} } \\
%     &\leq \sqrt{ \E_{X_{t} \sim \Tilde{Q}_{t}} \abs{ \partial^p_{\bm{a}} \log \Tilde{q}_{t-1}(m_t(X_t)) }^2 \max_{i \in \bm{a}} \brc{ \E_{X_{t} \sim \Tilde{Q}_{t}} \abs{ \Delta_t(X_t)^i }^{2m} } } \\
%     &\leq \sqrt{ \E_{X_{t} \sim \Tilde{Q}_{t}} \abs{ \partial^p_{\bm{a}} \log \Tilde{q}_{t-1}(m_t(X_t)) }^2 \brc{ \sqrt{d} \cdot \E_{X_{t} \sim \Tilde{Q}_{t}} \norm{\Delta_t(X_t)}^{2} } } = O(1).
% \end{align*}
% where $(i)$ follows by H\"older's inequality.
% Therefore, taking the expectation over $\Tilde{Q}_t$ does not change the rate w.r.t. $T$.
The proof is complete by noting that the rate does not change when we take the expectation over $\Tilde{Q}_t$ under \Cref{ass:regular-drv-general,ass:bdd-mismatch-general}.

\subsection{Proof of \texorpdfstring{\Cref{lem:small-grad-diff-genli}}{Lemma 5}}
\label{app:proof-lem-small-grad-diff-genli}

% We employ the same notations $\Tilde{Q}_t$, $\Tilde{q}_t$ and $m_t$ for all $t \geq 1$ as in \eqref{eq:thm1_abbrv_notation}. 
Note that $\Tilde{q}_{t|0}(x|x_0) = q_{t|0}(x|x_0)$ is the p.d.f. of $\calN(\sqrt{\Bar{\alpha}_t} x_0,(1-\Bar{\alpha}_t) I_d)$.
Thus, the gradient of $\log \Tilde{q}_t(x)$ equals
\begin{equation} \label{eq:log_qt_gradient}
    \nabla \log \Tilde{q}_t(x) = \frac{\int_{x_0 \in \mbR^d} \nabla \Tilde{q}_{t|0}(x|x_0) \d \Tilde{Q}_0(x_0)}{\Tilde{q}_t(x)} = -\frac{1}{1-\Bar{\alpha}_t} \int_{x_0 \in \mbR^d} (x-\sqrt{\Bar{\alpha}_t} x_0) \d \Tilde{Q}_{0|t}(x_0|x).
\end{equation}
Thus,
\begin{align} \label{eq:first-order-small-non-smooth-main}
    &\nabla \log \Tilde{q}_{t-1}(m_t) - \sqrt{\alpha_t} \nabla \log \Tilde{q}_t(x_t) \nonumber\\
    &= -\frac{1}{1-\Bar{\alpha}_{t-1}} \int_{x_0 \in \mbR^d} (m_t-\sqrt{\Bar{\alpha}_{t-1}} x_0) \d \Tilde{Q}_{0|t-1}(x_0|m_t) + \frac{\sqrt{\alpha_t}}{1-\Bar{\alpha}_t} \int_{x_0 \in \mbR^d} (x_t-\sqrt{\Bar{\alpha}_t} x_0) \d \Tilde{Q}_{0|t}(x_0 | x_t) \nonumber\\
    &= - \frac{1}{1-\Bar{\alpha}_t} \Bigg( \brc{\frac{1-\Bar{\alpha}_t}{1-\Bar{\alpha}_{t-1}} - 1} \int_{x_0 \in \mbR^d} (m_t-\sqrt{\Bar{\alpha}_{t-1}} x_0) \d \Tilde{Q}_{0|t-1}(x_0|m_t) \nonumber \\
    &\qquad + \int_{x_0 \in \mbR^d} (m_t-\sqrt{\Bar{\alpha}_{t-1}} x_0) \d \Tilde{Q}_{0|t-1}(x_0|m_t) - \sqrt{\alpha_t} \int_{x_0 \in \mbR^d} (x_t-\sqrt{\Bar{\alpha}_t} x_0) \d \Tilde{Q}_{0|t}(x_0 | x_t) \Bigg) \nonumber\\
    &\stackrel{(i)}{=} \frac{1}{1-\Bar{\alpha}_t} \Big( \brc{1-\Bar{\alpha}_t - (1-\Bar{\alpha}_{t-1})} \nabla \log \Tilde{q}_{t-1}(m_t) \Big) \nonumber \\
    &\qquad - \frac{1}{1-\Bar{\alpha}_t} \Bigg( \int_{x_0 \in \mbR^d} (m_t-\sqrt{\Bar{\alpha}_{t-1}} x_0) \d \Tilde{Q}_{0|t-1}(x_0|m_t) - \sqrt{\alpha_t} \int_{x_0 \in \mbR^d} (x_t-\sqrt{\Bar{\alpha}_t} x_0) \d \Tilde{Q}_{0|t}(x_0 | x_t) \Bigg) \nonumber\\
    &= \underbrace{\frac{\Bar{\alpha}_{t-1} (1-\alpha_t)}{1-\Bar{\alpha}_t} \nabla \log \Tilde{q}_{t-1}(m_t)}_{\text{term 1}} - \underbrace{\frac{1}{1-\Bar{\alpha}_t} (m_t - \sqrt{\alpha_t} x_t)}_{\text{term 2}} + \underbrace{\frac{\sqrt{\Bar{\alpha}_{t-1}} (1-\alpha_t) }{1-\Bar{\alpha}_t} \int_{x_0 \in \mbR^d} x_0 \d \Tilde{Q}_{0|t}(x_0 | x_t)}_{\text{term 3}} \nonumber \\
    &\qquad + \underbrace{\frac{\sqrt{\Bar{\alpha}_{t-1}}}{1-\Bar{\alpha}_t} \brc{\int_{x_0 \in \mbR^d} x_0 \d \Tilde{Q}_{0|t-1}(x_0|m_t) - \int_{x_0 \in \mbR^d} x_0 \d \Tilde{Q}_{0|t}(x_0|x_t)}}_{\text{term 4}}
\end{align}
where $(i)$ follows from Tweedie's formula. Among the four terms in \eqref{eq:first-order-small-non-smooth-main}, the first term satisfies that
\[ \E_{X_t \sim \Tilde{Q}_t} \norm{\frac{\Bar{\alpha}_{t-1} (1-\alpha_t)}{1-\Bar{\alpha}_t} \nabla \log \Tilde{q}_{t-1}(m_t)}^2 \lesssim \frac{d (1-\alpha_t)^2}{(1-\Bar{\alpha}_t)^2 (1-\Bar{\alpha}_{t-1})} \]
by \cite[Lemma~17]{liang2024discrete}. In the second term in \eqref{eq:first-order-small-non-smooth-main}, by Tweedie's formula,
\begin{align*}
    m_t - \sqrt{\alpha_t} x_t &= \frac{x_t}{\sqrt{\alpha_t}} + \frac{1-\alpha_t}{\sqrt{\alpha_t}} \nabla \log \Tilde{q}_t(x_t) - \sqrt{\alpha_t} x_t\\
    &= \frac{1-\alpha_t}{\sqrt{\alpha_t}} (x_t + \nabla \log \Tilde{q}_t(x_t)).
\end{align*}
Thus, by \cite[Lemma~15]{liang2024discrete} and \Cref{ass:m2-general}, the second term satisfies that
\[ \E_{X_t \sim \Tilde{Q}_t} \norm{\frac{1}{1-\Bar{\alpha}_t} (m_t - \sqrt{\alpha_t} x_t)}^2 \lesssim \frac{d (1-\alpha_t)^2}{(1-\Bar{\alpha}_t)^3}. \]
The third term in \eqref{eq:first-order-small-non-smooth-main} satisfies that
\[ \E_{X_t \sim \Tilde{Q}_t} \norm{ \frac{\sqrt{\Bar{\alpha}_{t-1}} (1-\alpha_t) }{1-\Bar{\alpha}_t} \int_{x_0 \in \mbR^d} x_0 \d \Tilde{Q}_{0|t}(x_0 | x_t) }^2 \lesssim \frac{d (1-\alpha_t)^2}{(1-\Bar{\alpha}_t)^2} \]
by Jensen's inequality and \Cref{ass:m2-general}.

To deal with the last term in \eqref{eq:first-order-small-non-smooth-main}, note that
\begin{align*}
    \d \Tilde{Q}_{0|t-1}(x_0|m_t) &= \frac{\Tilde{q}_{t-1|0}(m_t|x_0)}{\Tilde{q}_{t-1}(m_t)} \d \Tilde{Q}_0(x_0) = \frac{\Tilde{q}_{t-1|0}(m_t|x_0)}{\int_{y \in \mbR^d} \Tilde{q}_{t-1|0}(m_t|y) \d \Tilde{Q}_0(y)} \d \Tilde{Q}_0(x_0), \\
    \d \Tilde{Q}_{0|t}(x_0|x_t) &= \frac{\Tilde{q}_{t|0}(x_t|x_0)}{\Tilde{q}_{t}(x_t)} \d \Tilde{Q}_0(x_0) = \frac{\Tilde{q}_{t|0}(x_t|x_0)}{\int_{y \in \mbR^d} \Tilde{q}_{t|0}(x_t|y) \d \Tilde{Q}_0(y)} \d \Tilde{Q}_0(x_0).
\end{align*}
Thus, the last term in \eqref{eq:first-order-small-non-smooth-main} is equal to
\begin{align*}
    &\frac{\sqrt{\Bar{\alpha}_{t-1}}}{1-\Bar{\alpha}_t} \brc{\int_{x_0 \in \mbR^d} x_0 \d \Tilde{Q}_{0|t-1}(x_0|m_t) - \int_{x_0 \in \mbR^d} x_0 \d \Tilde{Q}_{0|t}(x_0|x_t)} \\
    &= \frac{\sqrt{\Bar{\alpha}_{t-1}}}{1-\Bar{\alpha}_t} \cdot \frac{1}{\Tilde{q}_{t-1}(m_t) \Tilde{q}_{t}(x_t)} \brc{\int_{x,y \in \mbR^d} x (\Tilde{q}_{t-1|0}(m_t|x) \Tilde{q}_{t|0}(x_t|y) - \Tilde{q}_{t|0}(x_t|x) \Tilde{q}_{t-1|0}(m_t|y)) \d \Tilde{Q}_{0}(x) \d \Tilde{Q}_{0}(y) }
\end{align*}
where
\begin{align*}
    &\Tilde{q}_{t-1|0}(m_t|x) \Tilde{q}_{t|0}(x_t|y) - \Tilde{q}_{t|0}(x_t|x) \Tilde{q}_{t-1|0}(m_t|y)\\
    &= \Tilde{q}_{t|0}(x_t|x) \Tilde{q}_{t-1|0}(m_t|y) \brc{ \frac{\Tilde{q}_{t-1|0}(m_t|x) \Tilde{q}_{t|0}(x_t|y)}{\Tilde{q}_{t|0}(x_t|x) \Tilde{q}_{t-1|0}(m_t|y)}  - 1 }\\
    &= \Tilde{q}_{t|0}(x_t|x) \Tilde{q}_{t-1|0}(m_t|y) \times \\
    &\quad \brc{ \exp\brc{-\frac{\norm{m_t - \sqrt{\Bar{\alpha}_{t-1}} x}^2}{2(1-\Bar{\alpha}_{t-1})} - \frac{\norm{x_t - \sqrt{\Bar{\alpha}_{t}} y}^2}{2(1-\Bar{\alpha}_t)} + \frac{\norm{m_t - \sqrt{\Bar{\alpha}_{t-1}} y}^2}{2(1-\Bar{\alpha}_{t-1})} + \frac{\norm{x_t - \sqrt{\Bar{\alpha}_{t}} x}^2}{2(1-\Bar{\alpha}_t)} } - 1 }\\
    &=: \Tilde{q}_{t|0}(x_t|x) \Tilde{q}_{t-1|0}(m_t|y) (e^\Delta - 1)
\end{align*}
in which we have defined the exponent as $\Delta$. Now,
\begin{align*}
    \Delta &= -\frac{\norm{m_t - \sqrt{\Bar{\alpha}_{t-1}} x}^2}{2(1-\Bar{\alpha}_{t-1})} - \frac{\norm{x_t - \sqrt{\Bar{\alpha}_{t}} y}^2}{2(1-\Bar{\alpha}_t)} + \frac{\norm{m_t - \sqrt{\Bar{\alpha}_{t-1}} y}^2}{2(1-\Bar{\alpha}_{t-1})} + \frac{\norm{x_t - \sqrt{\Bar{\alpha}_{t}} x}^2}{2(1-\Bar{\alpha}_t)} \\
    &= \frac{\sqrt{\Bar{\alpha}_{t-1}} (x-y)^\T m_t + \Bar{\alpha}_{t-1} \norm{y}^2 - \Bar{\alpha}_{t-1} \norm{x}^2}{2(1-\Bar{\alpha}_{t-1})} - \frac{\sqrt{\Bar{\alpha}_t} (x-y)^\T x_t + \Bar{\alpha}_t \norm{y}^2 - \Bar{\alpha}_t \norm{x}^2}{2(1-\Bar{\alpha}_t)} \\
    &= \frac{1}{2}\brc{\frac{\sqrt{\Bar{\alpha}_t} (1-\alpha_t)}{\alpha_t (1 - \Bar{\alpha}_{t-1})(1 - \Bar{\alpha}_t)} x_t + \frac{\sqrt{\Bar{\alpha}_{t-1}}(1-\alpha_t)}{\sqrt{\alpha_t}(1 - \Bar{\alpha}_{t-1})} \nabla \log \Tilde{q}_t(x_t)}^\T (x-y)\\
    &\qquad + \frac{\Bar{\alpha}_{t-1} (1-\alpha_t)}{(1 - \Bar{\alpha}_{t-1})(1 - \Bar{\alpha}_t)}(\norm{y}^2 - \norm{x}^2).
\end{align*}
Now, with the $\alpha_t$ defined in \eqref{eq:alpha_genli}, following from \Cref{lem:alpha_genli_rate},
\begin{align*}
    \frac{\sqrt{\Bar{\alpha}_t} (1-\alpha_t)}{\alpha_t (1 - \Bar{\alpha}_{t-1})(1 - \Bar{\alpha}_t)} &= O\brc{\frac{1-\alpha_t}{(1-\Bar{\alpha}_{t-1})^2}} = O\brc{\frac{\log T}{T}},\\
    \frac{1-\alpha_t}{1 - \Bar{\alpha}_{t-1}} &= O\brc{\frac{\log T}{T}},\\
    \frac{\Bar{\alpha}_{t-1} (1-\alpha_t)}{(1 - \Bar{\alpha}_{t-1})(1 - \Bar{\alpha}_t)} &= O\brc{\frac{1-\alpha_t}{(1-\Bar{\alpha}_{t-1})^2}} = O\brc{\frac{\log T}{T}}.
\end{align*}
Thus, for fixed $x,y,x_t$, $\Delta \to 0$ as $T \to \infty$, and thus when $T$ becomes large,
\[ e^{\Delta} - 1 = \Delta + O(\Delta^2) \lesssim \abs{\Delta},\quad \forall x_t \in \mbR^d. \]
Also, since $\Tilde{q}_{t|0}(x_t|x)$ and $\Tilde{q}_{t-1|0}(m_t|y)$ decay exponentially in terms of $x$ and $y$ (for any fixed $x_t$), we have
\begin{align*}
    \int \Tilde{q}_{t|0}(x_t|x) \mathrm{poly}(x) \d \Tilde{Q}_{0}(x) &< \infty,\\
    \int \Tilde{q}_{t-1|0}(m_t|y) \mathrm{poly}(y) \d \Tilde{Q}_{0}(y) &< \infty.
\end{align*}
Thus, the limit and the integral can be exchanged due to Dominated Convergence Theorem. Thus, the fourth term in \eqref{eq:first-order-small-non-smooth-main} gives us
\begin{align*}
    &\frac{\sqrt{\Bar{\alpha}_{t-1}}}{1-\Bar{\alpha}_t} \brc{\int_{x_0 \in \mbR^d} x_0 \d \Tilde{Q}_{0|t-1}(x_0|m_t) - \int_{x_0 \in \mbR^d} x_0 \d \Tilde{Q}_{0|t}(x_0|x_t)} \\
    &\lesssim \frac{\sqrt{\Bar{\alpha}_{t-1}}}{1-\Bar{\alpha}_t} \cdot \frac{1}{\Tilde{q}_{t-1}(m_t) \Tilde{q}_{t}(x_t)} \brc{\int_{x,y \in \mbR^d} x \Tilde{q}_{t|0}(x_t|x) \Tilde{q}_{t-1|0}(m_t|y) \abs{\Delta} \d \Tilde{Q}_{0}(x) \d \Tilde{Q}_{0}(y) }\\
    &= \frac{\sqrt{\Bar{\alpha}_{t-1}}}{1-\Bar{\alpha}_t} \brc{\int_{x,y \in \mbR^d} (x \cdot \abs{\Delta}) \d \Tilde{Q}_{0|t}(x|x_t) \d \Tilde{Q}_{0|t-1}(y|m_t) }
\end{align*}
and, from definition of $\Delta$ and using Cauchy-Schwartz and Jensen's inequality, we have
\begin{align*}
    &\E_{X_t \sim \Tilde{Q}_t} \norm{ \frac{\sqrt{\Bar{\alpha}_{t-1}}}{1-\Bar{\alpha}_t} \int_{x,y \in \mbR^d} (x \cdot \abs{\Delta}) \d \Tilde{Q}_{0|t}(x|X_t) \d \Tilde{Q}_{0|t-1}(y|m_t(X_t)) }^2 \\
    &\lesssim \frac{(1-\alpha_t)^2}{(1 - \Bar{\alpha}_{t-1})^2(1 - \Bar{\alpha}_t)^4} \cdot \\
    &\qquad \E_{\substack{X_t \sim \Tilde{Q}_t \\ X \sim \Tilde{Q}_{0|t}(\cdot|X_t) \\ Y \sim \Tilde{Q}_{0|t-1}(\cdot|m_t(X_t))}} \bigg[ \norm{\sqrt{\Bar{\alpha}_t} X}^2 \bigg( (\norm{X_t}^2 + (1-\Bar{\alpha}_t)^2 \norm{\nabla \log \Tilde{q}_t(X_t)}^2) \\
    &\qquad (\norm{\sqrt{\Bar{\alpha}_t} X}^2 + \norm{\sqrt{\Bar{\alpha}_{t-1}} Y}^2) + (\norm{\sqrt{\Bar{\alpha}_t} X}^4 + \norm{\sqrt{\Bar{\alpha}_{t-1}} Y}^4) \bigg) \bigg]\\
    &= \frac{(1-\alpha_t)^2}{(1 - \Bar{\alpha}_{t-1})^2(1 - \Bar{\alpha}_t)^4} \cdot\\
    &\qquad \E_{\substack{X_t \sim \Tilde{Q}_t \\ X \sim \Tilde{Q}_{0|t}(\cdot|X_t) \\ Y \sim \Tilde{Q}_{0|t-1}(\cdot|m_t(X_t))}} \bigg[ \norm{\sqrt{\Bar{\alpha}_t} X}^4 (\norm{X_t}^2 + (1-\Bar{\alpha}_t)^2 \norm{\nabla \log \Tilde{q}_t(X_t)}^2) \\
    &\qquad + \norm{\sqrt{\Bar{\alpha}_t} X}^2 \norm{\sqrt{\Bar{\alpha}_{t-1}} Y}^2 (\norm{X_t}^2 + (1-\Bar{\alpha}_t)^2 \norm{\nabla \log \Tilde{q}_t(X_t)}^2)\\
    &\qquad + \norm{\sqrt{\Bar{\alpha}_t} X}^6 + \norm{\sqrt{\Bar{\alpha}_t} X}^2 \norm{\sqrt{\Bar{\alpha}_{t-1}} Y}^4 \bigg]\\
    &\leq \frac{(1-\alpha_t)^2}{(1 - \Bar{\alpha}_{t-1})^2(1 - \Bar{\alpha}_t)^4} \cdot\\
    &\qquad \brc{ \E_{X \sim \Tilde{Q}_0} \norm{\sqrt{\Bar{\alpha}_t} X}^6 }^{2/3} \brc{ \E_{X_t \sim \Tilde{Q}_t} (\norm{X_t}^6 + (1-\Bar{\alpha}_t)^6 \norm{\nabla \log \Tilde{q}_t(X_t)}^6) }^{1/3} \\
    &\qquad + \brc{\E_{X \sim \Tilde{Q}_0} \norm{\sqrt{\Bar{\alpha}_t} X}^6}^{1/3} \brc{ \E_{\substack{X_t \sim \Tilde{Q}_t \\ Y \sim \Tilde{Q}_{0|t-1}(\cdot|m_t(X_t))}} \norm{\sqrt{\Bar{\alpha}_{t-1}} Y}^6 }^{1/3} \\
    &\qquad \quad \brc{ \E_{X_t \sim \Tilde{Q}_t} (\norm{X_t}^6 + (1-\Bar{\alpha}_t)^6 \norm{\nabla \log \Tilde{q}_t(X_t)}^6) }^{1/3}\\
    &\qquad + \E_{X \sim \Tilde{Q}_0} \norm{\sqrt{\Bar{\alpha}_t} X}^6 + \brc{\E_{X \sim \Tilde{Q}_0} \norm{\sqrt{\Bar{\alpha}_t} X}^6}^{1/3} \brc{\E_{\substack{X_t \sim \Tilde{Q}_t \\ Y \sim \Tilde{Q}_{0|t-1}(\cdot|m_t(X_t))}} \norm{\sqrt{\Bar{\alpha}_{t-1}} Y}^6}^{2/3} \\
    &\stackrel{(ii)}{\lesssim} \frac{d^3 (1-\alpha_t)^2}{(1 - \Bar{\alpha}_{t-1})^2(1 - \Bar{\alpha}_t)^4},
\end{align*}
where $(ii)$ follows because, following \cite[Lemmas~15--17]{liang2024discrete} and by the lemma assumption that $\E_{X_0 \sim \Tilde{Q}_0} \norm{X_0}^6 \lesssim d^3$, we have
\begin{align*}
    &\E_{X \sim \Tilde{Q}_0} \norm{\sqrt{\Bar{\alpha}_t} X}^6 \lesssim d^3,\\
    &\E_{X_t \sim \Tilde{Q}_t} \norm{X_t}^6 \leq \E_{X_0 \sim \Tilde{Q}_0} \norm{\sqrt{\Bar{\alpha}_t} X_0}^6 + (1-\Bar{\alpha}_t)^3 \E_{\Bar{W} \sim \calN(0,I_d)} \norm{\Bar{W}}^6 \lesssim d^3,\\
    &\E_{X_t \sim \Tilde{Q}_t} \norm{\nabla \log \Tilde{q}_t(X_t)}^6 \lesssim \frac{d^3}{(1-\Bar{\alpha}_t)^3},\\
    &\E_{\substack{X_t \sim \Tilde{Q}_t \\ Y \sim \Tilde{Q}_{0|t-1}(\cdot|m_t(X_t))}} \norm{\sqrt{\Bar{\alpha}_{t-1}} Y}^6\\
    &\quad \leq \E_{\substack{X_t \sim \Tilde{Q}_t \\ Y \sim \Tilde{Q}_{0|t-1}(\cdot|m_t(X_t))}} \norm{m_t - \sqrt{\Bar{\alpha}_{t-1}} Y}^6 + \E_{X_t \sim \Tilde{Q}_t} \norm{m_t}^6 \lesssim d^3.
\end{align*}

Hence, combining the rates of all parts, we obtain that
\[ (1-\alpha_t) \sqrt{ \E_{X_t \sim \Tilde{Q}_t} \norm{\nabla \log \Tilde{q}_{t-1}(m_t(X_t)) - \sqrt{\alpha_t} \nabla \log \Tilde{q}_t(X_t)}^2 } \lesssim \frac{d^{3/2} (1-\alpha_t)^2}{(1-\Bar{\alpha}_{t-1})^3}. \]
% The proof is complete by noting that $1-\Bar{\alpha}_{t-1} \geq 1 - \alpha_1 = \delta$.

\subsection{\texorpdfstring{\Cref{lem:alpha_genli_rate}}{Lemma 6} and its proof}

\begin{lemma} \label{lem:alpha_genli_rate}
The $\alpha_t$ defined in \eqref{eq:alpha_genli} (with $c > 1$) satisfy
\[ \frac{1 - \alpha_t}{(1 - \Bar{\alpha}_{t-1})^p} \lesssim \frac{\log T \log(1/\delta)}{\delta^{p-1} T}~~\text{while}~~ \Bar{\alpha}_T = o(T^{-1}),\quad \forall 2 \leq t \leq T,~p \geq 1. \]
\end{lemma}

\begin{proof}
The proof is similar to that of \cite[Eq (39)]{li2023faster}. We first prove the second relationship. First, note that if $T$ is large,
\[ \delta \brc{1 + \frac{c \log T}{T}}^{\frac{T}{\log T}} \asymp \delta e^c > 1. \]
Thus, with any fixed $r \in (0,1)$ such that $t \geq r T~(\geq \frac{T}{\log T})$, we have
\[ 1-\alpha_t = \frac{c \log T}{T} \min\cbrc{\delta \brc{1 + \frac{c \log T}{T}}^t, 1} = \frac{c \log T}{T}.\]
As a result,
\begin{equation} \label{eq:alpha_genli_rate_alphabar}
    \Bar{\alpha}_T \leq \prod_{t= \floor{r T}}^T \alpha_t = \brc{1 - \frac{c \log T}{T}}^{\ceil{(1-r) T}} \asymp \exp \brc{\ceil{(1-r) T} \brc{- \frac{c \log T}{T}}} = O( T^{-(1-r)c} ). 
\end{equation}
Given any $c > 1$, we can always find some $r$ such that $(1-r)c > 1$ (say, $r = (c-1)/2$ if $c \in (1,2)$ and $r = 1/4$ if $c \geq 2$). This shows that $\alpha_t$ satisfies $\Bar{\alpha}_T = o\brc{T^{-1}}$ if $c > 1$.

Now, for the first relatinoship, define $\tau$ such that
\begin{equation} \label{eq:alpha_proof_tau_def}
    \delta \brc{1 + \frac{c \log T}{T}}^\tau \leq 1 < \delta \brc{1 + \frac{c \log T}{T}}^{\tau+1}.
\end{equation}
Here $\tau$ is unique since $1-\alpha_t$ is non-decreasing. In other words, $\tau$ is the last time that $1-\alpha_t$ is exponentially growing.
Assume that $T$ is large enough such that $\tau \geq 2$. Below, we show that
\begin{equation} \label{eq:alpha_proof_denom_lower_bound}
    1 - \Bar{\alpha}_{t-1} \geq \frac{1}{3} \delta \brc{1 + \frac{c \log T}{T}}^t,\quad \forall 2 \leq t \leq \tau.
\end{equation}
If $t = 2$,
\[1 - \Bar{\alpha}_{t-1} = 1 - \Bar{\alpha}_1 = 1 - \alpha_1 = \delta \geq \frac{1}{3} \delta \brc{1 + \frac{c \log T}{T}}. \]
% \yuchen{This is wrong. Indeed, $1-\alpha_1$ cannot shrink too slowly...}
Here the last inequality holds when $T$ is sufficiently large. For $t > 2$, suppose for purpose of contradiction that there exists $2 < t_0 \leq \tau$ such that
\[ 1 - \Bar{\alpha}_{t_0-1} < \frac{1}{3} \delta \brc{1 + \frac{c \log T}{T}}^{t_0}~\text{while}~ 1 - \Bar{\alpha}_{t-1} \geq \frac{1}{3} \delta \brc{1 + \frac{c \log T}{T}}^{t},~\forall 2 \leq t \leq t_0-1. \]
In words, $t_0$ is defined as the \textit{first} time that \eqref{eq:alpha_proof_denom_lower_bound} is violated. To arrive at a contradiction, we first write
\begin{align*}
    1 - \Bar{\alpha}_{t_0-1} &= (1 - \Bar{\alpha}_{t_0-2}) \brc{1 + \frac{\Bar{\alpha}_{t_0-2} (1 - \alpha_{t_0-1})}{1 - \Bar{\alpha}_{t_0-2}}} \\
    &\geq \frac{1}{3} \delta \brc{1 + \frac{c \log T}{T}}^{t_0-1} \brc{1 + \frac{\Bar{\alpha}_{t_0-2} (1 - \alpha_{t_0-1})}{1 - \Bar{\alpha}_{t_0-2}}}.
\end{align*}
Here the inequality holds because $t_0$ is the first time that \eqref{eq:alpha_proof_denom_lower_bound} is violated, and thus \eqref{eq:alpha_proof_denom_lower_bound} stills holds for $t=t_0-1$. Also,
\[ 1 - \Bar{\alpha}_{t_0-2} \leq 1 - \Bar{\alpha}_{t_0-1} \stackrel{(i)}{<} \frac{1}{3} \delta \brc{1 + \frac{c \log T}{T}}^{t_0} \stackrel{(ii)}{\leq} \frac{1}{2} \delta \brc{1 + \frac{c \log T}{T}}^{t_0-1} \stackrel{(iii)}{\leq} \frac{1}{2} \]
where $(i)$ holds because \eqref{eq:alpha_proof_denom_lower_bound} is violated at $t=t_0$, $(ii)$ holds when $T$ is sufficiently large, and $(iii)$ holds because $t_0 -1 \leq \tau$ and by the definition of $\tau$ in \eqref{eq:alpha_proof_tau_def}. Thus,
\[ \frac{\Bar{\alpha}_{t_0-2} (1 - \alpha_{t_0-1})}{1 - \Bar{\alpha}_{t_0-2}} \geq \frac{\frac{1}{2} \frac{c \log T}{T} \delta \brc{1 + \frac{c \log T}{T}}^{t_0-1} }{\frac{1}{2} \delta \brc{1 + \frac{c \log T}{T}}^{t_0-1}} = \frac{c \log T}{T}, \]
and thus
\[ 1 - \Bar{\alpha}_{t_0-1} \geq \frac{1}{3} \delta \brc{1 + \frac{c \log T}{T}}^{t_0-1} \brc{1 + \frac{\Bar{\alpha}_{t_0-2} (1 - \alpha_{t_0-1})}{1 - \Bar{\alpha}_{t_0-2}}} \geq \frac{1}{3} \delta \brc{1 + \frac{c \log T}{T}}^{t_0}. \]
We have reached a contradiction. Therefore, we have shown that \eqref{eq:alpha_proof_denom_lower_bound} holds.

Now, \eqref{eq:alpha_proof_denom_lower_bound} implies that
\[ 1 - \Bar{\alpha}_{t-1} \geq \frac{1}{3} \delta \brc{1 + \frac{c \log T}{T}}^{t} \geq \frac{1}{3} \delta \brc{1 + \frac{c \log T}{T}}^{t/p},\quad \forall 2 \leq t \leq \tau. \]
There are two cases:
\begin{itemize}
    \item If $2 \leq t \leq \tau$, then
    \[ \frac{1 - \alpha_t}{(1 - \Bar{\alpha}_{t-1})^p} \leq \frac{\frac{c \log T}{T} \delta \brc{1 + \frac{c \log T}{T}}^t}{\frac{1}{3^p} \delta^p \brc{1 + \frac{c \log T}{T}}^t} = \frac{3^p c \log T}{\delta^{p-1} T}. \]
    \item If $t > \tau$, then
    \begin{align*}
        \frac{1 - \alpha_t}{(1 - \Bar{\alpha}_{t-1})^p} \leq \frac{1 - \alpha_t}{(1 - \Bar{\alpha}_{\tau-1})^p} \leq \frac{\frac{c \log T}{T}}{\frac{1}{3^p} \delta^p \brc{1 + \frac{c \log T}{T}}^{\tau}} &= \frac{\frac{c \log T}{T} \brc{1 + \frac{c \log T}{T}}}{3^{-p} \delta^{p-1} \brc{1 + \frac{c \log T}{T}}^{\tau+1}} \\
        &< \frac{3^p c \log T}{\delta^{p-1} T} \brc{1 + \frac{c \log T}{T}}.
    \end{align*}
\end{itemize}
In both cases, if $T$ is large enough, noting that $c \gtrsim \log(1/\delta)$, we have
\[ \frac{1 - \alpha_t}{(1 - \Bar{\alpha}_{t-1})^p} \leq \frac{4^p c \log T}{\delta^{p-1} T} \lesssim \frac{\log T \log(1/\delta)}{\delta^{p-1} T},\quad \forall 2 \leq t \leq T \]
because $p$ and $c$ are constants (that do not depend on $T$, $d$, and $\delta$). The proof is now complete.
\end{proof}

\subsection{\texorpdfstring{\Cref{lem:nonvanish_coef_sum_genli}}{Lemma 7} and its proof}

\begin{lemma} \label{lem:nonvanish_coef_sum_genli}
With the $\alpha_t$ defined in \eqref{eq:alpha_genli}, given any $p > 0$, if $\delta p < 1$,
\[ \sum_{t=2}^T (1-\alpha_t) \Bar{\alpha}_t^p \leq \brc{ \frac{1}{p} (1-\delta)^p e^{-p \delta \log(1/\delta)} + (1-\delta)^p \frac{e^{-p \delta \log(1/\delta)} - 1}{1 - \delta p} } \brc{1 + O\brc{\frac{\log T}{T}} }. \]
Further, when $\delta \ll 1$,
\[ \sum_{t=2}^T (1-\alpha_t) \Bar{\alpha}_t^p \leq \frac{1}{p} - \brc{ 1 + \frac{p+1}{2 p} } \frac{c \log T}{T} + \Tilde{O}\brc{\frac{1}{T^2}}. \]
\end{lemma}

\begin{proof}
Define the sum as $s_T$. Recall that
\[ 1-\alpha_1 = \delta,\quad 1-\alpha_t = \frac{c \log T}{T} \min\cbrc{\delta \brc{1 + \frac{c \log T}{T}}^t, 1},~\forall 2 \leq t \leq T. \]
We first note a relationship that for fixed $\delta \neq 0$ and $p > 0$. As $z \to \infty$,
\begin{equation} \label{eq:nonvanish_coef_sum_genli_exp_rate}
    (1-\delta z^{-1})^{p z} = e^{p z \log(1-\delta z^{-1}) } = e^{p z (-\delta z^{-1} + \delta^2 z^{-2} / 2 + O(z^{-3})) } = e^{-\delta p} (1 + \delta^2 p z^{-1}/2) + O(z^{-2}).
\end{equation}
We also use the fact from binomial series that
\begin{equation} \label{eq:nonvanish_coef_sum_genli_binom}
    (1-z^{-1})^p = 1 - p z^{-1} + \frac{p(p-1)}{2} z^{-2} + O(z^{-3}).
\end{equation}

Define $t^* := \sup\cbrc{ t \in [1,T]: \delta \brc{1 + \frac{c \log T}{T}}^t \leq 1}$. Thus, $\alpha_t \equiv 1 - \frac{c \log T}{T}$ for all $t > t^*$. Note that when $T$ becomes large, $t^* = \Theta\brc{\frac{T}{\log T}}$.
To further understand the big-$\Theta$ term, note that using \eqref{eq:nonvanish_coef_sum_genli_exp_rate},
% \[ \delta \brc{1 + \frac{c \log T}{T}}^{\frac{T \log(1/\delta)}{c \log T}} = 1 + \frac{\log(1/\delta) c \log T}{2 T} + \Tilde{O}\brc{\frac{1}{T^2}} > 1 \]
% while
\begin{align*}
    & \delta \brc{1 + \frac{c \log T}{T}}^{\frac{T \log(1/\delta)}{c \log T} - \log(1/\delta)} \\
    &= \brc{ 1 + \frac{\log(1/\delta) c \log T}{2 T} + \Tilde{O}\brc{\frac{1}{T^2}} } \brc{ 1 - \frac{\log(1/\delta) c \log T}{T} + \Tilde{O}\brc{\frac{1}{T^2}} } \\
    &= 1 - \frac{\log(1/\delta) c \log T}{2 T} + \Tilde{O}\brc{\frac{1}{T^2}} \\
    &< 1 ~~ \text{as $T \to \infty$}.
\end{align*}
This implies that
\begin{equation} \label{eq:nonvanish_coef_sum_genli_tstar_bound}
    t^* \geq \log(1/\delta) \brc{\frac{T}{c \log T} - 1}.
\end{equation}

To start, we suppose $T > t^*$ is large enough and decompose the sum as
\begin{align} \label{eq:nonvanish_coef_sum_genli_main}
    s_T &= \sum_{t=2}^{t^*} (1-\alpha_t) \Bar{\alpha}_t^p + \sum_{t=t^*+1}^{T} (1-\alpha_t) \Bar{\alpha}_t^p \nonumber\\
    &= \frac{c \log T}{T} \delta (1-\delta)^p \sum_{t=2}^{t^*} \brc{1 + \frac{c \log T}{T}}^t \prod_{i=2}^t \brc{1 - \delta \frac{c \log T}{T} \brc{1+\frac{c \log T}{T}}^i}^p \nonumber\\
    &\qquad + \Bar{\alpha}_{t^*}^p \frac{c \log T}{T} \sum_{t=t^*+1}^T \brc{1 - \frac{c \log T}{T}}^{p (t - t^*)}.
\end{align}
Now we first focus on the second term in \eqref{eq:nonvanish_coef_sum_genli_main}.
\begin{align*}
    &\frac{c \log T}{T} \sum_{t=t^*+1}^T \brc{1 - \frac{c \log T}{T}}^{p (t - t^*)} \\
    &= \brc{1 - \frac{c \log T}{T}}^p \frac{c \log T}{T} \cdot \frac{1 - \brc{1 - \frac{c \log T}{T}}^{p (T-t^*)}}{1 - \brc{1 - \frac{c \log T}{T}}^p}\\
    &\stackrel{(i)}{=} \brc{1 - \frac{c \log T}{T}}^p \frac{c \log T}{T} \cdot \frac{1 - \brc{1 - \frac{c \log T}{T}}^{p (T-t^*)}}{1 - \brc{1 - \frac{p c \log T}{T} + p (p-1) \frac{c^2 (\log T)^2}{2 T^2} + \Tilde{O}\brc{\frac{1}{T^3}} }}\\
    &\stackrel{(ii)}{=} \frac{1}{p} \brc{1 - \frac{p c \log T}{T} + \Tilde{O}\brc{\frac{1}{T^2}} } \brc{1 + \frac{(p-1) c \log T}{2 T} + \Tilde{O}\brc{\frac{1}{T^2}} - O\brc{\frac{1}{T^{p c / 2}}}}\\
    &= \frac{1}{p} \brc{1 - \frac{(p+1) c \log T}{2 T}} + \Tilde{O}\brc{\frac{1}{T^2}}
\end{align*}
where $(i)$ follows from \eqref{eq:nonvanish_coef_sum_genli_binom}, and $(ii)$ is because $t^* = \Theta(T / \log T)$ and thus $T - t^* > T / 2$ for large $T$.
Also, for all $t = 2,\dots,t^*$,
\begin{align*}
    \Bar{\alpha}_{t}^p &= (1-\delta)^p \prod_{i=2}^t \brc{1 - \delta \frac{c \log T}{T} \brc{1+\frac{c \log T}{T}}^i}^p\\
    &\leq (1-\delta)^p \brc{1 - \delta \frac{c \log T}{T}}^{p (t-1)}
\end{align*}
% and $\Bar{\alpha}_{t}^p$
Thus, we have
\begin{align*}
    \Bar{\alpha}_{t^*}^p &\leq (1-\delta)^p \brc{1 - \delta \frac{c \log T}{T}}^{p (t^*-1)} \\
    &= (1-\delta)^p \brc{1 + \delta p \frac{c \log T}{T} + \Tilde{O}\brc{\frac{1}{T^2}}} \brc{1 - \delta \frac{c \log T}{T}}^{p t^*} \\
    % &\stackrel{(ii)}{\leq} (1-\delta)^p \brc{1 + \delta p \frac{c \log T}{T} + \Tilde{O}\brc{\frac{1}{T^2}}} \brc{1 - \delta \frac{c \log T}{T}}^{p\log(1/\delta) \brc{\frac{T}{c \log T} - 1}}\\
    &\stackrel{(iii)}{\leq} (1-\delta)^p \brc{1 + \delta p (1 + \log(1/\delta)) \frac{c \log T}{T}} \times \\
    &\qquad \quad e^{-p \delta \log(1/\delta)} \brc{ 1 + \delta^2 p \log(1/\delta) \frac{c \log T}{2 T} } + \Tilde{O}\brc{\frac{1}{T^2}} \\
    &= (1-\delta)^p e^{-p \delta \log(1/\delta)} \brc{1 + \delta p \brc{1 + \log(1/\delta) + (\delta / 2) \log(1/\delta) } \frac{c \log T}{T}} + \Tilde{O}\brc{\frac{1}{T^2}}.
\end{align*}
Here $(iii)$ follows because using \eqref{eq:nonvanish_coef_sum_genli_tstar_bound} and \eqref{eq:nonvanish_coef_sum_genli_exp_rate}, we have that
\begin{align} \label{eq:nonvanish_coef_sum_genli_delta_power_tstar}
    &\brc{1 - \delta \frac{c \log T}{T}}^{p t^*} \leq \brc{1 - \delta \frac{c \log T}{T}}^{p \log(1/\delta) \brc{\frac{T}{c \log T} - 1}} \nonumber \\
    &= \brc{1 + \delta p \log(1/\delta) \frac{c \log T}{T} } e^{-p \delta \log(1/\delta)} \brc{ 1 + \delta^2 p \log(1/\delta) \frac{c \log T}{2 T}} + \Tilde{O}\brc{\frac{1}{T^2}}.
\end{align}
Thus, the second term in \eqref{eq:nonvanish_coef_sum_genli_main} satisfies that
\begin{align} \label{eq:nonvanish_coef_sum_genli_term2}
    &\sum_{t=t^*+1}^{T} (1-\alpha_t) \Bar{\alpha}_t^p \nonumber \\
    &\leq \frac{1}{p} (1-\delta)^p e^{-p \delta \log(1/\delta)} \brc{1 - \brc{\frac{p+1}{2} - \delta p \brc{1 + \log(1/\delta) + (\delta / 2) \log(1/\delta) }}  \frac{c \log T}{T} } \nonumber\\
    &\qquad + \Tilde{O}\brc{\frac{1}{T^2}}.
    % \nonumber\\
    % &= \frac{1}{p} (1-\delta)^p e^{-p \delta \log(1/\delta)} \brc{1 + O\brc{\frac{\log T}{T}} }.
\end{align}

Now we turn to the first term in \eqref{eq:nonvanish_coef_sum_genli_main}, in which the summation can be upper-bounded as
\begin{align*}
    &\sum_{t=2}^{t^*} \brc{1 + \frac{c \log T}{T}}^t \prod_{i=2}^t \brc{1 - \delta \frac{c \log T}{T} \brc{1+\frac{c \log T}{T}}^i}^p\\
    &\leq \sum_{t=2}^{t^*} \brc{1 + \frac{c \log T}{T}}^t \brc{1 - \delta \frac{c \log T}{T}}^{p (t-1)} =: \brc{1 + \frac{c \log T}{T}} \sum_{t=1}^{t^*-1} q^t
    % &\qquad 
\end{align*}
where
\begin{align*}
    q &:= \brc{1 + \frac{c \log T}{T}} \brc{1 - \delta \frac{c \log T}{T}}^{p}\\
    &= \brc{1 + \frac{c \log T}{T}} \brc{1 - \delta p \frac{c \log T}{T} + \delta^2 p (p-1) \frac{c^2 (\log T)^2}{2 T^2} + \Tilde{O}\brc{\frac{1}{T^3}}}\\
    &= 1 + (1-\delta p) \frac{c \log T}{T} + \brc{\frac{\delta^2 p (p-1)}{2} - \delta p} \frac{c^2 (\log T)^2}{T^2} + \Tilde{O}\brc{\frac{1}{T^3}}.
\end{align*}
Note that by assumption $\delta p < 1$.
Also, by definition of $t^*$ and \eqref{eq:nonvanish_coef_sum_genli_delta_power_tstar},
\begin{align*}
    &\delta q^{t^*} \leq \brc{1 - \delta \frac{c \log T}{T}}^{p t^*}\\
    &\leq e^{-p \delta \log(1/\delta)} \brc{ 1 + \delta p \log(1/\delta) (1 + \delta/2) \frac{c \log T}{T}} + \Tilde{O}\brc{\frac{1}{T^2}}.
\end{align*}
Thus, we have
\begin{align*}
    &\delta \frac{c \log T}{T} \sum_{t=1}^{t^*-1} q^t = \frac{c \log T}{T} \times \frac{\delta q^{t^*} - \delta q}{q-1}\\
    & \leq \frac{c \log T}{T} \times \frac{e^{-p \delta \log(1/\delta)} \brc{ 1 + \delta p \log(1/\delta) (1 + \delta/2) \frac{c \log T}{T}} + \Tilde{O}\brc{\frac{1}{T^2}} - \delta q}{q-1} \\
    &\stackrel{(iv)}{=} \frac{c \log T}{T} \times \frac{e^{-p \delta \log(1/\delta)} \brc{1 + \delta p \log(1/\delta) (1 + \delta/2) \frac{c \log T}{T}} + \Tilde{O}\brc{\frac{1}{T^2}} - 1 - (1-\delta p) \frac{c \log T}{T} + \Tilde{O}\brc{\frac{1}{T^2}} }{(1-\delta p) \frac{c \log T}{T} +  \brc{\frac{\delta^2 p (p-1)}{2} - \delta p} \frac{c^2 (\log T)^2}{T^2} + \Tilde{O}\brc{\frac{1}{T^3}} } \\
    &= \brc{ \frac{e^{-p \delta \log(1/\delta)} - 1}{1 - \delta p} + \brc{\frac{\delta p \log(1/\delta) (1 + \delta/2)}{1-\delta p} - 1} \frac{c \log T}{T} } \times\\
    &\qquad \brc{1 + \frac{\delta p - \frac{\delta^2 p (p-1)}{2}}{1 - \delta p} \cdot \frac{c \log T}{T}} + \Tilde{O}\brc{\frac{1}{T^2}}\\
    &= \frac{e^{-p \delta \log(1/\delta)} - 1}{1 - \delta p} + \brc{\frac{\delta p \log(1/\delta) (1 + \delta/2)}{1-\delta p} - 1 + \frac{e^{-p \delta \log(1/\delta)} - 1}{1 - \delta p} \cdot \frac{\delta p (1 - \delta (p-1) / 2 )}{1 - \delta p}}  \frac{c \log T}{T} \\
    &\qquad + \Tilde{O}\brc{\frac{1}{T^2}}.
\end{align*}
where $(iv)$ follows from \eqref{eq:nonvanish_coef_sum_genli_binom}.
Therefore,
\begin{align}\label{eq:nonvanish_coef_sum_genli_term1_neq}
    &\sum_{t=2}^{t^*} (1-\alpha_t) \Bar{\alpha}_t^p = (1-\delta)^p \brc{1 + \frac{c \log T}{T}} \brc{ \delta \frac{c \log T}{T} \sum_{t=1}^{t^*-1} q^t } \nonumber\\
    &\leq (1-\delta)^p \frac{e^{-p \delta \log(1/\delta)} - 1}{1 - \delta p} \nonumber\\
    &\qquad + (1-\delta)^p \bigg(\frac{\delta p \log(1/\delta) (1 + \delta/2)}{1-\delta p} - 1 \nonumber \\
    &\qquad \quad + \frac{e^{-p \delta \log(1/\delta)} - 1}{1 - \delta p} \brc{1 + \frac{\delta p (1 - \delta (p-1) / 2 )}{1 - \delta p} } \bigg) \frac{c \log T}{T} + \Tilde{O}\brc{\frac{1}{T^2}}.
\end{align}

Combining \eqref{eq:nonvanish_coef_sum_genli_term2} and \eqref{eq:nonvanish_coef_sum_genli_term1_neq}, we have that
\[ s_T \leq \underbrace{ \brc{ \frac{1}{p} (1-\delta)^p e^{-p \delta \log(1/\delta)} + (1-\delta)^p \frac{e^{-p \delta \log(1/\delta)} - 1}{1 - \delta p} } }_{=: s_\infty} \brc{1 + O\brc{\frac{\log T}{T}} }.  \]
Also, for all large $T$'s, since $s_{\infty} \to \frac{1}{p}$ and $\delta \log(1/\delta) \to 0$ as $\delta \to 0$, when $\delta \ll 1$,
\[ s_T \leq \frac{1}{p} - \brc{ 1 + \frac{p+1}{2 p} } \frac{c \log T}{T} + \Tilde{O}\brc{\frac{1}{T^2}}. \]
The proof is now complete.
\end{proof}

\section{Proofs in \texorpdfstring{\Cref{sec:delta_ty}}{Section 4}}

\subsection{Proof of \texorpdfstring{\Cref{lem:cond_score_proj}}{Theorem 3}}

Fix $t \geq 1$. Using the forward model in \eqref{eq:def_cond_fwd2}, we have that $q_{t|0,y}$ is the p.d.f. of $\calN(\sqrt{\Bar{\alpha}_t} (I_d - H^\dagger H) x_0 + \sqrt{\Bar{\alpha}_t} H^\dagger y, \Sigma_{t|0,y})$. Thus,
\begin{align*}
    &\nabla \log q_{t|y}(x) = \frac{1}{q_{t|y}(x)} \int_{x_0 \in \mbR^d} \nabla q_{t|0,y}(x|x_0) \d Q_{0|y}(x_0) \\
    &= -\frac{1}{q_{t|y}(x)} \Sigma_{t|0,y}^{-1} \int_{x_0 \in \mbR^d} q_{t|0,y}(x|x_0) (x-\sqrt{\Bar{\alpha}_t} (I_d - H^\dagger H) x_0 - \sqrt{\Bar{\alpha}_t} H^\dagger y) \d Q_{0|y}(x_0)\\
    &= - \Sigma_{t|0,y}^{-1} (x-\sqrt{\Bar{\alpha}_t} H^\dagger y)\\
    &\qquad + \frac{\sqrt{\Bar{\alpha}_t}}{q_{t|y}(x)} \Sigma_{t|0,y}^{-1} (I_d - H^\dagger H) \int_{x_0 \in \mbR^d} q_{t|0,y}(x|x_0) x_0 \d Q_{0|y}(x_0).
\end{align*}
Thus, the equality for $\nabla \log q_{t|y}$ is established because by \Cref{lem:gauss_cond_var_proj},
\[ (\sigma_y^2 H^\dagger (H^\dagger)^\T + (1-\Bar{\alpha}_t) I_d)^{-1} (I_d - H^\dagger H) = (1-\Bar{\alpha}_t)^{-1} (I_d - H^\dagger H). \]

To see the optimality with $f_{t,y}^*$, fix $t \geq 1$ and $x \in \mbR^d$. First note that $(I_d-H^\dagger H) f_{t,y}^*(x) = (I_d-H^\dagger H) \Sigma_{t|0,y}^{-1} (\sqrt{\Bar{\alpha}_t} H^\dagger y - H^\dagger H x) = 0$ by \Cref{lem:gauss_cond_var_proj}. Now, suppose that $f_{t,y} = f_{t,y}^* + v$ such that $(I_d-H^\dagger H) f_{t,y} = 0 \implies (I_d-H^\dagger H) v = 0$. From the definition of $\Delta_{t,y}$ in \eqref{eq:def_delta_linear_cond},
\begin{equation*}
    \Delta_{t,y}(x) = (I_d-H^\dagger H) (\nabla \log q_{t|y} (x) - \nabla \log q_t (x)) + (H^\dagger H) \nabla \log q_{t|y} (x) - f_{t,y}(x)
\end{equation*}
where
\begin{align*}
    &(H^\dagger H) \nabla \log q_{t|y} (x) - f_{t,y}(x) \\
    &= (H^\dagger H) \Sigma_{t|0,y}^{-1} (\sqrt{\Bar{\alpha}_t} H^\dagger y - x) - \Sigma_{t|0,y}^{-1} (\sqrt{\Bar{\alpha}_t} H^\dagger y - H^\dagger H x) - v\\
    &= - (H^\dagger H) \Sigma_{t|0,y}^{-1} (I_d - H^\dagger H) x - (I_d - H^\dagger H) \Sigma_{t|0,y}^{-1} (\sqrt{\Bar{\alpha}_t} H^\dagger y - H^\dagger H x) - v\\
    &= -v
\end{align*}
where the last line follows from \Cref{lem:gauss_cond_var_proj}. 

Thus, if $v = 0$, then $f_{t,y} = f_{t,y}^*$, and we have
\begin{equation} \label{eq:delta_using_fy}
    \Delta_{t,y} = (I_d-H^\dagger H) (\nabla \log q_{t|y} (x) - \nabla \log q_t (x)).
\end{equation}
Also, if $v \neq 0$, since $v$ is orthogonal to the space induced by $(I_d-H^\dagger H)$, we have
\begin{equation} \label{eq:delta_norm2_with_v}
    \norm{\Delta_{t,y}(x)}^2 = \norm{ (I_d-H^\dagger H) (\nabla \log q_{t|y} (x) - \nabla \log q_t (x)) }^2 + \norm{v}^2
\end{equation}
which is minimized at $v = 0$. The proof is now complete.

% Thus,
% \begin{align*}
%     &\nabla \log q_{t|y}(x_t) - (I_d - H^\dagger H) \nabla \log q_{t}(x_t)\\
%     &= -\frac{1}{(1-\Bar{\alpha}_t)} (x_t-\sqrt{\Bar{\alpha}_t} H^\dagger y) + \frac{1}{(1-\Bar{\alpha}_t)} (I_d - H^\dagger H) x_t + v_t\\
%     &= \frac{1}{(1-\Bar{\alpha}_t)} (\sqrt{\Bar{\alpha}_t} H^\dagger y - H^\dagger H x_t) + v_t.
% \end{align*}
% Here $v_t := (I_d - H^\dagger H) \frac{\sqrt{\Bar{\alpha}_t}}{(1-\Bar{\alpha}_t)} \brc{ \frac{\int_{x_0 \in \mbR^d} q_{t|0,y}(x|x_0) x_0 \d Q_{0|y}(x_0)}{q_{t|y}(x)} - \frac{\int_{x_0 \in \mbR^d} q_{t|0}(x|x_0) x_0 \d Q_0(x_0)}{q_t(x)} }$ is supported on $\range(I_d-H^\dagger H)$. Finally, since $(I_d - H^\dagger H) (H^\dagger H) = 0$ and $H^\dagger H$ is a projection operator,
% \begin{align*}
%     (H^\dagger H) \nabla \log q_{t|y}(x_t) &= (H^\dagger H) (\nabla \log q_{t|y}(x_t) - (I_d - H^\dagger H) \nabla \log q_{t}(x_t))\\
%     &= (H^\dagger H) \frac{1}{(1-\Bar{\alpha}_t)} (\sqrt{\Bar{\alpha}_t} H^\dagger y - H^\dagger H x_t) + (H^\dagger H) v_t\\
%     &= \frac{1}{(1-\Bar{\alpha}_t)} (\sqrt{\Bar{\alpha}_t} H^\dagger y - H^\dagger H x_t).
% \end{align*}

\subsection{Proof of \texorpdfstring{\Cref{lem:norm2_bd_bdsupp}}{Theorem 4}}

Fix $t \geq 2$. Recall that the unconditional score $\nabla \log q_t(x)$ is
\begin{align*}
    \nabla \log q_t(x) &= \frac{1}{q_t(x)} \int_{x_0 \in \mbR^d} \nabla q_{t|0}(x|x_0) \d Q_0(x_0) \\
    &= -\frac{1}{(1-\Bar{\alpha}_t) q_t(x)} \int_{x_0 \in \mbR^d} q_{t|0}(x|x_0) (x-\sqrt{\Bar{\alpha}_t} x_0) \d Q_0(x_0)\\
    &= -\frac{1}{(1-\Bar{\alpha}_t)} x + \frac{\sqrt{\Bar{\alpha}_t}}{(1-\Bar{\alpha}_t) q_t(x)} \int_{x_0 \in \mbR^d} q_{t|0}(x|x_0) x_0 \d Q_0(x_0)
\end{align*}
since $q_{t|0}$ is the p.d.f. of $\calN(\sqrt{\Bar{\alpha}_t} x_0, (1-\Bar{\alpha}_t) I_d)$. 

In the first half, we consider the case where $\sigma_y^2$ is known, and thus $f_{t,y} = f_{t,y}^*$ in \eqref{eq:def_fy_star}. Note that from \Cref{lem:cond_score_proj},
\begin{align*}
    \nabla \log q_{t|y}(x) &= \Sigma_{t|0,y}^{-1} (\sqrt{\Bar{\alpha}_t} H^\dagger y - x)\\
    &\qquad + \frac{\sqrt{\Bar{\alpha}_t}}{q_{t|y}(x)} \Sigma_{t|0,y}^{-1} (I_d - H^\dagger H) \int_{x_0 \in \mbR^d} q_{t|0,y}(x|x_0) x_0 \d Q_{0|y}(x_0).
\end{align*}
Here we also recall from \Cref{lem:cond_score_proj} that
\[ \Sigma_{t|0,y} := \Bar{\alpha}_t \sigma_y^2 H^\dagger (H^\dagger)^\T + (1-\Bar{\alpha}_t) I_d. \]
Since $H^\dagger (H^\dagger)^\T$ is positive semi-definite, all its eigenvalues are non-negative. Write the eigen-decomposition as $H^\dagger (H^\dagger)^\T = P \mathrm{diag}(D_1,\dots,D_d) P^\T$ where $D_1 \geq \dots \geq D_d \geq 0,~\forall i \in [d]$. Then, $\lambda_{min}(\Sigma_{t|0,y}) \geq \Bar{\alpha}_t \sigma_y^2 D_d + 1-\Bar{\alpha}_t \geq 1-\Bar{\alpha}_t$, and we get
\begin{equation} \label{eq:bdsupp_inv_cov_bound}
    \norm{\Sigma_{t|0,y}^{-1}} \leq \frac{1}{1-\Bar{\alpha}_t}.
\end{equation}
Also, from \eqref{eq:def_cond_spl}, with the $f_{t,y}^*$ in \eqref{eq:def_fy_star},
\begin{align*}
    g_{t,y}(x) &= f_{t,y}^*(x) + (I_d-H^\dagger H) \nabla \log q_t (x) \\
    &= \Sigma_{t|0,y}^{-1} \brc{\sqrt{\Bar{\alpha}_t} H^\dagger y - H^\dagger H x} - \frac{1}{(1-\Bar{\alpha}_t)} (I_d-H^\dagger H) x \\
    &\qquad + \frac{\sqrt{\Bar{\alpha}_t}}{(1-\Bar{\alpha}_t) q_t(x)} (I_d-H^\dagger H) \int_{x_0 \in \mbR^d} q_{t|0}(x|x_0) x_0 \d Q_0(x_0)\\
    &\stackrel{(i)}{=} \Sigma_{t|0,y}^{-1} \brc{\sqrt{\Bar{\alpha}_t} H^\dagger y - H^\dagger H x} - \Sigma_{t|0,y}^{-1} (I_d-H^\dagger H) x \\
    &\qquad + \frac{\sqrt{\Bar{\alpha}_t}}{q_t(x)} \Sigma_{t|0,y}^{-1} (I_d-H^\dagger H) \int_{x_0 \in \mbR^d} q_{t|0}(x|x_0) x_0 \d Q_0(x_0)\\
    &= \Sigma_{t|0,y}^{-1} \brc{\sqrt{\Bar{\alpha}_t} H^\dagger y - x} + \frac{\sqrt{\Bar{\alpha}_t}}{q_t(x)} \Sigma_{t|0,y}^{-1} (I_d-H^\dagger H) \int_{x_0 \in \mbR^d} q_{t|0}(x|x_0) x_0 \d Q_0(x_0)
\end{align*}
where $(i)$ follows from \Cref{lem:gauss_cond_var_proj}. 
Then, the norm-squared of the score mismatch at time $t \geq 2$ is
\begin{align} \label{eq:lem_norm2_bd_bdsupp_main_res_1}
    &\norm{\Delta_{t,y}}^2 = \norm{\nabla \log q_{t|y} - g_{t,y}}^2 = \norm{(I_d - H^\dagger H) (\nabla \log q_{t|y} - \nabla \log q_t) }^2 \nonumber\\
    &\leq \Bar{\alpha}_t \norm{\Sigma_{t|0,y}^{-1}}^2 \norm{\frac{\int_{x_0 \in \mbR^d} q_{t|0,y}(x|x_0) x_0 \d Q_{0|y}(x_0)}{q_{t|y}(x)} - \frac{\int_{x_0 \in \mbR^d} q_{t|0}(x|x_0) x_0 \d Q_0(x_0)}{q_t(x)} }^2 \nonumber\\
    &\stackrel{(ii)}{\leq} \Bar{\alpha}_t \norm{\Sigma_{t|0,y}^{-1}}^2 \int_{x_a,x_b \in \mbR^d} \norm{x_a - x_b}^2 \d Q_{0|t,y}(x_a) \d Q_{0|t}(x_b) \nonumber\\
    &\leq \Bar{\alpha}_t \norm{\Sigma_{t|0,y}^{-1}}^2 \max_{\substack{x_a \in \mathrm{supp}(Q_{0|y}) \\ x_b \in \mathrm{supp}(Q_{0})}} \norm{x_a-x_b}^2 \nonumber\\
    % &\stackrel{(iii)}{\lesssim} \Bar{\alpha}_t \brc{\frac{1}{(1-\Bar{\alpha}_t)} + \frac{\Bar{\alpha}_t \sigma_y^2}{(1-\Bar{\alpha}_t)^2} }^2 d.
    &\stackrel{(iii)}{\lesssim} \frac{\Bar{\alpha}_t}{(1-\Bar{\alpha}_t)^2} d.
\end{align}
Here $(ii)$ follows from Jensen's inequality, and $(iii)$ follows by \eqref{eq:bdsupp_inv_cov_bound} and from the assumption that $Q_0$ has bounded support (and thus also for both $Q_{0|t}$ and $Q_{0|t,y}$). Therefore, with the $\alpha_t$ in \eqref{eq:alpha_genli} (cf. \Cref{lem:alpha_genli_rate}), since $1-\Bar{\alpha}_t \geq 1 - \delta$ which is a constant, \Cref{ass:bdd-mismatch-general} is satisfied for all $\sigma_y^2 \geq 0$. 
% Thus, if $\sigma_y^2 = 0$,  then \Cref{thm:gen_q0_kl} holds with $\gamma = 1$ and $r = 2$. If $\sigma_y^2 > 0$, then \Cref{thm:gen_q0_kl} holds with $\gamma = 1$ and $r = 4$.
Thus, \Cref{thm:gen_q0_kl_general} holds with $\gamma = 1$ and $r = 2$.

Now, we consider the case where $\sigma_y^2$ is unknown, and the conditional sampler of interest is $g_{t,y}^N(x) = f_{t,y}^N(x) + (I_d-H^\dagger H) \nabla \log q_t (x)$ where $f_{t,y}^N(x) = (1-\Bar{\alpha}_t)^{-1} \brc{\sqrt{\Bar{\alpha}_t} H^\dagger y - H^\dagger H x}$. With the same notation as in the proof of \Cref{lem:cond_score_proj}, we can write $v = f_{t,y}^N - f_{t,y}^* = ( (1-\Bar{\alpha}_t)^{-1} I_d - \Sigma_{t|0,y}^{-1}) \brc{\sqrt{\Bar{\alpha}_t} H^\dagger y - H^\dagger H x}$. Note that $v$ still satisfies that $(I_d - H^\dagger H) v = 0$.
% \begin{align*}
%     g_{t,y}'(x) &= f_{t,y}(x) + (I_d-H^\dagger H) \nabla \log q_t (x) \\
%     &= \brc{\sqrt{\Bar{\alpha}_t} H^\dagger y - H^\dagger H x} - \frac{1}{(1-\Bar{\alpha}_t)} (I_d-H^\dagger H) x \\
%     &\qquad + \frac{\sqrt{\Bar{\alpha}_t}}{(1-\Bar{\alpha}_t) q_t(x)} (I_d-H^\dagger H) \int_{x_0 \in \mbR^d} q_{t|0}(x|x_0) x_0 \d Q_0(x_0).
% \end{align*}
Using the result in \eqref{eq:delta_norm2_with_v}, we have
\begin{multline*}
    \norm{\Delta_{t,y}^N}^2 = \norm{(I_d-H^\dagger H) (\nabla \log q_{t|y} (x) - \nabla \log q_t (x))}^2 \\
    + \norm{(\Sigma_{t|0,y}^{-1} - (1-\Bar{\alpha}_t)^{-1} I_d) \brc{\sqrt{\Bar{\alpha}_t} H^\dagger y - H^\dagger H x} }^2 
\end{multline*}
where the first term is the same as in \eqref{eq:lem_norm2_bd_bdsupp_main_res_1} which can be upper-bounded in a similar way. To upper-bound the second term, note that by Woodbury matrix identity,
\begin{align} \label{eq:noise_var_norm2_bd}
    \norm{\Sigma_{t|0,y}^{-1} - \frac{1}{1-\Bar{\alpha}_t} I_d} &= \frac{\Bar{\alpha}_t \sigma_y^2}{(1-\Bar{\alpha}_t)^2} \norm{ H^\dagger \brc{I_p + \frac{\sigma_y^2}{1-\Bar{\alpha}_t} (H^\dagger)^\T H^\dagger}^{-1} (H^\dagger)^\T } \nonumber\\
    &\lesssim \frac{\Bar{\alpha}_t \sigma_y^2}{(1-\Bar{\alpha}_t)^2}
\end{align} 
where the inequality follows because $\norm{H^\dagger} \lesssim 1$ is a constant and the minimum eigenvalue of $(I_p + \frac{\sigma_y^2}{1-\Bar{\alpha}_t} (H^\dagger)^\T H^\dagger)$ is at least 1.
Thus,
\begin{align*}
    &\E_{Q_{t|y}} \norm{(\Sigma_{t|0,y}^{-1} - (1-\Bar{\alpha}_t)^{-1} I_d) \brc{\sqrt{\Bar{\alpha}_t} H^\dagger y - H^\dagger H X_t} }^2 \\
    &\stackrel{(iv)}{\leq} \frac{\Bar{\alpha}_t^2 \sigma_y^4}{(1-\Bar{\alpha}_t)^4} \E_{Q_{t|y}} \norm{\sqrt{\Bar{\alpha}_t} H^\dagger y - H^\dagger H X_t }^2 \\
    &= \frac{\Bar{\alpha}_t^2 \sigma_y^4}{(1-\Bar{\alpha}_t)^4} \E_{Q_{0|y}} \E_{Q_{t|0,y}} \norm{\sqrt{\Bar{\alpha}_t} H^\dagger y - H^\dagger H X_t }^2 \\
    &= \frac{\Bar{\alpha}_t^2 \sigma_y^4}{(1-\Bar{\alpha}_t)^4} \E_{Q_{0|y}} \E_{Q_{t|0}} \norm{\sqrt{\Bar{\alpha}_t} H^\dagger y - H^\dagger H X_t }^2 \\
    &\leq \frac{2 \Bar{\alpha}_t^2 \sigma_y^4}{(1-\Bar{\alpha}_t)^4} \E_{Q_{0|y}} \sbrc{ \Bar{\alpha}_t \norm{H^\dagger y - H^\dagger H X_0 }^2 + \E_{Q_{t|0}} \norm{H^\dagger H (X_t - \sqrt{\Bar{\alpha}_t} X_0)}^2 } \\
    &\leq \frac{2 \Bar{\alpha}_t^2 \sigma_y^4}{(1-\Bar{\alpha}_t)^4} \E_{Q_{0|y}} \sbrc{ \Bar{\alpha}_t \norm{H^\dagger y}^2 + \Bar{\alpha}_t \norm{X_0}^2 + d (1-\Bar{\alpha}_t) } \\
    &\stackrel{(v)}{\lesssim} \frac{\Bar{\alpha}_t^2 \sigma_y^4}{(1-\Bar{\alpha}_t)^4} d
\end{align*}
where $(iv)$ follows from \eqref{eq:noise_var_norm2_bd}, and $(v)$ follows from the fact that $Q_{0|y}$ has bounded support. Similarly, for general moments $\ell \geq 2$,
\[ \E_{Q_{t|y}} \norm{(\Sigma_{t|0,y}^{-1} - (1-\Bar{\alpha}_t)^{-1} I_d) \brc{\sqrt{\Bar{\alpha}_t} H^\dagger y - H^\dagger H X_t} }^\ell \lesssim \brc{ \frac{\Bar{\alpha}_t \sigma_y^2}{(1-\Bar{\alpha}_t)^2} d }^{\ell / 2}. \]
Therefore, with the $\alpha_t$ in \eqref{eq:alpha_genli}, since $1-\Bar{\alpha}_t \geq 1 - \delta$ which is a constant, we still have that \Cref{ass:bdd-mismatch-general} is satisfied (see \Cref{lem:alpha_genli_rate}), and \Cref{thm:gen_q0_kl_general} still holds with $\gamma = 1$ and $r = 4$. The proof is now complete.

\subsection{\texorpdfstring{\Cref{thm:kl_bd_gauss}}{Theorem 6} and its proof}

Before we enter the proof of \Cref{lem:norm2_bd_gauss_mix,thm:kl_bd_gauss_mix}, we first state a similar set of results for Gaussian $Q_0$, which turns out to be useful for analyzing Gaussian mixture $Q_0$'s. To begin, the following lemma investigates $\E_{X_t \sim Q_{t|y}} \norm{\Delta_{t,y}(X_t)}^2$ when $Q_0$ is Gaussian. This quantity is proportional to the asymptotic bias $\mathcal{W}_{\text{bias}}$. 

\begin{proposition} \label{lem:norm2_bd_gauss}
For $Q_0 = \calN(\mu_0, \Sigma_0)$, if $f_{t,y} = f_{t,y}^*$ in \eqref{eq:def_fy_star} and $H = \begin{pmatrix} I_p & 0 \end{pmatrix}$, with the $\alpha_t$'s according to \Cref{def:noise_smooth}, \Cref{ass:bdd-mismatch-general} is satisfied, and
\begin{align*}
    \E_{X_t \sim Q_{t|y}} \norm{\Delta_{t,y}(X_t)}^2 &\leq \Bar{\alpha}_t^2 \frac{\max\{\norm{H^\dagger y - H^\dagger H \mu_{0}}^2 + d (\lambda_1 + \sigma_y^2), d\}}{\min\{\lambda_d, 1\}^2 \min\{\Tilde{\lambda}_{d-p}, 1\}^2} \norm{[\Sigma_0]_{y \Bar{y}} [\Sigma_0]_{\Bar{y} y}} \\
    &\lesssim \Bar{\alpha}_t^2 \cdot (\norm{H^\dagger y - H^\dagger H \mu_{0}}^2+d)
\end{align*}
where $\lambda_1$ is the largest eigenvalue of $\Sigma_0$, and $\lambda_d$ and $\Tilde{\lambda}_{d-p}$ are the smallest eigenvalues of $\Sigma_0$ and $[\Sigma_0]_{\Bar{y}\Bar{y}}$, respectively.
\end{proposition}

\begin{proof}
    See \Cref{app:proof_lem_norm2_bd_gauss}.
\end{proof}

% We also verified this phenomenon in numerical simulations.
With this lemma, the following theorem characterizes the conditional KL divergence when $Q_0$ is Gaussian.

\begin{theorem} \label{thm:kl_bd_gauss}
Suppose that $\sigma_y^2 > 0$. Suppose that \Cref{ass:m2-general,ass:score} hold.
Under the same conditions as in \Cref{lem:norm2_bd_gauss}, 
% with the $\alpha_t$'s in \eqref{eq:alpha_genli} where $\delta \ll 1$ with $c \asymp \log(1/\delta)$, 
if $\alpha_t$ further satisfies $\sum_{t=1}^T (1-\alpha_t) \Bar{\alpha}_t = 1 + o(1)$, we have
\begin{align*}
    % &\KL{Q_{1|y}}{\widehat{P}_{1|y}} \lesssim (\norm{H^\dagger y - H^\dagger H \mu_0}^2 + d) \brc{1-\frac{3 \log(1/\delta) \log T}{T}} \\
    &\KL{Q_{0|y}}{\widehat{P}_{0|y}} \lesssim (\norm{H^\dagger y - H^\dagger H \mu_0}^2 + d) \\
    &\qquad + (\norm{H^\dagger y - H^\dagger H \mu_0}^2 + d) \frac{(\log T)^2}{T} + \sqrt{\norm{H^\dagger y - H^\dagger H \mu_0}^2 + d} \cdot (\log T) \eps.
\end{align*}
\end{theorem}

% This is the first convergence result for zero-shot samplers where explicit dependency on the conditioning $y$ is expressed on Gaussian target. 
% The extra condition on $\alpha_t$ can be verified for constant $\alpha_t$ (\Cref{lem:nonvanish_coef_sum}) and for that in \eqref{eq:alpha_genli} (\Cref{lem:nonvanish_coef_sum_genli}). 
Note that a similar result can be obtained for $\KL{Q_{1|y}}{\widehat{P}_{1|y}}$ (where $\rmW_2(Q_{1|y}, Q_{0|y})^2 \lesssim \delta d$) for any general $\sigma_y^2 \geq 0$ using the $\alpha_t$ in \eqref{eq:alpha_genli} (see \Cref{rmk:gauss_kl_genli}).

\subsubsection{Proof of \texorpdfstring{\Cref{thm:kl_bd_gauss}}{Theorem 6}}

Throughout the proof we use the same notations as in \eqref{eq:gauss_notations}. Since \Cref{ass:bdd-mismatch-general} is satisfied from \Cref{lem:norm2_bd_gauss}, in order to invoke \Cref{thm:main-general}, we still need to check \Cref{ass:regular-drv-general}. Since each $Q_{t|y}~(\forall t \geq 0)$ is Gaussian, all partial derivatives of its log-p.d.f. higher than third-order equal zero. For the first and second-order, note that $\Sigma_{t|y} = \Bar{\alpha}_t (I_d - H^\dagger H) \Sigma_{0} (I_d - H^\dagger H) + (1-\Bar{\alpha}_t) I_d + \Bar{\alpha}_t \sigma_y^2 H^\dagger H$. Thus, when $\sigma_y^2 > 0$, $\lambda_{min}(\Sigma_{t-1|y}) \geq \min\{ 1-\Bar{\alpha}_t+\Bar{\alpha}_t \sigma_y^2, 1-\Bar{\alpha}_t+\Bar{\alpha}_t \Tilde{\lambda}_d \} \geq \min\{1, \sigma_y^2, \Tilde{\lambda}_d\} > 0$, which yields
\begin{equation} \label{eq:gauss_cov_inv_norm_upper}
    \norm{\Sigma_{t-1|y}^{-1}} \lesssim 1,\quad \forall t \geq 1.
\end{equation}
Thus, we have, $\forall \ell \geq 1$,
\begin{align*}
    &\E_{Q_{t|y}} \norm{\nabla \log q_{t|y}(X_t)}^\ell \\
    &\quad = \E_{Q_{t|y}} \norm{\Sigma_{t|y}^{-1} (X_t - \mu_{t|y})}^\ell \leq \norm{\Sigma_{t|y}^{-\frac{1}{2}}}^\ell \E_{Q_{t|y}} \norm{\Sigma_{t|y}^{-\frac{1}{2}} (X_t - \mu_{t|y})}^\ell \\
    &\quad \lesssim d^{\ell / 2} = O(1),\\
    &\E_{Q_{t|y}} \norm{\nabla \log q_{t-1|y}(m_{t,y}(X_t))}^\ell \\
    &\quad = \E_{Q_{t|y}} \norm{\Sigma_{t|y}^{-1} (m_{t,y}(X_t) - \mu_{t|y})}^\ell \\
    &\quad \lesssim \norm{\Sigma_{t|y}^{-\frac{1}{2}}}^\ell \E_{Q_{t|y}} \norm{\Sigma_{t|y}^{-\frac{1}{2}} (X_t - \mu_{t|y})}^\ell + \norm{\Sigma_{t|y}^{-1}}^\ell \E_{Q_{t|y}} \norm{\nabla \log q_{t|y}(X_t)}^\ell \\
    &\quad \lesssim d^{\ell / 2} = O(1), \\
    &\E_{Q_{t|y}} \norm{\nabla^2 \log q_{t|y}(X_t)}^\ell = \norm{\Sigma_{t|y}^{-1}}^\ell = O(1),\\
    &\E_{Q_{t|y}} \norm{\nabla^2 \log q_{t-1|y}(m_{t,y}(X_t))}^\ell = \norm{\Sigma_{t-1|y}^{-1}}^\ell = O(1).
\end{align*}
Thus, \Cref{ass:regular-drv-general} holds when $1-\alpha_t$ satisfies \Cref{def:noise_smooth}.

Now, we can invoke \Cref{thm:main-general} and get $\KL{Q_{0|y}}{\widehat{P}_{0|y}} \lesssim \mathcal{W}_{\text{oracle}} + \mathcal{W}_{\text{bias}} + \mathcal{W}_{\text{vanish}}$,
where
\begin{align*}
\mathcal{W}_{\text{oracle}} &= \sum_{t=1}^T \frac{(1-\alpha_t)^2}{2 \alpha_t} \E_{X_t \sim Q_{t|y}} \sbrc{ \Tr\Big(\nabla^2 \log q_{t-1|y}(m_{t,y}(X_t))\nabla^2\log q_{t|y}(X_t)  \Big) } + (\log T) \eps^2\\
\mathcal{W}_{\text{bias}} &= \sum_{t=1}^T (1-\alpha_t) \E_{X_t \sim Q_{t|y}} \norm{\Delta_{t,y}(X_t)}^2\\
\mathcal{W}_{\text{vanish}} &= \sum_{t=1}^T \frac{1-\alpha_t}{\sqrt{\alpha_t}} \E_{X_t \sim Q_{t|y}} \bigg[(\nabla \log q_{t-1|y}(m_{t,y}(X_t)) - \sqrt{\alpha_t} \nabla \log q_{t|y}(X_t))^\T \Delta_{t,y}(X_t)\bigg]\\
    &- \sum_{t=1}^T \frac{(1-\alpha_t)^2}{2 \alpha_t} \E_{X_t \sim Q_{t|y}}\sbrc{\Delta_{t,y}(X_t)^\T \nabla^2 \log q_{t-1|y}(m_{t,y}(X_t)) \Delta_{t,y}(X_t) } \\
    &+ \sum_{t=1}^T \frac{(1-\alpha_t)^2}{3! {\alpha_t}^{3/2}} \E_{X_t \sim Q_{t|y}} \bigg[ 3 \sum_{i=1}^d \partial^3_{iii} \log q_{t-1|y}(m_{t,y}(X_t)) \Delta_{t,y}(X_t)^i \\
    &\qquad + \sum_{\substack{i,j=1 \\ i\neq j}}^d \partial^3_{iij} \log q_{t-1|y}(m_{t,y}(X_t)) \Delta_{t,y}(X_t)^j \bigg]\\
    &+ \max_{t \geq 1}\sqrt{\E_{X_t \sim Q_{t|y}} \norm{\Delta_{t,y}(X_t)}^2} (\log T) \eps.
\end{align*}

We first consider the estimation error (in both $\mathcal{W}_{\text{oracle}}$ and $\mathcal{W}_{\text{vanish}}$), which can be upper-bounded as 
\[ \max_{t \geq 1}\sqrt{\E_{X_t \sim Q_{t|y}} \norm{\Delta_{t,y}(X_t)}^2} (\log T) \eps + (\log T) \eps^2 \lesssim \brc{\norm{H^\dagger y - H^\dagger H \mu_0}^2+d}^{\frac{1}{2}} (\log T) \eps \] from \Cref{lem:norm2_bd_gauss}. 
Also, when $Q_{t|y}$ is Gaussian, we can calculate, for any $x_t \in \mbR^d$,
\begin{align*}
    &\Tr\Big(\nabla^2 \log q_{t-1|y}(m_{t,y}(x_t))\nabla^2\log q_{t|y}(x_t) \Big) = \Tr(\Sigma_{t-1|y}^{-1} \Sigma_{t|y}^{-1})\\
    &= \Tr(\Sigma_{t-1|y}^{-1} (\alpha_t \Sigma_{t-1|y} + (1-\alpha_t) I_d)^{-1}) \\
    &\stackrel{(i)}{=} \Tr(\Sigma_{t-1|y}^{-1} (\alpha_t^{-1} \Sigma_{t-1|y}^{-1} - \frac{1-\alpha_t}{\alpha_t^2} \Sigma_{t-1|y}^{-2} + O((1-\alpha_t)^2)))\\
    &\lesssim \frac{1}{\alpha_t} \Tr(\Sigma_{t-1|y}^{-2}) 
    % \lesssim \frac{d}{(1-\Bar{\alpha}_{t-1})^2}
\end{align*}
where $(i)$ follows from Taylor expansion when $1-\alpha_t$ is small. Using \eqref{eq:gauss_cov_inv_norm_upper}, this implies that
% Thus, with the $\alpha_t$ in \eqref{eq:alpha_genli},
\[  \mathcal{W}_{\text{oracle}} \lesssim \frac{d (\log T)^2}{T} + (\log T) \eps^2. \]
% \yuchen{If $\sigma_y^2 > 0$, this bound will not have $1-\Bar{\alpha}_{t-1}$ on the denominator, which allows for any $1-\alpha_t$ (not necessarily that in \eqref{eq:alpha_genli}). Thus, if $\sigma_y^2 > 0$, then the same inequality will hold for $\KL{Q_{0|y}}{\widehat{P}_{0|y}}$.}
% With the $\alpha_t$ in \eqref{eq:alpha_genli}, we can also upper-bound $\mathcal{W}_{\text{bias}}$ using \Cref{lem:norm2_bd_gauss,lem:nonvanish_coef_sum_genli} when $\delta \ll 1$ as
Also, from the condition on $\alpha_t$,
\[ \sum_{t=1}^T (1-\alpha_t) \E_{X_t \sim Q_{t|y}} \norm{\Delta_{t,y}(X_t)}^2 \lesssim (\norm{H^\dagger y - H^\dagger H \mu_0}^2+d). \]

Now we focus on $\mathcal{W}_{\text{vanish}}$ (except the estimation error).
Since $Q_{t|y}$ is Gaussian, all third-order partial derivatives are zero, and only the first two terms in $\mathcal{W}_{\text{vanish}}$ remain.
In the following we fix $t \geq 1$. Also recall from \eqref{eq:delta_y_gauss} that when $H = \begin{pmatrix} I_p & 0 \end{pmatrix}$,
\begin{align*}
    \Delta_{t,y} &= - \Bar{\alpha}_t (I_d - H^\dagger H) \Sigma_{t,sig}^{-1} (I_d - H^\dagger H) \Sigma_{0} (H^\dagger H) \Sigma_t^{-1} (x_t - \sqrt{\Bar{\alpha}_t} \mu_{0}) \\
    &= -\Bar{\alpha}_t (\Bar{\alpha}_t [\Sigma_{0}]_{\Bar{y} \Bar{y}} + (1-\Bar{\alpha}_t) I_{d-p})^{-1} [\Sigma_0]_{\Bar{y} y} [\Sigma_t^{-1}]_{y:} (x_t - \sqrt{\Bar{\alpha}_t} \mu_{0}).
\end{align*}

For the first term of $\mathcal{W}_{\text{vanish}}$, we first calculate for each $x_t$ that
\begin{align*}
    &\nabla \log q_{t-1|y}(m_{t,y}) - \sqrt{\alpha_t} \nabla \log q_{t|y}(x_t)\\
    &= \sqrt{\alpha_t} \Sigma_{t|y}^{-1} (x_t - \sqrt{\Bar{\alpha}_t} \mu_{0|y}) - \Sigma_{t-1|y}^{-1} (m_{t,y} - \sqrt{\Bar{\alpha}_{t-1}} \mu_{0|y})
\end{align*}
Recall that 
\begin{align*}
    m_{t,y} &= \frac{1}{\sqrt{\alpha_t}} x_t + \frac{1-\alpha_t}{\sqrt{\alpha_t}} \nabla \log q_{t|y}(x_t) = \frac{1}{\sqrt{\alpha_t}} x_t - \frac{1-\alpha_t}{\sqrt{\alpha_t}} \Sigma_{t|y}^{-1}(x_t - \sqrt{\Bar{\alpha}_t} \mu_{0|y}),\\
    \Sigma_{t|y}^{-1} &= (\alpha_t \Sigma_{t-1|y} + (1-\alpha_t) I_d)^{-1} = \frac{1}{\alpha_t} \Sigma_{t-1|y}^{-1} - \frac{1-\alpha_t}{\alpha_t^2} \Sigma_{t-1|y}^{-2} + O((1-\alpha_t)^2). 
\end{align*}
Thus,
\begin{align*}
    &\nabla \log q_{t-1|y}(m_{t,y}) - \sqrt{\alpha_t} \nabla \log q_{t|y}(x_t)\\
    &= \sqrt{\alpha_t} \brc{ \frac{1}{\alpha_t} \Sigma_{t-1|y}^{-1} - \frac{1-\alpha_t}{\alpha_t^2} \Sigma_{t-1|y}^{-2} } (x_t - \sqrt{\Bar{\alpha}_t} \mu_{0|y}) \\
    &\qquad - \Sigma_{t-1|y}^{-1} \brc{\frac{1}{\sqrt{\alpha_t}} x_t + \frac{1-\alpha_t}{\sqrt{\alpha_t}} \nabla \log q_{t|y}(x_t) - \sqrt{\Bar{\alpha}_{t-1}} \mu_{0|y}} + O((1-\alpha_t)^2)\\
    &= - \frac{1-\alpha_t}{\alpha_t^{3/2}} \Sigma_{t-1|y}^{-2} (x_t - \sqrt{\Bar{\alpha}_t} \mu_{0|y}) + \frac{1-\alpha_t}{\sqrt{\alpha_t}} \Sigma_{t-1|y}^{-1} \Sigma_{t|y}^{-1}(x_t - \sqrt{\Bar{\alpha}_t} \mu_{0|y}) + O((1-\alpha_t)^2).
\end{align*}
Combining with the definition for $\Delta_{t,y}$ in \eqref{eq:delta_y_gauss} and using \Cref{lem:gauss_cond_var_proj}, we have
\begin{align*}
    &\E_{X_t \sim Q_{t|y}} \sbrc{\Delta_{t,y}(X_t)^\T (\nabla \log q_{t-1|y}(m_{t,y}(X_t)) - \sqrt{\alpha_t} \nabla \log q_{t|y}(X_t))} \\
    &= \Bar{\alpha}_t \E_{X_t \sim Q_{t|y}} \bigg[ (X_t - \sqrt{\Bar{\alpha}_t} \mu_{0})^\T \Sigma_t^{-1} (H^\dagger H) \Sigma_{0} (I_d - H^\dagger H) \Sigma_{t,sig}^{-1} (I_d - H^\dagger H) \\
    &\qquad \brc{ \frac{1-\alpha_t}{\alpha_t^{3/2}} \Sigma_{t-1|y}^{-2} - \frac{1-\alpha_t}{\sqrt{\alpha_t}} \Sigma_{t-1|y}^{-1} \Sigma_{t|y}^{-1} } (X_t - \sqrt{\Bar{\alpha}_t} \mu_{0|y}) \bigg] + O((1-\alpha_t)^2)\\
    &\stackrel{(ii)}{=} \Bar{\alpha}_t \E_{X_t \sim Q_{t|y}} \bigg[ (X_t - \sqrt{\Bar{\alpha}_t} \mu_{0})^\T \Sigma_t^{-1} (H^\dagger H) \Sigma_{0} (I_d - H^\dagger H) \Sigma_{t,sig}^{-1} (I_d - H^\dagger H) \\
    &\qquad \brc{ \frac{1-\alpha_t}{\alpha_t^{3/2}} \Sigma_{t-1|y}^{-1} (I_d - H^\dagger H) \Sigma_{t-1|y}^{-1} - \frac{1-\alpha_t}{\sqrt{\alpha_t}} \Sigma_{t-1|y}^{-1} (I_d - H^\dagger H) \Sigma_{t|y}^{-1} } \\
    &\qquad (I_d - H^\dagger H) (X_t - \sqrt{\Bar{\alpha}_t} \mu_{0}) \bigg] + O((1-\alpha_t)^2)\\
    &= \Bar{\alpha}_t \E_{X_t \sim Q_{t|y}} \bigg[ (X_t - \sqrt{\Bar{\alpha}_t} \mu_{0})^\T [\Sigma_t^{-1}]_{:y} [\Sigma_0]_{y \Bar{y}} (\Bar{\alpha}_t [\Sigma_{0}]_{\Bar{y} \Bar{y}} + (1-\Bar{\alpha}_t) I_{d-p})^{-1} \\
    &\qquad \brc{ \frac{1-\alpha_t}{\alpha_t^{3/2}} [\Sigma_{t-1|y}^{-1}]_{\Bar{y} \Bar{y}}^2 - \frac{1-\alpha_t}{\sqrt{\alpha_t}} [\Sigma_{t-1|y}^{-1}]_{\Bar{y} \Bar{y}} [\Sigma_{t|y}^{-1}]_{\Bar{y} \Bar{y}} } \begin{pmatrix}0 & I_{d-p}\end{pmatrix} (X_t - \sqrt{\Bar{\alpha}_t} \mu_{0}) \bigg] \\
    &\qquad + O((1-\alpha_t)^2)\\
    &= \Bar{\alpha}_t \Tr\bigg( [\Sigma_t^{-1}]_{:y} [\Sigma_0]_{y \Bar{y}} (\Bar{\alpha}_t [\Sigma_{0}]_{\Bar{y} \Bar{y}} + (1-\Bar{\alpha}_t) I_{d-p})^{-1} \brc{ \frac{1-\alpha_t}{\alpha_t^{3/2}} [\Sigma_{t-1|y}^{-1}]_{\Bar{y} \Bar{y}}^2 - \frac{1-\alpha_t}{\sqrt{\alpha_t}} [\Sigma_{t-1|y}^{-1}]_{\Bar{y} \Bar{y}} [\Sigma_{t|y}^{-1}]_{\Bar{y} \Bar{y}} } \\
    &\qquad \begin{pmatrix}0 & I_{d-p}\end{pmatrix} \E_{X_t \sim Q_{t|y}} \sbrc{(X_t - \sqrt{\Bar{\alpha}_t} \mu_{0}) (X_t - \sqrt{\Bar{\alpha}_t} \mu_{0})^\T } \bigg) \\
    &\qquad + O((1-\alpha_t)^2)\\
    &\stackrel{(iii)}{\leq} \norm{[\Sigma_t^{-1}]_{:y} } \norm{[\Sigma_0]_{y \Bar{y}} } \norm{(\Bar{\alpha}_t [\Sigma_{0}]_{\Bar{y} \Bar{y}} + (1-\Bar{\alpha}_t) I_{d-p})^{-1} } \times\\
    &\qquad \brc{ \frac{1-\alpha_t}{\alpha_t^{3/2}} \norm{[\Sigma_{t-1|y}^{-1}]_{\Bar{y} \Bar{y}} }^2 + \frac{1-\alpha_t}{\sqrt{\alpha_t}} \norm{[\Sigma_{t-1|y}^{-1}]_{\Bar{y} \Bar{y}} } \norm{[\Sigma_{t|y}^{-1}]_{\Bar{y} \Bar{y}} } }  \times \\
    &\qquad \Tr\brc{\E_{X_t \sim Q_{t|y}} \sbrc{(X_t - \sqrt{\Bar{\alpha}_t} \mu_{0}) (X_t - \sqrt{\Bar{\alpha}_t} \mu_{0})^\T }}\\
    &\stackrel{(iv)}{\lesssim} \brc{\frac{1-\alpha_t}{\alpha_t^{3/2}} + \frac{1-\alpha_t}{\sqrt{\alpha_t}} } \max \cbrc{\norm{H^\dagger y - H^\dagger H \mu_0}^2 + d (\lambda_1 + \sigma_y^2), d } \\
    &\lesssim (1-\alpha_t) \brc{ \norm{H^\dagger y - H^\dagger H \mu_0}^2 + d }
\end{align*}
where $(ii)$ follows by \Cref{lem:gauss_cond_var_proj} and from definition that $(I_d - H^\dagger H) \mu_{0|y} = (I_d - H^\dagger H) \mu_{0}$, $(iii)$ follows because $\abs{\Tr(U V)} \leq \norm{U} \Tr(V)$ if $V$ is positive semi-definite, and $(iv)$ follows from \eqref{eq:gauss_mismatch_norm_2} and the same reasons for \eqref{eq:gauss_mismatch_2norm_bound}. In particular, we note that $[\Sigma_{t|y}^{-1}]_{\Bar{y} \Bar{y}} = [\Sigma_{t,sig}^{-1}]_{\Bar{y} \Bar{y}} = (\Bar{\alpha}_t [\Sigma_{0}]_{\Bar{y} \Bar{y}} + (1-\Bar{\alpha}_t) I_{d-p})^{-1}$ and $\norm{(\Bar{\alpha}_t [\Sigma_{0}]_{\Bar{y} \Bar{y}} + (1-\Bar{\alpha}_t) I_{d-p})^{-1}} \leq \frac{1}{\min\{\Tilde{\lambda}_{d-p}, 1\}} < \infty$.

For the second term of $\mathcal{W}_{\text{vanish}}$, we use the fact that $[\Sigma_{t-1|y}^{-1}]_{\Bar{y}\Bar{y}} = (\Bar{\alpha}_{t-1} [\Sigma_{0}]_{\Bar{y} \Bar{y}} + (1-\Bar{\alpha}_{t-1}) I_{d-p})^{-1}$ and have
\begin{align*}
    &-\E_{X_t \sim Q_{t|y}} \Delta_{t,y}(X_t)^\T \nabla^2 \log q_{t-1|y}(m_{t,y}(X_t)) \Delta_{t,y}(X_t)\\
    &= \Bar{\alpha}_t^2 \E_{X_t \sim Q_{t|y}} \bigg[ (X_t - \sqrt{\Bar{\alpha}_t} \mu_{0})^\T \Sigma_t^{-1} (H^\dagger H) \Sigma_{0} (I_d - H^\dagger H) \Sigma_{t,sig}^{-1} (I_d - H^\dagger H) \Sigma_{t-1|y}^{-1} \\
    &\qquad (I_d - H^\dagger H) \Sigma_{t,sig}^{-1} (I_d - H^\dagger H) \Sigma_{0} (H^\dagger H) \Sigma_t^{-1} (X_t - \sqrt{\Bar{\alpha}_t} \mu_{0}) \bigg] \\
    &= \Bar{\alpha}_t^2 \Tr\bigg([\Sigma_t^{-1}]_{:y} [\Sigma_0]_{y \Bar{y}}(\Bar{\alpha}_t [\Sigma_{0}]_{\Bar{y} \Bar{y}} + (1-\Bar{\alpha}_t) I_{d-p})^{-1} (\Bar{\alpha}_{t-1} [\Sigma_{0}]_{\Bar{y} \Bar{y}} + (1-\Bar{\alpha}_{t-1}) I_{d-p})^{-1}\\
    &\qquad (\Bar{\alpha}_t [\Sigma_{0}]_{\Bar{y} \Bar{y}} + (1-\Bar{\alpha}_t) I_{d-p})^{-1} [\Sigma_0]_{\Bar{y} y} [\Sigma_t^{-1}]_{y:} \E_{X_t \sim Q_{t|y}} (X_t - \sqrt{\Bar{\alpha}_t} \mu_{0}) (X_t - \sqrt{\Bar{\alpha}_t} \mu_{0})^\T \bigg)\\
    &\stackrel{(v)}{\leq} \max\cbrc{\norm{H^\dagger y - H^\dagger H \mu_0}^2 + d (\lambda_1 + \sigma_y^2), d } \norm{(\Bar{\alpha}_t [\Sigma_{0}]_{\Bar{y} \Bar{y}} + (1-\Bar{\alpha}_t) I_{d-p})^{-1}}^2 \\
    &\qquad \norm{(\Bar{\alpha}_{t-1} [\Sigma_{0}]_{\Bar{y} \Bar{y}} + (1-\Bar{\alpha}_{t-1}) I_{d-p})^{-1}} \norm{\Sigma_t^{-1}}^2 \norm{[\Sigma_0]_{y\Bar{y}} [\Sigma_0]_{\Bar{y} y}}\\
    &\lesssim \max \cbrc{\norm{H^\dagger y - H^\dagger H \mu_0}^2 + d (\lambda_1 + \sigma_y^2), d } \\
    &\lesssim \norm{H^\dagger y - H^\dagger H \mu_0}^2 + d.
\end{align*}
Here $(v)$ follows from the fact that $\abs{\Tr(U V)} \leq \norm{U} \Tr(V)$ if $V$ is positive semi-definite and \eqref{eq:gauss_mismatch_norm_2}.
The proof is complete by plugging all the results above into \Cref{thm:main-general}.

\begin{remark} \label{rmk:gauss_kl_genli}
Before we end the proof, we leave a note for the case of $\sigma_y^2 = 0$ (indeed, for any general $\sigma_y^2 \geq 0$). The only difference is how to upper-bound $\mathcal{W}_{\text{oracle}}$. In particular, if $\sigma_y^2 = 0$, \eqref{eq:gauss_cov_inv_norm_upper} no longer holds (i.e., we can no longer upper-bound $\norm{\Sigma_{t-1|y}^{-1}}$ as a constant). Instead, we can obtain an upper bound as $\norm{\Sigma_{t-1|y}^{-1}} \lesssim (1-\Bar{\alpha}_{t-1})^{-1}$. Then, with the $\alpha_t$ in \eqref{eq:alpha_genli}, we have
\[  \mathcal{W}_{\text{oracle}} \lesssim \frac{d (\log T)^2 \log(1/\delta)^2 }{T} + (\log T) \eps^2. \]
The rest of the proof still follows because the $\alpha_t$ satisfies \Cref{def:noise_smooth} when $t \geq 2$.
Combining with \Cref{lem:nonvanish_coef_sum_genli}, we would finally obtain
\begin{align*}
    &\KL{Q_{1|y}}{\widehat{P}_{1|y}} \lesssim (\norm{H^\dagger y - H^\dagger H \mu_0}^2 + d) \brc{1-\frac{3 \log(1/\delta) \log T}{T}} \\
    &\qquad + (\norm{H^\dagger y - H^\dagger H \mu_0}^2 + d) \frac{(\log T)^2 \log(1/\delta)^2}{T} + \sqrt{\norm{H^\dagger y - H^\dagger H \mu_0}^2 + d} \cdot (\log T) \eps.
\end{align*}
Here $\rmW_2(Q_{1|y}, Q_{0|y})^2 \lesssim \delta d$.
\end{remark}

\subsection{Proof of \texorpdfstring{\Cref{lem:norm2_bd_gauss_mix}}{Proposition 1}}
\label{app:proof_lem_norm2_bd_gauss_mix}

We first introduce some useful notations for this subsection.
Recall that $Q_{0}$ has mixture p.d.f. in which the mixture prior $\pi_n$ is independent of $y\ (= H x_0 + n)$. Thus, using the fact that $x_0 = (I_d - H^\dagger H) x_0 + H^\dagger y - H^\dagger n$, we can define $Q_{0,n|y}$ as (cf. \cite{gauss-mix-char})
\begin{align*}
    Q_{0|y} &= \sum_{n=1}^N \pi_n Q_{0,n|y} \\
    &:= \sum_{n=1}^N \pi_n \calN((I_d - H^\dagger H) \mu_{0,n} + H^\dagger y, (I_d - H^\dagger H) \Sigma_0 (I_d - H^\dagger H) + \sigma_y^2 H^\dagger (H^\dagger)^\T).
\end{align*}
Note that when $H = \begin{pmatrix} I_p & 0 \end{pmatrix}$ and $\sigma_y^2 > 0$, $q_{0|y}$ exists.
From the conditional forward model in \eqref{eq:def_cond_fwd2}, we further define
\begin{align} \label{eq:gauss_mix_all_notations}
    Q_{t} &= \sum_{n=1}^N \pi_n Q_{t,n},~~ Q_{t|y} = \sum_{n=1}^N \pi_n Q_{t,n|y},~~ Q_{t,n} := \calN(\mu_{t,n}, \Sigma_{t}), ~~ Q_{t,n|y} := \calN(\mu_{t,n|y}, \Sigma_{t|y}) \nonumber
    \\
    \mu_{t,n} &:= \sqrt{\Bar{\alpha}_t} \mu_{0,n},~~  \Sigma_t := \Bar{\alpha}_t \Sigma_{0} + (1-\Bar{\alpha}_t) I_d,~~ \mu_{t,n|y} := \sqrt{\Bar{\alpha}_t} (I_d - H^\dagger H) \mu_{0,n} + \sqrt{\Bar{\alpha}_t} H^\dagger y \nonumber \\
    \Sigma_{t|y} &:= \Sigma_{t,sig} + \Bar{\alpha}_t \sigma_y^2 H^\dagger (H^\dagger)^\T,~~ \Sigma_{t,sig} := \Bar{\alpha}_t (I_d - H^\dagger H) \Sigma_{0} (I_d - H^\dagger H) + (1-\Bar{\alpha}_t) I_d.
\end{align}
Similar to \eqref{eq:gauss_cov_inv_norm_upper}, we still have
\[ \norm{\Sigma_{t-1|y}^{-1}} \lesssim 1,\quad \forall t \geq 1. \]
We can also calculate the scores of $Q_t$ and $Q_{t|y}$ in as follows.
\begin{align} \label{eq:gauss_mix_scores}
    \nabla \log q_t(x_t) &= - \frac{1}{q_t(x_t)} \sum_{n=1}^N \pi_n q_{t,n}(x_t) \Sigma_{t}^{-1} (x_t - \mu_{t,n}), \nonumber\\
    \nabla \log q_{t|y}(x_t) &= - \frac{1}{q_{t|y}(x_t)} \sum_{n=1}^N \pi_n q_{t,n|y}(x_t) \Sigma_{t|y}^{-1} (x_t - \mu_{t,n|y}).
\end{align}

% First, it can be shown that $Q_{0|y}$ is also Gaussian mixture \cite{gauss-mix-char}:
% \[ Q_{0|y} = \sum_{n=1}^N \pi_n \calN((I_d - H^\dagger H) \mu_{0,n} + H^\dagger y, (I_d - H^\dagger H) \Sigma_0 (I_d - H^\dagger H)) =: \sum_{n=1}^N \pi_n Q_{0,n|y}. \]
% We also employ the following notations:
% \begin{align*}
%     Q_{t} &= \sum_{n=1}^N \pi_n Q_{t,n},\quad Q_{t|y} = \sum_{n=1}^N \pi_n Q_{t,n|y} \\
%     %,\quad \mu_0 = \sum_{n=1}^N \pi_n \mu_{0,n},\quad \mu_t = \sqrt{\alpha_t} \mu_0\\
%     Q_{t,n} &= \calN(\sqrt{\Bar{\alpha}_t} \mu_{0,n}, \Bar{\alpha}_t \Sigma_{0} + (1-\Bar{\alpha}_t) I_d) =: \calN(\mu_{t,n}, \Sigma_{t})\\
%     Q_{t,n|y} &= \calN( \sqrt{\Bar{\alpha}_t} (I_d - H^\dagger H) \mu_{0,n} + \sqrt{\Bar{\alpha}_t} H^\dagger y, \Bar{\alpha}_t (I_d - H^\dagger H) \Sigma_{0} (I_d - H^\dagger H) + (1-\Bar{\alpha}_t) I_d) \\
%     &\qquad =: \calN(\mu_{t,n|y}, \Sigma_{t|y}).
% \end{align*}

Now, with $f_{t,y} = f_{t,y}^*$ (in \eqref{eq:def_fy_star}), from the expression of $\Delta_{t,y}$ in \eqref{eq:delta_using_fy}, under the assumption $H = \begin{pmatrix} I_p & 0 \end{pmatrix}$, the score mismatch at each diffusion step is equal to
\begin{align} \label{eq:delta_gauss_mix}
    \Delta_{t,y} &= (I_d-H^\dagger H) (\nabla \log q_{t|y} - \nabla \log q_t) \nonumber\\
    &= \frac{1}{q_t(x_t)} \sum_{n=1}^N \pi_n q_{t,n}(x_t) (I_d-H^\dagger H) \Sigma_{t}^{-1} (x_t - \sqrt{\Bar{\alpha}_t} \mu_{0,n}) \nonumber \\
    &\quad - \frac{1}{q_{t|y}(x_t)} \sum_{n=1}^N \pi_n q_{t,n|y}(x_t) (I_d-H^\dagger H) \Sigma_{t|y}^{-1} (x_t - \sqrt{\Bar{\alpha}_t} \mu_{0,n|y}) \nonumber\\
    &= \sum_{n=1}^N \pi_n \brc{ \frac{q_{t,n}(x_t)}{q_t(x_t)} - \frac{q_{t,n|y}(x_t)}{q_{t|y}(x_t)}} (I_d-H^\dagger H) \Sigma_{t}^{-1} (x_t - \sqrt{\Bar{\alpha}_t} \mu_{0,n}) \nonumber\\
    &\quad + \sum_{n=1}^N \pi_n \frac{q_{t,n|y}(x_t)}{q_{t|y}(x_t)}  (I_d-H^\dagger H) \brc{\Sigma_{t}^{-1} (x_t - \sqrt{\Bar{\alpha}_t} \mu_{0,n}) - \Sigma_{t|y}^{-1} (x_t - \sqrt{\Bar{\alpha}_t} \mu_{0,n|y})} \nonumber\\
    &= - \sqrt{\Bar{\alpha}_t} \sum_{n=1}^N \pi_n \brc{ \frac{q_{t,n}(x_t)}{q_t(x_t)} - \frac{q_{t,n|y}(x_t)}{q_{t|y}(x_t)}} (I_d-H^\dagger H) \Sigma_{t}^{-1}  \mu_{0,n} \nonumber\\
    &\quad + \sum_{n=1}^N \pi_n \frac{q_{t,n|y}(x_t)}{q_{t|y}(x_t)}  (I_d-H^\dagger H) \brc{\Sigma_{t}^{-1} (x_t - \sqrt{\Bar{\alpha}_t} \mu_{0,n}) - \Sigma_{t|y}^{-1} (x_t - \sqrt{\Bar{\alpha}_t} \mu_{0,n|y})} \nonumber\\
    &\stackrel{(i)}{=} - \sqrt{\Bar{\alpha}_t} \sum_{n=1}^N \pi_n \brc{ \frac{q_{t,n}(x_t)}{q_t(x_t)} - \frac{q_{t,n|y}(x_t)}{q_{t|y}(x_t)}} (I_d-H^\dagger H) \Sigma_{t}^{-1}  \mu_{0,n} \nonumber\\
    &\quad - \Bar{\alpha}_t \sum_{n=1}^N \pi_n \frac{q_{t,n|y}(x_t)}{q_{t|y}(x_t)} A_t \Sigma_{0} (H^\dagger H) \Sigma_t^{-1} (x_t - \sqrt{\Bar{\alpha}_t} \mu_{0,n})
\end{align}
where $A_t := (I_d - H^\dagger H) \Sigma_{t,sig}^{-1} (I_d - H^\dagger H)$. Here $(i)$ follows from similar arguments as in \eqref{eq:delta_y_gauss}. Note that since $H = \begin{pmatrix} I_p & 0 \end{pmatrix}$, we have equivalently $A_t = (I_d - H^\dagger H) \Sigma_{t|y}^{-1} (I_d - H^\dagger H)$.
Since $H^\dagger H = \begin{pmatrix}
I_p & 0 \\
0 & 0 
\end{pmatrix}$, we can also re-express the second term in $\Delta_{t,y}$ such that $[\Delta_{t,y}]_{y} = 0$ and
\begin{align} \label{eq:delta_gauss_mix_type2}
    [\Delta_{t,y}]_{\Bar{y}} 
    % &= - \sqrt{\Bar{\alpha}_t} \sum_{n=1}^N \pi_n \brc{ \frac{q_{t,n}(x_t)}{q_t(x_t)} - \frac{q_{t,n|y}(x_t)}{q_{t|y}(x_t)}} (I_d-H^\dagger H) \Sigma_{t}^{-1}  \mu_{0,n} \nonumber\\
    % &\quad - \Bar{\alpha}_t \sum_{n=1}^N \pi_n \frac{q_{t,n|y}(x_t)}{q_{t|y}(x_t)} A_t \Sigma_{0} (H^\dagger H) \Sigma_t^{-1} (x_t - \sqrt{\Bar{\alpha}_t} \mu_{0,n}) \nonumber\\
    &= - \sqrt{\Bar{\alpha}_t} \sum_{n=1}^N \pi_n \brc{ \frac{q_{t,n}(x_t)}{q_t(x_t)} - \frac{q_{t,n|y}(x_t)}{q_{t|y}(x_t)}} [\Sigma_{t}^{-1}]_{\Bar{y} :} \mu_{0,n} \nonumber \\
    &\quad - \Bar{\alpha}_t \sum_{n=1}^N \pi_n \frac{q_{t,n|y}(x_t)}{q_{t|y}(x_t)}  (\Bar{\alpha}_t [\Sigma_{0}]_{\Bar{y} \Bar{y}} + (1-\Bar{\alpha}_t) I_{d-p})^{-1} [\Sigma_0]_{\Bar{y} y} [\Sigma_t^{-1}]_{y:} (x_t - \sqrt{\Bar{\alpha}_t} \mu_{0,n})
\end{align}
since when $H^\dagger H = \begin{pmatrix} I_p & 0 \\ 0 & 0 \end{pmatrix}$, $A_t = \begin{pmatrix} 0 & 0 \\ 0 & (\Bar{\alpha}_t [\Sigma_{0}]_{\Bar{y} \Bar{y}} + (1-\Bar{\alpha}_t) I_{d-p})^{-1} \end{pmatrix}$.

Now, for the second moment, we follow similar analyses in \eqref{eq:gauss_mismatch_2norm_bound} and get
\begin{align*}
    &\E_{X_t \sim Q_{t|y}} \norm{\Delta_{t,y}}^2 \\
    &\leq 4 \Bar{\alpha}_t \E_{X_t \sim Q_{t|y}} \max_{n \in [N]} \norm{\Sigma_{t}^{-1}  \mu_{0,n}}^2 \\
    &\quad + 2 \Bar{\alpha}_t^2 \norm{(\Bar{\alpha}_t [\Sigma_{0}]_{\Bar{y} \Bar{y}} + (1-\Bar{\alpha}_t) I_{d-p})^{-1}}^2 \norm{[\Sigma_t^{-1}]_{y:}}^2  \norm{[\Sigma_0]_{y \Bar{y}}}^2 \times\\
    &\qquad \quad \E_{X_t \sim Q_{t|y}} \sbrc{ \sum_{n=1}^N \pi_n \frac{q_{t,n|y}(X_t)}{q_{t|y}(X_t)} \norm{X_t - \sqrt{\Bar{\alpha}_t} \mu_{0,n}}^2 }
\end{align*}
where
\begin{align*}
    &\E_{X_t \sim Q_{t|y}} \sbrc{ \sum_{n=1}^N \pi_n \frac{q_{t,n|y}(X_t)}{q_{t|y}(X_t)} \norm{X_t - \sqrt{\Bar{\alpha}_t} \mu_{0,n}}^2 } \\
    &= \E_{X_t \sim Q_{t|y}} \E_{N \sim \Pi_{\cdot|t,y}} \norm{X_t - \sqrt{\Bar{\alpha}_t} \mu_{0,N}}^2\\
    &= \E_{N \sim \Pi_{\cdot|y}} \E_{X_t \sim Q_{t,N|y}} \norm{X_t - \sqrt{\Bar{\alpha}_t} \mu_{0,N}}^2\\
    &\stackrel{(ii)}{=} \E_{N \sim \Pi_{\cdot|y}} \sbrc{ \Tr( \Sigma_{t|y}) + \Bar{\alpha}_t \norm{H^\dagger y - H^\dagger H \mu_{0,N}}^2 }\\
    &= \Tr( \Sigma_{t|y}) + \Bar{\alpha}_t \sum_{n=1}^N \pi_n \norm{H^\dagger y - H^\dagger H \mu_{0,n}}^2
    % &\leq \max \cbrc{ \sum_{n=1}^N \pi_n \norm{H^\dagger y - H^\dagger H \mu_{0,n}}^2 + d (\lambda_1 + \sigma_y^2), d } 
\end{align*}
where $(ii)$ follows from \eqref{eq:gauss_mismatch_norm_2} and note that $Q_{t,N|y}$ is Gaussian for each $N = n$.
Denote $\lambda_1 \geq \dots \geq \lambda_{d} > 0$ and $\Tilde{\lambda}_1 \geq \dots \geq \Tilde{\lambda}_{d-p} > 0$ to be the eigenvalues of $\Sigma_{0}$ and $[\Sigma_{0}]_{\Bar{y} \Bar{y}}$, respectively. Similarly as the proof of \Cref{lem:norm2_bd_gauss}, we have $\norm{[\Sigma_0]_{\Bar{y} y}} \leq \norm{\Sigma_0} = \lambda_1$, $\norm{[\Sigma_t^{-1}]_{y:}} \leq \norm{\Sigma_{t}^{-1}} \leq (\Bar{\alpha}_t \lambda_d + (1-\Bar{\alpha}_t))^{-1} \leq \frac{1}{\min\{\lambda_d, 1\}}$, $\norm{(\Bar{\alpha}_t [\Sigma_{0}]_{\Bar{y} \Bar{y}} + (1-\Bar{\alpha}_t) I_{d-p})^{-1}} \leq \frac{1}{\Bar{\alpha}_t \Tilde{\lambda}_{d-p} + (1-\Bar{\alpha}_t)} \leq \frac{1}{\min\{\Tilde{\lambda}_{d-p}, 1\}}$, and $\norm{\Sigma_{t|y}} \leq \Bar{\alpha}_t (\lambda_1 + \sigma_y^2) + (1-\Bar{\alpha}_t)$. 
Therefore,
\begin{align*}
    &\E_{X_t \sim Q_{t|y}} \norm{\Delta_{t,y}}^2 \\ &\lesssim \Bar{\alpha}_t d + \Bar{\alpha}_t^2 \frac{\norm{[\Sigma_0]_{y \Bar{y}}}^2}{\min\{\Tilde{\lambda}_{d-p}, 1\}^2 \min\{\lambda_d, 1\}^2 } \times \\
    &\qquad \brc{d  (1-\Bar{\alpha}_t) + \Bar{\alpha}_t d (\lambda_1 + \sigma_y^2) + \Bar{\alpha}_t \sum_{n=1}^N \pi_n \norm{H^\dagger y - H^\dagger H \mu_{0,n}}^2}\\
    &\lesssim \Bar{\alpha}_t d + \Bar{\alpha}_t^2 \frac{\norm{[\Sigma_0]_{y \Bar{y}}}^2}{\min\{\Tilde{\lambda}_{d-p}, 1\}^2 \min\{\lambda_d, 1\}^2 } \max\cbrc{d (\lambda_1 + \sigma_y^2) + \sum_{n=1}^N \pi_n \norm{H^\dagger y - H^\dagger H \mu_{0,n}}^2, d }.
\end{align*}
% Since $\norm{[\Sigma_0]_{y \Bar{y}} [\Sigma_0]_{\Bar{y} y}} \leq \norm{\Sigma_0}^2 = \lambda_1^2$, we further have
% \[ \E_{X_t \sim Q_{t|y}} \norm{\Delta_{t,y}}^2 \lesssim \Bar{\alpha}_t d. \]
The proof is complete.

\subsection{Proof of \texorpdfstring{\Cref{thm:kl_bd_gauss_mix}}{Theorem 5}}
\label{app:thm5-proof}

We first recall all the notations in \eqref{eq:gauss_mix_all_notations} under Gaussian mixture.
We also recall the scores from \eqref{eq:gauss_mix_scores}:
\begin{align*}
    \nabla \log q_t(x_t) &= -\frac{1}{q_t(x_t)} \sum_{n=1}^N \pi_n q_{t,n}(x_t) \Sigma_{t}^{-1} (x_t - \mu_{t,n}) \nonumber\\
    &\qquad = - \Sigma_{t}^{-1} x_t + \frac{1}{q_t(x_t)} \sum_{n=1}^N \pi_n q_{t,n}(x_t) \Sigma_{t}^{-1} \mu_{t,n} \nonumber\\
    \nabla \log q_{t|y}(x_t) &= -\frac{1}{q_{t|y}(x_t)} \sum_{n=1}^N \pi_n q_{t,n|y}(x_t) \Sigma_{t|y}^{-1} (x_t - \mu_{t,n|y}) \nonumber\\
    &\qquad = - \Sigma_{t|y}^{-1} x_t + \frac{1}{q_{t|y}(x_t)} \sum_{n=1}^N \pi_n q_{t,n|y}(x_t) \Sigma_{t|y}^{-1} \mu_{t,n|y}
\end{align*}
Also, we recall the explicit expression of $\Delta_{t,y}$ from \eqref{eq:delta_gauss_mix_type2}, such that $[\Delta_{t,y}]_{y} = 0$ and
\begin{align*}
    [\Delta_{t,y}]_{\Bar{y}} &= - \sqrt{\Bar{\alpha}_t} \sum_{n=1}^N \pi_n \brc{ \frac{q_{t,n}(x_t)}{q_t(x_t)} - \frac{q_{t,n|y}(x_t)}{q_{t|y}(x_t)}}[\Sigma_{t}^{-1}]_{\Bar{y}:}  \mu_{0,n} \nonumber \\
    &\quad - \Bar{\alpha}_t \sum_{n=1}^N \pi_n \frac{q_{t,n|y}(x_t)}{q_{t|y}(x_t)}  (\Bar{\alpha}_t [\Sigma_{0}]_{\Bar{y} \Bar{y}} + (1-\Bar{\alpha}_t) I_{d-p})^{-1} [\Sigma_0]_{\Bar{y} y} [\Sigma_t^{-1}]_{y:} (x_t - \sqrt{\Bar{\alpha}_t} \mu_{0,n}).
\end{align*}

In order to invoke \Cref{thm:main-general}, we need to check \Cref{ass:regular-drv-general,ass:bdd-mismatch-general}. From \cite[Lemmas 13 and 14]{liang2024discrete}, since $\norm{\Sigma_{t|y}^{-1}} \lesssim 1$ for all $t \geq 0$, the absolute values of any-order partial derivative are bounded by $O(1)$ in expectation, and thus \Cref{ass:regular-drv-general} is satisfied. 
The following lemma verifies \Cref{ass:bdd-mismatch-general} using the $\alpha_t$ in \Cref{def:noise_smooth}.

% Also, \Cref{ass:bdd-mismatch-general} is satisfied using the $\alpha_t$ in \Cref{def:noise_smooth} according to \Cref{lem:norm2_bd_gauss_mix}.

\begin{lemma} \label{lem:bdd-mismatch-gauss-mix}
Under the same condition of \Cref{thm:kl_bd_gauss_mix}, \Cref{ass:bdd-mismatch-general} holds if the $\alpha_t$ satisfies \Cref{def:noise_smooth}.
\end{lemma}
\begin{proof}
See \Cref{app:proof-lem-bdd-mismatch-gauss-mix}.
\end{proof}

% In the following lemma, we characterize the asymptotic bias and check \Cref{ass:bdd-mismatch-general} with the above $\Delta_{t,y}$ for Gaussian mixture $Q_0$, according to the $\alpha_t$ in \Cref{def:noise_smooth}.

% \begin{lemma} \label{lem:norm2_bd_gauss_mix}
% With the $\alpha_t$ according to \Cref{def:noise_smooth}, for the distribution in \Cref{thm:kl_bd_gauss_mix},
% \Cref{ass:bdd-mismatch-general} is satisfied, and
% \begin{align*}
%     &\E_{X_t \sim Q_{t|y}} \norm{\Delta_{t,y}(X_t)}^2 \\
%     &\lesssim \Bar{\alpha}_t d + \Bar{\alpha}_t^2 \frac{\norm{[\Sigma_0]_{y \Bar{y}}}^2}{\min\{\Tilde{\lambda}_{d-p}, 1\}^2 \min\{\lambda_d, 1\}^2 } \max\cbrc{d (\lambda_1 + \sigma_y^2) + \sum_{n=1}^N \pi_n \norm{H^\dagger y - H^\dagger H \mu_{0,n}}^2, d }\\
%     &\lesssim \Bar{\alpha}_t \brc{d + \sum_{n=1}^N \pi_n \norm{H^\dagger y - H^\dagger H \mu_{0,n}}^2}.
% \end{align*}
% where $\lambda_1$ is the largest eigenvalue of $\Sigma_0$, and $\lambda_d$ and $\Tilde{\lambda}_{d-p}$ are the smallest eigenvalues of $\Sigma_0$ and $[\Sigma_0]_{\Bar{y}\Bar{y}}$, respectively.
% \end{lemma}

% \begin{proof}
%     See \Cref{app:proof_lem_norm2_bd_gauss_mix}.
% \end{proof}

Now we start to upper-bound the conditional KL-divergence of interest. 
% We use the $\alpha_t$ in \eqref{eq:alpha_genli}. 
Recall that from \Cref{thm:main-general}, $\KL{Q_{0|y}}{\widehat{P}_{0|y}} \lesssim \mathcal{W}_{\text{oracle}} + \mathcal{W}_{\text{bias}} + \mathcal{W}_{\text{vanish}}$,
where
\begin{align*}
\mathcal{W}_{\text{oracle}} &= \sum_{t=1}^T \frac{(1-\alpha_t)^2}{2 \alpha_t} \E_{X_t \sim Q_{t|y}} \sbrc{ \Tr\Big(\nabla^2 \log q_{t-1|y}(m_{t,y}(X_t))\nabla^2\log q_{t|y}(X_t)  \Big) } + (\log T) \eps^2\\
\mathcal{W}_{\text{bias}} &= \sum_{t=1}^T (1-\alpha_t) \E_{X_t \sim Q_{t|y}} \norm{\Delta_{t,y}(X_t)}^2\\
\mathcal{W}_{\text{vanish}} &= \sum_{t=1}^T \frac{1-\alpha_t}{\sqrt{\alpha_t}} \E_{X_t \sim Q_{t|y}} \bigg[(\nabla \log q_{t-1|y}(m_{t,y}(X_t)) - \sqrt{\alpha_t} \nabla \log q_{t|y}(X_t))^\T \Delta_{t,y}(X_t)\bigg]\\
    &- \sum_{t=1}^T \frac{(1-\alpha_t)^2}{2 \alpha_t} \E_{X_t \sim Q_{t|y}}\sbrc{\Delta_{t,y}(X_t)^\T \nabla^2 \log q_{t-1|y}(m_{t,y}(X_t)) \Delta_{t,y}(X_t) } \\
    &+ \sum_{t=1}^T \frac{(1-\alpha_t)^2}{3! {\alpha_t}^{3/2}} \E_{X_t \sim Q_{t|y}} \bigg[ 3 \sum_{i=1}^d \partial^3_{iii} \log q_{t-1|y}(m_{t,y}(X_t)) \Delta_{t,y}(X_t)^i \\
    &\qquad + \sum_{\substack{i,j=1 \\ i\neq j}}^d \partial^3_{iij} \log q_{t-1|y}(m_{t,y}(X_t)) \Delta_{t,y}(X_t)^j \bigg]\\
    &+ \max_{t \geq 1}\sqrt{\E_{X_t \sim Q_{t|y}} \norm{\Delta_{t,y}(X_t)}^2} (\log T) \eps.
\end{align*}

From \cite[Theorem~2]{liang2024discrete} (and by assumption $N \leq d$), if the $\alpha_t$ satisfies \Cref{def:noise_smooth},
% with the $\alpha_t$ in \eqref{eq:alpha_genli},
\[ \mathcal{W}_{\text{oracle}} \lesssim \frac{d^2 (\log T)^2}{T} + (\log T) \eps^2. \]
% \yuchen{If $\sigma_y^2 > 0$, then the same inequality will hold for $\KL{Q_{0|y}}{\widehat{P}_{0|y}}$.}

Also, from \Cref{lem:norm2_bd_gauss_mix}, under the assumption on $\alpha_t$, 
% \Cref{lem:nonvanish_coef_sum_genli,lem:norm2_bd_gauss_mix}, with the $\alpha_t$ in \eqref{eq:alpha_genli} and when $\delta \ll 1$ and $c \asymp \log(1/\delta)$, 
$\mathcal{W}_{\text{bias}}$ can be upper-bounded as
% \[ \mathcal{W}_{\text{bias}} \lesssim  \brc{d + \sum_{n=1}^N \pi_n \norm{H^\dagger y - H^\dagger H \mu_{0,n}}^2} \brc{1 - \frac{2 \log(1/\delta) \log T}{T}}. \]
\[ \mathcal{W}_{\text{bias}} \lesssim d + \sum_{n=1}^N \pi_n \norm{H^\dagger y - H^\dagger H \mu_{0,n}}^2. \]
Among the terms in $\mathcal{W}_{\text{vanish}}$, the last estimation error term can be upper-bounded using \Cref{lem:norm2_bd_gauss_mix} as
\[ \max_{t \geq 1}\sqrt{\E_{X_t \sim Q_{t|y}} \norm{\Delta_{t,y}(X_t)}^2} (\log T) \eps \lesssim \sqrt{d + \sum_{n=1}^N \pi_n \norm{H^\dagger y - H^\dagger H \mu_{0,n}}^2} (\log T) \eps. \]

It remains to analyze the rest of the terms in $\mathcal{W}_{\text{vanish}}$. 
% We now begin analyzing the four terms in $\mathcal{W}_{\text{vanish}}$. 
In the following we fix $t \geq 1$.
We remind readers of the notations in \eqref{eq:gauss_mix_all_notations}.
For the first term in $\mathcal{W}_{\text{vanish}}$, we first provide the following useful calculations. Note that by exchanging the order of expectation, for any function $\mathrm{fn}$ we have
\[ \E_{X_t \sim Q_{t|y}} \sbrc{ \sum_{n=1}^N \pi_n \frac{q_{t,n|y}(X_t)}{q_{t|y}(X_t)} \mathrm{fn}(X_t, n)} = \sum_{n=1}^N \pi_n \E_{X_t \sim Q_{t,n|y}} \mathrm{fn}(X_t, n). \]
Thus,
\begin{align} \label{eq:gauss_mix_temp1}
    &\E_{X_t \sim Q_{t|y}} \sum_{n=1}^N \pi_n \frac{q_{t,n|y}(x_t)}{q_{t|y}(x_t)} \abs{(X_t - \sqrt{\Bar{\alpha}_t} \mu_{0,n})^\T X_t} \nonumber \\
    &= \sum_{n=1}^N \pi_n \E_{X_t \sim Q_{t,n|y}} \abs{(X_t - \sqrt{\Bar{\alpha}_t} \mu_{0,n})^\T X_t} \nonumber \\
    &\leq \sum_{n=1}^N \pi_n \E_{X_t \sim Q_{t,n|y}} \norm{X_t - \sqrt{\Bar{\alpha}_t} \mu_{0,n}}^2 + \sqrt{\Bar{\alpha}_t} \sum_{n=1}^N \pi_n \norm{\mu_{0,n}} \sqrt{\E_{X_t \sim Q_{t,n|y}} \norm{X_t - \sqrt{\Bar{\alpha}_t} \mu_{0,n}}^2 } \nonumber \\
    &\stackrel{(i)}{=} \Tr( \Sigma_{t|y}) + \Bar{\alpha}_t \sum_{n=1}^N \pi_n \norm{H^\dagger y - H^\dagger H \mu_{0,n}}^2 \nonumber \\
    &\qquad + \sqrt{\Bar{\alpha}_t} \sum_{n=1}^N \pi_n \norm{\mu_{0,n}} \sqrt{\Tr( \Sigma_{t|y}) + \Bar{\alpha}_t \sum_{n=1}^N \pi_n \norm{H^\dagger y - H^\dagger H \mu_{0,n}}^2 } \nonumber \\
    &\lesssim d + \sum_{n=1}^N \pi_n \norm{H^\dagger y - H^\dagger H \mu_{0,n}}^2 
\end{align}
where $(i)$ follows from \eqref{eq:gauss_mismatch_norm_2}.

Also, note that $(I_d - H^\dagger H) \Sigma_{t-1|y}^{-r} (H^\dagger H) = 0$ using the following simple induction argument. For the base case, we have $(I_d - H^\dagger H) \Sigma_{t-1|y}^{-1} (H^\dagger H) = 0$ from \Cref{lem:gauss_cond_var_proj}. Then, suppose $(I_d - H^\dagger H) \Sigma_{t-1|y}^{-(r-1)} (H^\dagger H) = 0$, we have $(I_d - H^\dagger H) \Sigma_{t-1|y}^{-r} (H^\dagger H) = (I_d - H^\dagger H) \Sigma_{t-1|y}^{-(r-1)} (I_d - H^\dagger H + H^\dagger H) \Sigma_{t-1|y}^{-1} (H^\dagger H) = (I_d - H^\dagger H) \Sigma_{t-1|y}^{-(r-1)} (I_d - H^\dagger H) \Sigma_{t-1|y}^{-1} (H^\dagger H) + (I_d - H^\dagger H) \Sigma_{t-1|y}^{-(r-1)} (H^\dagger H) \Sigma_{t-1|y}^{-1} (H^\dagger H) = 0$.
Thus, for all $r \geq 1$ and any fixed vector $v$, with the definition of $\Delta_{t,y}$ in \eqref{eq:delta_gauss_mix} and \eqref{eq:delta_gauss_mix_type2},
\begin{align} \label{eq:gauss_mix_temp3}
    &\E_{X_t \sim Q_{t|y}} \abs{ (\Sigma_{t-1|y}^{-r} v)^\T \Delta_{t,y}} \nonumber \\
    &\leq \norm{v} \E_{X_t \sim Q_{t|y}} \norm{ \Sigma_{t-1|y}^{-r} \Delta_{t,y}} = \norm{v} \E_{X_t \sim Q_{t|y}} \norm{ A_{t-1}^r \Delta_{t,y}} \nonumber \\
    &\leq 4 \norm{v} \norm{ (\Bar{\alpha}_{t-1} [\Sigma_{0}]_{\Bar{y} \Bar{y}} + (1-\Bar{\alpha}_{t-1}) I_{d-p})^{-r} } \max_{n \in [N]} \norm{[\Sigma_{t}^{-1}]_{\Bar{y}:} \mu_{0,n}} \nonumber \\
    &\quad + 2 \norm{v} \norm{ (\Bar{\alpha}_{t-1} [\Sigma_{0}]_{\Bar{y} \Bar{y}} + (1-\Bar{\alpha}_{t-1}) I_{d-p})^{-r} } \norm{(\Bar{\alpha}_t [\Sigma_{0}]_{\Bar{y} \Bar{y}} + (1-\Bar{\alpha}_t) I_{d-p})^{-1} [\Sigma_0]_{\Bar{y} y} [\Sigma_t^{-1}]_{y:} } \times \nonumber\\
    &\qquad \sqrt{ \E_{X_t \sim Q_{t|y}} \sum_{n=1}^N \pi_n \frac{q_{t,n|y}(X_t)}{q_{t|y}(X_t)} \norm{X_t - \sqrt{\Bar{\alpha}_t} \mu_{0,n}}^2 } \nonumber\\
    &\lesssim d + \sum_{n=1}^N \pi_n \norm{H^\dagger y - H^\dagger H \mu_{0,n}}^2
\end{align}
where the last line follows from \eqref{eq:gauss_mismatch_norm_2}. Similarly,
\[ \E_{X_t \sim Q_{t|y}} \abs{ (\Sigma_{t|y}^{-r} v)^\T \Delta_{t,y}} \lesssim d + \sum_{n=1}^N \pi_n \norm{H^\dagger y - H^\dagger H \mu_{0,n}}^2. \]

Also, for all $r \geq 1$,
\begin{align} \label{eq:gauss_mix_temp4}
    &\E_{X_t \sim Q_{t|y}} \abs{ (\Sigma_{t-1|y}^{-r} X_t)^\T \Delta_{t,y}} = \E_{X_t \sim Q_{t|y}} \abs{ X_t^\T \Sigma_{t-1|y}^{-r} (I_d - H^\dagger H) \Delta_{t,y}} \nonumber\\
    &= \E_{X_t \sim Q_{t|y}} \abs{ X_t^\T A_{t-1}^r \Delta_{t,y}} \nonumber\\
    &\leq 4 \norm{(\Bar{\alpha}_{t-1} [\Sigma_{0}]_{\Bar{y} \Bar{y}} + (1-\Bar{\alpha}_{t-1}) I_{d-p})^{-r}} \max_{n \in [N]} \norm{[\Sigma_{t}^{-1}]_{\Bar{y}:} \mu_{0,n}} \E_{X_t \sim Q_{t|y}} \norm{X_t} \nonumber\\
    &\quad + 2 \norm{(\Bar{\alpha}_{t-1} [\Sigma_{0}]_{\Bar{y} \Bar{y}} + (1-\Bar{\alpha}_{t-1}) I_{d-p})^{-r}} \norm{(\Bar{\alpha}_t [\Sigma_{0}]_{\Bar{y} \Bar{y}} + (1-\Bar{\alpha}_t) I_{d-p})^{-1} [\Sigma_0]_{\Bar{y} y} [\Sigma_t^{-1}]_{y:}} \times \nonumber\\
    &\qquad \E_{X_t \sim Q_{t|y}} \sum_{n=1}^N \pi_n \frac{q_{t,n|y}(x_t)}{q_{t|y}(x_t)} \abs{(X_t - \sqrt{\Bar{\alpha}_t} \mu_{0,n})^\T X_t} \nonumber\\
    &\lesssim d + \sum_{n=1}^N \pi_n \norm{H^\dagger y - H^\dagger H \mu_{0,n}}^2
\end{align}
where the last line follows from \eqref{eq:gauss_mix_temp1} and the fact that, from \eqref{eq:gauss_mismatch_norm_2},
\begin{align*}
    \sqrt{d} \cdot \E_{X_t \sim Q_{t|y}} \norm{X_t} &\leq \sqrt{d} \sum_{n=1}^N \pi_n \sqrt{2 \E_{X_t \sim Q_{t,n|y}} \norm{X_t - \mu_{t,n}}^2 + 2 \norm{\mu_{t,n}}^2 } \\
    &\lesssim \sqrt{d} \sum_{n=1}^N \pi_n \sqrt{ \Tr( \Sigma_{t|y}) + \Bar{\alpha}_t \norm{H^\dagger y - H^\dagger H \mu_{0,n}}^2 + d}\\
    &\lesssim d + \sqrt{d} \sum_{n=1}^N \pi_n \norm{H^\dagger y - H^\dagger H \mu_{0,n}} \\
    &\lesssim d + \sum_{n=1}^N \pi_n \norm{H^\dagger y - H^\dagger H \mu_{0,n}}^2.
\end{align*}
Similarly, we also have $\E_{X_t \sim Q_{t|y}} \abs{ (\Sigma_{t|y}^{-r} X_t)^\T \Delta_{t,y}} \lesssim d + \sum_{n=1}^N \pi_n \norm{H^\dagger y - H^\dagger H \mu_{0,n}}^2$.

Also, for all $r \geq 1$, using the expression of $\nabla \log q_{t|y}$ in \eqref{eq:gauss_mix_scores}, and noting that by definition $(I_d - H^\dagger H) (x_t - \mu_{t,n|y}) = (I_d - H^\dagger H) (x_t - \mu_{t,n})$, we have
\begin{align} \label{eq:gauss_mix_temp5}
    &\E_{X_t \sim Q_{t|y}} \abs{ (\Sigma_{t-1|y}^{-r} \nabla \log q_{t|y}(X_t))^\T \Delta_{t,y} } \nonumber \\
    &= \E_{X_t \sim Q_{t|y}} \abs{ \brc{A_{t-1}^r \sum_{n=1}^N \pi_n \frac{q_{t,n|y}(X_t)}{q_{t|y}(X_t)} A_t (I_d - H^\dagger H) (X_t - \mu_{t,n|y}) }^\T \Delta_{t,y} } \nonumber \\
    &\leq 4 \norm{(\Bar{\alpha}_{t-1} [\Sigma_{0}]_{\Bar{y} \Bar{y}} + (1-\Bar{\alpha}_{t-1}) I_{d-p})^{-r}} \norm{(\Bar{\alpha}_t [\Sigma_{0}]_{\Bar{y} \Bar{y}} + (1-\Bar{\alpha}_t) I_{d-p})^{-1}} \times \nonumber \\
    &\quad \E_{X_t \sim Q_{t|y}} \norm{ \sum_{\ell=1}^N \pi_\ell \frac{q_{t,\ell|y}(X_t)}{q_{t|y}(X_t)} [X_t - \sqrt{\Bar{\alpha}_t} \mu_{0,\ell}]_{\Bar{y}} } \times \max_{n \in [N]} \norm{ [\Sigma_{t}^{-1}]_{\Bar{y}:}  \mu_{0,n} } \nonumber \\
    &\quad + 2 \norm{(\Bar{\alpha}_{t-1} [\Sigma_{0}]_{\Bar{y} \Bar{y}} + (1-\Bar{\alpha}_{t-1}) I_{d-p})^{-r}} \norm{(\Bar{\alpha}_t [\Sigma_{0}]_{\Bar{y} \Bar{y}} + (1-\Bar{\alpha}_t) I_{d-p})^{-1}}^2 \times \nonumber \\
    &\quad \E_{\substack{X_t \sim Q_{t|y} \\ N,L \sim \Pi_{\cdot|t,y}(\cdot|X_t)}} \bigg[ [X_t - \sqrt{\Bar{\alpha}_t} \mu_{0,L}]_{\Bar{y}}^\T [\Sigma_0]_{\Bar{y} y} [\Sigma_t^{-1}]_{y:} (X_t - \sqrt{\Bar{\alpha}_t} \mu_{0,N}) \bigg] \nonumber \\
    &\lesssim \sqrt{\sum_{n=1}^N \pi_n \E_{X_t \sim Q_{t,n|y}} \norm{X_t - \sqrt{\Bar{\alpha}_t} \mu_{0,n}}^2 } \times \sqrt{d} \nonumber \\
    &\quad + \E_{\substack{X_t \sim Q_{t|y} \\ N,L \sim \Pi_{\cdot|t,y}(\cdot|X_t)}} \norm{X_t - \sqrt{\Bar{\alpha}_t} \mu_{0,L}} \norm{X_t - \sqrt{\Bar{\alpha}_t} \mu_{0,N}}  \nonumber \\
    % &\lesssim \sqrt{\sum_{n=1}^N \pi_n \E_{X_t \sim Q_{t,n|y}} \norm{X_t - \sqrt{\Bar{\alpha}_t} \mu_{0,n}}^2 } \times \sqrt{d} \nonumber \\
    % &\quad + \brc{\sum_{n=1}^N \pi_n \E_{X_t \sim Q_{t,n|y}} \norm{X_t - \sqrt{\Bar{\alpha}_t} \mu_{0,n}}^2} \nonumber \\
    &\lesssim d + \sum_{n=1}^N \pi_n \norm{H^\dagger y - H^\dagger H \mu_{0,n}}^2
\end{align}
where the last line follows because, from \eqref{eq:gauss_mismatch_norm_2},
\begin{align*}
    &\E_{\substack{X_t \sim Q_{t|y} \\ N,L \sim \Pi_{\cdot|t,y}(\cdot|X_t)}} \sbrc{ \norm{X_t - \sqrt{\Bar{\alpha}_t} \mu_{0,L}} \norm{X_t - \sqrt{\Bar{\alpha}_t} \mu_{0,N}} } \\
    &\leq \sqrt{\E_{\substack{X_t \sim Q_{t|y} \\ L \sim \Pi_{\cdot|t,y}(\cdot|X_t)}} \norm{X_t - \sqrt{\Bar{\alpha}_t} \mu_{0,L}}^2} \times \sqrt{ \E_{\substack{X_t \sim Q_{t|y} \\ N \sim \Pi_{\cdot|t,y}(\cdot|X_t)}} \norm{X_t - \sqrt{\Bar{\alpha}_t} \mu_{0,N}}^2 } \\
    &= \E_{\substack{X_t \sim Q_{t|y} \\ N \sim \Pi_{\cdot|t,y}(\cdot|X_t)}} \norm{X_t - \sqrt{\Bar{\alpha}_t} \mu_{0,N}}^2 \\
    &\lesssim d + \sum_{n=1}^N \pi_n \norm{H^\dagger y - H^\dagger H \mu_{0,n}}^2.
\end{align*}

Now, we start to analyze the first term of $\mathcal{W}_{\text{vanish}}$. Recall that $m_{t,y}(x_t) = \E_{X_{t-1} \sim Q_{t-1|t,y}} [X_{t-1}] = \frac{1}{\sqrt{\alpha_t}} x_t + \frac{1-\alpha_t}{\sqrt{\alpha_t}} \nabla \log q_{t|y}(x_t)$. Using the score expressions in \eqref{eq:gauss_mix_scores}, we can calculate that given $x_t$ (and thus $m_{t,y} = m_{t,y}(x_t)$),
\begin{align*}
    &\nabla \log q_{t-1|y}(m_{t,y}) - \sqrt{\alpha_t} \nabla \log q_{t|y}(x_t) \\
    &= \sqrt{\alpha_t} \Sigma_{t|y}^{-1} x_t - \Sigma_{t-1|y}^{-1} m_{t,y} \\
    &\qquad - \sqrt{\alpha_t} \Sigma_{t|y}^{-1} \sum_{n=1}^N \pi_n \frac{q_{t,n|y}(x_t)}{q_{t|y}(x_t)} \mu_{t,n|y} + \Sigma_{t-1|y}^{-1} \sum_{n=1}^N \pi_n \frac{q_{t-1,n|y}(m_{t,y})}{q_{t-1|y}(m_{t,y})} \mu_{t-1,n|y}\\
    &= \brc{\Sigma_{t|y}^{-1} - \Sigma_{t-1|y}^{-1} } \sqrt{\alpha_t} x_t - \frac{1-\alpha_t}{\sqrt{\alpha_t}} \Sigma_{t-1|y}^{-1}  (x_t + \nabla \log q_{t|y}(x_t)) \\
    &\qquad + (1-\alpha_t) \Sigma_{t-1|y}^{-1} \sum_{n=1}^N \pi_n \frac{q_{t-1,n|y}(m_{t,y})}{q_{t-1|y}(m_{t,y})} \mu_{t-1,n|y} \\
    &\qquad + \alpha_t \frac{\sum_{n,\ell=1}^N \pi_n \pi_\ell \brc{ q_{t-1,n|y}(m_{t,y}) q_{t,\ell|y}(x_t) \Sigma_{t-1|y}^{-1} - q_{t-1,\ell|y}(m_{t,y}) q_{t,n|y}(x_t) \Sigma_{t|y}^{-1} } \mu_{t-1,n|y}}{q_{t-1|y}(m_{t,y}) q_{t|y}(x_t)}.
\end{align*}
Here, using similar analyses as in the proof of \Cref{lem:small-grad-diff-genli}, we get
\begin{align*}
    &(m_{t,y} - \mu_{t-1,n|y}) - (x_t - \mu_{t,n|y}) = \frac{1-\sqrt{\alpha_t}}{\sqrt{\alpha_t}} x_t + \frac{1-\alpha_t}{\sqrt{\alpha_t}} \nabla \log q_{t|y}(x_t) - (1-\sqrt{\alpha_t}) \mu_{t-1,n|y} \\
    &\Sigma_{t|y}^{-1} - \Sigma_{t-1|y}^{-1} = \frac{1-\alpha_t}{\alpha_t} \Sigma_{t-1,n}^{-1} + \frac{1-\alpha_t}{\alpha_t^2} \Sigma_{t-1,n}^{-2} + O((1-\alpha_t)^2)\\
    &q_{t-1,n|y}(m_{t,y}) q_{t,\ell|y}(x_t) \Sigma_{t-1|y}^{-1} - q_{t-1,\ell|y}(m_{t,y}) q_{t,n|y}(x_t) \Sigma_{t|y}^{-1} \\
    &= q_{t-1,n|y}(m_{t,y}) q_{t,\ell|y}(x_t) (\Sigma_{t-1|y}^{-1} - \Sigma_{t|y}^{-1}) \\
    &\qquad + (q_{t-1,n|y}(m_{t,y}) q_{t,\ell|y}(x_t) - q_{t-1,\ell|y}(m_{t,y}) q_{t,n|y}(x_t)) \Sigma_{t|y}^{-1}\\
    &= q_{t-1,n|y}(m_{t,y}) q_{t,\ell|y}(x_t) (\Sigma_{t-1|y}^{-1} - \Sigma_{t|y}^{-1}) \\
    &\qquad + \bigg( \frac{1}{2} ((m_{t,y} - \mu_{t-1,\ell|y})-(x_t - \mu_{t,\ell|y}))^\T \Sigma_{t-1|y}^{-1} (m_{t,y} - \mu_{t-1,\ell|y}) \\
    &\quad \qquad + \frac{1}{2} (x_t - \mu_{t,\ell|y})^\T (\Sigma_{t-1|y}^{-1} - \Sigma_{t|y}^{-1}) (m_{t,y} - \mu_{t-1,\ell|y}) \\
    &\quad \qquad + \frac{1}{2} (x_t - \mu_{t,\ell|y})^\T \Sigma_{t|y}^{-1} ((m_{t,y} - \mu_{t-1,\ell|y}) - (x_t - \mu_{t,\ell|y})) \\
    &\quad \qquad - \frac{1}{2} ((m_{t,y} - \mu_{t-1,n|y})-(x_t - \mu_{t,n|y}))^\T \Sigma_{t-1|y}^{-1} (m_{t,y} - \mu_{t-1,n|y}) \\
    &\quad \qquad - \frac{1}{2} (x_t - \mu_{t,n|y})^\T (\Sigma_{t-1|y}^{-1} - \Sigma_{t|y}^{-1}) (m_{t,y} - \mu_{t-1,n|y}) \\
    &\quad \qquad - \frac{1}{2} (x_t - \mu_{t,n|y})^\T \Sigma_{t|y}^{-1} ((m_{t,y} - \mu_{t-1,n|y}) - (x_t - \mu_{t,n|y})) \bigg) \Sigma_{t|y}^{-1} \\
    &\qquad + O((1-\alpha_t)^2) 
\end{align*}
Thus, 
\begin{align*}
    &\abs{\E_{X_t \sim Q_{t|y}} \bigg[(\nabla \log q_{t-1|y}(m_{t,y}) - \sqrt{\alpha_t} \nabla \log q_{t|y})^\T \Delta_{t,y}\bigg]}\\
    &\leq \E_{X_t \sim Q_{t|y}} \bigg[\abs{X_t^\T (\Sigma_{t|y}^{-1} - \Sigma_{t-1|y}^{-1}) \Delta_{t,y} }\bigg] \\
    &\quad + \frac{1-\alpha_t}{\sqrt{\alpha_t}} \E_{X_t \sim Q_{t|y}} \bigg[\abs{(X_t + \nabla \log q_{t|y})^\T (\Sigma_{t-1|y}^{-1}) \Delta_{t,y} }\bigg]\\
    &\quad + (1-\alpha_t) \E_{X_t \sim Q_{t|y}} \bigg[\abs{ \brc{\sum_{n=1}^N \pi_n \frac{q_{t-1,n|y}(m_{t,y})}{q_{t-1|y}(m_{t,y})} \mu_{t-1,n|y}}^\T \Sigma_{t-1|y}^{-1} \Delta_{t,y} }\bigg]\\
    &\quad + \E_{X_t \sim Q_{t|y}} \abs{ \Delta_{t,y}^\T \sum_{n,\ell=1}^N \pi_n \pi_\ell \brc{ \frac{q_{t-1,n|y}(m_{t,y}) q_{t,\ell|y}(X_t)}{q_{t-1|y}(m_{t,y}) q_{t|y}(X_t)} \Sigma_{t-1|y}^{-1} - \frac{q_{t-1,\ell|y}(m_{t,y}) q_{t,n|y}(X_t)}{q_{t-1|y}(m_{t,y}) q_{t|y}(X_t)} \Sigma_{t|y}^{-1} } \mu_{t-1,n|y} }.
    % \\
    % % &\lesssim (1-\alpha_t) d + (1-\alpha_t) d + (1-\alpha_t) N d + (1-\alpha_t) N d^2\\
    % &\lesssim (1-\alpha_t) N \brc{d^2 + \sum_{n=1}^N \pi_n \norm{H^\dagger y - H^\dagger H \mu_{0,n}}^4 },
\end{align*}
Among the four terms above, the first term $\lesssim d + \sum_{n=1}^N \pi_n \norm{H^\dagger y - H^\dagger H \mu_{0,n}}^2$ from \eqref{eq:gauss_mix_temp4} (along with the similar result for $\Sigma_{t|y}^{-1}$), the second term $\lesssim d + \sum_{n=1}^N \pi_n \norm{H^\dagger y - H^\dagger H \mu_{0,n}}^2$ from \eqref{eq:gauss_mix_temp4} and \eqref{eq:gauss_mix_temp5}, and both the third and the fourth term $\lesssim d + \sum_{n=1}^N \pi_n \norm{H^\dagger y - H^\dagger H \mu_{0,n}}^2$ from \eqref{eq:gauss_mix_temp3} (along with the similar result for $\Sigma_{t|y}^{-1}$). Thus,
\[ \abs{\E_{X_t \sim Q_{t|y}} \bigg[(\nabla \log q_{t-1|y}(m_{t,y}) - \sqrt{\alpha_t} \nabla \log q_{t|y}(X_t))^\T \Delta_{t,y}\bigg]} \lesssim d + \sum_{n=1}^N \pi_n \norm{H^\dagger y - H^\dagger H \mu_{0,n}}^2, \]
and the first term in $\mathcal{W}_{\text{vanish}}$ satisfies that
\begin{multline*}
    \sum_{t=1}^T \frac{1-\alpha_t}{\sqrt{\alpha_t}} \E_{X_t \sim Q_{t|y}} \bigg[(\nabla \log q_{t-1|y}(m_{t,y}(X_t)) - \sqrt{\alpha_t} \nabla \log q_{t|y}(X_t))^\T \Delta_{t,y}(X_t)\bigg] \\
    \lesssim \brc{d + \sum_{n=1}^N \pi_n \norm{H^\dagger y - H^\dagger H \mu_{0,n}}^2} \frac{\log(1/\delta)^2 (\log T)^2}{T}.
\end{multline*}

For the second term in $\mathcal{W}_{\text{vanish}}$, we first provide the following useful calculation.
Similar to \eqref{eq:gauss_mix_temp3}, for all $r \geq 1$ and any fixed vector $v$,
\begin{align} \label{eq:gauss_mix_temp6}
    &\E_{X_t \sim Q_{t|y}} \abs{ (\Sigma_{t|y}^{-r} v)^\T \Delta_{t,y}}^2 \nonumber \\
    &\leq \norm{v}^2 \E_{X_t \sim Q_{t|y}} \norm{ \Sigma_{t|y}^{-r} \Delta_{t,y}}^2 = \norm{v}^2 \E_{X_t \sim Q_{t|y}} \norm{ A_{t}^r \Delta_{t,y}}^2 \nonumber \\
    &\lesssim \norm{v}^2 \norm{ (\Bar{\alpha}_{t} [\Sigma_{0}]_{\Bar{y} \Bar{y}} + (1-\Bar{\alpha}_{t}) I_{d-p})^{-r} }^2 \max_{n \in [N]} \norm{[\Sigma_{t}^{-1}]_{\Bar{y}:} \mu_{0,n}}^2 \nonumber \\
    &\quad + \norm{v}^2 \norm{ (\Bar{\alpha}_{t} [\Sigma_{0}]_{\Bar{y} \Bar{y}} + (1-\Bar{\alpha}_{t}) I_{d-p})^{-r} }^2 \norm{(\Bar{\alpha}_t [\Sigma_{0}]_{\Bar{y} \Bar{y}} + (1-\Bar{\alpha}_t) I_{d-p})^{-1} [\Sigma_0]_{\Bar{y} y} [\Sigma_t^{-1}]_{y:} }^2 \times \nonumber\\
    &\qquad \E_{X_t \sim Q_{t|y}} \sum_{n=1}^N \pi_n \frac{q_{t,n|y}(X_t)}{q_{t|y}(X_t)} \norm{X_t - \sqrt{\Bar{\alpha}_t} \mu_{0,n}}^2 \nonumber\\
    &\lesssim d^2 + d \sum_{n=1}^N \pi_n \norm{H^\dagger y - H^\dagger H \mu_{0,n}}^2
\end{align}
where the last line follows from \eqref{eq:gauss_mismatch_norm_2}.

Also, similar to \eqref{eq:gauss_mix_temp4}, for all $r \geq 1$,
\begin{align} \label{eq:gauss_mix_temp7}
    &\E_{X_t \sim Q_{t|y}} \abs{ (\Sigma_{t|y}^{-r} m_{t,y})^\T \Delta_{t,y}}^2 \nonumber \\
    &\lesssim \E_{X_t \sim Q_{t|y}} \abs{X_t^\T \Sigma_{t|y}^{-r} \Delta_{t,y}}^2 + (1-\alpha_t) \E_{X_t \sim Q_{t|y}} \abs{(\nabla \log q_{t|y}(X_t))^\T \Sigma_{t|y}^{-r} \Delta_{t,y}}^2 \nonumber \\
    &\stackrel{(ii)}{\lesssim} \E_{X_t \sim Q_{t|y}} \abs{ X_t^\T \Sigma_{t|y}^{-r} (I_d - H^\dagger H) \Delta_{t,y}}^2 = \E_{X_t \sim Q_{t|y}} \abs{ X_t^\T A_{t}^r \Delta_{t,y}}^2 \nonumber\\
    &\lesssim \norm{(\Bar{\alpha}_{t} [\Sigma_{0}]_{\Bar{y} \Bar{y}} + (1-\Bar{\alpha}_{t}) I_{d-p})^{-r}}^2 \max_{n \in [N]} \norm{[\Sigma_{t}^{-1}]_{\Bar{y}:} \mu_{0,n}}^2 \E_{X_t \sim Q_{t|y}} \norm{X_t}^2 \nonumber\\
    &\quad + \norm{(\Bar{\alpha}_{t} [\Sigma_{0}]_{\Bar{y} \Bar{y}} + (1-\Bar{\alpha}_{t}) I_{d-p})^{-r}}^2 \norm{(\Bar{\alpha}_t [\Sigma_{0}]_{\Bar{y} \Bar{y}} + (1-\Bar{\alpha}_t) I_{d-p})^{-1} [\Sigma_0]_{\Bar{y} y} [\Sigma_t^{-1}]_{y:}}^2 \times \nonumber\\
    &\qquad \E_{X_t \sim Q_{t|y}} \sum_{n=1}^N \pi_n \frac{q_{t,n|y}(X_t)}{q_{t|y}(X_t)} \abs{(X_t - \sqrt{\Bar{\alpha}_t} \mu_{0,n})^\T X_t}^2 \nonumber\\
    &\lesssim d^2 + \sum_{n=1}^N \pi_n \norm{H^\dagger y - H^\dagger H \mu_{0,n}}^4
\end{align}
where $(ii)$ follows from the fact that $\E_{X_t \sim Q_{t|y}} \abs{(\nabla \log q_{t|y}(X_t))^\T \Sigma_{t|y}^{-r} \Delta_{t,y}}^2 \lesssim d^2$ (using a similar argument for deriving \eqref{eq:gauss_mix_temp5}), and the last line follows because
\begin{align*}
    d \cdot \E_{X_t \sim Q_{t|y}} \norm{X_t}^2 &\leq d \sum_{n=1}^N \pi_n \brc{ 2 \E_{X_t \sim Q_{t,n|y}} \norm{X_t - \mu_{t,n}}^2 + 2 \norm{\mu_{t,n}}^2 } \\
    &\stackrel{(iii)}{\lesssim} d \sum_{n=1}^N \pi_n \brc{ \Tr( \Sigma_{t|y}) + \Bar{\alpha}_t \norm{H^\dagger y - H^\dagger H \mu_{0,n}}^2 + d}\\
    &\lesssim d^2 + d \sum_{n=1}^N \pi_n \norm{H^\dagger y - H^\dagger H \mu_{0,n}}^2
\end{align*}
where $(iii)$ follows from \eqref{eq:gauss_mismatch_norm_2}, and also
\begin{align*}
    &\E_{X_t \sim Q_{t|y}} \sum_{n=1}^N \pi_n \frac{q_{t,n|y}(X_t)}{q_{t|y}(X_t)} \abs{(X_t - \sqrt{\Bar{\alpha}_t} \mu_{0,n})^\T X_t}^2 \\
    &= \E_{\substack{X_t \sim Q_{t|y} \\ N \sim \Pi_{\cdot|t,y}(\cdot|X_t)}} \brc{(X_t - \sqrt{\Bar{\alpha}_t} \mu_{0,N})^\T X_t}^2 \\
    &\leq 2 \E_{\substack{X_t \sim Q_{t|y} \\ N \sim \Pi_{\cdot|t,y}(\cdot|X_t)}} \norm{X_t - \sqrt{\Bar{\alpha}_t} \mu_{0,N}}^4 + 2 \sqrt{\E_{\substack{X_t \sim Q_{t|y} \\ N \sim \Pi_{\cdot|t,y}(\cdot|X_t)}} \norm{X_t - \sqrt{\Bar{\alpha}_t} \mu_{0,N}}^4 } \sqrt{\E_{N \sim \Pi} \norm{\mu_{0,N}}^4 } \\
    &\lesssim d^2 + \sum_{n=1}^N \pi_n \norm{H^\dagger y - H^\dagger H \mu_{0,n}}^4
\end{align*}
where the last line above follows because for all $r \geq 1$ and each $n \in [N]$,
\begin{align} \label{eq:gauss_mix_high_power}
    &\E_{X_t \sim Q_{t,n|y}} \norm{X_t - \sqrt{\Bar{\alpha}_t} \mu_{0,n}}^r \nonumber\\
    &\lesssim \E_{X_t \sim Q_{t,n|y}} \norm{X_t - \sqrt{\Bar{\alpha}_t} \mu_{0,n|y}}^r + \norm{\sqrt{\Bar{\alpha}_t} \mu_{0,n|y} - \sqrt{\Bar{\alpha}_t} \mu_{0,n}}^r \nonumber\\
    &\leq \norm{\Sigma_{t|y}^{\frac{1}{2}}}^r \E_{X_t \sim Q_{t,n|y}} \norm{\Sigma_{t|y}^{-\frac{1}{2}} (X_t - \sqrt{\Bar{\alpha}_t} \mu_{0,n|y})}^r + (\Bar{\alpha}_t)^{r/2} \norm{H^\dagger y - H^\dagger H \mu_{0,n}}^r \nonumber\\
    &\lesssim d^{r/2} + \norm{H^\dagger y - H^\dagger H \mu_{0,n}}^r.
\end{align}

Now we are ready to analyze the second term of $\mathcal{W}_{\text{vanish}}$. Note that
\begin{align*}
    &\nabla^2 \log q_{t-1|y}(m_{t,y})\\
    &= \sum_{n=1}^N \pi_n \frac{q_{t-1,n|y}(m_{t,y})}{q_{t-1|y}(m_{t,y})} \brc{\Sigma_{t|y}^{-1}(m_{t,y}-\mu_{t,n|y})(m_{t,y}-\mu_{t,n|y})^\T \Sigma_{t|y}^{-1}} - \Sigma_{t|y}^{-1} \\
    &\qquad - \brc{\sum_{n=1}^N \pi_n \frac{q_{t-1,n|y}(m_{t,y})}{q_{t-1|y}(m_{t,y})} \Sigma_{t|y}^{-1}(m_{t,y}-\mu_{t,n|y})}\brc{\sum_{n=1}^N \pi_n \frac{q_{t-1,n|y}(m_{t,y})}{q_{t-1|y}(m_{t,y})} \Sigma_{t|y}^{-1}(m_{t,y}-\mu_{t,n|y})}^\T.
\end{align*}
Thus,
\begin{align*}
    &\E_{X_t \sim Q_{t|y}}\abs{\Delta_{t,y}^\T \nabla^2 \log q_{t-1|y}(m_{t,y}) \Delta_{t,y} } \\
    &\leq 3 \E_{X_t \sim Q_{t|y}}\sbrc{\sum_{n=1}^N \pi_n \frac{q_{t-1,n|y}(m_{t,y})}{q_{t-1|y}(m_{t,y})} \brc{\Delta_{t,y}^\T \Sigma_{t|y}^{-1}(m_{t,y}-\mu_{t,n|y})}^2 } \\
    &\quad + 3 \E_{X_t \sim Q_{t|y}}\abs{\Delta_{t,y}^\T \Sigma_{t|y}^{-1} \Delta_{t,y} }\\
    &\quad + 3 \E_{X_t \sim Q_{t|y}} \brc{\sum_{n=1}^N \pi_n \frac{q_{t-1,n|y}(m_{t,y})}{q_{t-1|y}(m_{t,y})} \Delta_{t,y}^\T \Sigma_{t|y}^{-1}(m_{t,y}-\mu_{t,n|y})}^2 \\
    &\leq 3 \E_{X_t \sim Q_{t|y}}\abs{\Delta_{t,y}^\T \Sigma_{t|y}^{-1} \Delta_{t,y} }\\
    &\quad + 6 \E_{X_t \sim Q_{t|y}}\sbrc{\sum_{n=1}^N \pi_n \frac{q_{t-1,n|y}(m_{t,y})}{q_{t-1|y}(m_{t,y})} \brc{\Delta_{t,y}^\T \Sigma_{t|y}^{-1}(m_{t,y}-\mu_{t,n|y})}^2 }.
\end{align*}
To determine the rate of these two terms, we get
\[ \E_{X_t \sim Q_{t|y}} \abs{ \Delta_{t,y}^\T \Sigma_{t|y}^{-1}\Delta_{t,y}} = \E_{X_t \sim Q_{t|y}} \abs{ \Delta_{t,y}^\T A_t \Delta_{t,y}} \lesssim d + \sum_{n=1}^N \pi_n \norm{H^\dagger y - H^\dagger H \mu_{0,n}}^2, \]
and
\begin{align*}
    &\E_{X_t \sim Q_{t|y}} \sum_{n=1}^N \pi_n \frac{q_{t-1,n|y}(m_{t,y})}{q_{t-1|y}(m_{t,y})} \abs{(m_{t,y}-\mu_{t,n|y})^\T \Sigma_{t|y}^{-1}\Delta_{t,y}}^2\\
    &\qquad \leq \E_{X_t \sim Q_{t|y}} \max_{n \in [N]} \abs{(m_{t,y}-\mu_{t,n|y})^\T \Sigma_{t|y}^{-1}\Delta_{t,y}}^2 \\
    &\qquad \lesssim \E_{X_t \sim Q_{t|y}} \abs{m_{t,y}^\T \Sigma_{t|y}^{-1}\Delta_{t,y}}^2 + \E_{X_t \sim Q_{t|y}} \max_{n \in [N]} \abs{\mu_{t,n|y}^\T \Sigma_{t|y}^{-1}\Delta_{t,y}}^2\\
    % &\qquad \lesssim d^2 + N d^2\\
    &\qquad \lesssim N \brc{d^2 + \sum_{n=1}^N \pi_n \norm{H^\dagger y - H^\dagger H \mu_{0,n}}^4}
\end{align*}
where the last line follows from \eqref{eq:gauss_mix_temp6} and \eqref{eq:gauss_mix_temp7}.
Thus, the second term of $\mathcal{W}_{\text{vanish}}$ satisfies that
\begin{align*}
    &\sum_{t=1}^T \frac{(1-\alpha_t)^2}{2 \alpha_t} \E_{X_t \sim Q_{t|y}}\abs{\Delta_{t,y}(X_t)^\T \nabla^2 \log q_{t-1|y}(m_{t,y}(X_t)) \Delta_{t,y}(X_t) } \\
    &\lesssim N \brc{d^2 + \sum_{n=1}^N \pi_n \norm{H^\dagger y - H^\dagger H \mu_{0,n}}^4} \frac{c^2 (\log T)^2}{T}.
\end{align*}

For the third term of $\mathcal{W}_{\text{vanish}}$, we provide the following useful calculations. Denote $v^{\circ 3}$ as the element-wise (Hadamard) third power of a vector $v$. For each $n \in [N]$, we have
\begin{align} \label{eq:gauss_mix_temp8}
    &\E_{X_t \sim Q_{t|y}} \abs{(m_{t,y} - \mu_{t,n|y})^{\circ 3} (I_d - H^\dagger H) \Delta_{t,y}} = \E_{X_t \sim Q_{t|y}} \abs{[\Delta_{t,y}]_{\Bar{y}}^\T [m_{t,y} - \mu_{t,n}]^{\circ 3}_{\Bar{y}} } \nonumber\\
    &\lesssim \E_{X_t \sim Q_{t|y}} \abs{[\Delta_{t,y}]_{\Bar{y}}^\T [X_t - \mu_{t,n}]^{\circ 3}_{\Bar{y}} } + (1-\alpha_t) \sqrt{\E_{X_t \sim Q_{t|y}} \norm{\Delta_{t,y}}^2} \sqrt{\E_{X_t \sim Q_{t|y}} \norm{\nabla \log q_{t|y}(X_t)}_6^6} \nonumber\\
    &\stackrel{(iv)}{\lesssim} \E_{X_t \sim Q_{t|y}} \abs{[\Delta_{t,y}]_{\Bar{y}}^\T [X_t - \mu_{t,n}]^{\circ 3}_{\Bar{y}} } \nonumber \\
    &\lesssim \max_{\ell} \norm{[\Sigma_{t}^{-1}]_{\Bar{y} :} \mu_{0,\ell}} \sqrt{\E_{X_t \sim Q_{t|y}} \norm{X_t - \mu_{t,n}}_6^6} + \E_{\substack{X_t \sim Q_{t|y} \\ L \sim \Pi_{\cdot|t,y}(\cdot|X_t)}} \norm{X_t - \mu_{t,L}}_4^4 \nonumber \\
    &\lesssim \sqrt{\E_{\substack{X_t \sim Q_{t|y} \\ L \sim \Pi_{\cdot|t,y}(\cdot|X_t)}} \norm{X_t - \mu_{t,L}}^6} + \sqrt{\E_{L \sim \Pi} \norm{\mu_{t,n} - \mu_{t,L}}^6} + \E_{\substack{X_t \sim Q_{t|y} \\ L \sim \Pi_{\cdot|t,y}(\cdot|X_t)}} \norm{X_t - \mu_{t,L}}^4 \nonumber\\
    &\lesssim d^2 + \sum_{n=1}^N \pi_n \norm{H^\dagger y - H^\dagger H \mu_{0,n}}^4
\end{align}
where $(iv)$ follows from \Cref{lem:alpha_genli_rate} (using the $\alpha_t$ in \eqref{eq:alpha_genli}) and \cite[Lemma~15]{liang2024discrete}, and the last line follows from \eqref{eq:gauss_mix_high_power}. With a similar argument,
\begin{align} \label{eq:gauss_mix_temp9}
    \E_{X_t \sim Q_{t|y}} \abs{(m_{t,y} - \mu_{t,n|y}) (I_d - H^\dagger H) \Delta_{t,y}} &\lesssim \E_{\substack{X_t \sim Q_{t|y} \\ L \sim \Pi_{\cdot|t,y}(\cdot|X_t)}} \norm{X_t - \mu_{t,L}}^2 \nonumber\\
    &\lesssim d + \sum_{n=1}^N \pi_n \norm{H^\dagger y - H^\dagger H \mu_{0,n}}^2.
\end{align}

Now, employing the notations from \cite[Section~G.1]{liang2024discrete}, we define
\[ z_{t,n}(x) := \Sigma_{t|y}^{-1}(x-\mu_{t,n|y}), \quad \xi_t(x,i) := \max_n \abs{z_{t,n}^i(x)}, \quad \Bar{\Sigma}_t^{ij} := \max_n \abs{ [\Sigma_{t|y}^{-1}]^{ij} }. \]
When $H = \begin{pmatrix} I_p & 0 \end{pmatrix}$, we note that
\[ \Sigma_{t|y}^{-1} = \begin{pmatrix} (1-\Bar{\alpha}_t + \Bar{\alpha}_t \sigma_y^2)^{-1} I_p & 0 \\ 0 & (\Bar{\alpha}_{t} [\Sigma_{0}]_{\Bar{y} \Bar{y}} + (1-\Bar{\alpha}_{t}) I_{d-p})^{-1} \end{pmatrix}. \]
Thus, we have $\Bar{\Sigma}_t^{ij} \equiv 0$ whenever $(i,j) \in [1,p] \times [p+1,d]$ or $(i,j) \in [p+1,d] \times [1,p]$ and $\max_{i,j \in [p+1,d]} \Bar{\Sigma}_t^{ij} = O(1)$.
% We also have $\norm{\Sigma_{t|y}^{-1}} \lesssim (1-\Bar{\alpha}_t)^{-1}$. 
Since $\sigma_y^2 > 0$, we also have $\norm{\Sigma_{t|y}^{-1}} \lesssim 1$.

From \cite[Section~G.1.2]{liang2024discrete}, an upper bound for third-order partial derivatives is
\[ \abs{\partial^3_{ijk} \log q_{t|y} (x)} \leq 6 \xi_t(x,i) \xi_t(x,j) \xi_t(x,k) + 2 \Bar{\Sigma}_t^{ij} \xi_t(x,k) + 2 \Bar{\Sigma}_t^{ik} \xi_t(x,j) + 2 \Bar{\Sigma}_t^{jk} \xi_t(x,i). \]
We also remind readers that $\Delta_{t,y}$ is supported on $\range(I_d-H^\dagger H)$, namely that $[\Delta_{t,y}]_y \equiv 0$.

Now, the third term of $\mathcal{W}_{\text{vanish}}$ can be upper-bounded as
\begin{align*}
    &\E_{X_t \sim Q_{t|y}} \abs{ \sum_{i=1}^d \partial^3_{iii} \log q_{t-1|y}(m_{t,y}(X_t)) \Delta_{t,y}(X_t)^i }\\
    &= \E_{X_t \sim Q_{t|y}} \abs{ \sum_{i=p+1}^d \partial^3_{iii} \log q_{t-1|y}(m_{t,y}(X_t)) \Delta_{t,y}(X_t)^i} \\
    &\lesssim \E_{X_t \sim Q_{t|y}} \abs{ \sum_{i=p+1}^d \xi_t(m_{t,y}(X_t),i)^3 \Delta_{t,y}(X_t)^i} + \E_{X_t \sim Q_{t|y}} \abs{ \sum_{i=p+1}^d \Bar{\Sigma}_t^{ii} \xi_t(m_{t,y}(X_t),i) \Delta_{t,y}(X_t)^i }\\
    &\lesssim \sum_{n=1}^N \norm{(\Bar{\alpha}_{t} [\Sigma_{0}]_{\Bar{y} \Bar{y}} + (1-\Bar{\alpha}_{t}) I_{d-p})^{-1}}^3 \E_{X_t \sim Q_{t|y}} \abs{\sum_{i=p+1}^d (m_{t,y}(X_t)^i - \mu_{t,n|y}^i)^3 \Delta_{t,y}(X_t)^i} \\
    &\quad + \sum_{n=1}^N \norm{(\Bar{\alpha}_{t} [\Sigma_{0}]_{\Bar{y} \Bar{y}} + (1-\Bar{\alpha}_{t}) I_{d-p})^{-1}} \E_{X_t \sim Q_{t|y}} \abs{ \sum_{i=p+1}^d (m_{t,y}(X_t)^i - \mu_{t,n|y}^i) \Delta_{t,y}(X_t)^i} \\
    &\lesssim N \brc{d^2 + \sum_{n=1}^N \pi_n \norm{H^\dagger y - H^\dagger H \mu_{0,n}}^4}.
\end{align*}
Here the last line follows from \eqref{eq:gauss_mix_temp8} and \eqref{eq:gauss_mix_temp9}.

% \yuchen{Check below...}
% \begin{align*}
%     [\Delta_{t,y}]_{\Bar{y}}
%     &= - \sqrt{\Bar{\alpha}_t} \sum_{n=1}^N \pi_n \brc{ \frac{q_{t,n}(x_t)}{q_t(x_t)} - \frac{q_{t,n|y}(x_t)}{q_{t|y}(x_t)}} [\Sigma_{t}^{-1}]_{\Bar{y} :} \mu_{0,n} \nonumber \\
%     &\quad - \Bar{\alpha}_t \sum_{n=1}^N \pi_n \frac{q_{t,n|y}(x_t)}{q_{t|y}(x_t)}  (\Bar{\alpha}_t [\Sigma_{0}]_{\Bar{y} \Bar{y}} + (1-\Bar{\alpha}_t) I_{d-p})^{-1} [\Sigma_0]_{\Bar{y} y} [\Sigma_t^{-1}]_{y:} (x_t - \sqrt{\Bar{\alpha}_t} \mu_{0,n})
% \end{align*}

We provide the following useful calculations to upper-bound the fourth term of $\mathcal{W}_{\text{vanish}}$. First, for all $r \geq 1$ and any fixed vector $v$,
\begin{align} \label{eq:gauss_mix_temp13}
    \E_{X_t \sim Q_{t|y}} \norm{m_{t,y}-v}^r &= \E_{X_t \sim Q_{t|y}} \norm{\frac{1}{\sqrt{\alpha_t}} X_t + \frac{1-\alpha_t}{\sqrt{\alpha_t}} \nabla \log q_{t|y}(X_t)-v}^r \nonumber \\
    &\lesssim \E_{X_t \sim Q_{t|y}} \norm{X_t - v}^r + (1-\alpha_t) \E_{X_t \sim Q_{t|y}} \norm{\nabla \log q_{t|y}(X_t)}^r \nonumber \\
    &\stackrel{(v)}{\lesssim} \E_{\substack{X_t \sim Q_{t|y} \\ L \sim \Pi_{\cdot|t,y}(\cdot|X_t) }} \norm{X_t - \mu_{t,L}}^r + \E_{L \sim \Pi} \norm{\mu_{t,L} - v}^r \nonumber \\
    &\lesssim d^{r/2} + \sum_{n=1}^N \pi_n \norm{H^\dagger y - H^\dagger H \mu_{0,n}}^r
\end{align}
where $(v)$ follows from \Cref{lem:alpha_genli_rate} (using the $\alpha_t$ in \eqref{eq:alpha_genli}) and \cite[Lemma~15]{liang2024discrete}, and the last line follows from \eqref{eq:gauss_mix_high_power}.
Now,
% \begin{align} \label{eq:gauss_mix_temp10}
%     &\E_{X_t \sim Q_{t|y}} \sum_{i=1}^d \xi_t(m_{t,y},i)^2 \abs{\sum_{j=p+1}^d \xi_t(m_{t,y},j) \Delta_{t,y}(X_t)^j } \nonumber \\
%     &\lesssim \frac{1}{(1-\Bar{\alpha}_t)^2} \sum_{n,\ell=1}^N \E_{X_t \sim Q_{t|y}} \norm{m_{t,y} - \mu_{t,n|y}}^2 \abs{(m_{t,y} -\mu_{t,\ell|y})^\T \Delta_{t,y}(X_t)} \nonumber \\
%     &\leq \frac{1}{(1-\Bar{\alpha}_t)^2} \sum_{n,\ell=1}^N \sqrt{ \E_{X_t \sim Q_{t|y}} \norm{m_{t,y} - \mu_{t,n|y}}^4 } \brc{\E_{X_t \sim Q_{t|y}} \norm{m_{t,y} -\mu_{t,\ell|y}}^4 \E_{X_t \sim Q_{t|y}} \norm{\Delta_{t,y}(X_t)}^4 }^{1/4} \nonumber \\
%     &\lesssim \frac{N^2}{(1-\Bar{\alpha}_t)^2} \brc{d^2 + \sum_{n=1}^N \pi_n \norm{H^\dagger y - H^\dagger H \mu_{0,n}}^4}
% \end{align}
\begin{align} \label{eq:gauss_mix_temp10}
    &\E_{X_t \sim Q_{t|y}} \sum_{i=1}^d \xi_t(m_{t,y},i)^2 \abs{\sum_{j=p+1}^d \xi_t(m_{t,y},j) \Delta_{t,y}(X_t)^j } \nonumber \\
    &\lesssim \sum_{n,\ell=1}^N \E_{X_t \sim Q_{t|y}} \norm{m_{t,y} - \mu_{t,n|y}}^2 \abs{(m_{t,y} -\mu_{t,\ell|y})^\T \Delta_{t,y}(X_t)} \nonumber \\
    &\leq \sum_{n,\ell=1}^N \sqrt{ \E_{X_t \sim Q_{t|y}} \norm{m_{t,y} - \mu_{t,n|y}}^4 } \brc{\E_{X_t \sim Q_{t|y}} \norm{m_{t,y} -\mu_{t,\ell|y}}^4 \E_{X_t \sim Q_{t|y}} \norm{\Delta_{t,y}(X_t)}^4 }^{1/4} \nonumber \\
    &\lesssim N^2 \brc{d^2 + \sum_{n=1}^N \pi_n \norm{H^\dagger y - H^\dagger H \mu_{0,n}}^4}
\end{align}
where the first inequality follows from $\norm{\Sigma_{t|y}^{-1}} \lesssim 1$ and the last line follows from \eqref{eq:gauss_mix_temp13} and \eqref{eq:gauss_mix_delta_high_power}.
Also,
\begin{align} \label{eq:gauss_mix_temp11}
    &\E_{X_t \sim Q_{t|y}} \abs{ \sum_{i,j=p+1}^d \Bar{\Sigma}_t^{ii} \xi_t(m_{t,y},j) \Delta_{t,y}(X_t)^j } \nonumber\\
    &= \abs{ \sum_{i=p+1}^d \Bar{\Sigma}_t^{ii} } \cdot \E_{X_t \sim Q_{t|y}} \abs{ \sum_{j=p+1}^d \xi_t(m_{t,y},j) \Delta_{t,y}(X_t)^j } \nonumber\\
    &\lesssim \abs{ \sum_{i=p+1}^d \Bar{\Sigma}_t^{ii} } \cdot \sum_{n=1}^N \E_{X_t \sim Q_{t|y}} \abs{(m_{t,y}-\mu_{t,n})^\T \Delta_{t,y}(X_t)} \nonumber\\
    &\lesssim \abs{ \sum_{i=p+1}^d \Bar{\Sigma}_t^{ii} } \cdot \sum_{n=1}^N \sqrt{\E_{X_t \sim Q_{t|y}} \norm{m_{t,y}-\mu_{t,n}}^2 } \sqrt{\E_{X_t \sim Q_{t|y}} \norm{ \Delta_{t,y}(X_t)}^2 } \nonumber\\
    &\lesssim N d \brc{d + \sum_{n=1}^N \pi_n \norm{H^\dagger y - H^\dagger H \mu_{0,n}}^2}
\end{align}
where the last line follows from \Cref{lem:norm2_bd_gauss_mix} and \eqref{eq:gauss_mix_temp13}.
Also,
\begin{align} \label{eq:gauss_mix_temp12}
    &\E_{X_t \sim Q_{t|y}} \abs{ \sum_{i,j=p+1}^d \Bar{\Sigma}_t^{ij} \xi_t(m_{t,y},i) \Delta_{t,y}(X_t)^j} \nonumber\\
    &\lesssim \E_{X_t \sim Q_{t|y}} \brc{\sum_{i=p+1}^d \xi_t(m_{t,y},i)} \abs{ \sum_{j=p+1}^d \Delta_{t,y}(X_t)^j} \nonumber\\
    &\leq \sqrt{d \cdot \E_{X_t \sim Q_{t|y}} \brc{\sum_{i=p+1}^d \xi_t(m_{t,y},i)^2 } } \sqrt{d \cdot \E_{X_t \sim Q_{t|y}} \norm{\Delta_{t,y}(X_t)}^2 } \nonumber\\
    &\lesssim \sqrt{d \cdot \sum_{n=1}^N \E_{X_t \sim Q_{t|y}} \norm{m_{t,y} - \mu_{t,n}}^2 } \sqrt{d \cdot \E_{X_t \sim Q_{t|y}} \norm{\Delta_{t,y}(X_t)}^2 } \nonumber\\
    &\lesssim \sqrt{N} d \brc{d + \sum_{n=1}^N \pi_n \norm{H^\dagger y - H^\dagger H \mu_{0,n}}^2}
\end{align}
where the last line follows from \Cref{lem:norm2_bd_gauss_mix} and \eqref{eq:gauss_mix_temp13}.

Now, the fourth term of $\mathcal{W}_{\text{vanish}}$ can be upper-bounded as
\begin{align*}
    &\E_{X_t \sim Q_{t|y}} \abs{ \sum_{i,j=1}^d \partial^3_{iij} \log q_{t-1|y}(m_{t,y}(X_t)) \Delta_{t,y}(X_t)^j } \\
    &= \E_{X_t \sim Q_{t|y}} \abs{ \sum_{i=1}^d \sum_{j=p+1}^d \partial^3_{iij} \log q_{t-1|y}(m_{t,y}(X_t)) \Delta_{t,y}(X_t)^j }\\
    &\lesssim \E_{X_t \sim Q_{t|y}} \sum_{i=1}^d \xi_t(m_{t,y}(X_t),i)^2 \abs{\sum_{j=p+1}^d \xi_t(m_{t,y}(X_t),j) \Delta_{t,y}(X_t)^j } \\
    &\quad + \E_{X_t \sim Q_{t|y}} \abs{ \sum_{i,j=p+1}^d \Bar{\Sigma}_t^{ii} \xi_t(m_{t,y}(X_t),j) \Delta_{t,y}(X_t)^j } \\
    &\quad + \E_{X_t \sim Q_{t|y}} \abs{ \sum_{i,j=p+1}^d \Bar{\Sigma}_t^{ij} \xi_t(m_{t,y}(X_t),i) \Delta_{t,y}(X_t)^j}.
\end{align*}
For the three terms above, the first term $\lesssim N^2 \brc{d^2 + \sum_{n=1}^N \pi_n \norm{H^\dagger y - H^\dagger H \mu_{0,n}}^4}$ from \eqref{eq:gauss_mix_temp10}, the second term $\lesssim N d \brc{d + \sum_{n=1}^N \pi_n \norm{H^\dagger y - H^\dagger H \mu_{0,n}}^2}$ from \eqref{eq:gauss_mix_temp11}, and the last term $\lesssim \sqrt{N} d \brc{d + \sum_{n=1}^N \pi_n \norm{H^\dagger y - H^\dagger H \mu_{0,n}}^2}$ from \eqref{eq:gauss_mix_temp12}.

Thus, overall, with the $\alpha_t$ in \eqref{eq:alpha_genli} (cf. \Cref{lem:alpha_genli_rate}), the third and fourth terms give us
\begin{align*}
    &\sum_{t=1}^T \frac{(1-\alpha_t)^2}{3! {\alpha_t}^{3/2}} \E_{X_t \sim Q_{t|y}} \bigg[ 3 \sum_{i=1}^d \partial^3_{iii} \log q_{t-1|y}(m_{t,y}(X_t)) \Delta_{t,y}(X_t)^i \\
    &\qquad + \sum_{\substack{i,j=1 \\ i\neq j}}^d \partial^3_{iij} \log q_{t-1|y}(m_{t,y}(X_t)) \Delta_{t,y}(X_t)^j \bigg] \\
    &\lesssim N^2 \brc{d^2 + \sum_{n=1}^N \pi_n \norm{H^\dagger y - H^\dagger H \mu_{0,n}}^4} \frac{(\log T)^2}{T}.
\end{align*}
Therefore, combining all the above, since $N$ is constant,
\begin{align*}
    \KL{Q_{0|y}}{\widehat{P}_{0|y}} &\lesssim \brc{d + \sum_{n=1}^N \pi_n \norm{H^\dagger y - H^\dagger H \mu_{0,n}}^2} \\
    &\quad + \brc{d^2 + \sum_{n=1}^N \pi_n \norm{H^\dagger y - H^\dagger H \mu_{0,n}}^4} \frac{(\log T)^2}{T}\\
    &\quad + \sqrt{d + \sum_{n=1}^N \pi_n \norm{H^\dagger y - H^\dagger H \mu_{0,n}}^2} (\log T) \eps.
\end{align*}

% \begin{align*}
%     &\KL{Q_{1|y}}{\widehat{P}_{1|y}} \\
%     &\lesssim \sum_{t=2}^T (1-\alpha_t) \E_{X_t \sim Q_{t|y}} \norm{\Delta_{t,y}(X_t)}^2 \\
%     &+ \sum_{t=2}^T \frac{1-\alpha_t}{\sqrt{\alpha_t}} \E_{X_t \sim Q_{t|y}} \bigg[(\nabla \log q_{t-1|y}(m_{t,y}(X_t)) - \sqrt{\alpha_t} \nabla \log q_{t|y}(X_t))^\T \Delta_{t,y}(X_t)\bigg]\\
%     &+ \sum_{t=2}^T \frac{(1-\alpha_t)^2}{2 \alpha_t} \E_{X_t \sim Q_{t|y}}\sbrc{\Tr\Big(\nabla^2 \log q_{t-1|y}(m_{t,y}(X_t)) \brc{ \nabla^2\log q_{t|y}(X_t) - \Delta_{t,y}(X_t) \Delta_{t,y}(X_t)^\T } \Big) } \\
%     &+ \sum_{t=2}^T \frac{(1-\alpha_t)^2}{3! {\alpha_t}^{3/2}} \E_{X_t \sim Q_{t|y}} \bigg[ 3 \sum_{i=1}^d \partial^3_{iii} \log q_{t-1|y}(m_{t,y}(X_t)) \Delta_{t,y}(X_t)^i \\
%     &\qquad + \sum_{\substack{i,j=1 \\ i\neq j}}^d \partial^3_{iij} \log q_{t-1|y}(m_{t,y}(X_t)) \Delta_{t,y}(X_t)^j \bigg]\\
%     &+ \max_{t \geq 2}\sqrt{\E_{X_t \sim Q_{t|y}} \norm{\Delta_{t,y}(X_t)}^2} (\log T) \eps + (\log T) \eps^2.
% \end{align*}

\begin{remark} \label{rmk:gauss_mix_kl_genli}
In case of general $\sigma_y \geq 0$, the same upper bound can be applied to $\norm{\Sigma_{t-1|y}^{-1}}$ as detailed in \Cref{rmk:gauss_kl_genli}. Thus, with the $\alpha_t$ in \eqref{eq:alpha_genli}, we similarly have
\[  \mathcal{W}_{\text{oracle}} \lesssim \frac{d^2 (\log T)^2 \log(1/\delta)^2 }{T} + (\log T) \eps^2. \]
The rest of the proof is similar.
Combining with \Cref{lem:nonvanish_coef_sum_genli}, we would finally obtain
\begin{align*}
    &\KL{Q_{1|y}}{\widehat{P}_{1|y}} \lesssim  \brc{d + \sum_{n=1}^N \pi_n \norm{H^\dagger y - H^\dagger H \mu_{0,n}}^2} \brc{1 - \frac{2 \log(1/\delta) \log T}{T}} \\
    & + \brc{d^2 + \sum_{n=1}^N \pi_n \norm{H^\dagger y - H^\dagger H \mu_{0,n}}^4} \frac{(\log T)^2 \log(1/\delta)^2}{T} + \sqrt{d + \sum_{n=1}^N \pi_n \norm{H^\dagger y - H^\dagger H \mu_{0,n}}^2} (\log T) \eps.
\end{align*}
Here $\rmW_2(Q_{1|y}, Q_{0|y})^2 \lesssim \delta d$.
\end{remark}

\section{Auxiliary Lemmas and Proofs in \texorpdfstring{\Cref{sec:delta_ty}}{Section 4}}

\subsection{Proof of \texorpdfstring{\Cref{lem:norm2_bd_gauss}}{Proposition 2}} \label{app:proof_lem_norm2_bd_gauss}

Given $Q_0 = \calN(\mu_0, \Sigma_0)$, from the conditional forward model in \eqref{eq:def_cond_fwd2}, we can calculate
\begin{align*}
    Q_{t} &= \calN(\sqrt{\Bar{\alpha}_t} \mu_{0}, \Bar{\alpha}_t \Sigma_{0} + (1-\Bar{\alpha}_t) I_d) =: \calN(\mu_{t}, \Sigma_{t})\\
    Q_{t|y} &= \calN( \sqrt{\Bar{\alpha}_t} (I_d - H^\dagger H) \mu_{0} + \sqrt{\Bar{\alpha}_t} H^\dagger y, \\
    &\qquad \Bar{\alpha}_t (I_d - H^\dagger H) \Sigma_{0} (I_d - H^\dagger H) + \Bar{\alpha}_t \sigma_y^2 H^\dagger (H^\dagger)^\T + (1-\Bar{\alpha}_t) I_d).
    % &\qquad =: \calN(\mu_{t|y}, \Sigma_{t|y}).
\end{align*}
Note that when $H = \begin{pmatrix} I_p & 0 \end{pmatrix}$ and $\sigma_y^2 > 0$, $q_{0|y}$ exists. Define
\begin{align} \label{eq:gauss_notations}
    \mu_{t|y} &:= \sqrt{\Bar{\alpha}_t} (I_d - H^\dagger H) \mu_{0} + \sqrt{\Bar{\alpha}_t} H^\dagger y \nonumber\\
    \Sigma_{t,sig} &:= \Bar{\alpha}_t (I_d - H^\dagger H) \Sigma_{0} (I_d - H^\dagger H) + (1-\Bar{\alpha}_t) I_d \nonumber\\
    \Sigma_{t|y} &:= \Sigma_{t,sig} + \Bar{\alpha}_t \sigma_y^2 H^\dagger (H^\dagger)^\T.
\end{align}
Here $\Sigma_{t,sig}$ is the signal variance at time $t$, and $\Sigma_{t|y}$ is the total variance of the signal and the measurement noise. Note that when $H = \begin{pmatrix} I_p & 0 \end{pmatrix}$, $[\Sigma_{t|y}^{-1}]_{\Bar{y} \Bar{y}} = [\Sigma_{t,sig}^{-1}]_{\Bar{y} \Bar{y}}$.
We also calculate the respective scores of $Q_t$ and $Q_{t|y}$:
\[ \nabla \log q_t(x_t) = - \Sigma_{t}^{-1} (x_t - \mu_{t}),\quad \nabla \log q_{t|y}(x_t) = -\Sigma_{t|y}^{-1} (x_t - \mu_{t|y}). \]
Since $f_{t,y} = f_{t,y}^*$ (defined in \eqref{eq:def_fy_star}), from \eqref{eq:delta_using_fy}, the bias at each time is equal to
\begin{align*}
    \Delta_{t,y} &= (I_d-H^\dagger H) (\nabla \log q_{t|y}(x_t) - \nabla \log q_t(x_t))\\
    &= (I_d-H^\dagger H) \brc{\Sigma_{t}^{-1} (x_t - \sqrt{\Bar{\alpha}_t} \mu_{0}) - \Sigma_{t|y}^{-1} (x_t - \sqrt{\Bar{\alpha}_t} \mu_{0|y})} \\
    &= (I_d-H^\dagger H) \Sigma_{t}^{-1} (x_t - \sqrt{\Bar{\alpha}_t} \mu_{0}) \\
    &\quad - (I_d-H^\dagger H) \Sigma_{t|y}^{-1} (x_t - \sqrt{\Bar{\alpha}_t} (I_d - H^\dagger H) \mu_{0} - \sqrt{\Bar{\alpha}_t} H^\dagger y).
\end{align*}

Now, define
\begin{align*}
    V_t &:= (H^\dagger H) \Sigma_{0} (I_d - H^\dagger H) + (I_d - H^\dagger H) \Sigma_{0} (H^\dagger H) + (H^\dagger H) \Sigma_{0} (H^\dagger H) \\
    % - \sigma_y^2 H^\dagger (H^\dagger)^\T\\
    % &\stackrel{(i)}{=} (H^\dagger H) \Sigma_{0} (I_d - H^\dagger H) + (I_d - H^\dagger H) \Sigma_{0} (H^\dagger H) + (H^\dagger H) \Sigma_{0} (H^\dagger H) - \sigma_y^2 H^\dagger H\\
    A_t &:= (I_d - H^\dagger H) \Sigma_{t,sig}^{-1} (I_d - H^\dagger H)
    % B_t &:= (H^\dagger H) \Sigma_t^{-1} (\Sigma_{t|y} + \Bar{\alpha}_t (H^\dagger y - H^\dagger H \mu_0) (H^\dagger y - H^\dagger H \mu_0)^\T) \Sigma_t^{-1} (H^\dagger H)
\end{align*}
Thus, we have $\Sigma_t = \Sigma_{t,sig} + \Bar{\alpha}_t V_t$ and $\Sigma_{t|y} = \Sigma_{t,sig} + \Bar{\alpha}_t \sigma_y^2 H^\dagger (H^\dagger)^\T$. By Woodbury matrix identity, for any two matrices $A$ and $B$, their sum can be inversed as $(A+B)^{-1} = A^{-1} - A^{-1} B (A+B)^{-1}$. Thus, we get
\begin{align*}
    \Sigma_t^{-1} &= (\Sigma_{t,sig} + \Bar{\alpha}_t V_t)^{-1} = \Sigma_{t,sig}^{-1} - \Bar{\alpha}_t \Sigma_{t,sig}^{-1} V_t \Sigma_t^{-1} \\
    \Sigma_{t|y}^{-1} &= (\Sigma_{t,sig} + \Bar{\alpha}_t \sigma_y^2 H^\dagger (H^\dagger)^\T)^{-1} = \Sigma_{t,sig}^{-1} - \Bar{\alpha}_t \sigma_y^2 \Sigma_{t,sig}^{-1} H^\dagger (H^\dagger)^\T \Sigma_{t|y}^{-1} \\
    &\stackrel{(i)}{=} \Sigma_{t,sig}^{-1} - \Bar{\alpha}_t \sigma_y^2 \Sigma_{t,sig}^{-1} (H^\dagger H) \Sigma_{t|y}^{-1}
\end{align*}
where $(i)$ holds under assumption $H = \begin{pmatrix} I_p & 0 \end{pmatrix}$.
Thus,
\begin{align} \label{eq:delta_y_gauss}
    \Delta_{t,y} &= (I_d-H^\dagger H) \brc{\Sigma_{t}^{-1} (x_t - \sqrt{\Bar{\alpha}_t} \mu_{0}) - \Sigma_{t|y}^{-1} (x_t - \sqrt{\Bar{\alpha}_t} \mu_{0|y})} \nonumber\\
    &= (I_d-H^\dagger H) \brc{\Sigma_{t,sig}^{-1} - \Bar{\alpha}_t \Sigma_{t,sig}^{-1} V_t \Sigma_t^{-1}} (x_t - \sqrt{\Bar{\alpha}_t} \mu_{0}) \nonumber \\
    &\quad - (I_d-H^\dagger H) \brc{\Sigma_{t,sig}^{-1} - \Bar{\alpha}_t \sigma_y^2 \Sigma_{t,sig}^{-1} (H^\dagger H) \Sigma_{t|y}^{-1}} (x_t - \sqrt{\Bar{\alpha}_t} (I_d - H^\dagger H) \mu_{0} - \sqrt{\Bar{\alpha}_t} H^\dagger y) \nonumber \\
    &\stackrel{(ii)}{=} (I_d-H^\dagger H) \brc{\Sigma_{t,sig}^{-1} - \Bar{\alpha}_t \Sigma_{t,sig}^{-1} V_t \Sigma_t^{-1}} (x_t - \sqrt{\Bar{\alpha}_t} \mu_{0}) \nonumber \\
    &\quad - (I_d-H^\dagger H) \Sigma_{t,sig}^{-1} (x_t - \sqrt{\Bar{\alpha}_t} (I_d - H^\dagger H) \mu_{0} - \sqrt{\Bar{\alpha}_t} H^\dagger y) \nonumber \\
    &= - \Bar{\alpha}_t (I_d-H^\dagger H) \Sigma_{t,sig}^{-1} V_t \Sigma_t^{-1} (x_t - \sqrt{\Bar{\alpha}_t} \mu_{0}) \nonumber\\
    &\quad + (I_d-H^\dagger H) \Sigma_{t,sig}^{-1} \Big( (I_d-H^\dagger H) (x_t - \sqrt{\Bar{\alpha}_t} \mu_{0}) + H^\dagger H (x_t - \sqrt{\Bar{\alpha}_t} \mu_{0}) \Big) \nonumber \\
    &\quad - (I_d-H^\dagger H) \Sigma_{t,sig}^{-1} \Big( (I_d - H^\dagger H) (x_t - \sqrt{\Bar{\alpha}_t} \mu_{0}) + ( (H^\dagger H) x_t - \sqrt{\Bar{\alpha}_t} H^\dagger y) \Big) \nonumber \\
    &\stackrel{(iii)}{=} - \Bar{\alpha}_t (I_d-H^\dagger H) \Sigma_{t,sig}^{-1} V_t \Sigma_t^{-1} (x_t - \sqrt{\Bar{\alpha}_t} \mu_{0}) \nonumber \\
    &\stackrel{(iv)}{=} - \Bar{\alpha}_t (I_d-H^\dagger H) \Sigma_{t,sig}^{-1} (I_d - H^\dagger H) \Sigma_{0} (H^\dagger H) \Sigma_t^{-1} (x_t - \sqrt{\Bar{\alpha}_t} \mu_{0}) \nonumber \\
    &= - \Bar{\alpha}_t A_t \Sigma_{0} (H^\dagger H) \Sigma_t^{-1} (x_t - \sqrt{\Bar{\alpha}_t} \mu_{0})
\end{align}
% Thus,
% \begin{align} \label{eq:delta_y_gauss}
%     \Delta_{t,y} &= (I_d-H^\dagger H) \brc{\Sigma_{t|y}^{-1} - \Bar{\alpha}_t \Sigma_{t|y}^{-1} V_t \Sigma_t^{-1}} (x_t - \sqrt{\Bar{\alpha}_t} \mu_{0}) \nonumber \\
%     &\quad - (I_d-H^\dagger H) \Sigma_{t|y}^{-1} (x_t - \sqrt{\Bar{\alpha}_t} (I_d - H^\dagger H) \mu_{0} - \sqrt{\Bar{\alpha}_t} H^\dagger y) \nonumber \\
%     &\stackrel{(i)}{=} - \Bar{\alpha}_t (I_d-H^\dagger H) \Sigma_{t|y}^{-1} V_t \Sigma_t^{-1} (x_t - \sqrt{\Bar{\alpha}_t} \mu_{0}) \nonumber \\
%     &\stackrel{(ii)}{=} - \Bar{\alpha}_t (I_d-H^\dagger H) \Sigma_{t|y}^{-1} (I_d - H^\dagger H) \Sigma_{0} (H^\dagger H) \Sigma_t^{-1} (x_t - \sqrt{\Bar{\alpha}_t} \mu_{0}) \nonumber \\
%     &= - \Bar{\alpha}_t A_t \Sigma_{0} (H^\dagger H) \Sigma_t^{-1} (x_t - \sqrt{\Bar{\alpha}_t} \mu_{0})
% \end{align}
where $(ii)$--$(iv)$ hold because $(I_d-H^\dagger H) \Sigma_{t,sig}^{-1} (H^\dagger H) = 0$ by \Cref{lem:gauss_cond_var_proj}.

Now, since $H^\dagger H = \begin{pmatrix}
I_p & 0 \\
0 & 0 
\end{pmatrix}$, we can re-express $A_t$ and $\Delta_{t,y}$ as follows.
\begin{align*}
    A_t &= (I_d - H^\dagger H) \Sigma_{t,sig}^{-1} (I_d - H^\dagger H)\\
    &= \begin{pmatrix} 0 & 0 \\ 0 & I_{d-p} \end{pmatrix} \begin{pmatrix} (1-\Bar{\alpha}_t) I_p & 0 \\ 0 & \Bar{\alpha}_t [\Sigma_{0}]_{\Bar{y} \Bar{y}} + (1-\Bar{\alpha}_t) I_{d-p} \end{pmatrix}^{-1} \begin{pmatrix} 0 & 0 \\ 0 & I_{d-p} \end{pmatrix} \\
    &= \begin{pmatrix}
        0 & 0 \\
        0 & (\Bar{\alpha}_t [\Sigma_{0}]_{\Bar{y} \Bar{y}} + (1-\Bar{\alpha}_t) I_{d-p})^{-1}
        \end{pmatrix}\\
    % B_t &= (H^\dagger H) \Sigma_t^{-1} \begin{pmatrix}
    %         \Bar{\alpha}_t (y - H \mu_0) (y - H \mu_0)^\T & 0 \\
    %         0 & \Bar{\alpha}_t [\Sigma_{0}]_{\Bar{y} \Bar{y}} + (1-\Bar{\alpha}_t) I_{d-p}
    %         \end{pmatrix} \Sigma_t^{-1} (H^\dagger H)\\
    \Delta_{t,y} &= \begin{pmatrix} 0 \\ -\Bar{\alpha}_t (\Bar{\alpha}_t [\Sigma_{0}]_{\Bar{y} \Bar{y}} + (1-\Bar{\alpha}_t) I_{d-p})^{-1} [\Sigma_0]_{\Bar{y} y} [\Sigma_t^{-1}]_{y:} (x_t - \sqrt{\Bar{\alpha}_t} \mu_{0}) \end{pmatrix},
\end{align*}
and we have
\begin{align} \label{eq:gauss_mismatch_2norm_bound}
    &\E_{X_t \sim Q_{t|y}} \norm{\Delta_{t,y}}^2 = \Bar{\alpha}_t^2 \E_{X_t \sim Q_{t|y}} \norm{ (\Bar{\alpha}_t [\Sigma_{0}]_{\Bar{y} \Bar{y}} + (1-\Bar{\alpha}_t) I_{d-p})^{-1} [\Sigma_0]_{\Bar{y} y} [\Sigma_t^{-1}]_{y:} (X_t - \sqrt{\Bar{\alpha}_t} \mu_{0}) }^2 \nonumber\\
    &= \Bar{\alpha}_t^2 \E_{X_t \sim Q_{t|y}} \Tr\bigg((\Bar{\alpha}_t [\Sigma_{0}]_{\Bar{y} \Bar{y}} + (1-\Bar{\alpha}_t) I_{d-p})^{-1} [\Sigma_0]_{\Bar{y} y} [\Sigma_t^{-1}]_{y:} (X_t - \sqrt{\Bar{\alpha}_t} \mu_{0}) (X_t - \sqrt{\Bar{\alpha}_t} \mu_{0})^\T \nonumber\\
    &\qquad [\Sigma_t^{-1}]_{:y} [\Sigma_0]_{y \Bar{y}} (\Bar{\alpha}_t [\Sigma_{0}]_{\Bar{y} \Bar{y}} + (1-\Bar{\alpha}_t) I_{d-p})^{-1}\bigg) \nonumber\\
    &\stackrel{(v)}{\leq} \Bar{\alpha}_t^2 \norm{(\Bar{\alpha}_t [\Sigma_{0}]_{\Bar{y} \Bar{y}} + (1-\Bar{\alpha}_t) I_{d-p})^{-1}}^2 \norm{[\Sigma_t^{-1}]_{y:}}^2  \norm{[\Sigma_0]_{y \Bar{y}} [\Sigma_0]_{\Bar{y} y}} \E_{X_t \sim Q_{t|y}} \norm{X_t - \sqrt{\Bar{\alpha}_t} \mu_{0}}^2
\end{align}
where $(v)$ follows from the fact that $\abs{\Tr(U V)} \leq \norm{U} \Tr(V)$ if $V$ is positive semi-definite. 
To analyze each norm above, denote $\lambda_1 \geq \dots \geq \lambda_{d} > 0$ to be the eigenvalues of $\Sigma_{0}$, and note that $\norm{[\Sigma_0]_{\Bar{y} y}} \leq \norm{\Sigma_0} = \lambda_1$. The largest eigenvalue of $\Sigma_{t}^{-1}$ is  $(\Bar{\alpha}_t \lambda_d + (1-\Bar{\alpha}_t))^{-1}$, and note that $\norm{[\Sigma_t^{-1}]_{y:}} \leq \norm{\Sigma_t^{-1}}$. 
Also, since $[\Sigma_{0}]_{\Bar{y} \Bar{y}}$ is positive semi-definite, denote $\Tilde{\lambda}_1 \geq \dots \geq \Tilde{\lambda}_{d-p} > 0$ to be its eigenvalues, and thus
\[ \norm{(\Bar{\alpha}_t [\Sigma_{0}]_{\Bar{y} \Bar{y}} + (1-\Bar{\alpha}_t) I_{d-p})^{-1}} \leq \frac{1}{\Bar{\alpha}_t \Tilde{\lambda}_{d-p} + (1-\Bar{\alpha}_t)} \leq \frac{1}{\min\{\Tilde{\lambda}_{d-p}, 1\}} < \infty. \]
Also, since 
\begin{align*}
    &\E_{X_{t} \sim Q_{t|y}} (X_t - \sqrt{\Bar{\alpha}_t} \mu_{0}) (X_t - \sqrt{\Bar{\alpha}_t} \mu_{0})^\T \\
    &= \E_{X_{t} \sim Q_{t|y}} (X_t - \mu_{t|y} + \sqrt{\Bar{\alpha}_t} (H^\dagger y - H^\dagger H \mu_0) ) (X_t - \mu_{t|y} + \sqrt{\Bar{\alpha}_t} (H^\dagger y - H^\dagger H \mu_0))^\T \\
    &= \Sigma_{t|y} + \Bar{\alpha}_t (H^\dagger y - H^\dagger H \mu_0) (H^\dagger y - H^\dagger H \mu_0)^\T,
\end{align*}
we have
% \[ \E_{X_{t} \sim Q_{t|y}} \norm{\Delta_{t,y}}^2 = \Bar{\alpha}_t^2 \cdot \Tr(A_t \Sigma_0 B_t \Sigma_0 A_t). \]
\begin{align} \label{eq:gauss_mismatch_norm_2}
    &\E_{X_t \sim Q_{t|y}} \norm{X_t - \sqrt{\Bar{\alpha}_t} \mu_{0} }^2 = \Tr(\Sigma_{t|y}) + \Bar{\alpha}_t \norm{H^\dagger y - H^\dagger H \mu_0}^2 \nonumber\\
    &\stackrel{(vi)}{\leq} \Bar{\alpha}_t (\lambda_1 + \sigma_y^2) + (1-\Bar{\alpha}_t) + \Bar{\alpha}_t \norm{H^\dagger y - H^\dagger H \mu_0}^2 \nonumber\\
    &\leq \max \cbrc{\norm{H^\dagger y - H^\dagger H \mu_0}^2 + d (\lambda_1 + \sigma_y^2), d }  
\end{align}
where $(vi)$ is because $\Sigma_{t|y}$ is positive definite with
\begin{align*}
    \norm{\Sigma_{t|y}} &= \norm{\Bar{\alpha}_t (I_d - H^\dagger H) \Sigma_{0} (I_d - H^\dagger H) + \Bar{\alpha}_t \sigma_y^2 H^\dagger (H^\dagger)^\T + (1-\Bar{\alpha}_t) I_d} \\
    &\leq \Bar{\alpha}_t \norm{\Sigma_{0}} + \Bar{\alpha}_t \sigma_y^2 + (1-\Bar{\alpha}_t) \\
    &= \Bar{\alpha}_t (\lambda_1 + \sigma_y^2) + (1-\Bar{\alpha}_t).
\end{align*}
% Since $\norm{(H^\dagger y - H^\dagger H \mu_0) (H^\dagger y - H^\dagger H \mu_0)^\T} \leq \norm{H^\dagger y - H^\dagger H \mu_0}^2$, we have
% \[ \norm{[B_t]_{y y}} \leq \norm{B_t} \leq \frac{\Bar{\alpha}_t \norm{H^\dagger y - H^\dagger H \mu_0}^2 + \Bar{\alpha}_t \lambda_1 + (1-\Bar{\alpha}_t)}{(\Bar{\alpha}_t \lambda_d + (1-\Bar{\alpha}_t))^2} \leq \frac{\max\{\norm{H^\dagger y - H^\dagger H \mu_0}^2 + \lambda_1, 1\}}{\min\{\lambda_d, 1\}^2} \]

Therefore, for the second moment,
\begin{align*}
    \E_{X_{t} \sim Q_{t|y}} \norm{\Delta_{t,y}}^2 &\leq \Bar{\alpha}_t^2 \frac{\max\cbrc{ \norm{H^\dagger y - H^\dagger H \mu_0}^2 + d (\lambda_1 + \sigma_y^2), d }}{\min\{\lambda_d, 1\}^2 \min\{\Tilde{\lambda}_{d-p}, 1\}^2} \norm{[\Sigma_0]_{y \Bar{y}} [\Sigma_0]_{\Bar{y} y}}\\
    &\lesssim \Bar{\alpha}_t^2 \cdot \max\cbrc{ \norm{H^\dagger y - H^\dagger H \mu_0}^2 + d (\lambda_1 + \sigma_y^2), d }
\end{align*}
since $\norm{[\Sigma_0]_{y \Bar{y}} [\Sigma_0]_{\Bar{y} y}} \leq \norm{\Sigma_0}^2 = \lambda_1^2$.
Also, for general moments with $m \geq 2$,
\begin{align*}
    \norm{\Delta_{t,y}}^m &\leq \Bar{\alpha}_t^m \norm{(\Bar{\alpha}_t [\Sigma_{0}]_{\Bar{y} \Bar{y}} + (1-\Bar{\alpha}_t) I_{d-p})^{-1} [\Sigma_0]_{\Bar{y} y} [\Sigma_t^{-1}]_{y:} (x_t - \sqrt{\Bar{\alpha}_t} \mu_{0})}^m\\
    &\leq \Bar{\alpha}_t^m \norm{(\Bar{\alpha}_t [\Sigma_{0}]_{\Bar{y} \Bar{y}} + (1-\Bar{\alpha}_t) I_{d-p})^{-1}}^m \norm{\Sigma_0}^m \norm{\Sigma_t^{-1}}^m \norm{x_t - \sqrt{\Bar{\alpha}_t} \mu_{0}}^m\\
    &\lesssim \Bar{\alpha}_t^m \norm{x_t - \sqrt{\Bar{\alpha}_t} \mu_{0}}^m
\end{align*}
and thus
\begin{align*}
    \E_{X_{t} \sim Q_{t|y}} \norm{\Delta_{t,y}}^m &\lesssim \Bar{\alpha}_t^m \brc{ \E_{X_{t} \sim Q_{t|y}} \norm{\Sigma_{t|y}^{-\frac{1}{2}} (X_t - \mu_{t|y})}^m + \norm{\sqrt{\Bar{\alpha}_t} (H^\dagger y - H^\dagger H \mu_{0})}^m }\\
    &\leq \Bar{\alpha}_t^m \brc{ (m-1)!! \cdot d^{m/2-1} + \norm{\sqrt{\Bar{\alpha}_t} (H^\dagger y - H^\dagger H \mu_{0})}^m } = O(\Bar{\alpha}_t).
\end{align*}
Therefore, \Cref{ass:bdd-mismatch-general} is satisfied if the $\alpha_t$ satisfies \Cref{def:noise_smooth}. The proof is now complete.

\subsection{Proof of \texorpdfstring{\Cref{lem:bdd-mismatch-gauss-mix}}{Lemma 8}}
\label{app:proof-lem-bdd-mismatch-gauss-mix}

We continue from the expression of $\Delta_{t,y}$ in \eqref{eq:delta_gauss_mix} when $Q_0$ is Gaussian mixture. For $m \geq 2$, we have
\begin{align} \label{eq:gauss_mix_delta_high_power}
    &\E_{X_t \sim Q_{t|y}} \norm{\Delta_{t,y}}^m \nonumber\\
    &\leq 2^{m-1} (\Bar{\alpha}_t)^{m/2} \max_{n \in [N]} \norm{\Sigma_{t}^{-1}  \mu_{0,n}}^m \nonumber\\
    &\quad + 2^{m-1} (\Bar{\alpha}_t)^m \E_{\substack{X_t \sim Q_{t|y} \\ N \sim \Pi_{\cdot|t,y}(\cdot|X_t)}} \norm{ (I_d - H^\dagger H) \Sigma_{t,sig}^{-1} (I_d - H^\dagger H) \Sigma_{0} (H^\dagger H) \Sigma_t^{-1} (X_t - \sqrt{\Bar{\alpha}_t} \mu_{0,N}) }^m \nonumber\\
    &\stackrel{(ii)}{\lesssim} (\Bar{\alpha}_t)^{m/2} d^{m/2} + (\Bar{\alpha}_t)^m \E_{\substack{X_t \sim Q_{t|y} \\ N \sim \Pi_{\cdot|t,y}(\cdot|X_t)}} \norm{X_t - \sqrt{\Bar{\alpha}_t} \mu_{0,N} }^m \nonumber\\
    &= (\Bar{\alpha}_t)^{m/2} d^{m/2} + (\Bar{\alpha}_t)^m \E_{\substack{N \sim \Pi \\ X_t \sim Q_{t,N|y}}} \norm{X_t - \sqrt{\Bar{\alpha}_t} \mu_{0,N} }^m \nonumber\\
    &\stackrel{(iii)}{\lesssim} (\Bar{\alpha}_t)^{m/2} d^{m/2} + (\Bar{\alpha}_t)^m \brc{ d^{m/2} + \sum_{n=1}^N \pi_n \norm{H^\dagger y - H^\dagger H \mu_{0,n}}^m } \\
    &= O(\Bar{\alpha}_t). \nonumber
\end{align}
Here $(ii)$ follows because
\begin{align*}
    \norm{(I_d - H^\dagger H) \Sigma_{t,sig}^{-1} (I_d - H^\dagger H) \Sigma_{0} (H^\dagger H) \Sigma_t^{-1}} &= \norm{(\Bar{\alpha}_t [\Sigma_{0}]_{\Bar{y} \Bar{y}} + (1-\Bar{\alpha}_t) I_{d-p})^{-1} [\Sigma_0]_{\Bar{y} y} [\Sigma_t^{-1}]_{y:}}\\
    &\leq \frac{\lambda_1}{\min\{\Tilde{\lambda}_{d-p}, 1\} \min\{\lambda_d, 1\} } = O(1),
\end{align*}
and $(iii)$ follows from \eqref{eq:gauss_mix_high_power}.
% because
% \begin{align*}
%     &\E_{X_t \sim Q_{t|y}} \sbrc{ \max_{n \in [N]} \norm{X_t - \sqrt{\Bar{\alpha}_t} \mu_{0,n}}^m } \\
%     &\leq \sum_{n=1}^N \E_{X_t \sim Q_{t|y}} \norm{X_t - \sqrt{\Bar{\alpha}_t} \mu_{0,n}}^m \\
%     &= \sum_{n=1}^N \sum_{\ell=1}^N \pi_\ell \E_{X_t \sim Q_{t,\ell|y}} \norm{X_t - \sqrt{\Bar{\alpha}_t} \mu_{0,n}}^m\\
%     &\lesssim \sum_{n=1}^N \sum_{\ell=1}^N \pi_\ell \E_{X_t \sim Q_{t,\ell|y}} \norm{\Sigma_{t|y}^{-\frac{1}{2}} (X_t - \mu_{t,\ell|y}) }^m + \norm{\sqrt{\Bar{\alpha}_t} \mu_{0,\ell|y} - \sqrt{\Bar{\alpha}_t} \mu_{0,n}}^m\\
%     &\lesssim N (m-1)!! \cdot d^{m/2-1} + N d^{m/2} \lesssim N d^{m/2}.
% \end{align*}
Therefore, this verifies \Cref{ass:bdd-mismatch-general} when the $\alpha_t$ satisfies \cref{def:noise_smooth}. The proof is now complete.

\subsection{\texorpdfstring{\Cref{lem:gauss_cond_var_proj}}{Lemma 9} and its proof}

\begin{lemma} \label{lem:gauss_cond_var_proj}
Given a positive semi-definite matrix $\Sigma$, $\sigma \geq 0$, and $\alpha \in (0,1)$,
\begin{align*}
    &(\alpha \sigma H^\dagger (H^\dagger)^\T + (1-\alpha) I_d)^{-1} (I_d - H^\dagger H) = \frac{1}{1-\alpha} (I_d - H^\dagger H)\\
    &(I_d - H^\dagger H) (\alpha (I_d - H^\dagger H) \Sigma (I_d - H^\dagger H) + \alpha \sigma H^\dagger H + (1-\alpha) I_d)^{-1} (H^\dagger H) = 0.
\end{align*}
\end{lemma}

\begin{proof}
The key of the proof is the Woodbury matrix identity, which states that for any matrices $U \in \mbR^{d\times p},V \in \mbR^{p \times d}$,
\[ (I_d+UV)^{-1} = I_d - U (I_p+VU)^{-1} V. \]

For the first equality, we apply Woodbury with $U = \sqrt{\frac{\alpha \sigma}{1 - \alpha}} H^\dagger$ and $V = \sqrt{\frac{\alpha \sigma}{1 - \alpha}} (H^\dagger)^\T$ and we get
\[ (\alpha \sigma H^\dagger (H^\dagger)^\T + (1-\alpha) I_d)^{-1} = \frac{1}{1-\alpha} \brc{ I_d - \frac{\alpha \sigma}{1-\alpha} H^\dagger \brc{I_p + \frac{\alpha \sigma}{1-\alpha} (H^\dagger)^\T H^\dagger}^{-1} (H^\dagger)^\T }. \]
Since $p \leq d$, the pseudo-inverse equals $H^\dagger = H^\T (H H^\T)^{-1}$, and by the orthogonal property we have
\[ (H^\dagger)^\T (I_d - H^\dagger H) = (H H^\T)^{-1} H (I_d - H^\dagger H) = 0. \]
We have thus shown the first equality.

For the second equality, we first consider the case where $\sigma = 0$. Write $S = (1-\alpha) I_d + \alpha (I_d - H^\dagger H) \Sigma (I_d - H^\dagger H)$. Since $\Sigma$ is positive semi-definite, there exists matrix $L$ such that $\Sigma = L L^\T$. Thus,
\begin{align*}
    S^{-1} &= \bigg( (1-\alpha) I_d + \alpha (I_d - H^\dagger H) \Sigma (I_d - H^\dagger H) \bigg)^{-1} \\
    &= \bigg( (1-\alpha) I_d + \alpha \Big((I_d - H^\dagger H) L \Big) \Big((I_d - H^\dagger H)^\T L \Big)^\T \bigg)^{-1}\\
    &= \frac{1}{1-\alpha} I_d - \frac{\alpha}{1-\alpha} ((I_d - H^\dagger H) L) \brc{I_d + \frac{\alpha}{1-\alpha} L^\T (I_d - H^\dagger H) L}^{-1} L^\T (I_d - H^\dagger H),
\end{align*}
where in the last line we have applied Woodbury with $U = V^\T = \sqrt{\frac{\alpha}{1 - \alpha}} (I_d - H^\dagger H) L$.
The equality is achieved because
\begin{align*}
    &(I_d - H^\dagger H) S^{-1} (H^\dagger H) \\
    &= \frac{1}{1-\alpha} \underbrace{(I_d - H^\dagger H) (H^\dagger H)}_{=0} \\
    &\quad - \frac{\alpha}{1-\alpha} (I_d - H^\dagger H) L \brc{I_d + \frac{\alpha}{1-\alpha} L^\T (I_d - H^\dagger H) L}^{-1} L^\T \underbrace{(I_d - H^\dagger H) (H^\dagger H)}_{=0} \\
    &= 0. \qedhere
\end{align*}

When $\sigma > 0$, we can apply Woodbury identity to sum of matrices $A$ and $B$ and get $(A+B)^{-1} = A^{-1} - A^{-1} B (A+B)^{-1}$. Thus,
\[ (S + \alpha \sigma H^\dagger H)^{-1} = S^{-1} - \alpha \sigma S^{-1} (H^\dagger H) (S+\alpha \sigma H^\dagger H)^{-1} \]
and
\begin{align*}
    &(I_d - H^\dagger H) (S + \alpha \sigma H^\dagger H)^{-1} (H^\dagger H) \\
    &= \underbrace{(I_d - H^\dagger H) S^{-1} (H^\dagger H)}_{=0} \\
    &\qquad - \alpha \sigma \underbrace{(I_d - H^\dagger H) S^{-1} (H^\dagger H)}_{=0} (S+\alpha \sigma H^\dagger H)^{-1}  (H^\dagger H) \\
    &= 0.
\end{align*}
The proof is now complete.

\end{proof}

\subsection{\texorpdfstring{\Cref{lem:nonvanish_coef_sum}}{Lemma 10} and its proof}

\begin{lemma} \label{lem:nonvanish_coef_sum}
With $1 - \alpha_t \equiv \frac{\log T}{T},~\forall t \geq 0$ (which satisfies \Cref{def:noise_smooth}), given any $p > 0$,
\[ \sum_{t=2}^T (1-\alpha_t) \Bar{\alpha}_t^p = \frac{1}{p} \brc{1-\frac{2 p c \log T}{T}} +\Tilde{O}\brc{\frac{1}{T^2}}. \]
\end{lemma}

\begin{proof}
Define the sum as $s_T$. Then,
\begin{align*}
    s_T &= \sum_{t=2}^T \frac{c \log T}{T} \brc{1-\frac{c \log T}{T}}^{p t} = \frac{c \log T}{T} \brc{1-\frac{c \log T}{T}}^{2p} \frac{1 - \brc{1-\frac{c \log T}{T}}^{p(T-1)}}{1- \brc{1-\frac{c \log T}{T}}^{p}}\\
    &= \frac{c \log T}{T} \brc{1-\frac{c \log T}{T}}^{2p} \frac{1}{1- \brc{1-\frac{c \log T}{T}}^{p}} (1+O(T^{-cp}))\\
    &= \frac{1}{p} \brc{1-\frac{2 p c \log T}{T}} + \Tilde{O}(T^{-2}). \qedhere
\end{align*}
\end{proof}

\end{document}